\documentclass[titlepage, oneside]{amsbook}
\usepackage{varwidth}
\usepackage{bm}
\usepackage{sagetex}
\usepackage{amsthm}
\usepackage{amsmath}
\usepackage[titletoc]{appendix}
\usepackage[mathscr]{euscript}
\usepackage{listings}
\usepackage{rotating}
\usepackage{subfigure}
\newtheorem{theorem}{Theorem}
\newtheorem{proposition}{Proposition}
\newtheorem{definition}{Definition}
\newtheorem{lem}{Lemma}
\newtheorem{conj}{Conjecture}
\newtheorem{corollary}{Corollary}
\newtheorem{rem}{Remark}
\usepackage{listings}
\usepackage[]{nomencl}
\usepackage{comment}

\makenomenclature

\usepackage{color}
 
\definecolor{dkgreen}{rgb}{0,0.6,0}
\definecolor{gray}{rgb}{0.5,0.5,0.5}
\definecolor{mauve}{rgb}{0.58,0,0.82}
\def\ci{\perp\!\!\!\perp}

\newcommand{\HRule}{\rule{\linewidth}{0.5mm}} 
\begin{document}
\begin{titlepage}

\begin{center}



\HRule \\[0.4cm]
{ \huge \bfseries Incorporating Type II Error Probabilities from Independence Tests into Score-Based Learning of Bayesian Network Structure}\\[0.4cm]

\HRule \\[1.5cm]


\newenvironment{texto}{%
    \begin{varwidth}[t]{0.8\textwidth}
	  }{%
	    \end{varwidth}
		  }
\hskip -1.2in
\begin{texto}
\large
\emph{Authors:}\\
Eliot \textsc{Brenner}, eliotpbrenner@gmail.com\\
David \textsc{Sontag}, dsontag@cs.nyu.edu
\end{texto}

\vspace{3\baselineskip}
\begin{texto}
\textsc{Abstract.}  We give a new consistent scoring function for structure learning of Bayesian
networks.
In contrast to traditional approaches to score-based structure
learning, such as BDeu or MDL, the complexity penalty that we propose
is data-dependent and is given by the probability that a conditional
independence test correctly shows that an edge cannot exist.
What really distinguishes this new scoring function from earlier
work is that it has the property of becoming computationally
{\em easier} to maximize as the amount of data increases.
We prove a polynomial sample complexity result, showing that maximizing
this score is guaranteed to correctly learn a structure
with no false edges and a distribution close to the generating distribution, 
whenever there exists a
Bayesian network which is a perfect map for the data generating
distribution. 
Although the new score can be used with any search algorithm, in \cite{brennerSontagUAI} 
we have given empirical results showing that it is particularly effective
when used together with a linear programming relaxation approach to
Bayesian network structure learning.  The present paper contains all details of the proofs
of the finite-sample complexity results in \cite{brennerSontagUAI} as well as detailed
explanation of the computation of the certain error probabilities called $\beta$-values,
whose precomputation and tabulation is necessary for the implementation of the algorithm
in \cite{brennerSontagUAI}.
\end{texto}


\vfill

{\large \today}

\end{center}

\end{titlepage}

\title[Learning Bayesian Network Structure Using $\beta$-values]{Incorporating $\beta-values from Independence Tests into Score-Based Learning of Bayesian Network Structure}

\tableofcontents

\cleardoublepage
\listoffigures

\printnomenclature[1in]

\chapter{Introduction}  
Bayesian networks are graphical representations
which encapsulate relations of causation and correlation between different observables
or hidden factors of a complex system.  The purpose of such a representation is to facilitate both conceptual
understanding and automated probabilistic reasoning (prediction) concerning
such systems.  

Let $\mathcal{X}=(X_1,X_2,\ldots, X_n)$ be a collection
of random variables representing such observables or hidden factors.
We will concern ourselves with the case when each $X_i$ is discrete,
that is takes values in some finite set $\mathrm{Val}(X_i)$.  For any probability
distribution $P$ over $\mathcal{X}$ we can write the joint probability as a product
of conditional probabilities
\begin{equation}\label{eqn:chainRule}
P(X_1=x_1,\dots, X_n=x_n) = P(X_1 = x_1)\prod_{\nu =2}^n P(X_\nu = x_\nu | X_{1}=x_1,
\ldots, X_{\nu-1}=x_{\nu-1}).
\end{equation}
The factorization \eqref{eqn:chainRule} is called the \textbf{chain rule}.
In the most general probabilistic model for $\mathcal{X}$, the number of parameters involved is exponential
in the number of variables $n$.  Since it is
common to have $n$ in the hundreds or thousands, this makes basic tasks such as parameter estimation impractical.  
Perhaps more fundamentally, such an overly complex model gives no insight
into the structural relationships between the factors represented by the $X_i$.
In many real-world systems modeled by collections of random variables---especially those considered in artificial 
intelligence such as speech recognition, natural language processing, biological
and medical networks---even when the total number of factors (variables)
is large (e.g. in the thousands), each factor interacts \textit{directly} with only a handful of other factors in the system.  Thus,
each factor turns out to be (at least conditionally) independent of the vast majority of other
factors in the system, and a sparse representation of the system is preferable.  

Bayesian network representations of such systems
simultaneously capture these conditional independence relationships and make
the inference tasks we wish to carry out more tractable.  They do this by eliminating
a large proportion of the conditioning variables within each the conditional probability distribution factor appearing the chain rule \eqref{eqn:chainRule}.   

Formally, we may define a Bayesian
network over $\mathcal{X}$ as a pair $(G,P)$ where $G=(V,E)$ is a directed
acyclic graph (DAG), satisfying the following conditions: the nodes $V$ correspond to the $X_i\in\mathcal{X}$
and the directed edges $E$ correspond to direct influence among the nodes; 
$P$ is a probability distribution over $\mathcal{X}$ such that
all the correlations observed in $P$ are ``represented" by the influence
relations shown in $G$.  We have to give a more precise explanation of the statement about $P$ and $G$.
The most simple-minded interpretation of the DAG representation $G$ would be that pairs of random variables
not directly connected by an edge in $G$ are independent in $P$.  This interpretation is \textit{not}
correct. Rather, the correct interpretation of the DAG in terms of conditional independence
is given by the so-called \textbf{local Markov property}, which we state as follows:
\begin{quotation}
 \textbf{Local Markov Property of a Bayesian Network:} every variable
is \textit{conditionally} independent of its \textit{non-descendants} given its \textit{parents}.
\end{quotation}
It is not difficult
to see that the following is equivalent to the local Markov property: if
$X_1,\ldots, X_n$ represents topological ordering of the $X_i$
with respect to $G$, we may rewrite the chain rule \eqref{eqn:chainRule}, omitting
all conditioning factors $X_j$ except those which are parents of $X_{\nu}$:
\begin{equation}\label{eqn:chainRuleParents}
P(X_1=x_1,\dots, X_n=x_n) = \prod_{\nu =1}^n P(X_\nu = x_\nu | X_j=x_j\;\text{for all}\,
X_j\in\mathrm{Pa}_G(X_\nu)).
\end{equation}
Here $\mathrm{Pa}_G(X_\nu)$ denotes the parents of $X_\nu$
in the DAG $G$.

In general, a BN G is called an \textbf{independence map} (or \textbf{I-map}) for a distribution $P$ if all the (conditional) independence relationships
implied by $G$ are present in $P$ (\textit{i.e.}, $G$ is consistent with $P$); 
in that case case, $G$ is also called a \textbf{perfect map} (or \textbf{P-map}) for $P$
if the conditional independence relationships implied by $G$ are the only ones present in $P$ ($G$ is consistent with $P$ and
no model simpler than $G$ is consistent with $P$). 

We are concerned with the problem of learning the structure of Bayesian networks.
Chapters 16--18 of \cite{koller} give a very general framework for considering
this problem and an overview of the progress in this already well-studied problem.
We will not attempt to summarize all of the relevant literature, 
but here and in the comments below concerning
relations of our results to the literature, we will explain why
the problem is both difficult and interesting. The simplest case in which the problem of
learning the BN structure from data can arise is when we have two random variables,
which we will call $X_A$ and $X_B$.  It is not difficult to see that there are only two
nonequivalent BN structures:
\[
\begin{aligned}
G_0:& X_A\qquad X_B\; \text{(``disconnected")},\\
G_1:& X_A\longrightarrow X_B\; \text{(``connected")}.
\end{aligned}
\]
A word about the notion of equivalence of BN structures: 
changing the direction of the arrow changes the DAG $G_1$, but not the probability
distributions which can be represented by $G_1$.  This substantiates the statements that
all connected BN structures on $V=\{X_A,X_B\}$ are equivalent to the one, $G_1$, which we have depicted.
The problem of learning the structure from data in this instance is as follows:
devise a decision procedure which, given a sequence $\omega_N$ of $N$ independent samples (i.e., observations) 
of the joint distribution, returns $G_0$ in the case that $X_A$ and $X_B$ are independent
in the generating distribution, and returns $G_1$ in the case that $X_A$ and $X_B$ are dependent
in the distribution.   As in all learning problems, we cannot avoid the possibility of error,
but seek to bound the error probability by some $\delta>0$.  In other words,
in this case, the structure learning problem is identical to one case of \textit{hypothesis
testing}, a well-studied and classic problem in statistics.  

Before discussing the methods in any detail we note that even in this simplest of cases
a number of the subtleties of the situation are evident.  First, there are two different types of
errors: learning the network $G_1$ when the underlying network in $G_0$ and learning
the network $G_0$ when the underlying network is $G_1$.  Note that the first situation
deserves to be treated differently from the second, because whereas the independence relationships
implied by $G_1$ (namely, $\emptyset$) are \textit{consistent} with the model $G_0$, the independence
relationships implied by $G_0$ ($\{X_A\ci X_B\}$) are \textit{not consistent} with the model $G_1$.   In the terminology
of Bayesian networks, in the case the underlying network is $G_0$ and we learn a dependent distribution
(corresponding to $G_1$), we have learned an I-map of $G_0$ which is not a P-map.  This is
strictly speaking a model which is consistent with the data, but nevertheless a model which is more complex than necessary
(violating Occam's Razor).   In the case the underlying network is $G_1$ and we learn an independent distribution (corresponding to $G_0$) we have learned a distribution which is not even an I-map of the true network, which is a more clear-cut case of error.
In light of these consideration we clarify the scope of the problem considered in this paper:
\begin{quotation}
\textit{In this paper, we assume $\omega_N$ is a sequence of observations
of the random variables $\mathcal{X}$ generated i.i.d.\ from an unknown
Bayesian network $(G,P)$, with $P$ a probability distribution over $\mathcal{X}$
(no hidden variables or unobserved factors).  We interpret the problem of Bayesian
structure learning as recovering from $\omega_N$ a perfect map $G$ for $P$. } 
\end{quotation}
Thus, learning a model which is an I-map but not a P-map of the underlying generating Bayesian network (e.g. $G_1$
in the case the samples are being generated from $G_0$) is counted as an error in our framework.  Depending on the practical application, the experimenter may choose
to count this as a less ``costly" error than the error of learning a network which is not even an I-map.

So far, by considering only the very simplest case of a two-node network, we have avoided confronting
a major source of complexity in the BN structure learning task.  In the two-node case the Bayesian
network structure is determined (up to equivalence of BN's) entirely by the \textit{skeleton} $\mathrm{Skel}(G)$ (undirected graph
obtained by making all edges of the DAG $G$ undirected).  Already in the case of three nodes, this is no longer true.
For example, the following structures, 
\[
G: X_A\longleftarrow X_B\longrightarrow X_C\; \text{(``common cause")}
\]
and
\[
G': X_A\longrightarrow X_B  \longleftarrow X_C \; \text{(``common effect")},
\]
share the same (connected) skeleton $X_A\text{---}X_B\text{---}X_C$.  Yet they are definitely
\textit{not} equivalent because in the case of case of the ``common cause" network, we have
$X_A\not\ci X_C$, but $X_A\ci X_C | X_B$, whereas for the ``common effect" network, the situation is exactly reversed, which
is to say that $X_A\not\ci X_C | X_B$ but $X_A\ci X_C$.  (Intuitively, when $B$ is a common cause of $A$ and $C$,
observing $B$ decouples, or renders independent, the effects $A$ and $C$ in our probability estimates; when $B$ is a common effect of independent causes $A$ and $C$,
observing $B$ couples, or renders dependent, the causes $A$ and $C$ in our probability estimates). So even if we had an ``oracle"
which handed us the $\mathrm{Skel}(G)$, we would still need to perform additional (conditional) independence tests to distinguish
$G$ from the nonequivalent network $G'$.  Some methods break down the structure learning
problem into the problem of learning
the $\mathrm{Skel}(G)$, followed by the problem of orienting the edges of the DAG correctly, but we do not. 
Terminologically, the structure seen in $G'$ is called (for obvious reasons) a ``v-structure''.
The possible appearance of ``v-structures" at many locations of the DAG $G$ is responsible for much of the complexity of 
the edge-orientation task in this learning scheme.
 
By concentrating on these small-$n$ examples, we would not want to give the reader the impression that a scalable strategy for learning the structure
of Bayesian networks is to enumerate \textit{all} the possible DAG's on the vertices and to sift through them in succession until we find the one best-fitting the data.
For one thing, there are too many such structures: the number $a_n$ of DAG's on $n$ vertices is given \cite{robinson1977counting} by the recurrence relation
\[
a_n=\sum_{k=1}^n(-1)^{k-1} \binom {n} {k}2^{k(n-k)}a_{n-k},
\]
and $a_n$ has exponential growth in $n$.  In fact it can be shown (\cite{Chickering96lns} and \cite{Dasgupta:UAI99}) that the problem of BN structure learning is NP-complete, 
and moreover this remains true even if we restrict the space of BN's to be considered to those of bounded in-degree.  
Further, the result of a single independence test may constrain the space of possible BN's
in a rather complicated way, and the naive way of iterating through the structures to look for the right one does not
take advantage of this.  With a more sophisticated strategy one can achieve a running time of much better than exponential in most cases (at least if the
goal of learning the ``correct structure" is relaxed in certain ways, as we will describe).

Structure learning in Bayesian networks is already a well-studied problem,
and there are two classes of techniques that are frequently used.
The first approach is to construct a scoring function which assigns
a value to every possible structure, and then to find the structure which
maximizes the score \cite{Lam_MDL94,heckerman}. The scores measure
how well a Bayesian network with a given structure can fit the data.
The scores also include a complexity penalty which biases toward structures
with fewer edges. The simplest type of score built out of these general principles is the 
BIC (Bayesian Information Criterion) score:
\begin{equation}\label{eqn:BICScore}
S_{\psi_{1}}(\omega_{N},G)=LL(\omega_{N}|G)-\psi_{1}(N)\cdot|G|.
\end{equation}
Here $LL(\omega_N|G)$ is the log-likelihood
of the data given $G$, $|G|$ is the number of parameters and the most frequent choice of \textit{weighting} function is \[\psi_1(N):=\frac{\log N}{2},\]
in which case the score is also known as the \textit{minimum description length} (MDL) score.
This choice of $\psi_1(N)$ has a theoretical justification in terms of Bayesian probability (see, e.g., p. 802 of 
\cite{koller}).  If no assumptions are made about the data generating
distribution, finding the structure which maximizes the score is known
to be NP-hard \cite{Chickering96lns, Chickering:JMLR04, Dasgupta:UAI99}.
Many heuristic algorithms have been proposed for maximizing this score,
such as greedy hill-climbing \cite{Chickering02JMLRa, friedman99} and, 
more recently, by formulating the structure learning problem
as an integer linear program and solving using branch-and-cut \cite{Cussens11, JaaSonGloMei_aistats10}.

The second approach is based on statistical hypothesis tests, where
one seeks conditional independencies that appear to hold true in the
observed data. Using these, one can constrain the space of Bayesian
networks that could have generated the data.  Such
approaches are commonly referred to as \textit{constraint-based approaches}:
the results of the tests of conditional independence provide the logical
``constraints" on the model.  If we make the key assumption that each variable has at most a fixed number of parents $d$, then this approach can
yield an algorithm for structure learning, see \cite{spirtes,PearlVerma},
which runs in polynomial time in the parameter $n=\mathrm{Card}(V)$.
However, this approach has a number of drawbacks: difficulty setting
thresholds, propagation of errors, and inconsistencies. It is futhermore
difficult to use this approach together with other prior knowledge
about the network structure that we are attempting to learn.

Our approach combines the two main approaches to structure learning of \linebreak Bayesian
networks in a novel way.  In keeping with the approach
via statistical independence testing, we view the structure
learning problem as \emph{statistical recovery}, that of recovering
the true underlying structure from samples from its distribution.
In keeping with the approach via score functions, we incorporate the statistical
hypothesis tests into our framework via certain \textit{additional terms} added onto the BIC
score \eqref{eqn:BICScore}. The \textit{additional terms} help the score function
rule out incorrect skeletons and hone in more quickly on the correct skeleton 
than the BIC score (see \cite{brennerSontagUAI} for experimental evidence).  In contrast, the log-likelihood
and complexity penalty terms from the BIC score act as a Bayesian prior: 
with small sample size, they prevent the sparsity boost from
having undue influence, and with large sample size, they select a model which is
``close" to the true $G$ among those models sharing $\mathrm{Skel}(G)$ (the \textit{additional terms}
are not sensitive to the orientation of DAG edges).
\section{Objective function with data-driven sparsity boost}
Our algorithm consists of computing and optimizing the following objective
function, $S_{\psi_{1},\psi_{2}}$, which is parameterized by two
``weighting'' functions $\psi_{1}$ and $\psi_{2}$ and depends
on the graph structure $G$ and the data in the form of a sequence
$\omega_{N}$ of independent samples from the underlying distribution:
\[
\begin{split}
S_{\eta,\psi_{1},\psi_{2}}(\omega_{N},G)&=LL(\omega_{N}|G)-\psi_{1}(N)\cdot|G|
\\ &+\psi_{2}(N)\sum_{(A,B)\notin G}\max_{S\in S_{A,B}(G) }
\min_{s\in\mathrm{val}(S)}-\ln\left[\beta_{N}^{p^{\eta}}\tau(p(\omega_N, A,B|s))\right].
\end{split}
\]
Here $\mathrm{LL}(\omega_{N}|G)$ is the maximum log likelihood of observing
the data given the graph structure $G$, while $|G|$ is the number
of parameters in the model represented by $G$, so that the first
two terms of $S_{\psi_{1},\psi_{2}}$ are a BIC score with function
$\psi_{1}$ weighting the complexity penalty $|G|$. The third term,
which is our novel contribution to the scoring function, is perhaps
better thought of as a \emph{sparsity boost} (as opposed to a \emph{complexity
penalty}), which rewards $G$ for each \emph{nonexistent edge}. 

The
reward going to an $(A,B)\notin E(G)$ increases as evidence accumulates
that there is some set $S$ of conditioning variables which renders
random variables $A$ and $B$ \emph{independent.}  It is well known
(Section 3.4.3.1 of \cite{koller}) that among all possible parent
sets that could render nodes $A$ and $B$ independent, it suffices
to consider only the two sets $\mathrm{Pa_{G}}A\backslash B$ and
$\mathrm{Pa_{G}}B\backslash A$. In this paper, we generally, assume that in all the networks we consider
the number of parents of any node is bounded by a constant $d$.
So, if we wish to iterate algorithmically over all the ``witnesses" $S$
to the possible independence of $A$ and $B$, there are two natural possibilities:
first pair of possibilities
 $\{\mathrm{Pa_{G}}A\backslash B,\mathrm{Pa_{G}}B\backslash A \}$,
 and second the collection of all subsets of cardinality $d$ of $V^{\times 2}\backslash
 \{A,B\}$.

Now conditioning on each possible witness set $S$ on each assignment
$s\in\mathrm{Val}(S)$ of the variables in $S$, we obtain a potentially different
sparsity boost: how do we produce \textit{one single} sparsity
boost for $(A,B)$?
On the one hand, the edge $(A,B)$ should not
exist in the graph if conditioning on any \emph{one} of the relevant
$S$ renders them independent, so we take the\emph{ maximum }sparsity
boost over all relevant $S$. One the other hand, for the conditioning
set $S$ to render $A$ and $B$ independent, \emph{all }of the conditional
joint distributions of $A$ and $B$, for the various $s\in \mathrm{Val}(S)$, have
to be independent, so for fixed $S$ we take the minimum sparsity boost over all
$s\in\mathrm{Val}(S)$. 

Intuitively, we can explain the sparsity boost, $-\ln\beta_{N}^{p^{\eta}}\tau(\omega_{N})$, as follows:
$\tau=\tau(\omega_N)$ is the (conditional) mutual independence of the two variables calculated
on the basis of the evidence $\omega_N$. With infinite data, if the two variables really were independent, $\tau$
should be zero, that is $\tau(\omega_N)\rightarrow 0$ as $N\rightarrow\infty$. 
Suppose we had some quantitative notion, $\eta>0$, of the minimal amount of dependence that we can expect for any edge that truly exists in the Bayesian network. 
We want to ensure that we do not confuse this minimally dependent edge with independence. 
If the two variables really are independent, then with enough data, we would see that $\tau$
goes to zero whereas the mutual information of the minimally dependent edge goes to $\eta$. 
Even though $\tau(\omega_N)\rightarrow 0$ as $N\rightarrow\infty$, a minimally dependent edge can produce an empirical $\tau$ which is smaller than $\tau(\omega_N)$.
The distribution $p^\eta$ is a ``reference" distribution with certain minimal dependence
$\eta:=\tau(p^\eta)$.   Since for two independent variables $\tau(\omega_N)\rightarrow 0$
as $N\rightarrow \infty$, \textit{the event that the reference distribution $p^\eta$
emits a $Y_N$ with $\tau(Y_N)<\tau(\omega_N)$ has probability shrinking
to zero as $N\rightarrow\infty$. The quantity $\beta_{N}^{p^{\eta}}\tau(\omega_{N})$ inside the logarithm is precisely the probability of this event.  In case of true independence
of the variables the probability of this event should approach $0$ and its $-\ln$ should approach $\infty$ as $N\rightarrow\infty$; in the case of dependence of the variables
of strength at least $\eta$, it should remain bounded away from $0$ (because of the assumption of a minimal dependence), and its $-\ln$ should remain bounded from above.}
  
The connection to hypothesis testing may be expressed as follows: the test statistic 
$\tau$ is a Neyman-Pearson
test statistic for the decision between null hypothesis $H_{0}$, independence,
and the alternative hypothesis $H_{1}$.  This is because 
$\tau(\omega_N)$ is the log of the likelihood ratio of observing
$\omega_N$ under hypothesis $H_0$ versus $H_1$. 
We compute from $\tau(\omega_N)$ the probability 
$\beta_{N}^{p^{\eta}}\tau(\omega_{N})$ of Type II
error associated with the Neyman-Pearson test with threshhold $\tau(\omega_N)$.   
In the objective function, we use not $\tau$ itself, but the CDF
of $\tau$, considered as a random variable
on the set of outcomes $\left\{\omega_N\right\}$.  
The justification for using the CDF of $\tau$, namely,
 $\beta_{N}^{p^{\eta}}\tau(\omega_{N})$ is that it is the
 probability of accepting $H_{0}$ under the Neyman-Pearson test
with threshold $\tau(\omega_{N})$, when the data
$Y_{N}$ is drawn from a distribution $p^{\eta}$ whose true mutual information (namely $\eta$) is unknown to the experimenter
but is really $\eta$ (for some $\eta>0$).  The reason for using $\tau(\omega_N)$ is that it is the smallest threshold for which the Neyman-Pearson test 
with that threshold returns $H_0$
as its decision given the evidence $\omega_N$ as input.  Therefore, $\tau(\omega_N)$
can also be categorized as the threshold of the Neyman-Pearson 
test with smallest $\beta$ which accepts $H_0$ from the evidence $\omega_N$.  Conventionally, in statistics, one denotes the 
complement of the Type II error probability as the \textbf{power of the hypothesis test}, so that we have relationship:
\[
\textbf{Power} = 1-\beta.
\]   
From this we obtain the most powerful explanation of the use of $\tau(\omega_N)$: $\tau(\omega_N)$ is the threshold of the most \textit{powerful} Neyman-Pearson test which would classify the pair $A,B$, based on the empirical evidence, as independent.
As the power increases to $1$, the minus log of its complement---the sparsity boost 
$-\ln\beta(\tau(\omega_N))$---increases to $\infty$.

The Type II error $\beta_N^{p^\eta}$ should be contrasted to the Type I
error associated with the test.  In this context, the Type I error is defined as probability
$\alpha_{N}^{p^{\eta}}\tau(\omega_{N})$ of the Neyman-Pearson test of threshold $\tau(\omega_N)$
accepting $H_1$ when the data $Y_N$ is drawn from a \textit{product} distribution $p_0$.
We use $\beta_N^{p^\eta}$ instead of $\alpha_N^{p^\eta}$ because 
we wish to produce a \textit{sparsity boost for absent edges}, rather than 
a \textit{complexity
penalty for parameters}, in the candidate network.
Formally, $\beta_{N}^{p^{\eta}}\tau(\omega_{N})$ is a ``Type II error 
probability associated to the most powerful test with Type I error
$\alpha_{N}^{p^\eta}\tau(\omega_N)$'', but henceforth we 
refer to  $\beta_N^{p^\eta}\tau(\omega_N)$ more concisely as a ``$\beta$-value".

\section{Description of main results}
As one of our main results, we obtain finite sample complexity results, Theorems \ref{thm:twoNodeCaseOptimalAsymptotics},
\ref{thm:nNodeCase}, and \ref{thm:nNodeCaseSanov}.   
In the PAC learning framework (see e.g. \cite{mohri2012foundations})
a finite sample complexity result refers to a formula for a function
$N: (0,1)\rightarrow \mathbf{Z}$ (or in some cases an asymptotically equivalent elementary function) 
such that the algorithm learns the correct \textit{model} from data $\omega_{N}$
with $N=N(\delta)$ datapoints, with probability $1-\delta$.  There are a few caveats about applying the PAC
model to our results.  First, our finite sample complexity result is about learning not necessarily the correct
Bayesian network $(G,P)$, but about one ``close" to the correct network, in a certain sense.  Our result is stated in two parts: the first part is about the speed at which the algorithm learns a Bayesian network $(G',P')$ for which $P'$ is ``close''
to $P$ in a sense which we will specify below.  The second part is about
the speed at which the algorithm learns a network $(G',P')$ for which $G'$
has no \textit{false edges}, meaning edges of the skeleton of $G'$
which do not appear in the skeleton of $G$.  The reason for doing this is that the second part is independent of the problem-parameter $\zeta$,
whereas the first part is dependent on $\zeta$. Ultimately we will derive a result (Theorem \ref{thm:nNodeCaseAsymptotic})
which gives the rate at which the algorithm learns a network $(G',P')$ for which $\mathrm{Skel}(G')\subseteq \mathrm{Skel}(G)$
\textit{and} $\mathrm{d}(P',P)\leq \zeta$.  This result will, of course, depend on $\zeta$.

In addition to the parameter $\zeta=\mathrm{d}(P',P)$ just mentioned, the result depends on certain parameters
of the hidden underlying network $(G,P)$.  
First we need to assume that the contingency table representing any conditional joint distribution
 occurring in $G$, if not a product distribution, has a minimial separation
from the set of product distributions. Similarly to the practice in 
\cite{zhang2002strong}, this minimal separation is expressed
as a KL-divergence $\epsilon$ of the contingency table from the product
distribution sharing the same marginals.  Second, we need to assume
that certain marginalizations of $P$ stay away from the walls of the probability simplex.  This separation-from-the-walls assumption
 is expressed as a pair of integers $m, \hat{m}$.  First, $m$ is an integer such that the reciprocals of all entries
 of all contingency tables obtained by marginalizing over \textit{all possible parent sets} of size $d$
 are less than $m$.  Second, $\hat{m}$ is an integer such that the reciprocals of all entries
 of all contingency tables obtained by marginalizing over a certain
 restricted subset of the possible parent sets of size $d$ are less than $\hat{m}$.
 This restricted subset only has to be large enough that it contains
 a \textit{certificate of independence} for each of the pairs of vertices \textit{not}
 connected by an edge.  For a more complete explanation of \textit{separating sets for independence},
 see \ref{subsec:separatingSets}, below.  For now, we just note
 that $\hat{m}$ is in general smaller than $m$, potentially substantially smaller.  
 Then there is the error-tolerance parameter $\zeta$
 which is unlike $\epsilon \,\&\, m$ because it is under the control of the experimenter.
 Finally there are certain hyperparameters: first $\eta$ and $\kappa$
 which, while constrained to lie in certain intervals, are ``free," in the sense
 that otherwise they are up to the experimenter.
 Second, there are hyperparameters $\theta, \Theta,\lambda,\mu$,
 which can be freely chosen in the unit interval $(0,1)$,
 and which do not appear in the objective function, but only in the formula for $N$
 in the finite sample complexity result. In order to separate the different variables by role, we write the function $N$
 in the finite sample complexity, as
 \[
 N(\epsilon,m,\hat{m},n;\delta;\zeta;\eta;\kappa,\Theta,\theta,\lambda,\mu).
 \]
 Sometimes, for the sake of readability of formulas, we
 drop the hyperparameters ($\eta,\kappa,\Theta,\theta,\lambda,\mu$) from our notation for $N$.  
Then our main theoretical result, Theorem \ref{thm:nNodeCase}(a) and \ref{thm:nNodeCaseSanov} is a formula for $N$.  As a corollary of the main formula, we readily deduce (Theorem \ref{thm:nNodeCaseAsymptotic} (a) and (c)) that
\begin{multline*}
  N(\epsilon,m,\hat{m},n;\delta;\zeta)\in \tilde{O}\left(\max\left(
\frac{\log(n) m}{\epsilon^2},\left(\frac{n}{\zeta}\right)^2,\left(\frac{\hat{m}n}{\epsilon}\right)^2\right)\cdot
\log\frac{1}{\delta}
\right),\\ \;\text{as}\; \epsilon,\delta,\zeta\rightarrow 0^+,\;\text{and}\; m,\hat{m},n\rightarrow\infty.
\end{multline*}%
\nomenclature{$f(n) = \tilde{O}(g(n))$}{$f(n) = O(g(n) \log^k g(n))$ for some $k$}%
for $\psi_1(N)=\kappa\log (N)$ and $\psi_2\equiv 1$ and only binary random variables in the networks considered.
The $\psi_i$ would be relatively easy to change, but the binary random variable
assumption seems a bit trickier to remove.  We do, however, conjecture
that a similar result holds for networks with multivalued discrete variables.

Our other main results are computational.  In order to turn the formula
for the objective function into the basis of a practical algorithm,
we have to have access to a table of values of $\beta_N^{p^\eta}(\gamma)$, allowing
us to look up values for the $\eta, \gamma=\tau(\omega_N)$ that we compute
from the data.  Since a table can contain $\beta$ values for only
finitely many $\eta, \gamma,N$, our software library uses not only a look-up, but also 
an interpolation algorithm to obtain an accurate approximation for  $\beta_N^{p^\eta}(\gamma)$.  In Chapter \ref{sec:computationOfBeta}, we describe the essential ideas used in our computation and tabulation of (approximations for)
a large number of $\beta^{p^\eta}_N(\gamma)$, as well as the interpolation used
by the algorithm to generate its approximations.  We have
made publicly available both the source code and the tables of $\beta$'s generated by the code in \cite{Brenner:2013:Online}.

\vspace*{0.3cm}
\noindent \textbf{Relation to results in the literature.}  
Our results are limited to the problem of learning a Bayesian network \textit{without any hidden variables}, 
and from a data
series \textit{without any missing values}.  In other words, we assume that every point of $\omega_N$ is a $|V|$-vector, where $V$ is the number of variables in the generating network.  In contrast, our results
do deal with the issue of \textit{generalization error} in that Theorems \ref{thm:nNodeCase}(a) and 
\ref{thm:nNodeCaseAsymptotic}(a) and (c) guarantee that the learned distribution $P'$
lies $\zeta$-close (in KL-divergence) to the generating distribution $P$. 

For the benefit
of readers who are familiar with the results on the BIC score in the literature,
we make a few preliminary comments relating our results to these results.  There
are three ``representative'' types of results, those of \cite{hoffgen1993learning}, 
those of \cite{Friedman:UAI96}, and those of \cite{DBLP:conf/uai/ZukMD06} (in \cite{DBLP:conf/uai/2006}).  Our result is most like that of  \cite{hoffgen1993learning}
in that we put a bound on the in-degree of the learned network and achieved a result which is polynomial in
the total number of network variables.  This may be contrasted to \cite{Friedman:UAI96} where
the bound is exponential in the number of variables, and may be contrasted to \cite{DBLP:conf/uai/ZukMD06}
in that the finite sample complexity holds for the family of competing
networks simultaneously rather than for individual competitors.
Unlike \cite{hoffgen1993learning} though, we do recover (with high probability) the correct skeleton with the finite sample.
We will discuss these matters further, and also make a comparison between our approach and the previous ``hybrid" approaches to the learning problem in the literature, in Section \ref{sec:discussionLiterature} of Chapter \ref{chap:Discussion}, ``Discussion".  Also in that section, we will briefly discuss existing hybrid approaches such as \textsc{Relax} of \cite{fast2010learning}
and MMHC of \cite{mmhc}, and the points of differentiation of our approach from their
approaches.
\vspace{0.3cm}

\noindent \textit{Acknowledgments.}  The authors acknowledge Profs.\  Mark Tygert, Keivan Mallahi-Karai, Nicholas Loehr, and Francisco-Javier Sayas for
useful discussions.  The user ``Bernikov'' on Mathoverflow provided the proof of Lemma \ref{lem:LambertWasymptotics}.

\chapter{Finite Sample Complexity of the Scoring Function}
\label{chap:finitesamplecomplexity}
In the next section we give a concise overview of the notation and definitions needed to understand the finite sample complexity
result.  Also, in the course of the proof that follows we will invoke, without detailed explanations,
several classic results from information theory, statistics, and the Theory
of Large Deviations.  Readers wishing a more leisurely introduction to any of these
topics are referred to the following works: for general notions concerning Bayesian
networks and probabilistic graphical models, the first few chapters of \cite{koller}; for information
theory, \cite{cover2006elements}, particularly Chapter 11; for the Theory of Large Deviations, \cite{dembo2009large}, particularly
Sections 2.1 and 3.6; for hypothesis testing in general, Chapter 9 of \cite{degrootprobability}; and for hypothesis
testing with composite hypotheses \cite{hoeffding1965asymptotically}.
\section{Notation and Preliminaries}
\label{sec:preliminaries}
In all cases, the problem of learning Bayesian network structure involves examining a contingency
table representing a (possibly conditional)
probability distribution over the (Cartesian) product $X=X_A\times X_B$ of two random variables, $X_A$ and $X_B$.   
For any of these random variables, say $X$, we denote by $|X|$ or $\mathrm{card}(X)$ the cardinality
of $\mathrm{val}(X)$.  In the case when the Bayesian network being learned is a structure over discrete
variables, which is the case we are examining in this paper, all the numbers $|X|$ are finite.
By convention, we use $k$ to denote the cardinality
of $X_A$ and $l$ to denote the cardinality of $X_B$, and we use  
$\mathcal{P}=\mathcal{P}_{k,l}$ to denote 
the probability simplex, identified
with the collection of $k\times l$-contingency tables filled with nonnegative numbers summing to $1$. These quantities
are related by
\[
|X|=\mathrm{card}(X)=k\times l=\mathrm{dim}(\mathcal{P})+1.
\]  
In the case of two binary random variables,
to which we will often restrict ourselves in various proofs below, $k=l=2$, $|X|=4$, $\mathrm{dim}(\mathcal{P})=3$.%
\nomenclature{$X_A, X_B,\ldots$}{Random variables corresponding to nodes of a space of Bayesian networks}%
\nomenclature{$X$}{Product random variable $X_A\times X_B$}%
\nomenclature{$\mathrm{card}(X)$}{$\mathrm{card}(\mathrm{val}(X))$, assumed finite}%
\nomenclature{$\lvert X\rvert$}{$\mathrm{card}(X)$}%
\nomenclature{$k,\; l$}{$\lvert X_A\rvert,\; \lvert X_B\rvert$}%
\nomenclature{$\mathcal{P}=\mathcal{P}_{k,l}$}{Probability simplex, identified
with the collection of $k\times l$-contingency tables filled with nonnegative numbers summing to $1$.}%

By $p$ we will denote an element of $\mathcal{P}$, 
which the standard parameterization of $\mathcal{P}$ identifies with a set (or contingency table) 
of $|X|$ nonnegative numbers summing to $1$.  Considering a joint distribution $p\in\mathcal{P}$ as fixed, 
$p_A, p_B$ denote the first and second marginal distributions of $p$, identifiable as vectors whose entries equal the sums of the rows
(in the case of $A$, say) and the columns (in the case of $B$, say) of the entries in the contingency table $p$.
The entries of the contingency table, denoted $p_{i,j}$, $1\leq i\leq k$, $1\leq j\leq l$ 
are the probabilities of the atomic events $(i,j)$ under the joint distribution $p$.
Similarly, the entries of the vector $p_A$, denoted $pA_i$, are the probabilities 
of the atomic event $i$ for the marginal distribution $pA$; similarly for the entries $pB_j$ of $p_B$.
\nomenclature{$p$}{Element of $\mathcal{P}$; contingency table of $\lvert X\rvert$ nonnegative numbers summing to $1$}%
\nomenclature{$p_A, p_B$}{first and second marginal distributions of $p$}%
\nomenclature{$p_{i,j}$}{probability of the atomic event $(i,j)$ under the joint distribution $p$}%
\nomenclature{ $pA_i,\; pB_j$  }{$p(A=i)$, resp $p(B=j)$; probability of the atomic event $i$, resp. $j$, for the marginal distribution $pA$, resp. $pB$}%

We introduce notation for representing certain distributions of particular interest.  For example,
$\mathcal{P}_0$ denotes the submanifold (of dimension $k+l-2$) of $\mathcal{P}$ consisting of probability distributions 
which are products of their marginal distributions.
For $p_A$, $p_B$ as above $p_0(p_A,p_B)$ denotes the unique product distribution with first marginal distribution 
$p_A$ and second marginal distribution $p_B$.  In the case of binary marginal random variables, $p_A$ and $p_B$
are represented by a single parameter each, so we will write this as $p_0(x,y)$ for $x=p_{A,0}, y=p_{B,0}$, both in $[0,1]$.
For $x\in[0,1]$, by $p_0(x)=p_0(x,x)$, we denote the special case of a product distribution with equal marginals $pA=pB=x$.
Further, $p^0$ is the uniform distribution; $p^0=p_0(1/2,1/2)$ in the binary case.%
\nomenclature{$\mathcal{P}_0$}{submanifold of $P$ consisting of product distributions }%
\nomenclature{$p_0(p_A,p_B)$}{unique product distribution with first marginal distribution $p_A$ and second marginal distribution $p_B$}%
\nomenclature{$p_0(x,y)$}{$p_0(p_A,p_B)$ for $A$, $B$ binary and  $x=p_{A,0},\; y=p_{B,0}\in[0,1]$}%
\nomenclature{$p_0(x)$}{$p_0(x,x)$ for $x\in[0,1]$}%
\nomenclature{$p^0$}{uniform distribution, in binary case equals $p_0(1/2,1/2)$}

\subsection{Quantifying Edge Strength}\label{subsec:EdgeStrength}
One of the most fundamental functions to our approach is $\tau$, defined as the
\textbf{test statistic} ${\tau(p):=H(p\|\mathcal{P}_0)}$ of the \linebreak\mbox{(Hoeffding-)Neyman-Pearson} test between null hypothesis 
$\mathcal{P}_0$ and alternative hypothesis $\mathcal{P}-\mathcal{P}_0$.  
\begin{definition}  The mutual information (test statistic) $\tau(p):=H(p\| \mathcal{P}_0)$ is the Kullback-Leibler divergence from $p$
to $p_0$, where $p_0$ minimizes the KL-divergence $H(p\| p_0)$, subject to the constraint $p_0\in \mathcal{P}_0$.
\end{definition}
The distribution $p_0$ is also known as the $M$-projection of $p$ onto $\mathcal{P}_0$ and is known
(Chapter 8 of \cite{koller}) to be the unique product distribution sharing the same marginals as $p$.
While $\mathcal{P}_\gamma$ denotes the (disconnected) submanifold of $\mathcal{P}$ consisting of probability distributions $p$ such that 
${\tau(p)\geq \gamma}$, 
$A_\gamma^0$ denotes the (complementary) set of $p\in\mathcal{P}$ with $\tau(p)<\gamma$: we have
$\mathcal{P}_{\gamma}=\mathcal{P}-A_\gamma^0$. 
One of the parameters of the finite sample complexity, the \textit{strength} of an edge in the Bayesian network, denoted $\epsilon$,
is defined below in terms of the test statistic $\tau$.
\nomenclature{$\tau(p)$}{$H(p\lvert\lvert \mathcal{P}_0)$}
\nomenclature{$p_0$}{$M$-projection of $p$ onto $P_0$; $\mathrm{argmin}_{q\in P_0}H(p\lvert \rvert q)$}%
\nomenclature{$A_\gamma^0$}{ set of $p\in\mathcal{P}$ with $\tau(p)<\gamma$}%
\nomenclature{$\mathcal{P}_{\gamma}$}{$\mathcal{P}-A_\gamma^0$}%

Frequently we will consider, a distribution $p(t)$ lying on a \textit{path} in $\mathcal{P}$ based at $p$, consisting
of a one-parameter family of distributions, all sharing the same marginals.
Usually, but not always, the parameterization is based at a distribution $p\in \mathcal{P}_0$, meaning that $p(0)\in\mathcal{P}_0$.
In the case of binary random variables, there is up to reparameterization only one choice of $p(t)$,
and we will assume contingency tables of the following form:
\[
 p(t)=\begin{bmatrix}p_{00}+t & p_{01}-t\\
 p_{10}-t & p_{11}+t \end{bmatrix},\; t\in (-\min\left(p_{00},p_{11}\right), \min\left(p_{01},p_{10}\right)).
\]
In cases of $k$ or $l>2$, there are many possible choices for $p(t)$, which will have to be specified depending on the context.
For example, one choice is (in the below $i$ ranges only from $0$ to $2\lfloor \frac{k}{2}\rfloor-1$, and $j$ ranges only from $0$ to $2\lfloor \frac{l}{2}\rfloor-1$)
\[
 p(t)=
 p+\left[(-1)^{i+j}
 \right]_{i,j}t,\;
  t\in\left(
 -\min\left\{  p_{i,j}| i+j\;\mathrm{even} \right\}
 ,\min\left\{   p_{i,j}| i+j\;\mathrm{odd} \right\}
 \right).
\]
\nomenclature{$p(t)$}{ distribution belonging to a \textit{path} of distributions in $\mathcal{P}$ sharing the marginals of $p=p(0)$; such a one-parameter family of distributions
is said to be \boldmath \textbf{based at $p=p_0$} \unboldmath}%

Once the path $p(t)$ is defined (as it is now for the binary variables case), we can define
$p^\eta$, for $\eta>0$ a ``standard reference" probability distribution having test statistic $\eta$, in the following sense.  By definition,
\[
 p^\eta = p^0(t)\;\text{for the unique}\;t>0\;\text{such that} \;\tau(p^0(t))=\eta.
\]
In particular, $p^\eta$ has uniform marginals.  Of course, $p^\eta$ is defined only for the (interval of) \textit{realizable} values of $\eta$, depending on $k$ and $l$.
Also once we have defined a path and fixed a realizable KL-divergence $\gamma$,
$t_\gamma^+$, denotes the unique positive parameter for which $q^\gamma=q_0(t_{\gamma}^+)$ for 
$q^\gamma$ some other distribution (defined in the context) satisfying
$\tau(q^\gamma)=\gamma$ and $q_0\in \mathcal{P}_0$ sharing the marginals of $q^\gamma$.
When the ``standard reference" distribution $p^{\gamma}$ is defined and it is clear no other path $q(t)$ is meant, $t_{\gamma}^+$
denotes the unique positive parameter value for which $p^\gamma=p^0(t_\gamma^+)$.

The \textit{test statistic} gives us \textit{one} of the main tools that we need to quantify edge strength.  The \textit{other} is \textit{separating
collections for independence}, the subject of Section \ref{subsec:separatingSets}, where we will complete the quantitative definition
of edge strength.
\nomenclature{$p^\eta$}{$p^0(t)$ for the unique $t>0$ such that $\tau(p^0(t))=\eta$}%
\nomenclature{$t_\gamma^+$}{the unique positive parameter for which $q^\gamma=q_0(t_{\gamma}^+)$ where $q_0\in\mathcal{P}_0$}%

\subsection{A $\beta$-value for Independence Tests}\label{subsec:pValue}
We now define the most important quantity of all, $\beta^{p^\eta}_N(\gamma)$, which we call for the sake of concision \boldmath
 the \textbf{$\beta$-value}\unboldmath.  It is more accurate, though not feasible, to call this quantity 
the \boldmath \textbf{probability of type II error of the Hoeffding-Neyman-Pearson test with null hypothesis $P_0$ and alternative hypothesis $p^\eta$},
\unboldmath
for it is defined as
\[
 \beta^{p^\eta}_N(\gamma):=\mathrm{Pr}_{\omega_N\sim p^\eta}\;\{\tau(p_{\omega_N})\leq \gamma\}.
\]
In order to apply these ideas of hypothesis testing to the setting of Bayesian
networks, we have to recall operations on joint probability distributions that
are considered frequently in all parts of the study of BN's.  
In the most general setting we are given a Bayesian network $(G,P)$, 
with $G$ the directed acyclic graph (DAG) $(V,E)$.  Let 
$S\subset V$, and $s\in \mathrm{Val}(S)$, $A,B\in V-S$.  
An empirical sequence $\omega_N$ of $N$ samples from the network, without any missing data,
is a set of $|V|$-tuples of values, with the coordinate element in the $|V|$-tuple corresponding to $A$ (say) coming from $\mathrm{Val}(X_A)$. 
Taking frequency counts
of these $|V|$-tuples, and normalizing by the sample size, we obtain the empirical joint probability distribution $p_{\omega_N}$ (also written as $p(\omega_N)$).  
Let $p$ be a joint distribution over the $|V|$ variables in $G$, for example, arising from frequency counts in the manner just described.
Then $p_{A,B}$ denotes the probability distribution over the pair of variables $X_A$, $X_B$, obtained by marginalizing out
all other variables (if any) in the network.  Further, for a joint assignment $s$ of the variables $S$,
we use $p_s$ to denote the probability distribution obtained by restriction, i.e., conditioning on this joint assignment.
Then, starting with joint distribution $p=p_{\omega_N}$ over \textit{all} the variables,
$(p_{A,B})_s$, also written as $p(A,B|s)$, denotes the probability distribution associated to the contingency table of the
variables $X_A$ and $X_B$,
conditioned on the joint assignment $s$ to the variables in $S$.  Since the multiple subscripts involved
in writing all this out are too cumbersome, in practice, when we consider several of these
operations applied in succession to $\omega_N$, we write them inside parentheses at normal text levels,
so that $p(\omega_N,A,B|s)$ is the joint distribution of $X_A$ and $X_B$ derived from taking the normalized
frequency counts of the empirical sequence $\omega_N$, conditioning
on $s$, marginalizing out variables other than $X_A$, $X_B$.%
\nomenclature{$\beta^{p^\eta}_N(\gamma)$}{$\mathrm{Pr}_{\omega_N\sim p^\eta}\;\{\tau(p_{\omega_N})\leq \gamma\}$}%
\nomenclature{$(G,P)$}{Bayesian network (generating empirical distribution)}
\nomenclature{$(G',P')$}{Bayesian network (competing)}
\nomenclature{$(V,E)$}{Directed acyclic graph (DAG)}
\nomenclature{$A,B,\ldots$}{Vertices in $V$}
\nomenclature{ $\omega_N$}{Empirical sequence sampled from $(G,P)$}
\nomenclature{$p_{\omega_N}$}{Probability distribution over $V$ obtained by taking normalized frequency counts from $\omega_N$}
\nomenclature{$p(A,B\lvert s)$}{Probability distribution produced by conditioning $p$ (over all $V$) on $S=V-\{A,B\}$}
\nomenclature{$p(\omega_N,A,B\lvert s)$}{Distribution $p(A,B\lvert s)$ associated to $p=p_{\omega_N}$}

\subsection{Separating Sets for Independence}\label{subsec:separatingSets}
In order to formulate the objective function and the finite sample complexity result, we also must
introduce some notions related to \textit{families} of Bayesian networks.  Let $\mathcal{G}$
be any subset of all the possible Bayesian networks on a set of random
variables $X_V$ indexed by a vertex set $V$.  We refer to $\mathcal{G}$ as a \textbf{family
of Bayesian networks}.  We consider the functions $S_{A,B}$ on $\mathcal{G}$, indexed by pairs
$A,B\in V$, taking values in \textit{subsets of the power set of $V-\{A,B\}$}.  In other words, using $V^{\times 2}\backslash \Delta$
to denote the Cartesian square of the vertex set $V$ minus the diagonal, we have
\[
\mathscr{S}= \left\{S_{A,B}:\mathcal{G}\rightarrow 2^{2^{|V|-2}}\;|\; A,B\in V^{\times 2}\backslash \Delta\right\}
\]
is a mapping that takes a pair of \textit{distinct} vertices $A,B$ in the subscripted places, and an element $G\in \mathcal{G}$
in the remaining place, and outputs a subset of the power set of $V-\{A,B\}$.
It is usually more convenient to think of the parameters $A,B\in V$ ($A\neq B$) as fixed and consider one mapping,
which is denoted in a shorthand notation by
\[
 S_{A,B}:\mathcal{G}\rightarrow 2^{2^{|V|-2}},
\]
where $|V|-2$ is understood to mean $V-\{A,B\}$. 
Note that we informally call such a collection of mappings $\mathscr{S}$ a \boldmath\textbf{collection of eligible separating sets for $\mathcal{G}$.} \unboldmath
\nomenclature{$\mathcal{G}$}{a family of BNs, defined as any subset of all possible BNs on a set of random variables $X_V$}
\nomenclature{$V^{\times 2}\backslash \Delta$}{Cartesian product of $V$ with itself, minus the diagonal $\Delta$.}

Here is the context for families $\mathcal{G}$ of families of BN's and collections
of eligible separating sets for $\mathcal{G}$: suppose that we are given a family $\mathcal{G}$ and $\omega_N$ generated from 
the underlying, but unknown network $G\in\mathcal{G}$ and wish to learn $G$ (or a network close to $G$).
In particular we often consider $\mathcal{G}=\mathcal{G}^d$ consisting of all candidate networks such that
every vertex has at most $d$ parents, for some positive integer $d$. 
Let $S=\{S_{A,B}\;|\; (A,B)\in V^{\times 2}\backslash \Delta\}$ 
be a collection of eligible separating sets for $\mathcal{G}$.  Common examples of $S_{A,B}$, include
\begin{itemize}
 \item $S_{A,B}(G) = \{ S\subset V\backslash \{A,B\}\;|\; |S|\leq d \},$
 \item $S_{A,B}(G) = \{ \mathrm{Pa}_G(A)\backslash B,\,\mathrm{Pa}_G(B)\backslash A \}$.
\end{itemize}
Note that the first $S_{A,B}(G)$ is actually independent of $G$.  Now we can make the following very important definition.
\begin{definition}\label{defn:separatingSet}
A collection \[\mathscr{S}=\left\{S_{A,B}:\mathcal{G}\rightarrow 2^{2^{|V|-2}}\;|\; A,B\in V^{\times 2}\backslash \Delta\right\}\] 
is called \boldmath \textbf{a separating collection for $G$ in $\mathcal{G}$} \unboldmath provided that for all
$(A,B)\in V^2\backslash \Delta$ we have
\[
 (A,B)\notin G \;\text{if and only if there exists}\;S\in S_{A,B}(G)\;\text{such that}\; A\ci B\, | \,S. 
\]
Here, as usual ``$(A,B)\in G$'' is shorthand for ``$(A,B)$ belongs to the edge set $E(G)$ of G''.  In this case, for 
any $G'\in\mathcal{G}$, the value $S_{A,B}(G')$ (which is itself a subset of the power set of $V-\{A,B\}$) is called a \textbf{separating
set}.
\end{definition}
\nomenclature{$\mathscr{S}$}{\boldmath \textbf{a separating collection for $G$ in $\mathcal{G}$} \unboldmath as in Definition \ref{defn:separatingSet}}
\nomenclature{$\mathscr{S}(G')$}{\boldmath \textbf{a separating set for $G$ and $G'\in\mathcal{G}$} \unboldmath as in Definition \ref{defn:separatingSet}}

The intuition behind this definition is that a separating set $\mathscr{S}$ \textit{separates}
$G$ from competing \textit{undirected structures} $\mathrm{Skel}(G')$
for all $G'\neq G$ in $\mathcal{G}$.
A separating collection $\mathscr{S}$ for $G$ in $\mathcal{G}$ provides a \textit{certificate for statistical recovery of $G$}.
We may think of the quantity $\epsilon$ as the \textit{strength of this certificate}, or alternatively,
as the \textbf{margin of the separating set} \boldmath $\mathscr{S}$ \unboldmath for $G\in \mathscr{G}$:
\[
 \epsilon = \epsilon(G, \mathcal{G}, \mathscr{S}): = \min_{(A,B)\in G}\;\min_{S\in S_{A,B}(\mathcal{G})}\;\max_{s\in \mathrm{Val}(S)}
 \tau(p(A,B|s)).
\]
where
\[
 S_{A,B}(\mathcal{G})\;\text{denotes}\; \cup_{G\in \mathcal{G}}S_{A,B}(G).
\]
The smaller $\epsilon$ is, the more ``noise'' there is in the network $G$, 
and the more data is needed to distinguish $G$ from the networks in $\mathcal{G}$ which differ from $G$
by dropping one or more edges from $G$.  Put another way, the reciprocal of $\epsilon$ is a measure
of the ``noise'' in $\omega_G$ considered as evidence for the (undirected) structure $\mathrm{Skel}(G)$.  
We will refer to $\epsilon$ as the \boldmath\textbf{minimum edge strength of the network $G$ with respect to $\mathcal{G}$ and $\mathscr{S}$}.
\unboldmath  The minimum edge strength will be an important factor in the finite sample complexity result.
\nomenclature{$S_{A,B}(\mathcal{G})$}{$\cup_{G\in \mathcal{G}}S_{A,B}(G)$}
\nomenclature{$\epsilon(G, \mathcal{G}, S)$}{ $\min_{(A,B)\in G}\;\min_{S\in S_{A,B}(\mathcal{G})}\;\max_{s\in \mathrm{Val}(S)}$}
\nomenclature{$\mathrm{Skel}(G)$}{Skeleton of $G$; undirected graph corresponding to DAG $(V,E)$
underlying $G$}

In addition to $S_{A,G}(\mathcal{G})$, defined above, we also have the following derived quantities
defined in terms of $\mathscr{S}$.  First, we define
\[
\Sigma_{\mathscr{S}}(G) := \max_{(A,B)\in G}|S_{A,B}(G)|,
\] 
and 
\[
\Sigma_{\mathscr{S}}(\mathcal{G}) := \max_{(A,B)\in V^{\times 2}}|S_{A,B}(\mathcal{G})|,
\]
which are two related notions of the \textit{number of possible separating sets
for a given edge}.  We also define
\[
\sigma_{\mathscr{S}}(G):=\max_{(A,B)\in G}\max_{S\in S_{A,B}(G)}|S|,
\]
which is the maximal size of any separating set in the collection for $G$.
In our sample complexity results, we will make the assumption that 
$\sigma_{\mathscr{S}}(G)$ and $\Sigma_{\mathscr{S}}(\mathcal{G})$
are at most polynomial in $n$ of degree at most $d$ 
(typically, $\sigma_{\mathscr{S}}(G)\leq d$ and $\Sigma_{\mathscr{S}}(\mathcal{G})<\binom{n}{d}$).

\subsection{New Objective Function}\label{subsec:newObjective}
At this point, we have introduced all the components making
up $S_{\eta}(G,\omega_N)$, the score function of the network $G$ on the empirical sequence $\omega_N$, which we may
now define as,
\[
\begin{split}
S_{\eta,\psi_{1},\psi_{2}}(\omega_{N},G)&=LL(\omega_{N}|G)-\psi_{1}(N)\cdot|G|
\\ &+\psi_{2}(N)\sum_{(A,B)\notin G}\max{}_{S\in S_{A,B}(G) }
\min_{s\in\mathrm{val}(S)}-\ln\left[\beta_{N}^{p^{\eta}}\tau(p(\omega_N, A,B|s))\right].
\end{split}
\]
The first line of the definition
is the BIC score, and the second line is the sparsity boost.  It is generally assumed that $\psi_1(N)\rightarrow\infty$
but that $\frac{\psi(N)}{N}\rightarrow 0$ as $N\rightarrow\infty$.  We mainly consider the case when $\psi_2(N)\equiv 1$,
although other choices are possible.  When $s$ is fixed it is much more convenient to write
$\tau(\omega_N)$ in place of the more correct but cumbersome $\tau(p(\omega_N, A,B|s))$.

The finite sample complexity result will be stated in terms of the Lambert-W function, 
in particular in terms of the branch $W_{-1}$, which is uniquely characterized as the unique
functional inverse of $ze^z$ which takes real, negative values between $-1/e$ and $0$ (see \cite{corless1996lambertw} for details).
In order to extract a finite sample complexity in terms of more elementary functions, we need to make a further examination of the asymptotic behavior
of the expressions involving $\mathcal{W}$ defined by
\[
 \mathcal{W}(x):=-W_{-1}(-x),\;\text{for}\, x\in \left( 0, \frac{1}{\epsilon} \right).
\]
The main thing that the reader has to understand about $\mathcal{W}(x)$ for our purposes
can be summed up as follows
\begin{quote}
as $x\in\left(0,\frac{1}{\epsilon}\right)$ approaches $0$ from the right, $\mathcal{W}(x)$
is pretty much the same as $-\ln(x)$.
\end{quote}  For the precise statements that are needed in our Theorems, see Lemmas \ref{lem:LambertWasymptotics}
and \ref{lem:LamberWasymptoticsApplied} in Section \ref{sec:TechnicalLemmas}.

\subsection{Error tolerance}\label{subset:errorTolerance}
Since, similar to the situation in \cite{Friedman:UAI96}, our finite sample complexity
results will allow for the learning algorithm to return some structure ``$\zeta$-close to"
the generating structure $(G,P)$, we have to explain how to quantify the distance between BN structures.
First, fix $\zeta>0$.  What we say here could in principle be stated for any notion of distance, but we will
adopt the most common notion of distance between distributions, the KL-divergence, and thus $\zeta$
will always denote a KL-divergence between distributions.

Suppose that $\mathcal{X}$ is an arbitrary variable set (e.g. the set corresponding to the vertices $V$)
and $B$ is an arbitrary distribution over $X$ (typically $B$ is the underlying distribution).  In order
to specify the notion of a \textit{distance of a BN structure $(G',P')$ over $\mathcal{X}$ from $B$}, the main notion that we
need is the distribution $p_{G',B}$.  Those familiar with the field may already know $p_{G',B}$ as the ``entropy minimizing"
distribution among distributions which factor according to $G'$.  Alternatively, $p_{G',B}$ is the M-projection
of $B$ onto the set of distributions which factor according to $G'$ (have $G'$) as an I-map.  Those
wishing more background and context for $p_{G',B}$ may see Appendix \ref{app:relEntropyLogLikelihood}.
The characterization of $p_{G',B}$ as an M-projection is most useful for our purposes
because it points out that $p_{G',B}$ is the solution to a constrained minimization problem where $H(B | p)$
is the objective being minimized and where the constraint is that $G'$ is an I-map for $p$.  Thus, we define
the distance of $G'$ from $B$ as the value of this minimum, namely $H(B | p_{G',B})$.

In order to derive from this notion of distance a notion of the distance of $G'$ from a Bayesian network $(G,P)$,
we simply adapt the above definition by setting $B$ equal to a distribution such that $(G,B)$ satisfies the \textit{faithfulness assumption}:
in other words $G$ is a perfect map (P-map) of $B$:
\begin{definition}
Let $(G,B)$, $(G',P')$ denote a pair of Bayesian networks over a variable set $\mathcal{X}$ satisfying the faithfulness assumption,
that $G$ is a perfect map for $B$ and $G'$ is a perfect map for $P'$.
Then the \boldmath\textbf{distance of $G'$ from $G$}\unboldmath \hspace*{0.03cm} 
is defined as the divergence of $p_{G',B}$ from $B$,
namely $H(B | p_{G',B})$.  Fix $\zeta>0$.  We say that a \boldmath\textbf{learning algorithm or decision procedure returns a learning
algorithm $\zeta$-close to $(G,B)$}\unboldmath \hspace*{.03cm} if it returns a network $(G',P')$ whose distance from $(G,B)$ is less than $\zeta$.
\end{definition}
The quantity $\zeta$, not our $\epsilon$, corresponds to the $\epsilon$
of \cite{Friedman:UAI96}.  Our $\epsilon$ is more closely (though somewhat more general) than the minimum ``information content''
considered in \cite{DBLP:conf/uai/ZukMD06} and \cite{zhang2002strong}.

\section{Concentration of Entropy and Mutual Information}
Let $X(p)$ be a Bernoulli random variable with parameter $p\in[0,1]$.
Let $\tilde{p}_N$ be an empirical estimate of the parameter $p$ derived from $N$
samples drawn independently from $X(p)$.  The Law of Large Numbers and Theory
of Large Deviations assure us that $\tilde{p}_N\rightarrow p$ almost surely
as $N\rightarrow\infty$ and give us upper bounds on the probability
of $\tilde{p}_N$ being ``far" from $p$ for specific values of $N$. 
In all parts of the proof below, we will need similar, quantitative results concerning the (probable)
convergence of estimates to various quantities derived from such estimates of $p$, for example,
the empirical estimates of the entropy of a (possibly multinomial) distribution to the true entropy,
and the mutual information between two marginal distributions to the true mutual information.  
Since the basic term appearing in both the entropy and mutual information is of the form $p\log p$,
our first task will be to combine the standard results of the Theory of Large Deviations with some analysis
to develop an upper bound on the probability of $\tilde{p}_N\log \tilde{p}_N$ being ``far" from $p\log p$ for specific values of $N$.

The techniques we are now going to apply can be used
much more generally to transform the estimates from the Theory
of Large Deviations into  effective estimates
concerning the convergence of $f(\tilde{p}_N)$ to $f(p)$ for $f$
various specific bounded, continuous functions.  Naturally,
the case we are primarily interested in is the case $f(x)=x\log x$
so all our efforts will be focused on that case, leading up to Proposition 
\ref{prop:plogp}, below.

\begin{lem}\label{lem:ChernoffMultiplicative} \textbf{Multiplicative form of the Chernoff bounds}.  Let $p$ be the parameter
of a Bernoulli random variable $X(p)$.  Let $\tilde{p}_N$ be the empirical
estimate of $p$ derived from the frequency counts of a sequence of $N$ samples drawn i.i.d. from $X(p)$.
\begin{enumerate}
\item  The probability that $\tilde{p}_N$ exceeds the expected value $p$ is bounded as follows:
\[
\mathrm{Pr}\left\{ \frac{\tilde{p}_N}{p} \geq (1+\epsilon)  \right\}\leq e^{-Np\epsilon^2/2}.
\]
\item  The probability that $\tilde{p}_N$ falls short of the expected value $p$ is bounded as follows:
\[
\mathrm{Pr}\left\{ \frac{\tilde{p}_N}{p} \leq (1-\epsilon)  \right\}\leq e^{-Np\epsilon^2/3}.
\]
\end{enumerate}
\end{lem}
\begin{proof}
This is \textit{one} of the forms of the Chernoff inequality, given for example as
Theorem 9.2 in \cite{kearns1994introduction}.
\end{proof}
It is very easy to deduce a two-sided version from the one-sided Chernoff bound. 
\begin{lem}\label{lem:ChernoffApplication}
With the notation as above we have
\[
\mathrm{Pr}\left\{\left|  1-\frac{\tilde{p}_N}{p} \right|\geq \epsilon \right\}\leq 2e^{-Np\epsilon^2/3}.
\]
\end{lem}
 For the proof of Lemma \ref{lem:ChernoffApplication} from
 Lemma \ref{lem:ChernoffMultiplicative}, see Section \ref{sec:TechnicalLemmas} below.

Our aim in the following is to prove an effective version of Lemma 1 from \cite{hoffgen1993learning}, which we will achieve in Proposition \ref{prop:plogp} 
below.  In order to do that, the next proposition will bound the probability of large deviations of $f(x)=x\log x$. 
The main techniques we will use in the proof are the Chernoff bounds (Lemma \ref{lem:ChernoffMultiplicative} above)
and the Mean-value Theorem.  Let us note that the principal difficulty in obtaining the result is the need to obtain
a bound which is independent of $p$, a task which is complicated by the fact that as $p\rightarrow 0^+$, we have $f'(p)\rightarrow-\infty$.  In Appendix \ref{app:relEntropyLogLikelihood}, we prove a much simpler ``one-sided" bound, Proposition \ref{prop:plogpPositiveDeviation}, which applies only when $\tilde{p}>p$.  The proof of that proposition makes use of the well-known inequality 
\begin{equation}\label{eqn:tangencyInequality}
\log(x)\leq x-1.
\end{equation}
(with equality at the point of tangency $x=1$); the multiplicative Chernoff bounds; and elementary analytic estimates.  The following
two-sided bound, is the only one we use in the sequel.

\begin{proposition}\label{prop:plogpNegativeDeviation}
Let $\epsilon\geq 0$, $p\in [0,1]$.  Denote by $\tilde{p}_N$ the empirical average of a sequence drawn from the Bernoulli distribution $X(p)$.  
Then we have 
 \[
  \mathrm{Pr}\left\{ 
\left|
\tilde{p}_N\log\tilde{p}_N-p\log p
\right|\geq \epsilon
 \right\}
\leq
3\mathrm{exp}\left(
-N\min\left[
\frac{ \epsilon^2}
{3(1+\log 2)^2},
\frac{1}{24},
\frac{\exp(W_{-1}(-\epsilon))}{12}
\right]
\right).
  \]
Note that this bound is \textit{independent of $p$}.
\end{proposition}
\begin{proof}
Let $f(x)=x\log x$.  Our task is to bound the probability under $\tilde{p}_N\sim X(p)$ of the event
\[
\left\{|f(\tilde{p}_N)-f(p)|\geq \epsilon\right\}.
\]
We consider two cases, first when $\tilde{p}_N\geq \frac{1}{2}p$ 
(we use the constant $\frac{1}{2}$ for convenience), and second when $\tilde{p}_N< \frac{1}{2}p$.  
The reason for breaking up the argument into these two cases is that if $\tilde{p}_N\geq \frac{1}{2}p$, then the derivative of $f(x)$, which is
\[
f'(x)=1+\log x,
\]
can be bounded, in absolute value, on the interval with endpoints $p$ and $\tilde{p}$, in terms of $p$, whereas when $\tilde{p}_N<\frac{1}{2}p$, this cannot be done because $f'(x)$
approaches negative infinity as $x\rightarrow 0^+$.

So assuming that $\tilde{p}_N\geq \frac{1}{2}p$, note that the interval with endpoints $\{\tilde{p}_N,p\}$ is a subinterval
of $(\frac{p}{2},1)$.  The derivative of $f'(x)$, that is $f''(x)$, equals $\frac{1}{x}$, which is positive on the entire unit interval.
Thus, the maximum of the absolute value $f'(x)$ on any subinterval of the entire unit interval can occur only at the 
endpoints.  That is to say, 
\[
|f'(p')|\leq \max\left\{  |f'(1)|,\left|f'\left(\frac{p}{2}\right)\right|  \right\}\;\text{for all}\;p'\in\left(\frac{p}{2},1\right).
\]
The mean value theorem says that there is a $p'\in\left(\frac{p}{2},1\right)$ satisfying the mean value property:
\[
f(\tilde{p})-f(p)=(\tilde{p}_N-p)f'(p').
\]
Now, take the absolute value of both sides, assume that $|f(\tilde{p}_N)-f(p)|\geq \epsilon$, and solve for $|\tilde{p}_N-p|$, to obtain
the inequality
\[
|p-\tilde{p}_N|\geq \frac{\epsilon}{\max(|f'(1)|, |f'(p/2)|)}=\frac{\epsilon}{\max(1, |1+\log p-\log 2|)}
\]
Dividing both sides by $p$, we obtain
\[
\left|\frac{\tilde{p}_N}{p}-1\right|\geq\frac{\epsilon}{p\max\{ 1,|1+\log p + \log 2| \}}.
\]
Thus far we have shown that
\[
|f(\tilde{p}_N)-f(p')|\geq \epsilon \Rightarrow \left| \frac{\tilde{p}_N}{p} -1 \right|
\geq \frac{\epsilon}{p\max\left\{ 1, |1+\log p + \log 2| \right\}}.
\]
Using Lemma \ref{lem:ChernoffApplication} (which applies independently of any assumption $\tilde{p}_N<p$!) 
we have the estimate
\begin{equation}\label{eqn:estimateptildelarge}
\mathrm{Pr}\left\{ |f(\tilde{p}_N)-f(p)|\geq \epsilon,\;\&\; \tilde{p}_N\geq \frac{1}{2}p \right\}\leq 2\mathrm{exp}\left(
\frac{-N\epsilon^2}
{3p\max\{1,|1+\log p+\log 2|^2\}}
\right).
\end{equation}
In order to obtain a bound independent of $p$, we have to find the global maximum of the function appearing
in the denominator, namely
\[
g(x)=x\max\{ 1,|1+\log x + \log 2|^2 \}
\]
on the unit interval.
By solving the equation $1+\log x + \log 2=1$, we eliminate the $\max$ and obtain that $g(x)$ is piecewise defined in terms of elementary functions by
\[
g(x)=\begin{cases}
x(1+\log x+\log 2)^2& 0<x<\frac{1}{2}e^{-2}\\
x                            & \frac{1}{2}e^{-2}\leq x<\frac{1}{2}\\
x(1+\log x+\log 2)^2 & \frac{1}{2}\leq x<1.
\end{cases}
\]
We calculate
\[
g'(x)=\begin{cases}
(\log(2) + \log(x) + 1)^2 + 2\log(2) + 2\log(x) + 2,& \;\text{if}\;0<x<\frac{1}{2}e^{-2}\\
 1,                            & \;\text{if}\;\frac{1}{2}e^{-2}\leq x<\frac{1}{2}\\
(\log(2) + \log(x) + 1)^2 + 2\log(2) + 2\log(x) + 2,& \;\text{if}\;\frac{1}{2}\leq x<1
\end{cases}.
\]
\begin{figure}\label{fig:gandgPrime}
\hspace*{-1.3cm}\mbox{\subfigure{\includegraphics[width=3in]{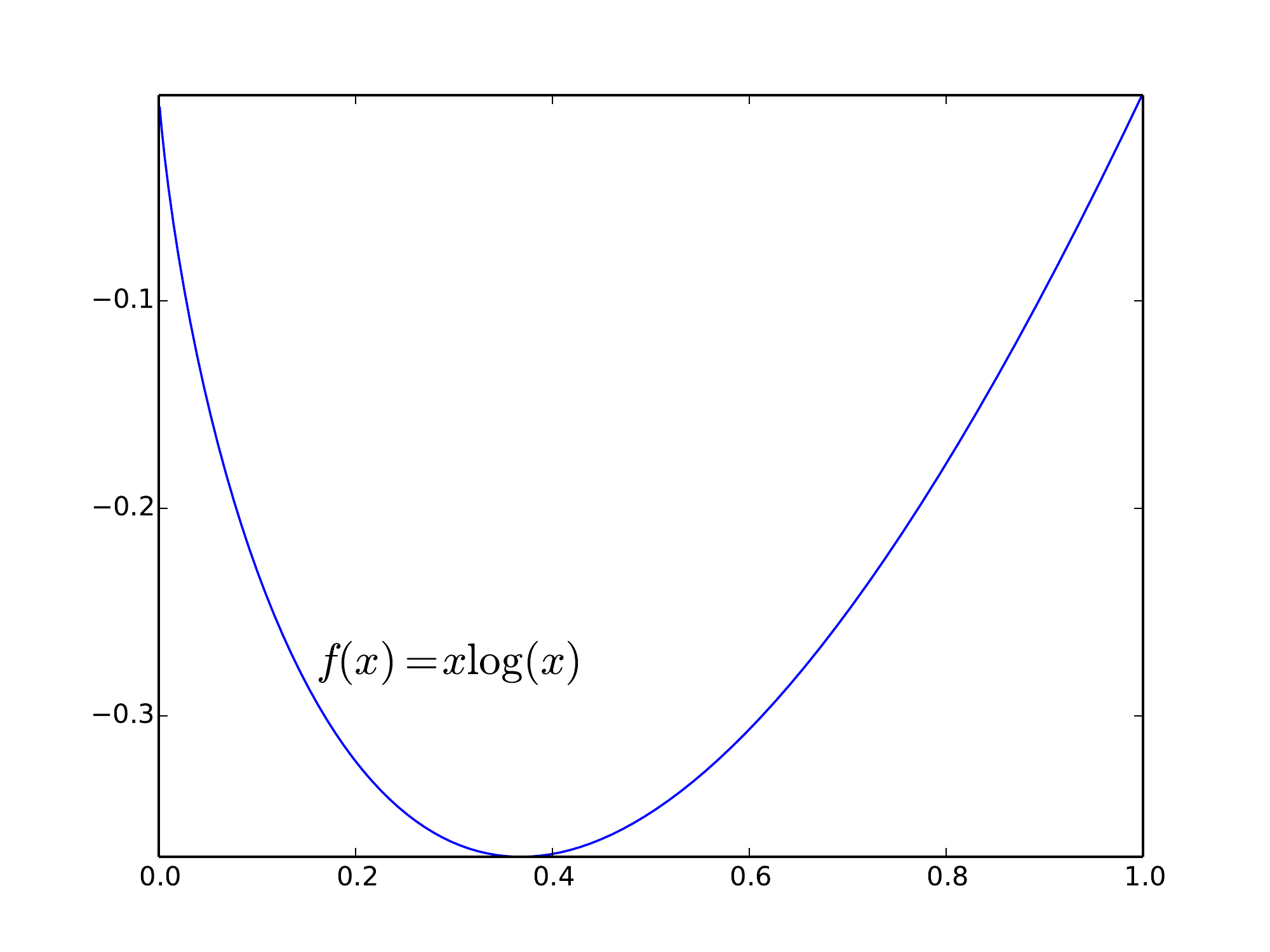}}\quad
\subfigure{\includegraphics[width=3in]{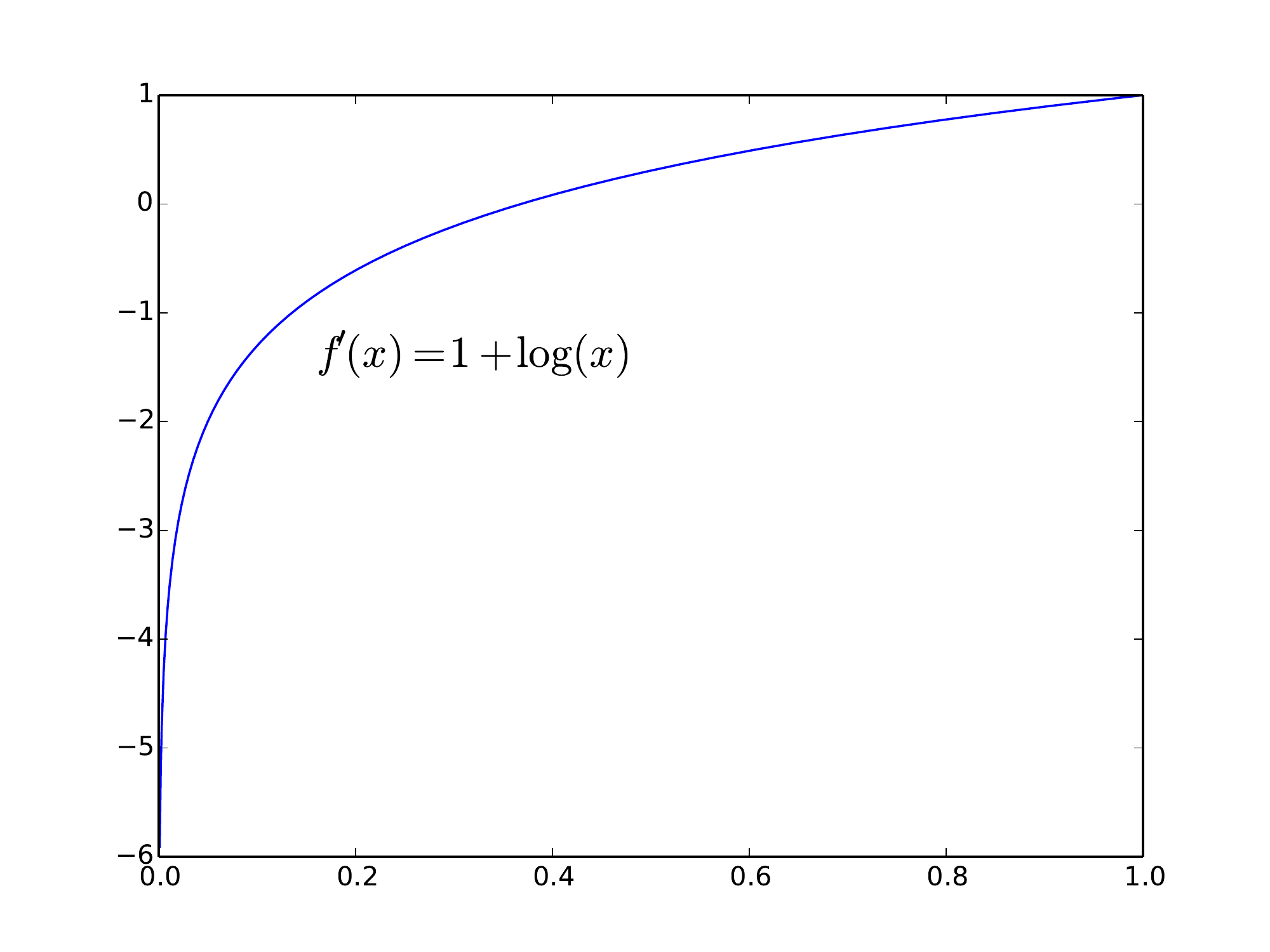} }}
\hspace*{-1.3cm}\mbox{\subfigure{\includegraphics[width=3in]{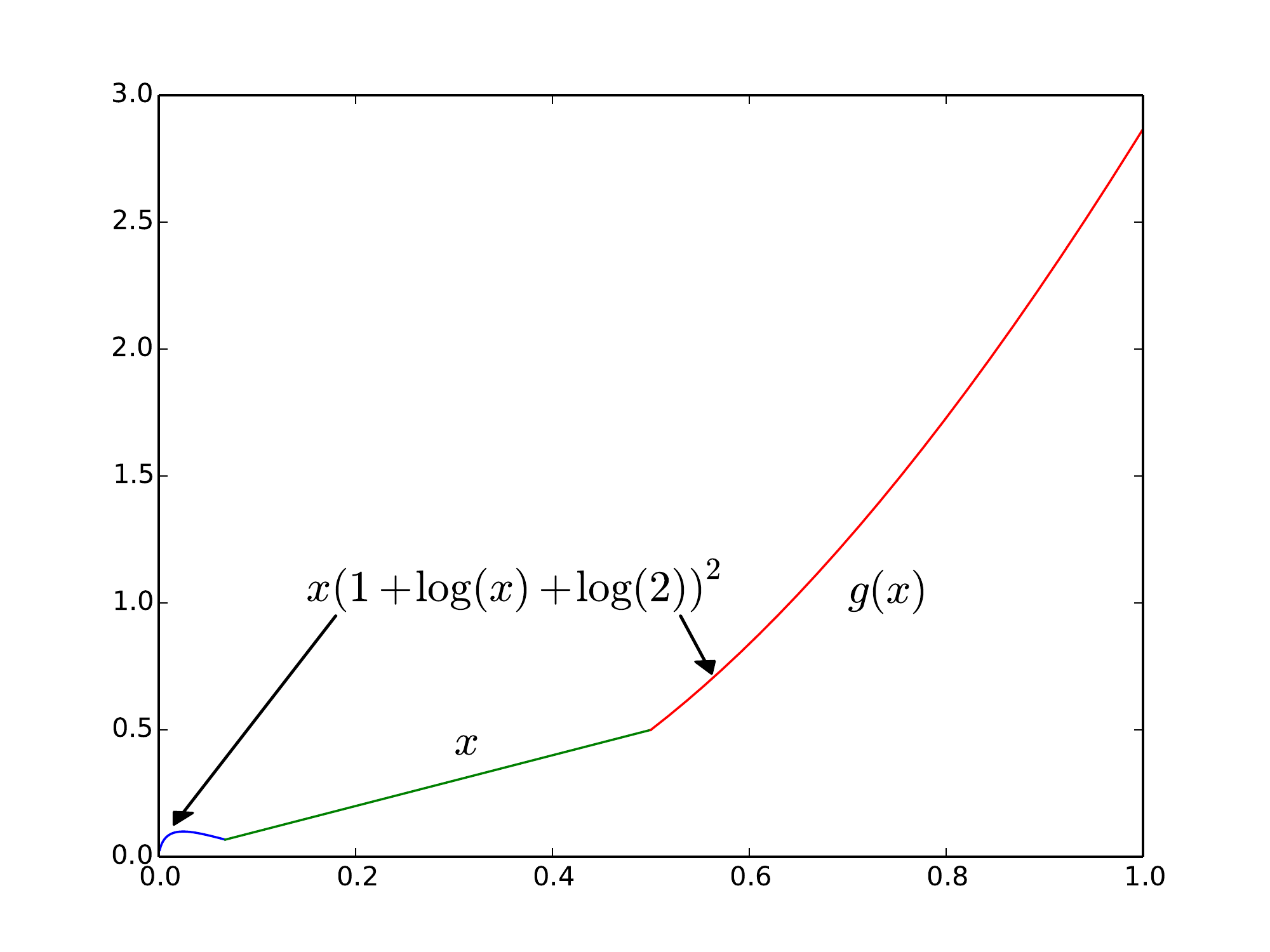}}\quad
\subfigure{\includegraphics[width=3in]{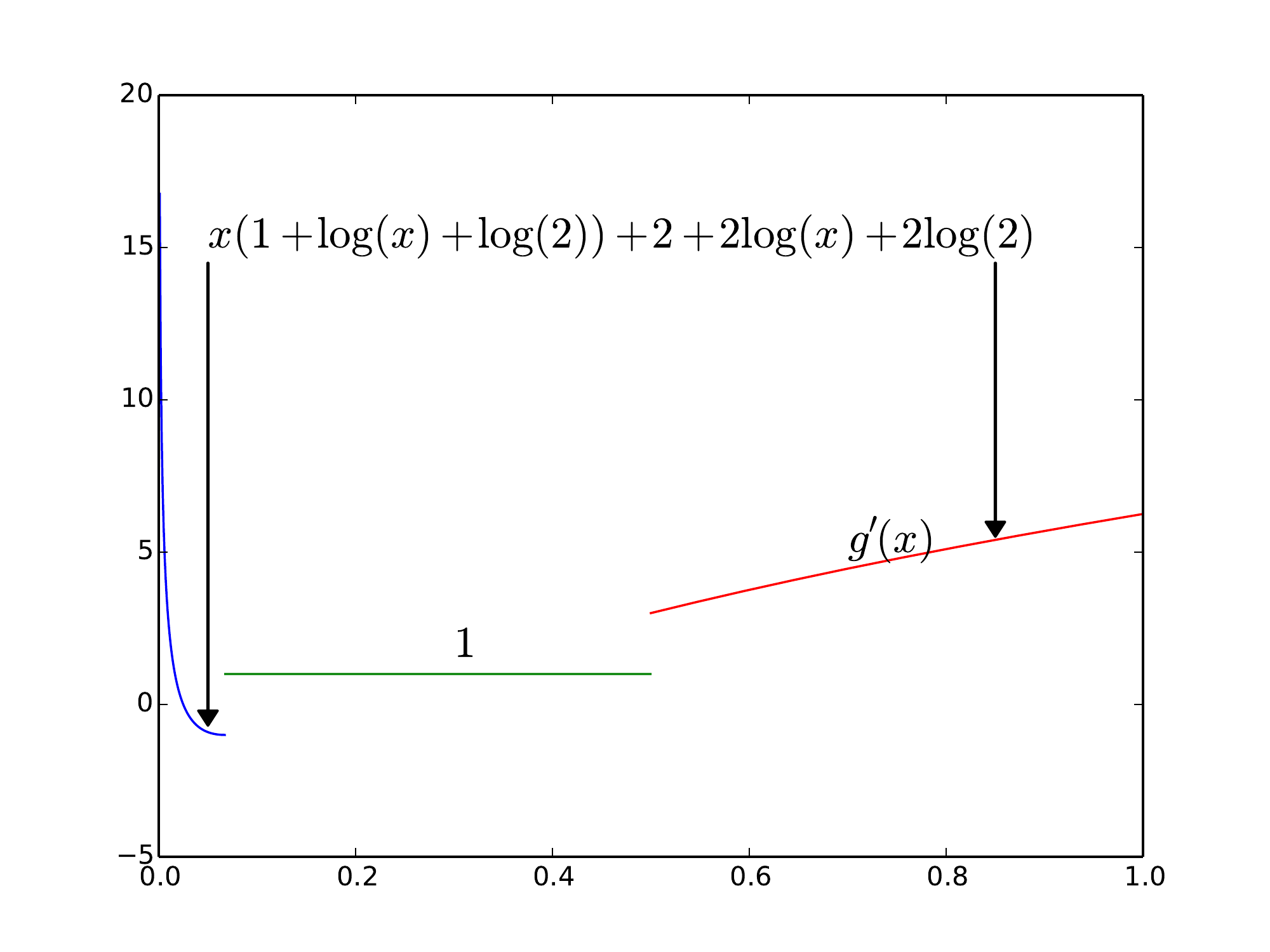} }}
\caption{\textit{Top left.}  Graph of $f$ illustrating that  if we fix $\tilde{p}_N=\frac{1}{2}p$, then $p=\frac{1}{2}$ is the unique value in the unit interval
at which $f(\tilde{p}_N)=f(p)$.  \textit{Top right.}  Graph of $f'$ illustrating its monotonicity.  \textit{Bottom left.}  
Graph of $g$ illustrating its continuity and the location of its unique global maximum on the unit interval at the right endpoint $1$.  \textit{Bottom right.}  
Graph of $g'$ illustrating that its only zero is at $\frac{1}{2}e^{-3}$. }
\end{figure}

The function $(\log(2) + \log(x) + 1)^2 + 2\log(2) + 2\log(x) + 2$ has zeros at $\frac{1}{2}e^{-3}$
and $\frac{1}{2}e^{-1}$.  Of these, only the first is a zero of $g'(x)$.  
By substituting in values of $x$ from the intervals, we see that $g'(x)$ is positive on $\left( 0,\frac{1}{2}e^{-3} \right)$,
negative on $\left( \frac{1}{2}e^{-3},\frac{1}{2}e^{-2} \right)$, and positive on $\left(\frac{1}{2}e^{-2},1\right)$.
Therefore, by one-variable differential calculus,
$g(x)$ is increasing on $(0,\frac{1}{2}e^{-3})$, decreasing on $(\frac{1}{2}e^{-3},\frac{1}{2}e^{-2})$, and increasing
on the intervals $(\frac{1}{2}e^{-2},1)$.  Further, note that whereas $g'(x)$ is not continuous,
$g(x)$ is continuous (which is clear from the definition of $g(x)$).  The continuity of $g(x)$,
combined with the preceding observation of the slope of $g(x)$ on the different intervals 
allows us to rule out $\frac{1}{2}$ as a possible local extremum of $g(x)$.
As a result, the only possible
points for the maximum of $g(x)$ on the interval are $\frac{1}{2}e^{-3}$ and $1$.  The
values of $g(x)$ at these points are $2e^{-3}$ and $(1+\log 2)^2$.  Of these two values, only
$(1+\log 2)^2>1$, so the maximum of $g(x)$ on the interval $[0,1]$ is $(1+\log 2)^2$, attained only at the right endpoint $1$.  
Using this maximum value to bound the denominator of \eqref{eqn:estimateptildelarge}, we obtain
\begin{equation}\label{eqn:estimateptildelargePindependent}
\mathrm{Pr}\left\{ |f(\tilde{p}_N)-f(p)|\geq \epsilon\;\&\; \tilde{p}\geq \frac{1}{2}p \right\}\leq 2\mathrm{exp}\left(
\frac{ - N\epsilon^2}
{3(1+\log 2)^2}
\right).
\end{equation}
Turning our attention now to the case $\tilde{p}_{N}<\frac{1}{2}p$, Chernoff's bound yields the estimate
\begin{equation}\label{eqn:estimateptildesmall}
\mathrm{Pr}\left\{\tilde{p}_N<\frac{1}{2}p \right\}\leq 
\mathrm{exp}\left(
\frac{-pN}{12}
\right).
\end{equation}
The problem is that the $p$ in the numerator of the exponent \eqref{eqn:estimateptildesmall} may render
this bound useless if $p$ approaches $0$.  In order to ensure that this does not happen, we use the following
claim:
\begin{quotation}\textit{Claim.}  Assuming that $p\leq \frac{1}{2}$ and 
$\tilde{p}_N\leq\frac{1}{2} p$, $|f(p)-f(\tilde{p}_N)|>\epsilon$ implies that $p\geq\exp(W_{-1}(-\epsilon))$, where
$W_{-1}$ is unique real non-principal branch of the Lambert W function on $(-\frac{1}{e},0)$.  
\end{quotation}
First, note that if we fix $\tilde{p}_N=\frac{1}{2}p$, then $p=\frac{1}{2}$ is the unique value in the unit interval
at which $f(\tilde{p}_N)=f(p)$.  Next, note that $f(x)$ is decreasing from $x=0$ to $x=e^{-1}$, where it attains minimum of $f(e^{-1})=-e^{-1}$, and
thereafter increases from $x=e^{-1}$ to $x=1$.  Consequently, continuing to enforce the equality $\tilde{p}_N=\frac{1}{2}p$
letting $p$ decrease from $\frac{1}{2}$ to $0$, we have $f(\tilde{p}_N)\geq f(p)$ for all $p$ between $0$ and $\frac{1}{2}$.  Finally, keeping $p$
fixed and letting $\tilde{p}_N$ vary in the interval $\left(0,\frac{1}{2}p\right)$, we also have $f(\tilde{p}_N)\geq f(p)$ for all $p\leq \frac{1}{2}$.  So if $|f(p)-f(\tilde{p}_N)|>\epsilon$,
then actually $f(\tilde{p}_N)-f(p)>0$, and so $f(\tilde{p}_N)>f(p)+\epsilon$.

Suppose, in order to obtain a contradiction, that $p< \exp(W_{-1}(-\epsilon))$.  Then $\log p>W_{-1}(-\epsilon)$, so that, applying the function
$ze^z$ (which is the functional inverse of $W_{-1}$) to both sides, we have $p\log p> -\epsilon$.  That is, $f(p)>-\epsilon$.
Consequently, since we are also assuming $f(\tilde{p}_N)>f(p)+\epsilon$, we have $f(\tilde{p}_N)>0$.  But this is impossible, since $f(\tilde{p}_N)<0$ for
$\tilde{p}_N\in (0,1)$.  So contrary to assumption, we have $p\geq \exp(W_{-1}(-\epsilon))$, and this completes the proof of the claim. 
  
From \eqref{eqn:estimateptildesmall}, the Claim, and the fact that the possibilities $p<\frac{1}{2}$ and $p\geq \frac{1}{2}$ are exhaustive,
mutually exclusive possibilities, we then obtain
\begin{equation}\label{eqn:estimateptildesmallPindependent}
\mathrm{Pr}\left\{|f(\tilde{p}_N)-f(p)|\geq \epsilon,\;\&\; \tilde{p}_N<\frac{1}{2}p  \right\}\leq 
\max\left(
\mathrm{exp}\left(
\frac{-N}{24}
\right),\mathrm{exp}\left(-
\frac{\exp(W_{-1}(-\epsilon))N}{12}
\right)
\right).
\end{equation}
Now, we have two bounds, both independent of $p$, for the probability of $\left\{ \tilde{p}_N<\frac{1}{2}p \right\}$
for the possibilities of $\tilde{p}_N\geq\frac{1}{2}p$, \eqref{eqn:estimateptildelargePindependent}, and  $\tilde{p}_N<\frac{1}{2}p$,
\eqref{eqn:estimateptildesmallPindependent}.  Since the two conditions on $\tilde{p}_N$, namely 
$\tilde{p}_N\geq \frac{1}{2}p$ and $\tilde{p}_N < \frac{1}{2}p$, are exhaustive,
\begin{multline}
\mathrm{Pr}\left\{|f(\tilde{p}_N)-f(p)|\geq \epsilon \right\}\leq \\
2\mathrm{exp}\left(
\frac{ - N\epsilon^2}
{3(1+\log 2)^2}
\right)+\max\left(
\mathrm{exp}\left(
\frac{-N}{24}
\right),
\mathrm{exp}\left(-
\frac{\exp(W_{-1}(-\epsilon))N}{12}
\right)
\right),
\end{multline}
an upper bound which is less than or equal to the upper bound given in the Proposition. 
\end{proof}

We are now able to derive an effective version of the bound in \cite{hoffgen1993learning}.
\begin{proposition}\label{prop:plogp}
Let $\epsilon, \delta>0$ be given.  Let
\[
N(\epsilon, \delta)=
 \max\left[
\frac{3(1+\log 2)^2}
{ \epsilon^2},24,
\frac{12}
{\exp(W_{-1}(-\epsilon))}
\right]\cdot 
\log\frac{3}{\delta}.
\]
Then 
\[
\mathrm{Pr}\left(|\tilde{p}_N\log\tilde{p}_N - p\log p|\geq \epsilon\right) \leq \delta.
\]
Note that $N=O((\frac{1}{\epsilon})^2\log\left(\frac{1}{\delta}\right))$ as $\epsilon,\delta\rightarrow 0$.
\end{proposition}
\begin{proof}
Solve for $N$ in the bound of Proposition \ref{prop:plogpNegativeDeviation} in order to obtain the bound on $N_n(\epsilon,\delta)$.
In terms of the notation we have introduced concerning the Lambert W function, the term involving the Lambert W 
may be written as $12e^{\mathcal{W}(\epsilon)}$.  By exponentiating Lemma \ref{lem:LambertWasymptotics}, we see that as $\epsilon\rightarrow 0^+$,
 $e^{\mathcal{W}(\epsilon)}$ goes to $\infty$ at the rate $O\left(\epsilon^{-1}\left(\log \epsilon^{-1}\right)\log\log\left(\epsilon^{-1}\right)\right)$.  Note that we actually obtain a tigher asymptotic for $N$
as $\epsilon,\delta\rightarrow 0$ than that stated in \cite{hoffgen1993learning}, which states the asymptotic
$N=O\left( (\frac{1}{\epsilon})^2\log\left( \frac{1}{\epsilon}\right)^2\log\left(\frac{1}{\delta}\right)\right)$.
\end{proof}

Now we are going to use Proposition \ref{prop:plogpNegativeDeviation} to estimate from below the quantity
\[
\beta^{p^\eta}_N(\gamma)\;\text{for}\; \gamma>\eta>0.
\]
The reason that the Proposition applies to estimate this quantity is that, setting
\[
\Delta = \gamma - \eta > 0,
\] 
we have
\begin{equation}\label{eqn:betaAsLargeDeviation}
\begin{aligned}
\beta^{p^\eta}(\gamma)&=\mathrm{Pr}_{Y_N\sim p^{\eta}}\{\tau(Y_N)<\gamma\} \\
                                     &=1-\mathrm{Pr}_{Y_N\sim p^{\eta}}\{\tau(Y_N)\geq\gamma\}\\
                                     &\geq 1- \mathrm{Pr}_{Y_N\sim p^{\eta}}\{  |\tau(Y_N) - \eta | \geq \Delta \}.
\end{aligned}
\end{equation}
There is a well-known relation expressing mutual information in terms of entropies (see e.g. (2.45), p. 21, of \cite{cover2006elements}),
which in our notation says that
\begin{equation}\label{eqn:mutualInformationInTermsOfOrdEntropy}
\tau(p) = H(p_A) + H(p_B) - H(p)\;\text{for all} \;p\in \mathcal{P}.
\end{equation}
In other words, the mutual information of the joint is the sum of the entropies of the marginals minus the entropy of the joint.
Further, each entropy on the right side of \eqref{eqn:mutualInformationInTermsOfOrdEntropy} can be expressed
as the sum of several terms of the form $p_i\log p_i$,  for $p_i$ an event in some Bernoulli distribution.  We can use this
observation concerning the entropy to prove the following.
\begin{lem}\label{lem:empiricalEntropyError}
Let $\gamma > \eta > 0$.  Let $\Delta:=\gamma-\eta$, so that $\Delta>0$.  Let $\mathcal{F}_N(\Delta)$ be the function of $\Delta$
defined as follows.  
\begin{equation}\label{eqn:scriptFNDefn}
\mathcal{F}_N(\Delta) := 24\mathrm{exp}\left(
-N\min\left[
\frac{ \Delta^2}
{3\cdot 64(1+\log 2)^2},\,
\frac{1}{24},\,
\frac{1}{12\exp(\mathcal{W}\left(\frac{\Delta}{8}
\right))}
\right]
\right).
\end{equation}
\nomenclature{$\mathcal{W}(\cdot)$}{$-W_{-1}(-\cdot)$, $W_{-1}$ the real non-principal branch of the Lambert W-function}%
Assuming that $p^\eta\in \mathcal{P}_{2,2}$ (i.e. both of the marginals of $p^\eta$ are \textit{binary} random variables), we have
\begin{equation}\label{eqn:LargeDeviationOfMIEstimate}
\mathrm{Pr}_{Y_N\sim p^{\eta}}\{  |\tau(Y_N) - \tau(p^\eta) | \geq \Delta \} \leq \mathcal{F}_N(\Delta).
\end{equation}
\end{lem}
\begin{proof}
Using \eqref{eqn:mutualInformationInTermsOfOrdEntropy} and the comments 
following that equation, we have that $\tau(p^\eta)$ (resp. $\tau(Y_N)$) is the sum, with appropriate signs, of a
total of $8$ terms of the form $p\log p$ (resp. $\tilde{p}_N\log\tilde{p}_N$).  In order to obtain the desired estimate, we must use the
function of Proposition \ref{prop:plogpNegativeDeviation} with $\epsilon = \frac{1}{8}\Delta$, so that we obtain
an upper bound on the probability of the event,
\begin{multline*}
\mathrm{Pr}\left\{ |\tilde{p}_N\log\tilde{p}_N - p\log p|\geq \frac{\Delta}{8}  \right\}\leq  \\
3\mathrm{exp}\left(
-N\min\left[
\frac{ \Delta^2}
{3\cdot 64(1+\log 2)^2},
\frac{1}{24},
\frac{\exp(W_{-1}(-\Delta/8))}{12}
\right]
\right).
\end{multline*}
We must also estimate the probability of $8$ events in this way, so from the union bound, we have a factor of $24$
outside the exponent.  
\end{proof}

\begin{proposition}\label{prop:mutualInformationLargeDeviation}  
Let $\gamma > \eta > 0$.  Let $\Delta:=\gamma-\eta$, so that $\Delta>0$.  Let $\mathcal{F}_N(\Delta)$ be the function of $\Delta$
defined as above.
Assuming that $p^\eta\in \mathcal{P}_{2,2}$ (i.e. both of the marginals of $p^\eta$ are \textit{binary} random variables), we have
\[
\beta_N^{p^\eta}(\gamma) \geq 1-\mathcal{F}_N(\Delta).
\]
\end{proposition}
\begin{proof}
Continuing from \eqref{eqn:betaAsLargeDeviation}, which bounds $\beta$ from below in terms of a large deviation, we have 
\begin{equation}\label{eqn:betaIntermsOfLargeDeviation}
\beta^{p^\eta}(\gamma)\geq1- \mathrm{Pr}_{Y_N\sim p^{\eta}}\{  |\tau(Y_N) - \tau(p^\eta) | \geq \Delta \}.
\end{equation}
In order to lower-bound $\beta_N^{p^\eta}(\gamma)$, we have to upper bound the probability of the empirical mutual information
$\tau(Y_N)$ deviating from the true value $\eta=\tau(p^\eta)$ by more than $\Delta$.  The Proposition follows immediately from  \eqref{eqn:LargeDeviationOfMIEstimate} and \eqref{eqn:betaIntermsOfLargeDeviation}.
\end{proof}
At this point we collect a number of useful facts concerning $\mathcal{F}_N(\Delta)$ and related
functions.  First, in order to simplify the formulas, make the abbreviation
\[
K_1=3\cdot 64(1+\log2)^2.
\]
\nomenclature{$K_1$}{$3\cdot 64(1+\log2)^2$}
Next, set
\[
F(\Delta)=\min\left[
\frac{\Delta^2}{K_1},\,
\frac{1}{24},\,
\frac{1}{12\exp\mathcal{W}\left(
\frac{\Delta}{8}
\right)}
\right].
\]
Further, set
\[
\tilde{F}(\Delta)=\min\left[
\frac{\Delta^2}{K_1},\,
\frac{1}{12\mathcal{W}\left(
\frac{\Delta}{8}
\right)}
\right].
\]
\nomenclature{$\tilde{F}(\Delta)$}{$\min\left[\frac{\Delta^2}{K_1},\,\frac{1}{12\mathcal{W}\left(\frac{\Delta}{8}\right)}\right]$}
Clearly we have the following relation
\[
F(\Delta) = \min\left[
\frac{1}{24},\,
\tilde{F}(\Delta)
\right],
\]
\nomenclature{$F(\Delta)$ }{$\min\left[\frac{1}{24},\,\tilde{F}(\Delta)\right]$}
and, by \eqref{eqn:scriptFNDefn}
\[
\mathcal{F}_N(\Delta)=24\exp(-NF(\Delta)).
\]
\nomenclature{$\mathcal{F}_N(\Delta)$}{$24\exp(-NF(\Delta)$}
and the positivities,
\[
F(\Delta), \tilde{F}(\Delta), \mathcal{F}_N(\Delta)>0.
\]
Further we define
\begin{equation}\label{eqn:scriptFtildedefn}
\tilde{\mathcal{F}}_N(\Delta):=24\exp(-N\tilde{F}(\Delta)).
\end{equation}
\nomenclature{$ \tilde{\mathcal{F}}_N(\Delta)$ }{$24\exp(-N\tilde{F}(\Delta))$  }
The point of this definition is that while $F_N(\Delta)$, and consequently $\mathcal{F}_N(\Delta)$ 
is constant as a function of $\Delta$ for large $\Delta$, $\tilde{F}_N(\Delta)$ and $\tilde{\mathcal{F}}_N(\Delta)$ are monotonic.
\begin{lem}\label{eqn:FunctionsMonotonicity}
On $\mathbf{R}^+$,  $\tilde{F}_N(\Delta)$ is increasing and $\tilde{\mathcal{F}}_N(\Delta)$ is decreasing. 
\end{lem}
\begin{proof}
We have that $\tilde{F}_N(\Delta)$ is the minimum of two functions and we show that these
two functions are increasing.  The derivatives of the two functions with respect to $\Delta$ are
\[
\frac{\partial}{\partial\Delta}\left(\frac{\Delta^2}{K_1}\right)=\frac{2\Delta}{K_1}>0,
\]
and 
\[
\frac{\partial}{\partial\Delta}\frac{1}{12\exp\mathcal{W}\left(
\frac{\Delta}{8}
\right)}=
-\frac{1}{8\cdot 12}\left[
\frac{1}{\exp\mathcal{W}\left(\frac{\Delta}{8}\right)}
\right]^2\exp\mathcal{W}\left(\frac{\Delta}{8}\right)\cdot
\mathcal{W}'\left(\frac{\Delta}{8}\right)>0.
\]
The reason for the second expression's being positive is that
$\mathcal{W}'$ is decreasing on $\mathbf{R}^+$, so that its derivative
is negative.

Thus we have that $\tilde{F}_N(\Delta)$ is the minimum of two increasing functions.
Further, except at the unique point at which these two functions are equal $\tilde{F}_N(\Delta)$
equals one of the two increasing functions, so is increasing.  At the unique point
at which the two functions are equal, $\tilde{F}_N(\Delta)$, though not differentiable,
is still continuous, so it is not difficult to see that $\tilde{F}_N(\Delta)$ is also increasing at that point.
From \eqref{eqn:scriptFtildedefn} $\tilde{\mathcal{F}}_N(\Delta)$
is decreasing,
\end{proof}
Therefore, $\tilde{\mathcal{F}}_N$ has a functional inverse $\tilde{\mathcal{F}}_N^{-1}$ satisfying
\[
\tilde{\mathcal{F}}_N^{-1}\left(
\tilde{\mathcal{F}}_N(\Delta)
\right)=
\tilde{\mathcal{F}}_N
\left(
\tilde{\mathcal{F}}_N^{-1}(\Delta)
\right)=\Delta.
\]
Further, we define the functions
\begin{equation}\label{eqn:ScriptGDefinition}
\mathcal{G}_\Delta(N)=\mathcal{F}_N(\Delta)>0,
\end{equation}
\nomenclature{$\mathcal{G}_\Delta(N)$}{$\mathcal{F}_N(\Delta)$ }
and
\[
\tilde{\mathcal{G}}_\Delta(N)=\tilde{\mathcal{F}}_N(\Delta)>0,
\]
\nomenclature{$\tilde{\mathcal{G}}_\Delta(N)$}{$\tilde{\mathcal{F}}_N(\Delta)$}
which are the functions $\mathcal{F}$ and $\tilde{\mathcal{F}}$
with the index and argument swapped.  Unlike $\mathcal{F}_N(\Delta)$, 
$\mathcal{G}_\Delta(N)$ is monotonic (decreasing) because
\[
\mathcal{G}_{\Delta}'(N) = -F(\Delta)\cdot \mathcal{G}_{\Delta}(N)<0.
\]
Therefore, there is a functional inverse $\mathcal{G}_{\Delta}^{-1}(\Gamma)$
such that
\[
\mathcal{G}_{\Delta}(\mathcal{G}_{\Delta}^{-1}(\Gamma)) = 
\mathcal{G}_{\Delta}^{-1}(\mathcal{G}_{\Delta}(\Gamma))=\Gamma.
\]
The functional inverses defined above are decreasing, a fact which we now record.
\begin{lem}\label{lem:functionalInversesDecreasing}
We have all three of $\tilde{\mathcal{F}}_{N}^{-1}(\Gamma)$ 
and $\mathcal{G}_{\Delta}^{-1}(\Gamma)$, and $\tilde{\mathcal{G}}_{\Delta}^{-1}(\Gamma)$ 
decreasing as a function of $\Gamma$.
Further, the latter two are given by the formula
\begin{equation}\label{eqn:Ginverse}
\mathcal{G}_{\Delta}^{-1}(\Gamma)=(\log 24 - \log\Gamma)\cdot \left[
F(\Delta)
\right]^{-1}
\end{equation}
\nomenclature{ $\mathcal{G}_{\Delta}^{-1}(\Gamma)$ }{ $(\log 24 - \log\Gamma)\cdot \left[F(\Delta)\right]^{-1}$  }
and
\begin{equation}\label{eqn:GTildeinverse}
\tilde{\mathcal{G}}_{\Delta}^{-1}(\Gamma)=(\log 24 - \log\Gamma)\cdot \left[
\tilde{F}(\Delta)
\right]^{-1}
\end{equation}
where 
\begin{equation}\label{eqn:FDeltaReciprocal}
 \left[
F(\Delta)
\right]^{-1}
=\max\left[
\frac{K_1}{\Delta^2},\,
24,\,
12\exp \mathcal{W}\left(
\frac{\Delta}{8}
\right)
\right].
\end{equation}
and
\begin{equation}\label{eqn:FDeltaTildeReciprocal}
 \left[
\tilde{F}(\Delta)
\right]^{-1}
=\max\left[
\frac{K_1}{\Delta^2},\,
12\exp \mathcal{W}\left(
\frac{\Delta}{8}
\right)
\right].
\end{equation}
\end{lem}
\begin{proof}
For readability, we drop the subscripts $N$
from the proof.  All points at which $\tilde{\mathcal{F}}^{-1}$ is differentiable, the formula for the derivative
of a functional inverse says that
\[
\left(
\tilde{\mathcal{F}}^{-1}\right)'(\Gamma)=
\frac{1}{
\tilde{\mathcal{F}}' 
\left(
\tilde{\mathcal{F}}^{-1}(\Gamma)
\right)}
<0,
\]
since $\tilde{\mathcal{F}}_N$, by Lemma \ref{eqn:FunctionsMonotonicity}, is decreasing.  Further 
$\tilde{\mathcal{F}}^{-1}$ is differentiable at all but one point, namely the point $\tilde{\mathcal{F}}_N(\Delta)$ for $\Delta$ the point at which the two expressions defining $\tilde{\mathcal{F}}_N$ are equal.  At that point $\tilde{\mathcal{F}}^{-1}$ is still continuous.  Thus $\tilde{\mathcal{F}}^{-1}$ is decreasing at every point. 
 
As for $\mathcal{G}_{\Delta}^{-1}$, we derive the formula \eqref{eqn:Ginverse}
by solving for $N$ in $\mathcal{G}_{\Delta}(N)=\Gamma$,
 and observe that it is decreasing in $\Gamma$.  Similarly for  $\tilde{\mathcal{G}}_{\Delta}^{-1}$.
\end{proof}

We will use the following fact a number of times.
\begin{lem} \label{lem:FNDeltaLowerBoundInterpretation} The bound 
\begin{equation}\label{eqn:FNlowerBound}
\mathcal{F}_N(\Delta)>\Gamma
\end{equation}
is equivalent to one of the following conditions being true:
\[
N<24\log \frac{24}{\Gamma},\quad\text{or}\quad \Delta<\tilde{\mathcal{F}}_N^{-1}(\Gamma).
\]
\end{lem}
\begin{proof}
Solving the bound \eqref{eqn:FNlowerBound} for $F(\Delta)$, we obtain that it is equivalent to
\[
F(\Delta) < \frac{\log \frac{24}{\Gamma}}{N}.
\]
By the definition of $\tilde{F}(\Delta)$, this is equivalent to
\[
\min\left[
\frac{1}{24},\,
\tilde{F}(\Delta)
\right] < \frac{\log \frac{24}{\Gamma}}{N}.
\]
For the minimum to be less than the quantity on the right side is equivalent to one of the two quantities inside the minimum to be less than the quantity on the right side.  Thus we obtain
\[
N<24\log\frac{24}{\Gamma},\quad\text{or}\quad \tilde{F}(\Delta)< \frac{\log \frac{24}{\Gamma}}{N}.
\]
The latter alternative is (by multiplying by $-N$, then exponentiating) easily seen to be equivalent to 
$\tilde{\mathcal{F}}(\Delta)>\Gamma$.
We obtain the second condition given above by applying $\tilde{\mathcal{F}}_N^{-1}$ to
both sides of this inequality and using the fact that $\tilde{\mathcal{F}}_N^{-1}$ is a decreasing
function.
\end{proof}
Lemma \ref{lem:FNDeltaLowerBoundInterpretation} implies the following:
if $N\geq 24\log \frac{24}{\Gamma}$, then $\mathcal{F}_N(\Delta)>\Gamma$
implies that $\Delta < \tilde{\mathcal{F}}_N^{-1}(\Gamma)$.

\begin{corollary}\label{cor:mutualInformationLargeDeviation} 
Use the same notation as in Proposition \ref{prop:mutualInformationLargeDeviation}. Fix
$\Gamma\in(0,1)$, a lower bound which we wish to impose on $\beta^{p^\eta}_N(\gamma)$.  
Assume that
\begin{equation}\label{eqn:NConditionNotTooSmall}
N\geq 24\log\frac{24}{1-\Gamma}.
\end{equation}
Then we have the following equivalent
statements relating $\Gamma, \gamma, \eta, \Delta$.
\begin{itemize}
\item[(i).]  $\Delta\geq \tilde{\mathcal{F}}_N^{-1}(1-\Gamma)$ implies that $\beta^{p^\eta}_N(\eta+\Delta)\geq \Gamma$.
\item[(ii).]  $\beta_{N}^{p^\eta}(\eta+\Delta)\leq \Gamma$ implies that $\Delta\leq \tilde{\mathcal{F}}_N^{-1}(1-\Gamma)$.
\end{itemize}
\end{corollary}
\begin{proof}
The two statements are contrapositives of one another, so we only have to prove one of them.
We will prove (ii).  Suppose that $\beta_{N}^{p^\eta}(\eta+\Delta)\leq \Gamma$.
By Proposition \ref{prop:mutualInformationLargeDeviation}, with $\gamma=\eta+\Delta$, 
we have $1-\mathcal{F}_N(\Delta)<\Gamma$,
which is to say that $\mathcal{F}_N(\Delta)>1-\Gamma$.  Under the condition \eqref{eqn:NConditionNotTooSmall}, this implies that $\Delta<\tilde{\mathcal{F}}_N^{-1}(1-\Gamma)$,
according to Lemma \ref{lem:FNDeltaLowerBoundInterpretation}.
\end{proof}
Now, let $p_1^\epsilon\in\mathcal{P}$ satisfy $\tau(p_1^\epsilon)=\epsilon$.  Unlike with $p^\epsilon$,
we make no assumption that the marginals of the distribution $p_1^\epsilon$ are uniform.  This is very
important because in the applications $p_1^\epsilon$ will be a probability distribution in an unknown generating
network, so the experimenter/learner can have no control over the marginals of $p_1^\epsilon$.  
We are now going to apply Corollary \ref{cor:mutualInformationLargeDeviation} to estimate the probability
of achieving a \textit{smaller than expected} value $\beta_N^{p^\eta}(\tau(\omega_N))$, where $\omega_N$
is formed from $N$ samples from $p_1^{\epsilon}$.  Naturally, because $\omega_N$ is an empirical sequence, our proposed upper bound 
$\Gamma\in (0,1)$ on the probability of a large deviation must be a probabilistic one, meaning it holds only with a certain
probability (ideally as close to $1$ as possible).  
\begin{proposition}\label{prop:betaUpperProbableEstimate}
Let $\epsilon, \eta>0$ be given as above so that $\Delta:=\epsilon-\eta>0$.
Let $\Gamma\in(0,1)$.
Suppose that $N, \Delta$ and $\Gamma$ satisfy the conditions
\begin{equation}\label{eqn:betaUpperProbableEstimateConditions}
\tilde{\mathcal{F}}_N^{-1}(1-\Gamma)<\Delta, \quad N\geq 24\log\frac{24}{1-\Gamma}.
\end{equation}
Then
\[
\mathrm{Pr}_{\omega_N\sim p_1^\epsilon}\left\{ 
\beta_N^{p^\eta}(\tau(\omega_N)) < \Gamma
\right\}
\leq \mathcal{F}_N\left( \Delta-\mathcal{\tilde{F}}_N^{-1}(1-\Gamma)  \right).
\]
\end{proposition}
\begin{proof}
Apply Corollary \ref{cor:mutualInformationLargeDeviation}(ii) 
with $\tau(\omega_N)=\eta+\Delta$, so that $\Delta=\tau(\omega_N)-\eta$.  Then
the corollary says that if the condition $N>24\log\frac{24}{1-\Gamma}$ holds, then 
\[
\beta_N^{p^\eta}(\tau(\omega_N))<\Gamma\;\text{implies}\; \tau(\omega_N)-\eta
\leq \tilde{\mathcal{F}}_N^{-1}(1-\Gamma).
\]
Using \eqref{eqn:LargeDeviationOfMIEstimate}, we calculate that, now with $\Delta:=\epsilon-\eta$
as in the hypothesis, 
\[
\begin{aligned}
\mathrm{Pr}_{\omega_N\sim p_1^{\epsilon}}\left\{
\beta_N^{p^\eta}(\tau(\omega_N))<\Gamma
\right\}
&\leq 
\mathrm{Pr}_{\omega_N\sim p_1^{\epsilon}}\left\{
\tau(\omega_N)-\eta<\tilde{\mathcal{F}}_N^{-1}(1-\Gamma)
\right\}\\
&=
\mathrm{Pr}_{\omega_N\sim p_1^{\epsilon}}\left\{
\tau(\omega_N)<\epsilon-\Delta+\tilde{\mathcal{F}}_N^{-1}(1-\Gamma)
\right\}\\
&\leq 
\mathrm{Pr}_{\omega_N\sim p_1^{\epsilon}}\left\{
|\tau(\omega_N)-\epsilon|\geq \Delta-\tilde{\mathcal{F}}_N^{-1}(1-\Gamma)
\right\}\\
&\leq
 \mathcal{F}_N\left( \Delta-\tilde{\mathcal{F}}_N^{-1}(1-\Gamma)  \right),
\end{aligned}
\]
where we have used the assumption, $\tilde{\mathcal{F}}_N^{-1}(1-\Gamma)<\Delta$,
in the form $\Delta-\tilde{\mathcal{F}}^{-1}_N(1-\Gamma)>0$ to obtain the next-to-last inequality.
\end{proof}

The condition in Proposition \ref{prop:betaUpperProbableEstimate} involves
$N$ in an implicit manner, but in most circumstances it will be preferable
to express the conditions on $N$ in an explicit manner.
\begin{lem}\label{lem:GetBackRomanF}
The condition first condition of \eqref{eqn:betaUpperProbableEstimateConditions}, 
namely that $\tilde{\mathcal{F}}_N^{-1}(1-\Gamma)<\Delta$, is equivalent to
\begin{equation}\label{eqn:NThetaDeltaConditionExplicit}
N>\left[
\tilde{F}(\Delta)
\right]^{-1}\log\frac{24}{1-\Gamma}.
\end{equation}
Further, the conjunction of the conditions found in \eqref{eqn:betaUpperProbableEstimateConditions} is equivalent to the condition
\[
N>\left[
F(\Delta)
\right]^{-1}\log\frac{24}{1-\Gamma}.
\]
\end{lem}
\begin{proof}
Because $\mathcal{F}_N^{-1}$ is decreasing, $\tilde{\mathcal{F}}_N^{-1}(1-\Gamma)<\Delta$ is
equivalent to $\tilde{\mathcal{F}}_N(\Delta)<1-\Gamma$.  This is equivalent to $\tilde{\mathcal{G}}_{\Delta}(N)<1-\Gamma$.  Solving for $N$ and using Lemma \ref{lem:functionalInversesDecreasing} we obtain
\[
N>\tilde{\mathcal{G}}_{\Delta}^{-1}(1-\Gamma)=\left[
\tilde{F}(\Delta)
\right]^{-1}\log\frac{24}{1-\Gamma}.
\]
For the second statement, \eqref{eqn:FDeltaReciprocal} and \eqref{eqn:FDeltaTildeReciprocal} together imply that
\[
\left[
F(\Delta)
\right]^{-1}=\max(24,\tilde{F}(\Delta)).
\]
\end{proof}
From Proposition \ref{prop:betaUpperProbableEstimate}
and Lemma \ref{lem:GetBackRomanF}
we obtain the following more usable form of Proposition
\ref{prop:betaUpperProbableEstimate}.
\begin{corollary}\label{cor:betaUpperProbableEstimate}
Let $\epsilon>\eta>0$, $\Delta:=\epsilon-\eta>0$, and $\Gamma\in(0,1)$.  Then 
for all $N$ large enough, specifically
\[
N>\left[
F(\Delta)
\right]^{-1}
\log\frac{24}{1-\Gamma},
\]
we have the following probable upper bound on $\beta_N^{p^\eta}$:
\[
\mathrm{Pr}_{\omega_N\sim p_1^\epsilon}\left\{ 
\beta_N^{p^\eta}(\tau(\omega_N)) < \Gamma
\right\}
\leq \mathcal{F}_N\left( \Delta-\mathcal{\tilde{F}}_N^{-1}(1-\Gamma)  \right).
\]
\end{corollary}

\section{Statement of Finite Sample Complexity}\label{sec:finiteSampleComplexityStatement}
There are four problem parameters
in terms of which we express our finite sample complexity result.
We have already explained the minimum edge strength $\epsilon$.  The second is the error probability $\delta$.  Next is the minimum separation $m$ of certain
local probability distributions associated with $p$.  Two versions
of this parameter $m$ will be introduced in the statements of the Theorems
below.  We discuss the final parameter $\zeta$, which is only relevant in the case of more than two random variables,
before stating the result for $n$ variables.

Now we come to the hyperparameters, the first of which is $\kappa$.
For the finite sample complexity result, we concentrate on the case when the weighting function
$\psi_1(N) = \kappa\log N$ for $\kappa\geq 0$ a constant, and $\psi_2(N)\equiv 1$.  In this case we refer to $S_{\eta,\psi_1,\psi_2}$
more simply as $S_{\eta,\kappa}$. We make this choice because then, with the choice $\kappa=\frac{1}{2}$ the initial part of the objective
function (without the sparsity boost) becomes equal to the MDL objective function
considered by, among others, \cite{friedman99}.  Although \cite{friedman99}
considers a number of other complexity-penalty weighting functions $\psi_1$ ($\psi$ in their terminology),
such as $\psi_1=\kappa$ (constant) and $\psi_1=\kappa x^{\alpha}, 0\leq \alpha\leq 1$, we will postpone consideration
of these until later work.  We note here only that our proof of finite sample
complexity adapts, in the case $n=2$, in a straightforward manner to a proof of finite sample
complexity in the case of constant weighting function, when $\psi_1=\kappa$, which contrasts somewhat with the setting of \cite{friedman99}
because when the sparsity boost is not present, there is no known finite sample complexity result in the case of constant weighting function.
Because the proof is currently complete only for the case of networks consisting of binary variables, we state
it only in the case $k=l=2$, for each node.  We believe that analogous formulas can be derived for the case of $k,l>2$, but leave the technical details for future works.
In the two-node case, the remaining parameters are the hyper parameters $\mu,\Theta\in(0,1)$.
The choice of these hyper parameters affects the constants, but not the asymptotic
growth properties of $N(\epsilon,\delta; \eta,\kappa;\lambda; \Theta, \mu)$.
 
 \begin{theorem}\label{thm:finiteSampleComplexity} 
Let a set of nodes $V=\{ A,B \}$, with each node representing a \textit{binary} random variable in $\mathcal{X}:=\left\{ X_A,X_B\right\}$
so that $k=l=2$ and \[|X|=kl=4.\]  Consider the family $\mathcal{G}$ of Bayesian
networks $\left\{ G_0=(G_{\emptyset},P_0), \, G_1=(G_{A\rightarrow B},P_1) \right\}$, where
$G_0$ has no edge between $A$ and $B$ and $G_1$ has an edge with edge strength at least 
$\epsilon>0$.  Let $\delta>0$ be a given error probability. Set
$\psi_1(N)=\kappa\log(N)$ and $\psi_2(N) \equiv 1$, with $\kappa$ a positive constant.
Choose $\eta\in(0,\epsilon)$ so that $\Delta:=\epsilon-\eta>0$.
 Choose parameters $\Theta, \mu\in(0,1)$.
Denote by
\[
\tilde{F}(\Delta) := \min\left[
\frac{\Delta^2}{K_1},\,
\frac{1}{12\exp \mathcal{W}\left(
\frac{\Delta}{8}
\right)}
\right],\quad K_1:= 3\cdot 64(1+\log 2)^2,
\]
and
\[
F(\Delta) := \min\left(\frac{1}{24},\tilde{F}(\Delta)\right),
\]
so that
\[
\left[\tilde{F}(\Delta)\right]^{-1} = \max\left[
\frac{K_1}{\Delta^2},\,
12\exp \mathcal{W}\left(
\frac{\Delta}{8}
\right)
\right],
\]
and
\[
[F(\Delta)]^{-1} = \max(24,[\tilde{F}(\Delta)]^{-1}).
\]
Choose $\lambda\in(0,1)$ satisfying
\begin{equation}\label{eqn:LambdaUpperCondition}
\lambda < \frac{F(\mu\eta)}{\eta},\;\text{equivalently}, \; F(\mu\eta)-\lambda\eta>0.
\end{equation}
Let $N$ be larger than all the
following quantities.
\begin{equation}\label{eqn:muBalance}
\frac{\log 24}{F(\mu\eta)-\lambda\eta},\quad [F(\eta(1-\mu))]^{-1}\log\frac{48}{\delta},
\end{equation}
\begin{equation}\label{eqn:NInTermsOfLambdaCondition}
\max\left(\left[
F\left(\lambda\eta\right)
\right]^{-1}
,\,
\left[
F(\epsilon(1-\lambda))
\right]^{-1}
\right)
\log\frac{48}{\delta}
\end{equation}
\begin{equation}\label{eqn:NInTermsOfThetaCondition}
\left[
F(\Delta)
\right]^{-1}\log\frac{24}{1-\Theta},\,
\left[
\tilde{F}\left(\frac{\Delta}{2}\right)
\right]^{-1}
\log\frac{24}{\Theta}
\end{equation}
\begin{equation}
\left[
F\left(\frac{\Delta}{2}\right)
\right]^{-1}\log\frac{48}{\delta},
\end{equation}
\begin{equation}
\frac{\kappa}{\epsilon\lambda}\mathcal{W}\left(
\frac{\epsilon\lambda\Theta^{\frac{1}{\kappa}}}
{\kappa}
\right).
\end{equation}
Then with probability $1-\delta$, we have
\[
S_{\eta}(G_0,\omega_N) - S_{\eta}(G_1, \omega_N)\begin{cases}
>0&\text{if}\; G=G_0\\
<0&\text{if}\; G=G_1
\end{cases}.
\]
 \end{theorem}
 In examining this result, one question that may arise is how to choose the ``free" hyperparameters 
$\Theta, \mu$, and in particular, is there obvious value, for example, very close
to $0$ or very close to $1$.  The answer is that there is no such obviously optimal value.
In order to see this, consider the case of $\Theta$ and the two conditions in 
\eqref{eqn:NInTermsOfThetaCondition}: in one $\Theta$ appears
in the denominator, and in the other $1-\Theta$ appears in the denominator.
Similar comments regarding $\mu$ can be made in relation to \eqref{eqn:muBalance}.
If it were important to determine the optimal values of these hyperparameters, we could
carry out a numerical optimization procedure to find the optimal value of, for example, $\Theta$,
as a function of $\Delta$.  However, we are primarily interested in the asymptotic
dependence of $N$ on problem parameters $\epsilon, \delta, \zeta$, rather than the specific value of 
$N$.  So we merely note that a choice of $\Theta=\mu=\frac{1}{2}$ will lead
to a specific sample size function $N=N(\epsilon, \delta; \eta, \kappa; \lambda)$
which has the properties stated in the Theorem.
 
Here is the corresponding asymptotic statement:
\begin{theorem}\label{thm:finiteSampleComplexityAsymptotics}
Let a set of nodes $V=\{ A,B \}$, with each node representing a \textit{binary} random variable, 
so that $k=l=2$ and \[|X|=kl=4.\]  Consider the family $\mathcal{G}$ of Bayesian
networks $\left\{ G_0=(G_{\emptyset},P_0), \, G_1=(G_{A\rightarrow B},P_1) \right\}$, where
$G_0$ has no edge between $A$ and $B$ and $G_1$ has an edge with edge strength at least $\epsilon>0$.  Let $\delta>0$ be a given error probability. Set
$\psi_1(N)=\kappa\log(N)$ and $\psi_2(N) \equiv 1$ with $\kappa$ a positive constant.
Choose $\eta\in (0,\epsilon)$ to be a fixed multiple of $\epsilon$, so that $\Delta:=\epsilon-\eta$
is also a fixed multiple of $\epsilon$.
 Choose parameters $\Theta, \mu\in(0,1)$.  Then there is a function  
$N(\epsilon, \delta; \eta, \kappa; \lambda; \Theta, \mu)$,
\[
N(\epsilon, \delta; \eta, \kappa; \lambda; \Theta, \mu) \in O\left(\frac{1}{\epsilon^4}\log\frac{1}{\delta}\right)\;\text{as}\; \epsilon, \delta\rightarrow 0^+.
\] 
such that with probability $1-\delta$, we have, for $N> N(\epsilon, \delta; \eta, \kappa; \lambda; \Theta, \mu)$,
\[
S_{\eta}(G_0,\omega_N) - S_{\eta}(G_1, \omega_N)\begin{cases}
>0&\text{if}\; G=G_0\\
<0&\text{if}\; G=G_1
\end{cases}.
\]
\end{theorem}
\begin{proof}
The asymptotic dependence on $\delta$ as $\delta\rightarrow 0^+$ follows from,
for example, \eqref{eqn:muBalance}, or any of the other conditions in which
the factor $\log\frac{48}{\delta}$ appears.  

The asymptotic dependence on $\epsilon$ is determined by \eqref{eqn:LambdaUpperCondition}
and \eqref{eqn:NInTermsOfLambdaCondition}.  First, note that, for the purposes
of determining asymptotic dependence as $\epsilon$, and therefore $\eta$ approaches
$0$, we may replace $F(\eta)$ with the term $\eta^2$.  Therefore, \eqref{eqn:LambdaUpperCondition}
says that $\lambda$ is of order $\eta^2/\eta=\eta$ as $\eta\rightarrow 0^+$.  Therefore,
the first element in the maximum in \eqref{eqn:NInTermsOfLambdaCondition} is order
$\left[ F(\eta^2)\right]^{-1}$.  As $\Delta\rightarrow 0^+$, for the purposes
of determining asymptotic dependence on $\Delta$, we may replace $\left[ F(\Delta)\right]^{-1}$
with $\Delta^{-2}$.  Consequently,  $\left[ F(\eta^2)\right]^{-1}$ is of order $(\eta^2)^{-2}=\frac{1}{\eta^4}$
as $\eta\rightarrow 0$.  Since $\eta$ and $\epsilon$ are of the same order, we
reach the conclusion that $N$ is of at least order $\frac{1}{\epsilon^4}$ as $\epsilon\rightarrow 0^+$.
All the remaining conditions are clearly $\tilde{O}(\frac{1}{\epsilon^k})$ for $k<4$.
\end{proof}

Note that the cause of the asymptotic dependence of $N$ on $\epsilon$
being $N=\tilde{O}\left(\frac{1}{\epsilon^4} \right)$ is because of the case when $G=G_0$.
The following Theorem shows that if we use a different technique,
namely Sanov's Theorem, instead of Chernoff's Theorem,
for the estimates in the $G=G_0$ case,
the dependence becomes $N=O\left( \frac{1}{\epsilon^2}\right)$.
\begin{theorem} \label{thm:TwoNodeComparisonIndependentCase}
Let $V, k,l, G_0, G_1, \psi_1, \psi_2, \epsilon,\delta$ be as in Theorem \ref{thm:finiteSampleComplexityAsymptotics}. 
 \boldmath \textbf{Assume that the underlying
network in $(G_0,p_0)$, so that $\omega_N$ is sampled from an independent
distribution over a pair of binary random variables.}\unboldmath 
 \hspace*{0.03cm} Choose $\eta\in (0,\epsilon)$.  
  \begin{itemize}
  \item[(a)]Let $\mu,\lambda\in (0,1)$ such that
\[
F(\mu\eta)>\lambda\eta.
\]
Then for any $N$ greater than
\[
N^{\rm C}(\eta,\delta,\kappa,\lambda,\mu):=\max\left[
\frac{\log 24}{F(\mu\eta)-\lambda\eta},\,
F\left(
\min(\lambda\eta,\eta(1-\mu))\right)\cdot
\log\frac{48}{\delta}
\right],
\]
we have $S_{\eta}(G_0,\omega_N)>S_{\eta}(G_1,\omega_N)$ with probability
at least $1-\delta$.
\item[(b)] Define 
\[
\eta_N^- := \left(
\frac{-\sqrt{\frac{\eta}{2\eta+1}} + \sqrt{  25\cdot\frac{\eta}{2\eta+1} -2400(|X-\kappa'|) \frac{\log N}{N}  }}
         {48}
\right)^2.
\]
 Let $\mu\in (0,1)$ be a free parameter.
Let the number of sample points be larger than
\[
N^{\rm S}(\eta,\delta,\kappa,\mu):=
\max\left[e^{\max\left(\frac{|X|}{\kappa\mu},\;\mathcal{W}\left(
\frac{\eta}
{96(2\eta+1)(|X|-\kappa(1-\mu))}
\right)\right)},\,
\frac{|X|}{\eta_N^-}
\mathcal{W}\left(
\frac{\eta_N^-\delta^{\frac{|X|}{\eta_N^-}}}
{|X|\exp\left(\frac{\eta_N^-}{|X|}\right)}
\right)
\right].
\]
Then with probability $1-\delta$, we have $S_{\eta}(G_0,\omega_N)>S_{\eta}(G_1,\omega_N)$.
\item[(c)]  The functions above have the following asymptotic
properties.  As $N\rightarrow\infty$,
\[
\eta_N^-\rightarrow \frac{\eta}{144(2\eta+1)}.
\]
As $\epsilon,\delta\rightarrow 0$,
\[
N^{\rm C}(\delta;\eta,\kappa;\lambda;\mu) =\tilde{O}\left(
\frac{1}{\epsilon^4}\log\frac{1}{\delta}
\right),
\]
and 
\[
N^{\rm S}(\delta;\eta,\kappa;\mu)=\tilde{O}\left(
\frac{1}{\epsilon^2}\log\frac{1}{\delta}
\right).
\]
\end{itemize}
\end{theorem}
 Combining the estimates from the Sanov bounds in the independent case
 Chernoff bounds in the dependent case, we obtain result with the best
 possible asymptotics.
 \begin{theorem}\label{thm:twoNodeCaseOptimalAsymptotics}
 Let $\epsilon>\eta>0$ and $\Delta:=\epsilon-\eta>0$.
 Let $\lambda,\mu,\Theta\in(0,1)$ be free parameters.
 Let $N>0$.  Define 
\[
\eta_N^- := \left(
\frac{-\sqrt{\frac{\eta}{2\eta+1}} + \sqrt{  25\cdot\frac{\eta}{2\eta+1} -2400(|X-\kappa'|) \frac{\log N}{N}  }}
         {48}
\right)^2.
\]
Suppose that 
$N=N(\epsilon,\delta; \eta; \kappa, \lambda, \Theta,\mu)$ is larger than all of the following quantities
\[
\max\left[e^{\max\left(\frac{|X|}{\kappa\mu},\;\mathcal{W}\left(
\frac{\eta}
{96(2\eta+1)(|X|-\kappa(1-\mu))}
\right)\right)},\,
\frac{|X|}{\eta_N^-}
\mathcal{W}\left(
\frac{\eta_N^-\delta^{\frac{|X|}{\eta_N^-}}}
{|X|\exp\left(\frac{\eta_N^-}{|X|}\right)}
\right)
\right],
\]
\[
\left[
F(\Delta)
\right]^{-1}\log\frac{24}{1-\Theta},\,
\left[
\tilde{F}\left(\frac{\Delta}{2}\right)
\right]^{-1}
\log\frac{24}{\Theta},
\]
\[
\left[
F\left(\frac{\Delta}{2}\right)
\right]^{-1}\log\frac{48}{\delta},
\]
\[
\left[
F(\epsilon(1-\lambda))
\right]^{-1}
\log\frac{48}{\delta},\;
\frac{\kappa}{\epsilon\lambda}\mathcal{W}\left(
\frac{\epsilon\lambda\Theta^{\frac{1}{\kappa}}}
{\kappa}
\right).
\]
Then with probability $1-\delta$, we have for
\[
S_{\eta}(G_0,\omega_N) - S_{\eta}(G_1, \omega_N)\begin{cases}
>0&\text{if}\; G=G_0\\
<0&\text{if}\; G=G_1
\end{cases}.
\]
Further, 
\[
N(\epsilon,\delta; \eta; \kappa, \lambda, \Theta,\mu)=\tilde{O}\left(\frac{1}{\epsilon^2}
\log\frac{1}{\delta}
\right)\;\text{as}\;\epsilon,\delta\rightarrow 0^+.
\]
 \end{theorem}
Here is the finite sample complexity result we obtain for the case of $n$
nodes, using only the concentration of mutual information results derived
from Chernoff's bounds.
 \begin{theorem}\label{thm:nNodeCase}
Let $\delta>0$, $\mathcal{G}=\mathcal{G}^d$, and $\epsilon>0$ the minimum
edge strength of the network $G$ with respect to $\mathcal{G}$ and $\mathscr{S}$.
Let $\eta\in(0,\epsilon)$ and $\Delta:=\epsilon-\eta>0$.  Let $G\in\mathcal{G}^d$
be a perfect map for $P$.  Suppose that the separating collection $\mathscr{S}$
is defined by 
\[
S_{A,B}(G):=\{\text{All subsets of}\;V-\{A,B\}\;\text{of cardinality}\; d\},
\]
so that in particular
\[
\sigma(\mathcal{G}):=\max_{G\in\mathcal{G}}\sigma(G)\;\text{is bounded by a constant, i.e.,}\;
d;
\]
\[
\Sigma_{\mathscr{S}}(\mathcal{G}):=\left|
\bigcup_{\stackrel{G\in\mathcal{G}}{(A,B)\in V^2\backslash \Delta}}
S_{A,B}(G)
\right|\;\text{is bounded by a polynomial in $n$ of degree $d$, i.e.}\; \binom{n}{d}.
\]
Let $F$, $\tilde{F}$ be as defined as in Theorem \ref{thm:finiteSampleComplexity}, and let
\[
N_{n}(\epsilon,\delta)=   \max\left[
\frac{3(1+\log 2)^24^{d+2}n^2}
{ \epsilon^2},24,
12\exp\left(\mathcal{W}\left(\frac{\epsilon}{n2^{d+2}}\right)\right)
\right]\cdot 
\log\frac{3\cdot 2^{d+1}\binom{n}{d+1}(n-d)}{\delta(n-(2d+1))}.
\]
\nomenclature{$N_{n}(\epsilon,\delta)$}{
$ \max\left[
\frac{3(1+\log 2)^24^{d+2}n^2}
{ \epsilon^2},24,
12\exp\left(\mathcal{W}\left(\frac{\epsilon}{n2^{d+2}}\right)\right)
\right]\cdot 
\log\frac{3\cdot 2^{d+1}\binom{n}{d+1}(n-d)}{\delta(n-(2d+1))}$
}
Set the weighting functions $\psi_1(N):=\kappa\log N$ and $\psi_2(N):=1$, $\kappa$
a constant, e.g., $\kappa=\frac{1}{2}$.
\begin{itemize}
\item[(a)]   Let $\zeta>0$ be a given error parameter.  Let $\theta,\Theta\in(0,1)$
be free parameters.  Let
\[
m=m_p(G,\mathcal{G},\mathscr{S})=\max_{(A,B)\in G}m_p((A,B),\mathcal{G},\mathscr{S}),
\]
\nomenclature{$m=m_p(G,\mathcal{G},\mathscr{S})$}{$\max_{(A,B)\in G}m_p((A,B),\mathcal{G},\mathscr{S})$}%
where
\[
m_p((A,B),\mathcal{G},\mathscr{S}):=\max_{\stackrel{S\in S_{A,B}(\mathcal{G})}{s\in\mathrm{Val}(S)}}\left[
p(S=s)
\right]^{-1}.
\]
\nomenclature{$m_p((A,B),\mathcal{G},\mathscr{S})$}{ $\max_{\stackrel{S\in S_{A,B}(\mathcal{G})}{s\in\mathrm{Val}(S)}}\left[
p(S=s)
\right]^{-1}$}
Take $N$ larger than all of the following quantities:
\[
N_{n}\left(\frac{\zeta}{3},\frac{\delta}{6} \right),
\]
\[
m(1-\theta)^{-1}\left[
F(\Delta)
\right]^{-1}\log
\frac{24}{1-\Theta},\;
 \frac{3\kappa(|G|-n)}{\zeta}\mathcal{W}
\left(
\frac{\zeta\Theta^{\frac{|E(G)|}{\kappa(|G|-n)}}}{3\kappa(|G|-n)}
\right),
\]
\[
\frac{3}{(1-\theta)\theta^2}\log\frac{3|E(G)|\Sigma_{\mathscr{S}}(\mathcal{G})2^{\sigma(\mathcal{G})}}{\delta},
\]
and
\[
\frac{m}{1-\theta}
\max\left(
\left[
\tilde{F}\left(
\frac{\Delta}{2}
\right)
\right]^{-1}
\log \frac{24}{1-\Theta},\,
\left[
F\left(
\frac{\Delta}{2}
\right)
\right]^{-1}
\log \frac{72 |E(G)|\Sigma_{\mathscr{S}}(\mathcal{G})  }{\delta}
\right).
\]
Then with probability $1-\delta$,
\[
S_{\eta,N}(G,\omega_N)>S_{\eta,N}(G',\omega_N),\;\text{for all}\; G'\;\text{such that}\;
H(P,p_{G',P})>\zeta.
\]
\item[(b)]  Let $\mu,\Theta,\theta\in(0,1)$ be free parameters, and
let $L$ be a positive integer $L<\mathrm{card}(G'\backslash G)$.
Let
\begin{equation}
\hat{m}_p(G,\mathscr{S}):=\max_{(A,B)\in V^{\times 2}\backslash G} 
m_p(\hat{S}_p((A,B),G)),
\end{equation}
and
\[
\mbox{$\hat{S}:=\hat{S}_p((A,B),G)$}\in S_{A,B}(G)\;\text{such that}\; A\ci B|\hat{S}\;\text{and}\;m_p(\hat{S})\;\text{as small as possible}.
\]
For compactness of notation, we abbreviate as follows:
\begin{equation}\label{eqn:mHatInTheorems}
\hat{m} := \hat{m}_p(G,\mathscr{S}),\; m:=m_p(G,\mathcal{G},\mathscr{S}). 
\end{equation}
Let $N$ be larger than all the following quantities:
\begin{equation}\label{eqn:asymptoticControllingCondition}
\max\left[m(1-\theta)^{-1}\left[
F(\Delta)
\right]^{-1}\log\frac{24}{1-\Theta},\,
N_{n}\left(
 \frac{L(1-\theta)F(\mu\eta)}{4\hat{m}},
\frac{\delta}{10}
\right)
\right],
\end{equation}
\begin{equation}\label{eqn:nonAsymptoticControllingCondition}
\frac{3}{(1-\theta)\theta^2}\log\frac{5|E(G)|\Sigma_{\mathscr{S}}(\mathcal{G})2^{\sigma(G)}}{\delta},
\end{equation}
\begin{equation}\label{eqn:secondAsymptoticControllingCondition}
\frac{\hat{m}}{1-\theta}\cdot\max\left(
\left[
\tilde{F}\left(\frac{\Delta}{2}
\right)
\right]^{-1}
\log\frac{48}{F(\mu\eta)},\,
\left[
F\left(\frac{\Delta}{2}
\right)
\right]^{-1}\log\frac{120|E(G)|\Sigma_{\mathscr{S}}(\mathcal{G})}{\delta}
\right),
\end{equation}
\begin{equation}\label{eqn:firstSuperfluousCondition}
\frac{3}{(1-\theta)\theta^2}\log\frac{5L\Sigma_{\mathscr{S}}(G)2^{\sigma(G)}}{\delta},
\end{equation}
\begin{equation}
\frac{\hat{m}}{1-\theta}\cdot\log\left(
\frac{120 \sigma_{\mathscr{S}}(G)L}{\delta}\right)
\left[
F(\eta(1-\mu))
\right]^{-1}.
\end{equation}
\begin{equation}\label{eqn:NlargerthanFactorOfFmueta}
\frac{4\hat{m}\log 24}{F(\mu\eta)(1-\theta)},
\end{equation}
\begin{equation}\label{eqn:nNodeSampleComplexityWFunction}
\left(
\frac{4\hat{m}\kappa (|G|-n)}
{(1-\theta)L F(\mu\eta)}
\right)
\mathcal{W}\left(
\frac{(1-\theta)LF(\mu\eta)\Theta^{\frac{|E(G)|}{(|G|-n)\kappa}}}
{4\hat{m}\kappa(|G|-n)}
\right).
\end{equation}
Then with probability $1-\delta$, 
\[
S_{\eta,N}(G,\omega_N)>S_{\eta,N}(G',\omega_N),\;\text{for all}\; G'\;\text{such that}\;\mathrm{Skel}(G')\not\subset 
\mathrm{Skel}(G). 
\]
\end{itemize}
 \end{theorem}
 \begin{rem}  The reader is cautioned not to confuse the following two quantities which appear
 in Theorem \ref{thm:nNodeCase} and related discussions: first, $|G|$, the number of parameters of the Bayesian network associated
 to graph $G$, and second, $|E(G)|$, the number of edges in the graph $G$.  In particular, the bounds
 on these quantities for $G$ ranging over $\mathcal{G}^d$ are different.  On the one hand, we have $|E(G)|<nd$,
 and on the other hand, we have $|G|<n2^d$, for all $G\in \mathcal{G}^d$.
 \end{rem}
Using an analysis based on Sanov's Theorem in place of Chernoff's Theorem,
we can obtain a refinement of part (b) of Theorem \ref{thm:nNodeCase}
with improved asymptotics (in terms of $\epsilon$):
 \begin{theorem}\label{thm:nNodeCaseSanov}
Let $\delta>0$ be given.  Let $G\in\mathcal{G}^d$
and let $m$ be as given in \eqref{eqn:mHatInTheorems}.
Let $\epsilon$ be the minimum edge strength of $G$ in $\mathcal{G},\mathscr{S}$,
as above.  Let $\eta\in(0,\epsilon)$.  Let $\theta, \Theta\in(0,1)$
be free parameters and $L<\mathrm{card}(G'\backslash G)$
as above. Let $N$ be larger than the following quantities
\begin{equation}\label{eqn:ConditionControllingAsymptotics}
\max\left[m(1-\theta)^{-1}\left[
F(\Delta)
\right]^{-1}\log\frac{24}{1-\Theta},\,
N_{n}\left(\frac{1-\theta}{\hat{m}}\frac{L}{800}\frac{\eta}{2\eta+1},
\frac{\delta}{6}
\right)\right]\,
\end{equation}
where $N_{n}(\epsilon,\delta)$ is the function defined in Theorem \ref{thm:nNodeCase},
\[
\frac{3}{\theta^2(1-\theta)}
\log\frac{3|E(G)|\Sigma_{\mathscr{S}}(G)2^{\sigma(G)}}{\delta},
\]
\[
16\frac{|X|(2\eta+1)\hat{m}}{\eta(1-\theta)}
\mathcal{W}\left(
\frac{\eta}
{16(2\eta+1)|X|\exp\left( \frac{\eta}{16(2\eta+1)|X|} \right)\left(
\frac{\delta}{3L}
\right)^{1/|X|}}
\right).
\]
Then for
\begin{multline*}
N>\frac{800\hat{m}(2\eta+1)\left(\kappa(|G|-n)-L|X|\right)}{(1-\theta)L\eta}\times \\
\mathcal{W}\left(
\frac{(1-\theta)L\eta}{800\hat{m}(2\eta+1)(\left(\kappa(|G|-n)-L|X|\right))}
\Theta^{\frac{|E(G)|}{\left(\kappa(|G|-n)-L|X|\right)}}
\right),
\end{multline*}
we have with probability $1-\delta$,
\[
S_{\eta,N}(G,\omega_N)>S_{\eta,N}(G',\omega_N),\;\text{for all}\; G'\;\text{such that}\;\mathrm{Skel}(G')\not\subset 
\mathrm{Skel}(G). 
\]
 \end{theorem}
We can sum up the asymptotics of all these functions as follows:
\begin{theorem}
\label{thm:nNodeCaseAsymptotic}
Let $(G,P)$, $\mathcal{G}^d, \epsilon, m, \hat{m}, m, \delta,$  be as in Theorem \ref{thm:nNodeCase}.
Let \linebreak $N(m,n;\delta;\zeta;\eta;\kappa,\Theta,\theta,L,\mu)$ be the function of part (a) of Theorem \ref{thm:nNodeCase}.
Define \linebreak $N^\mathrm{C}(\hat{m},m,n;\delta;\eta;\kappa, \theta,\Theta,L)$
to be the maximum of the functions bounding $N$ from below in Theorem
\ref{thm:nNodeCase}(b).
Define $N^\mathrm{S}(\hat{m},m,n;\delta;\eta;\kappa, \theta,\Theta,L)$ to be the maximum of the functions
bounding $N$ from below in Theorem \ref{thm:nNodeCaseSanov}.
Then we can sum up as follows the asymptotics of the functions appearing in Theorems
\ref{thm:nNodeCase} and \ref{thm:nNodeCaseSanov}, as 
\[
\zeta,\epsilon,\delta \rightarrow 0^+,\;\text{and}\quad m,\hat{m}\rightarrow\infty.
\]
\begin{itemize}
\item[(a)]
From part (a) of Theorem \ref{thm:nNodeCase},
\[
N(m,n;\delta;\zeta;\eta;\kappa,\Theta,\theta,L,\mu)\in\tilde{O}\left(\max\left(
\frac{\log(n) m}{\epsilon^2},\left(\frac{n}{\zeta}\right)^2
\right)\cdot
\log\frac{1}{\delta}\right)
\] 
\item[(b)]
From part (b) of Theorem \ref{thm:nNodeCase},
\[
N^\mathrm{C}(\hat{m},m,n;\delta;\eta;\kappa,\Theta,\theta,L,\mu) \in \tilde{O}\left(\max\left(
\frac{ m}{\epsilon^2},\left(\frac{n\hat{m}}{\epsilon^2}\right)^2\right)\cdot
\log\frac{1}{\delta}\right)
\]           
\item[(c)]
From Theorem \ref{thm:nNodeCaseSanov},
\[
N^\mathrm{S}(\hat{m},m,n;\delta;\eta;\kappa,\Theta,\theta,L)\in \tilde{O}\left(\max\left(
\frac{m}{\epsilon^2},\left(\frac{n\hat{m}}{\epsilon}\right)^2\right)\cdot
\log\frac{1}{\delta}\right)
\]
\end{itemize}
 \end{theorem}
 \begin{proof}
 We begin by collecting some elementary observations that will be used in the proof of all parts.
 First, let the function $N_{n}(\epsilon,\delta)$ defined in Theorem \ref{thm:nNodeCase}.
 Then,
 \begin{equation}\label{eqn:NsubscriptnAsymptotics}
 N_n(\epsilon,\delta)=\tilde{O}\left(\left(\frac{n}{\epsilon}\right)^2\log\frac{1}{\delta}\right),\;\text{as}\; n\rightarrow\infty,\, \epsilon,\,\delta\rightarrow 0^+.
 \end{equation}
 Second, from the definition of $[F(\Delta)]^{-1}$ as the maximum of a constant factor times $\Delta^{-2}$, a constant,
 and a constant factor times $\exp \mathcal{W}\left(\frac{\Delta}{8}\right)$, which is by Lemma \ref{lem:LambertWasymptotics},
 $\tilde{O}\left(\frac{1}{\Delta}\right)$ as $\Delta\rightarrow 0^+$, we have
 \[
 \left[F(\Delta)\right]^{-1}, \left[\tilde{F}(\Delta)\right]^{-1}=\tilde{O}\left(\frac{1}{\Delta^2}\right),\;\text{as}\; \Delta\rightarrow 0^+.
 \]
 Since $\Delta:=\epsilon-\eta$ and we may choose $\eta$ to be a constant multiple (in $(0,1)$)
 of $\epsilon$ as $\epsilon\rightarrow 0^+$, the asymptotics of any function $f(\Delta)$ (or of $f(\eta)$) as $\epsilon\rightarrow 0^+$
 are the same as the asymptotics of $f(\epsilon)$ as $\epsilon\rightarrow 0^+$.  Therefore, we may actually write the second observation as
 \begin{equation}\label{eqn:FasymptoticsInEspilon}
 \left[F(\Delta)\right]^{-1}, \left[\tilde{F}(\Delta)\right]^{-1}, \left[F(\eta)\right]^{-1}, \left[\tilde{F}(\eta)\right]^{-1}=\tilde{O}\left(\frac{1}{\epsilon^2}\right),\;\text{as}\; \epsilon\rightarrow 0^+.
 \end{equation}
 Finally, since we are assuming that $\Sigma_{\mathscr{S}}(\mathscr{G})$ is (bounded by) $\binom{n}{d}$, elementary combinatorics says that
 \begin{equation}\label{eqn:SigmaSubSOfGAsymptotics}
 \Sigma_{\mathscr{S}}(\mathscr{G})=O(n^d),\;  \log\Sigma_{\mathscr{S}}(\mathscr{G})=O(\log n),\; \text{as}\,\, n\rightarrow \infty.
 \end{equation}
  \begin{description}
 \item[(a)] \hfill \\
 By \eqref{eqn:NsubscriptnAsymptotics}, we have $N_{n}\left(\frac{\zeta}{3},\frac{\delta}{6} \right)=\tilde{O}
\left(\left(\frac{n}{\zeta}\right)^2\log\frac{1}{\delta}\right)$ as $n\rightarrow\infty$ and $\zeta,\delta\rightarrow 0^+$.
But $N(\cdots)$ the maximum of this function and the other functions listed in  part (a) of Theorem \ref{thm:nNodeCase}.
The last of these functions, once the logarithm is suitably expanded, amounts to a constant multiple of 
$m\left[ F(\Delta)\right]^{-1}\log\left(\frac{\Sigma_{\mathscr{S}}(\mathscr{G})}{\delta}\right)$.
and by \eqref{eqn:FasymptoticsInEspilon} and \eqref{eqn:SigmaSubSOfGAsymptotics}, 
we have
\[
m\left[ F(\Delta)\right]^{-1}\log\left(\frac{\Sigma_{\mathscr{S}}(\mathscr{G})}{\delta}\right)=\tilde{O}\left( \frac{\log(n)m}{\epsilon^2}\log\frac{1}{\delta}  \right),
\;\text{as}\, m,n\rightarrow\infty,\; \epsilon,\delta\rightarrow 0^+
\]
For the max of these two functions, we obtain the asymptotics of $N(\cdots)$ claimed in part (a).  It is not difficult
to go through the other functions in the maximum defining $N(\cdots)$ and verify that these two functions
dominate all the other functions asymptotically.  This completes the proof of part (a).
\item[(b)] \hfill \\
Using  \eqref{eqn:NsubscriptnAsymptotics}, \eqref{eqn:FasymptoticsInEspilon}, and \eqref{eqn:SigmaSubSOfGAsymptotics}, we can see that \eqref{eqn:asymptoticControllingCondition} is
\begin{equation}\label{eqn:partBDominatingCondition}
\tilde{O}\left(
\max
\left[
\frac{m}{\epsilon^2},
\left(\frac{n\hat{m}}{\epsilon^2}\right)^2 \log\frac{1}{\delta}
\right]
\right),\;\text{as}\, n,m,\hat{m}\rightarrow\infty,\; \epsilon,\delta\rightarrow 0^+,
\end{equation}
and that 
\eqref{eqn:secondAsymptoticControllingCondition} is
\[
\tilde{O}\left(
\frac{\hat{m}\log(n)}{\epsilon^2}\right),\;\text{as}\, n,\hat{m}\rightarrow\infty,\; \epsilon\rightarrow 0^+.
\]
But the latter function is dominated by \eqref{eqn:partBDominatingCondition}.  Furthermore, it is not difficult to verify that
\eqref{eqn:firstSuperfluousCondition} through \eqref{eqn:nNodeSampleComplexityWFunction}, the other
functions whose maximum defines $N^\mathrm{C}(\hat{m},m,n;\delta;\eta;\kappa,\Theta,\theta,L,\mu)$ are all
dominated by \eqref{eqn:partBDominatingCondition}.  So the asymptotics of $N^\mathrm{C}(\hat{m},m,n;\delta;\eta;\kappa,\Theta,\theta,L,\mu)$
are given by  \eqref{eqn:partBDominatingCondition}.
\item[(c)] \hfill \\
Using  \eqref{eqn:NsubscriptnAsymptotics} and \eqref{eqn:FasymptoticsInEspilon}, we can see that \eqref{eqn:ConditionControllingAsymptotics} is
\begin{equation}\label{eqn:partCDominatingCondition}
\tilde{O}\left(
\max
\left[
\frac{m}{\epsilon^2},
\left(\frac{n\hat{m}}{\epsilon}\right)^2 \log\frac{1}{\delta}
\right]
\right),\;\text{as}\, n,m,\hat{m}\rightarrow\infty,\; \epsilon,\delta\rightarrow 0^+.
\end{equation}
It is not difficult to see that the other functions whose maximum
is used to define $N^\mathrm{S}(\hat{m},m,n;\delta;\eta;\kappa,\Theta,\theta,L)$ all have asymptotics
which are dominated by \eqref{eqn:partBDominatingCondition}.  Thus, the asymptotics
of $N^\mathrm{S}(\hat{m},m,n;\delta;\eta;\kappa,\Theta,\theta,L)$ are determined by
\eqref{eqn:partCDominatingCondition}.
 \end{description} \end{proof}
 To obtain the asymptotic statement for $N(\epsilon,m,\hat{m}, n; \delta; \zeta)$ in the Introduction, 
 take the maximum of the functions in parts (a) and (c), setting the parameter $L=1$ throughout.  This is an allowable
 choice because for all $G'\in\mathcal{G}^d$, such that $\mathrm{Skel}(G')$ properly contains $\mathrm{Skel}(G)$ we have $L\geq 1$.
See Section \ref{sec:discussionLiterature} for further comments on the relation of our results in the literature.  

\section{Proof in the case of two binary random variables}
We begin by proving the theorem for the special case where $V$ has two nodes, 
because there are fewer technical complications and the essential ideas are still visible in this case.
The most significant reusable parts of the proof in the two-node case are Corollaries
\ref{cor:sparsityBoostLowerBoundIndependentNetwork} 
and \ref{cor:complexityPenaltyUpperBound}.
The reason is that these two results give lower and upper bounds, respectively, on the sparsity boost in the case when the two variables
being tested are independent, respectively dependent.  This will allow us, in the general case, to show the following: first, that the absence
of the sparsity boost from the objective function for a false network with an extra, false edge will cause that false network
to quickly (in terms of number of samples) lose out to the true network.  Second, that the presence of the extra sparsity
boost in a false network which is missing an edge in the true network, will have its positive influence on 
the objective function of the false network quickly die out.  Thus, these two propositions are key in showing
that the sparsity boosts, collectively, quickly winnow down the space of networks with large objective functions
to networks with the same skeleton as the true network.

In the case of two-node networks, there are only two distinct Bayesian network structures:
\begin{itemize}
\item $G_0$, the disconnected network on two nodes, with possible corresponding probability distributions $P_0\in \mathcal{P}_0$.
\item $G_1$, the fully connected structure on two nodes, with possible corresponding probability distributions $P_1\in \mathcal{P}_{\epsilon}.$
\end{itemize}
The inequality we are trying to prove in this case is equivalent to 
\begin{equation}\label{eqn:twoNodeCaseScriptSTwoConditions}
\mathcal{S}_\eta(G_0, G_1, \omega_N) := S_\eta(G_0,\omega_N) - S_\eta(G_1,\omega_N) \begin{cases} >0&\text{if}\; G=G_0\\ <0 &\text{if}\;G=G_1\end{cases}.
\end{equation}
\nomenclature{$\mathcal{S}_\eta(G_0, G_1, \omega_N)$}{$S_\eta(G_0,\omega_N) - S_\eta(G_1,\omega_N)$}
We calculate that
\[
 \begin{aligned}
  \mathcal{S}_\eta(G_0, G_1, \omega_N) &=\mathrm{LL}(G_0,\omega_N)-\mathrm{LL}(G_1,\omega_N)\\
  &-\psi_1(N)(|G_0|-|G_1|)+\psi_2(N)(0-\max_{S\in S_{A,B}(G)}\min_{s\in \mathrm{val}(S)}\ln\beta_N^{p^\eta}(\tau(p(\omega | s))))\\
  &= -NH(p(\omega_N)\| p(\omega_N)_0)-\psi_1(N)(2-3)+\psi_2(N)(0-\ln\beta_N^{p^\eta}(\tau(p(\omega)))) \\
  &=-N  \tau(\omega_N)+\kappa\log(N)-\ln\beta_N^{p^\eta}( \tau(\omega_N))\\
 \end{aligned}
\]
Note that in the two-node case, $V=\{A,B\}$, the dependence of the sparsity boost on the mapping $S_{A,B}(\cdot)$ in effect disappears, because
the only subset of $V-\{A,B\}$ is the empty set.  This accounts for the relative
simplicity of the objective function in the two-node case.

As a result of the simple form of the objective function in the two-node case, as reflected
in the above calculation, we can decompose the function $\mathcal{S}_\eta$ into the composition
\[
N \mapsto \tau(\omega_N) \mapsto \mathcal{S}_{N,\eta}(\tau(\omega_N)),
\]
where $\mathcal{S}_{\eta,N}:\,\mathbf{R}^+\rightarrow\mathbf{R}$ is defined by 
\begin{equation}\label{eqn:scriptSdefn}
\mathcal{S}_{\eta,N}(\gamma) := -N\gamma + \kappa\log(N) -\log\beta_N^{p^\eta}(\gamma).
\end{equation}
\nomenclature{$\mathcal{S}_{\eta,N}(\gamma)$}{$-N\gamma + \kappa\log(N) -\log\beta_N^{p^\eta}(\gamma)$}

The significance of this decomposition is that whereas the first function ${N\mapsto\tau(\omega_N)}$ is a random variable,
the second function $\gamma\mapsto S_{\eta,N}(\gamma)$ is a deterministic function of the test statistic $\gamma:=\tau(\omega_N)$.  
\subsection{Independent Network}
Suppose that the generating network is $(G_0,p_0)$ with $p_0\in \mathcal{P}_0$ a product distribution with arbitrary marginals.  Recall from the above derivation that 
in the two-node case $\mathcal{S}_{\eta}(G_0,G_1,\omega_N)$ has the following simple expression,
\[
\mathcal{S}_{\eta}(G_0,G_1,\omega_N)= \mathcal{S}_{N,\eta}(\tau(\omega_N)):=-N\tau(\omega_N)-\log\beta_N^{p^\eta}(\tau(\omega_N))+\kappa\log N.
\]
The main points of the proof will be to show that with ``high" probability (to be quantified below)
\begin{itemize}
\item The mutual information $\tau(\omega_N)$ converges to $0$ fast enough that for moderately large $N$ it is less than some multiplier $\lambda$ times $\eta$ (Proposition \ref{prop:MIestimateIndependentNetwork}).
\item The complexity penalty $-\log\beta_N^{p^\eta}(\tau(\omega_N))$ grows linearly with $N$ for moderately large $N$ with ``slope" $\Gamma>0$ depending
on $\eta$, but not on $p_0$ (Corollary \ref{cor:sparsityBoostLowerBoundIndependentNetworkVer2})
\end{itemize}
From these points, stated precisely and proved in Propositions \ref{prop:MIestimateIndependentNetwork} and Corollary \ref{cor:sparsityBoostLowerBoundIndependentNetwork} below, the finite sample complexity result will follow easily.

\begin{proposition}
\label{prop:MIestimateIndependentNetwork}
Let $\lambda\in(0,1)$ be a ``multiplier" for $\eta$.  Let $\eta>0$ and $p_0\in\mathcal{P}_0$.  Then
\[
\mathrm{Pr}_{\omega\sim p_0}
\left\{
\tau(\omega_N)>\lambda\eta
\right\}
\leq
\mathcal{F}_N(\lambda\eta).
\]
\end{proposition}
\begin{proof}
Because $p_0\in \mathcal{P}_0$, so that $\tau(p_0)=0$, we have 
\[
\mathrm{Pr}_{\omega_N\sim p_0}
\left\{
\tau(\omega_N)>\lambda\eta
\right\}
=
\mathrm{Pr}_{\omega_N\sim p_0}
\left\{
|\tau(\omega_N)-\tau(p_0)|>\lambda\eta
\right\}\leq \mathcal{F}_N(\lambda\eta).
\]
The latter inequality follows from 
Lemma \ref{lem:empiricalEntropyError}.
\end{proof}

In order to show the linear growth of the complexity penalty in $N$, it suffices
to show the linear exponential decay in $N$ of $\beta_N^{p^\eta}(\omega_N)$.  Thus we start with the following.

\begin{proposition}\label{prop:betaUpperBoundIndependentNetwork}
Let $\eta>0$, and choose $\mu\in (0,1)$, so that $\eta(1-\mu)>0$.  Let $\Gamma>0$ small enough and let $N$ a positive integer large enough
that they satisfy
\begin{equation}\label{eqn:DeltaCondition}
\mathcal{F}_N(\mu\eta)\leq e^{-N\Gamma}.
\end{equation}
Then we have
\begin{equation}\label{eqn:NegExponentialUpperBoundOnBeta}
\beta_N^{p^\eta}(\eta(1-\mu))\leq e^{-N\Gamma}.
\end{equation}
\end{proposition}
\begin{proof}
Using the definition of $\beta_N^{p^\eta}$ as the CDF of $\tau$, $\beta_N^{p^\eta}(\mu(1-\eta))$ is
\[
\begin{aligned}
\beta_N^{p^\eta}(\eta(1-\mu))&=\mathrm{Pr}_{Y_N\sim p^{\eta}}\left\{
\tau(Y_N)<\eta-\mu\eta
\right\}\\
&\leq \mathrm{Pr}_{Y_N\sim p^{\eta}}\left\{
|\tau(Y_N)-\eta|>\mu\eta
\right\}\\
&\leq \mathcal{F}_N(\mu\eta),
\end{aligned}
\]
where we have used Lemma \ref{lem:empiricalEntropyError} in the last inequality.  So
$\beta_N^{p^\eta}(\eta(1-\mu))\leq \mathcal{F}_N(\eta\mu)$, and using \eqref{eqn:DeltaCondition}
we obtain the conclusion \eqref{eqn:NegExponentialUpperBoundOnBeta}.
\end{proof}
From the upper bound on $\beta$ in Proposition \ref{prop:betaUpperBoundIndependentNetwork}, we readily
derive a (probable) lower bound on the sparsity boost for the independent network.
\begin{corollary}\label{cor:sparsityBoostLowerBoundIndependentNetwork}
Let $\Gamma, \mu\in(0,1)$ be as in Proposition \ref{prop:betaUpperBoundIndependentNetwork}, in particular
satisfying condition \eqref{eqn:DeltaCondition}. Then 
\[
\mathrm{Pr_{\omega_N\sim p_0}}\left\{
-\log\beta_N^{p^\eta}(\tau(\omega_N))<N\Gamma
\right\}\leq \mathcal{F}_N(\eta(1-\mu)).
\]
\end{corollary}
\begin{proof}
By assumption \eqref{eqn:DeltaCondition} and by Proposition \ref{prop:betaUpperBoundIndependentNetwork}, we have
\[
e^{-N\Gamma}\geq \beta_N^{p^\eta}(\eta(1-\mu)).
\]
So if $\tau(\omega_N)<\eta(1-\mu)$ then because $\beta_N^{p^\eta}$
is increasing, we also have also
\[
e^{-N\Gamma}\geq \beta_N^{p^\eta}(\tau(\omega_N)).
\]
Thus,
\[
\begin{aligned}
\mathrm{Pr_{\omega_N\sim p_0}}\left\{
-\log\beta_N^{p^\eta}(\tau(\omega_N))<N\Gamma
\right\}
&=\mathrm{Pr}_{\omega_N\sim p_0}\left\{
\beta_N^{p^\eta}(\tau(\omega_N))>e^{-N\Gamma}
\right\}\\
&\leq \mathrm{Pr}_{\omega_N\sim p_0}\left\{
\tau(\omega_N)>\eta-\eta\mu\right\}\\
&= \mathrm{Pr}_{\omega_N\sim p_0}\left\{
|\tau(\omega_N)-\tau(p_0)|>\eta-\eta\mu
\right\}\\
&\leq \mathcal{F}_N(\eta-\eta\mu),
\end{aligned}
\]
where we have used \eqref{eqn:LargeDeviationOfMIEstimate} to obtain the final inequality.
\end{proof}
For a given $N,\eta$ and chosen $\mu\in(0,1)$, set
\begin{equation}\label{eqn:GammaMaxDefinition}
\Gamma^{\max}(N,\mu,\eta)=\sup\left\{\Gamma\in(0,1)\,|\,
\mathcal{F}_N(\mu\eta)<e^{-N\Gamma}
\right\}.
\end{equation}
\nomenclature{$\Gamma^{\max}(N,\mu,\eta)$}{$\sup\left\{\Gamma\in(0,1)\,\lvert\,
\mathcal{F}_N(\mu\eta)<e^{-N\Gamma}
\right\}$, for $\mu\in(0,1)$}
By solving the inequality defining $\Gamma^{\max}(N,\mu,\eta)$ we obtain the formula,
\begin{equation}\label{eqn:GammaMaxFormula}
\Gamma^{\max}(N,\mu,\eta)=
\frac{-\log\mathcal{F}_N(\mu\eta)}{N}
=\frac{-\log[24\exp(-NF(\mu\eta))]}{N}
=-\frac{24}{N}+F(\mu\eta).
\end{equation}
The reason for defining this quantity is that $\Gamma^{\max}(N,\mu,\eta)$ is
the largest $\Gamma$ such we can apply Corollary \ref{cor:sparsityBoostLowerBoundIndependentNetwork}.  In terms of this
new quantity, we can state Corollary \ref{cor:sparsityBoostLowerBoundIndependentNetwork}
in a form that is more convenient for the applications.
\begin{corollary}\label{cor:sparsityBoostLowerBoundIndependentNetworkVer2}
Let $\mu\in(0,1)$, and let
\begin{equation}\label{eqn:GammaInInterval}
\Gamma\in(0,\Gamma^{\max}(N,\mu,\eta)),
\end{equation}
with $\Gamma^{\max}(N,\mu,\eta)$ as defined in \eqref{eqn:GammaMaxDefinition}, so that
in particular, $N,\Gamma,\mu,\eta$ satisfy \eqref{eqn:DeltaCondition}.  Then 
\[
\mathrm{Pr_{\omega_N\sim p_0}}\left\{
-\log\beta_N^{p^\eta}(\tau(\omega_N))\geq N\Gamma
\right\}\geq 1- \mathcal{F}_N(\eta(1-\mu)).
\]
\end{corollary}

Combining Corollary \ref{cor:sparsityBoostLowerBoundIndependentNetworkVer2} with Proposition \ref{prop:MIestimateIndependentNetwork}, we have
\begin{proposition} \label{prop:sampleComplexity2NodeIndependentCase} Let $\lambda,\mu\in(0,1)$.
Let 
\begin{equation}\label{eqn:GammaLessThanMax}
0<\Gamma<\Gamma^{\max}(N,\mu,\eta).
\end{equation}
Suppose that the generating network is $(G_0,p_0)$.  Then with probability at least
\begin{equation}\label{eqn:IndependentNetworkErrorProbability}
1-\mathcal{F}_N(\lambda\eta)-\mathcal{F}_N(\eta(1-\mu)),
\end{equation}
we have the lower bound.
\[
\mathcal{S}_{\eta}(G_0,G_1,\omega_N)\geq (\Gamma-\lambda\eta)N + \kappa\log N.
\]
Consequently, provided that
\begin{equation}\label{eqn:GammaLambdaEtaCondition}
\Gamma>\lambda\eta,
\end{equation}
then we have with probability at least \eqref{eqn:IndependentNetworkErrorProbability},
\[
\mathcal{S}_{\eta}(G_0,G_1,\omega_N)>0.
\]
\end{proposition}
\textbf{This concludes the particular treatment of the case $G=G_0$.  Next we deal with the other case, when $G=G_1$.}
\subsection{Dependent network}\label{subsec:dependentNetwork}
In the case when $G=G_1$, our goal is to show that, with high probability, $\mathscr{S}_{\eta}(G_0,G_1,\omega_N) < 0$.

First, by definition, we have
\[
\mathscr{S}_{\eta}(G_0, G_1, \omega_N) = -\log\beta_N^{p^\eta}(\omega_N)-N\tau(\omega_N)+\kappa \log N.
\]
From Proposition \ref{prop:betaUpperProbableEstimate}, we derive the following:
\begin{corollary}\label{cor:complexityPenaltyUpperBound}
Let $\epsilon>\eta>0$ and let $\Delta:=\epsilon-\eta>0$.  Let $\Theta\in(0,1)$ with
the following conditions satisfied
\begin{equation}\label{eqn:conditionsForSparsityBoostUpperBound}
\tilde{\mathcal{F}}^{-1}_N(1-\Theta)<\Delta\;\text{and}\; N>24\log \frac{24}{1-\Theta}.
\end{equation}
Then,
\[
\mathrm{Pr}_{\omega\sim p_1^{\epsilon}}
\left\{
-\log \beta(\tau(\omega_N)) > \log \Theta^{-1}\right\}
<\mathcal{F}_N(\Delta-\tilde{\mathcal{F}}_N^{-1}(1-\Theta)).
\]
\end{corollary}
\begin{proof}
Proposition \ref{prop:betaUpperProbableEstimate} implies that under the conditions 
\eqref{eqn:conditionsForSparsityBoostUpperBound}, with $\Gamma$ set equal to $\Theta$,
\[
\begin{aligned}
\mathrm{Pr}_{\omega\sim p_1^{\epsilon}}
\left\{
-\log \beta(\tau(\omega_N)) > \log \Theta^{-1}\right\}&=\mathrm{Pr}_{\omega\sim p_1^{\epsilon}}
\left\{
\beta(\tau(\omega_N)) < \Theta \right\}\\
&<\mathcal{F}_N(\Delta-\tilde{\mathcal{F}}_N^{-1}(1-\Theta)).
\end{aligned}
\]
\end{proof}
Consequently we have the following estimate of $\mathscr{S}_{\eta}(G_0, G_1, \omega_N)$.
\begin{proposition} \label{prop:SDependentCaseEstimate}
Let $\epsilon>\eta>0$ and let $\Delta:=\epsilon-\eta>0$ with the conditions 
\eqref{eqn:conditionsForSparsityBoostUpperBound} satisfied.
Then we have
\[
\mathcal{S}_{\eta}(G_0, G_1,\omega_N)\leq \log\Theta^{-1}-N\tau(\omega_N)+\kappa\log N,
\]
with probability at least $1-\mathcal{F}_N(\Delta-\tilde{\mathcal{F}}_N^{-1}(\Theta))$.
\end{proposition}
Let $\lambda\in(0,1)$ be a ``multiplier", meaning that we are going to choose $N$ so that $\tau(\omega_N)>\lambda\epsilon$,
the $\lambda$-proportion of its expected value, with some (close to $1$) probability.  The tradeoff in the selection
of $\lambda$ is that the closer $\lambda$ is to $1$, the larger $N$ will have to be to achieve $\tau(\omega_N)>\lambda\epsilon$
with any specified probability.   

\begin{proposition}\label{prop:TauDependentCaseEstimate}
With $\lambda\in(0,1)$ a fixed multiplier, we have
\[
\mathrm{Pr}_{\omega_N\sim p_1^\epsilon}
\left\{
\tau(\omega_N)>\lambda\epsilon
\right\}
\leq
1-\mathcal{F}_N((1-\lambda)\epsilon).
\]
\end{proposition}
\begin{proof}
\[
\begin{aligned}
\mathrm{Pr}_{\omega_N\sim p_1^\epsilon}
\left\{
\tau(\omega_N)>\lambda\epsilon
\right\}
&=1-\mathrm{Pr}_{\omega_N\sim p_1^\epsilon}
\left\{
\tau(\omega_N)\leq\lambda\epsilon
\right\}\\
&\leq
1-\mathrm{Pr}_{\omega_N\sim p_1^\epsilon}
\left\{
|\tau(\omega_N)-\epsilon|\geq (1-\lambda)\epsilon
\right\}\\
&\leq
1-\mathcal{F}_N((1-\lambda)\epsilon),
\end{aligned}
\]
where we have used Lemma \ref{lem:empiricalEntropyError} in the last line, applied with $\eta:=\epsilon$, $\gamma:=\lambda\epsilon$ (so $\Delta=\epsilon(1-\lambda)$).
\end{proof}
By the union bound, Propositions \ref{prop:SDependentCaseEstimate} and \ref{prop:TauDependentCaseEstimate} under the conditions \eqref{eqn:conditionsForSparsityBoostUpperBound}
we have the following upper bound on $\mathcal{S}_{\eta}(G_0,G_1,\omega_N)$:
\begin{equation}\label{eqn:finalSEstimateDependentCase}
S_{\eta}(G_0,G_1,\omega_N)<\log\Theta^{-1}-N\lambda\epsilon + \kappa\log N,
\end{equation}
with probability
at least
\[
1-\mathcal{F}_N(\Delta-\tilde{\mathcal{F}}_N^{-1}(\Theta)) - \mathcal{F}_N\left((1-\lambda)\epsilon\right).
\]
Thus we obtain the following final sample complexity result for the case of $G=G_1$.
\begin{proposition}\label{prop:finiteSampleComplexityDependentCase}
Let $\lambda\in(0,1)$ be a fixed multiplier.  
Let $\epsilon>\eta$ and define $\Delta:=\epsilon-\eta$.  Let $\Theta\in(0,1)$ be such that 
the conditions \eqref{eqn:conditionsForSparsityBoostUpperBound} are satisfied.
If in addition
\begin{equation}\label{eqn:NconditionDependentCase}
N>\frac{\kappa}{\epsilon\lambda}\mathcal{W}\left(
\frac{\epsilon\lambda\Theta^{\frac{1}{\kappa}}}
{\kappa}
\right),
\end{equation}
we have
\[
\mathcal{S}_{\eta}(G_0,G_1,\omega_N)<0,\;\text{with probability}\; 1-\mathcal{F}_N(\Delta-\tilde{\mathcal{F}}_N^{-1}(\Theta)) - \mathcal{F}_N\left((1-\lambda)\epsilon\right).
\]
\end{proposition}
\begin{proof}
Dividing \eqref{eqn:finalSEstimateDependentCase} by $-\kappa$, 
we have $\mathcal{S}_{\eta}<0$ if and only if
\[
\frac{\epsilon\lambda}{\kappa}N-\log N - \frac{\log\Theta^{-1}}{\kappa}>0.
\]
Solving this condition for $N$ we obtain \eqref{eqn:NconditionDependentCase}.
\end{proof}

\subsection{General Case}\label{subsec:GeneralCase}
Putting Propositions \ref{prop:sampleComplexity2NodeIndependentCase} and \ref{prop:finiteSampleComplexityDependentCase} together, we have the most
general form of the finite sample complexity result for $2$ variables.
\begin{proposition}\label{prop:mostGeneralSampleComplexity}
Let $\Theta,\lambda,\mu\in(0,1)$.  Let 
\begin{equation}\label{eqn:GammaInIntervalSecond}
\Gamma\in(\lambda\eta,\Gamma^{\max}(N,\mu,\eta)),
\end{equation}
where
\[
\Gamma^{\max}(N,\eta,\mu) = -\frac{24}{N} + F(\mu\eta).
\]
(Assume, for now, that the interval in \eqref{eqn:GammaInInterval} is nonempty.)
Let $N$ be a positive integer, such that the following conditions are satisfied,
\begin{equation}\label{eqn:NThetaCondition}
N>24\log \frac{24}{1-\Theta},
\end{equation}
\begin{equation}\label{eqn:NThetaDeltaCondition}
\tilde{\mathcal{F}}_N^{-1}(1-\Theta)<\Delta,
\end{equation}
\begin{equation}\label{eqn:NconditionDependentCaseRepeated}
N>\frac{\kappa}{\epsilon\lambda}\mathcal{W}\left(
\frac{\epsilon\lambda\Theta^{\frac{1}{\kappa}}}
{\kappa}
\right).
\end{equation}
Then with probability $1-\delta(\epsilon,\eta,\lambda,\mu,\Theta,\Delta,N)$, we have
\[
\mathcal{S}_{\eta}(G_0,G_1,\omega_N)\begin{cases}
>0&\text{if}\; G=G_0\\
<0&\text{if}\; G=G_1
\end{cases}.
\]
Here $\delta(\epsilon,\eta,\lambda,\mu,\Theta,\Delta,N)$ is defined as
\begin{equation}\label{eqn:rawDeltaDefinition}
\max\left(
\mathcal{F}_{N}(\lambda\eta)+\mathcal{F}_N(\eta(1-\mu)),\;
\mathcal{F}_N\left(\Delta - \mathcal{F}_N^{-1}(\Theta)\right)+
\mathcal{F}_N\left((1-\lambda)\epsilon\right)
\right)
\end{equation}
\end{proposition}
\begin{proof}
Gather together all the conditions of the Propositions: \eqref{eqn:NThetaCondition} and
\eqref{eqn:NThetaDeltaCondition} account for \eqref{eqn:conditionsForSparsityBoostUpperBound}.  
Condition
\eqref{eqn:GammaInInterval} is the combination of \eqref{eqn:GammaLessThanMax} and \eqref{eqn:GammaLambdaEtaCondition}.  Finally, \eqref{eqn:NconditionDependentCaseRepeated}
 comes from \eqref{eqn:NconditionDependentCase}.  Thus \eqref{eqn:GammaInInterval} through \eqref{eqn:NconditionDependentCaseRepeated} account for all the conditions in the Propositions.

Since $G=G_0$ or $G=G_1$ is an exhaustive set of conditions, the error probability is the max
of the error probabilities in these two cases.
\end{proof}
This most general formulation of the finite sample complexity is evidently
unsatisfactory as it stands for three reasons: it is not clear that the interval \eqref{eqn:GammaInInterval}
is nonempty, so that there is a possible choice of $\Gamma$.  Secondly, in
order to study the asymptotics of $N$, the number of samples, we need all the conditions
on $N$ to be expressed explicitly, with $``N>"$ on the left side and expressions involving
all the variables on the right side.  Third, since $\delta$ is a choice of the experimenter,
we should not have $\delta$ expressed in terms of the other parameters, but instead
we should have $N$ expressed in terms of the desired confidence $1-\delta$ and any other
parameters.

As for the first task, concerning the choice of $\Gamma$,
we have to show that there is a choice of $\lambda, N,\mu$
so that the interval \eqref{eqn:GammaInInterval} is nonempty, in other words so that
\begin{equation}\label{eqn:nonEmptyIntervalAsInequalityTwoNode}
\lambda\eta<\Gamma^{\max}(N,\mu,\eta).
\end{equation}  
\begin{lem}  \label{lem:lambdaLemmaTwoNode} The condition \eqref{eqn:nonEmptyIntervalAsInequalityTwoNode} is equivalent to the following 
two conditions
\begin{equation}\label{eqn:conditionBetweenMuAndLambdaTwoNode}
F(\mu\eta)>\lambda\eta,
\end{equation}
and
\begin{equation}\label{eqn:NConditionOnMuEtaLambdaTwoNode}
N>\frac{\log 24}{F(\mu\eta) - \lambda\eta}.
\end{equation}
\end{lem}
\begin{proof}
For the interval in \eqref{eqn:GammaInInterval} to be nonempty, it is equivalent to have
\begin{equation}\label{eqn:intervalNonemptyTwoNode}
\lambda\eta < -\frac{\log 8}{N} + F(\mu\eta).
\end{equation}
For a fixed choice of $\lambda, \mu$, there will be an $N$ satisfying \eqref{eqn:intervalNonemptyTwoNode}
if and only if \eqref{eqn:conditionBetweenMuAndLambdaTwoNode} holds, because the term
$-\frac{\log 8}{N}$ can only be negative.   If \eqref{eqn:conditionBetweenMuAndLambdaTwoNode} does hold,
then we can solve \eqref{eqn:intervalNonemptyTwoNode} to obtain \eqref{eqn:NConditionOnMuEtaLambdaTwoNode}. 
\end{proof}
The final point to address is how to make the error probability \linebreak $\delta(\epsilon,\eta,\lambda,\mu,\Theta,\Delta,N)$ less than a $\delta$ set by the experimenter.  There are many ways of going
about this.  We give just one that makes it possible to deduce the asymptotics.  Namely,
we assume that $\delta$ is given and we require each of the four terms in  $\delta(\epsilon,\eta,\lambda,\mu,\Theta,\Delta,N)$ to be less than $\frac{\delta}{2}$.  Further, we have to bound away
from zero, the
expression $\Delta - \tilde{\mathcal{F}}_N^{-1}(\Theta)$ appearing inside the $\mathcal{F}_N$
in the second term in the maximum.  In order to do this, we simply require that  $\tilde{\mathcal{F}}_N^{-1}(\Theta)<\frac{\Delta}{2}$, or equivalently, because $\tilde{\mathcal{F}}_N^{-1}$ is decreasing,
\begin{equation}\label{eqn:conditionWithNDeltaAndTheta}
\tilde{\mathcal{F}_N}\left(
\frac{\Delta}{2}
\right)<\Theta.
\end{equation}
Because $\mathcal{F}_N$ is decreasing, it will suffice to replace $\Delta - \tilde{\mathcal{F}}_N^{-1}(\Theta)$ inside $\mathcal{F}_N$ with its lower bound and require that
\[
\mathcal{F}_N\left(
\frac{\Delta}{2}
\right)<\frac{\delta}{2}.
\] 
Thus we arrive at the five conditions stated in the following Lemma.
\begin{lem}\label{lem:conditionsForDeltaInScriptFTerms}
Let $\delta>0$ be given.  In order to achieve $\delta(\epsilon,\eta,\lambda,\mu,\Theta,\Delta,N)<\delta$
it is sufficient to take $N$ large enough that the following four quantities are bounded by $\delta/2$
\begin{equation}\label{eqn:conditionsForDeltaInScriptFTerms}
\mathcal{F}_N(\lambda\eta),\,\mathcal{F}_N(\eta(1-\mu)),\, \mathcal{F}_N\left(\frac{\Delta}{2}\right),\,
\mathcal{F}_N((1-\lambda)\epsilon)<\frac{\delta}{2},
\end{equation}
and in addition take $N$ large enough so that the following holds,
\begin{equation}\label{eqn:scriptFNDeltaTheta}
\tilde{\mathcal{F}}_N\left(\frac{\Delta}{2}\right)<\Theta.
\end{equation}
\end{lem}
The following simple Lemma allows us to transform these conditions into explicit lower bounds on $N$.
\begin{lem} \label{lem:scriptFConditionsIntoExplicitNConditions} Let $A,B>0$.  Then the condition
\[
\mathcal{F}_N(A)<B
\]
is equivalent to the condition
\[
\mathcal{G}_A(N)<B.
\]
Solving for $N$, we have that both the conditions are also equivalent to the condition
\[
N>\mathcal{G}_A^{-1}(B)=(\log 24 - \log B)[F(A)]^{-1}.
\]
\end{lem}
\begin{proof}
Use \eqref{eqn:ScriptGDefinition} and then Lemma \ref{lem:functionalInversesDecreasing}.
\end{proof}
Applying Lemma \ref{lem:scriptFConditionsIntoExplicitNConditions} to \eqref{eqn:conditionsForDeltaInScriptFTerms} and \eqref{eqn:scriptFNDeltaTheta},
we have the following explicit bound for $N$, in terms of the given $\delta$,
for the error probability to be less than $\delta$.
\begin{lem}\label{lem:LemmaRelatingFAndDelta}
Assume that $N$ satisfies the conditions
\[
N>
\max\left(
\left[
F\left(\lambda\eta\right)
\right]^{-1},\,
\left[
F(\eta(1-\mu))
\right]^{-1},
\left[
F\left(\frac{\Delta}{2}\right)
\right]^{-1},
\left[
F(\epsilon(1-\lambda))
\right]^{-1}
\right)\log\frac{48}{\delta}
\]
and 
\[
N> \left[
\tilde{F}\left(\frac{\Delta}{2}\right)
\right]^{-1}
\log\frac{24}{\Theta}.
\]
Then $\delta(\epsilon,\eta,\lambda,\mu,\Theta,\Delta,N)<\delta$.
\end{lem}

\begin{proof}[Completion of Proof of Theorem \ref{thm:finiteSampleComplexity}]
In the Theorem, we have set an arbitrary $\delta\in(0,1)$.  
We have assumed that $G=(G_0,p_0)$, $p_0\in\mathcal{P}_0$, i.e. $\tau(p_0)=0$,
or that $G=(G_1,p_1)$, $p_1\in \mathcal{P}_{\epsilon}$, that is $\tau(p_1)\geq \epsilon$.
Let $\eta\in(0,\epsilon)$, so that $\Delta:=\epsilon-\eta>0$. 

Since the conditions \eqref{eqn:muBalance} and \eqref{eqn:LambdaUpperCondition}
in the Theorem imply that
\[
N>\frac{\log 24}{F(\mu\eta)-\lambda\eta}>0,
\]
we can define
\begin{equation}\label{eqn:gammaMaxDefn}
\Gamma^{\max}(N,\mu,\eta):=-\frac{\log 24}{N} + F(\mu\eta).
\end{equation}
and then the interval $(\lambda\eta,\Gamma^{\max})\subseteq(0,1)$ is nonempty.
Choose $\Gamma\in(\lambda\eta,\Gamma^{\max})$.
Apply the most general form of the finite sample complexity, Proposition 
\ref{prop:mostGeneralSampleComplexity}.  Then use 
three lemmas in this section, 
Lemma \ref{lem:lambdaLemmaTwoNode}, Lemma \ref{lem:scriptFConditionsIntoExplicitNConditions},
and Lemma
\ref{lem:LemmaRelatingFAndDelta}, to obtain the conditions given in Theorem
\ref{thm:finiteSampleComplexity}.
\end{proof}

\section{Proof in the Case of $n$ binary random variables}
The proof in the case of $n$ nodes is, in most respects, an elaboration of 
the proof in the case of two nodes with some additional combinatorial arguments.
In order to explain this, we introduce some notation convenient
for the expressing the difference of objective functions, which is denoted,
for $G, G'\in\mathcal{G}$, by
\[
\mathcal{S}_{\eta}(G, G',\omega_N) = S_{\eta}(G,\omega_N) - S_{\eta}(G',\omega_N). 
\]
\nomenclature{$\mathcal{S}_{\eta}(G, G',\omega_N)$}{$S_{\eta}(G,\omega_N) - S_{\eta}(G',\omega_N)$}
Usually, we will think of the network $G$ in the first position of the difference as the ``true generating network"
so that we are trying to show the (probable) positivity of $\mathscr{S}_{\eta}(G, G',\omega_N)$.
By $\Delta(G,G')$ we denote the symmetric difference of the $\mathrm{Skel}(G)$
and $\mathrm{Skel}(G')$.  
\nomenclature{$\Delta(G,G')$}{Symmetric difference of the $\mathrm{Skel}(G)$
and $\mathrm{Skel}(G')$}
Further, we define a sign associated to any edge $(A,B)$ in this symmetric difference by
\[
\sigma((A,B),G,G')=\begin{cases}
+1 & \text{for}\;(A,B)\;\in G'-G\\
-1 & \text{for}\;(A,B)\;\in G-G',
\end{cases}
\]
\nomenclature{$S_{\eta}(G,G',\omega_N)$}{$S_{\eta}(G,\omega_N) - S_{\eta}(G',\omega_N)$}%
\nomenclature{$\sigma((A,B),G,G')$}{$\begin{cases}
+1 & \text{for}\;(A,B)\;\in G'-G\\
-1 & \text{for}\;(A,B)\;\in G-G',\\
0  & \;\text{otherwise}.
\end{cases}$}
\nomenclature{$\Delta(G,G')$}{Symmetric difference of $\mathrm{Skel}(G)$ and $\mathrm{Skel}(G')$}
(which can simply be thought of as $0$ for $(A,B)\notin \Delta(G,G')$).  We then define
the set
\[
S_{A,B}(G,G')=\begin{cases}
S_{A,B}(G) & \text{for}\;(A,B)\in G'-G\\
S_{A,B}(G') & \text{for}\;(A,B)\;\in G-G'.
\end{cases}
\]
\nomenclature{$S_{A,B}(G,G')$}{$\begin{cases}
S_{A,B}(G) & \text{for}\;(A,B)\in G'-G\\
S_{A,B}(G') & \text{for}\;(A,B)\;\in G-G'\\
\emptyset    & \text{otherwise}
\end{cases}$}

Then in this notation, we have
\begin{multline}\label{eqn:generalScoreDifference}
\mathcal{S}_{\eta,N}(G,G',\omega_N)=\mathrm{LL}(\omega_{N},G)-\mathrm{LL}(\omega_{N},G')+\psi_{1}(N)\left(|G'|-|G|\right)+\\
\psi_{2}(N)\left[\sum_{(A,B)\in\Delta(G,G')}\sigma((A,B),G,G')\max_{S\in S_{A,B}(G,G')}\min_{s\in\mathrm{val}(S)}-\ln\left[\beta_{N}^{p^{\eta}}\tau(p(\omega_N, A,B|s))\right]\right]+\\
\psi_{2}(N)\left[\sum_{(A,B)\not\in E(G)\cup E(G')}\left(\max_{S\in S_{A,B}(G)} - \max_{S\in S_{A,B}(G')}\right)\left(\min_{s\in\mathrm{val}(S)}-\ln\left[\beta_{N}^{p^{\eta}}\tau(p(\omega_N, A,B|s))\right]\right)\right].
\end{multline}

Corollaries \ref{cor:sparsityBoostLowerBoundIndependentNetwork} 
and \ref{cor:complexityPenaltyUpperBound} continue
to provide the probabilistic information we need in order to estimate the second line of this expression,
that is the part of $\mathcal{S}_{\eta,N}(G,G')$ 
coming from the sparsity boosts.  In contrast, the situation with the first line of this expression, in particular,
the difference of log-likelihood terms, is somewhat different in the $n$-node case.
Recall that in the two-node case, the difference
of the log-likelihood terms is known to be precisely $-N$ times the test statistic $\gamma$,
while in the $n$-node case, it is clear we cannot hope for such a precise formula,
and we must find estimates instead.  Fortunately, because of previous work on
the MDL objective function, many such estimates are known, at least under
various auxiliary assumptions.  The auxiliary assumption we will adopt 
is the one followed by H\"{o}ffgen in \cite{hoffgen1993learning}, namely, that each node has at most $d$ parents, and we
will follow his methods to derive more precise versions of the estimates appearing in H\"{o}ffgen's paper. 
\nomenclature{$d$}{The maximum in-degree, that is maximum number of parents of a node in $G$.}

\subsection{Comparison between Empirical and Idealized Log-Likelihood}
We fix $d<n$, the maximum in-degree, that is maximum number of parents of a node in $G$.  
\nomenclature{$d$}{bound on in-degree of a graph}
We let $\mathcal{G}=\mathcal{G}^{d}$ denote the space of Bayesian
network structures with bounded in-degree $d$.
\nomenclature{$\mathcal{G}^d$}{Space of BN structures with bounded in-degree $d$}
 The significance of the restriction on the space of possible structures to $\mathcal{G}^d$ is as follows: if we consider
all the log-likelihoods for $G\in \mathcal{G}^d$, there are super-exponentially
many values to be estimated, and a simple-minded application of bounds for
individual log-likelihoods followed by the union bound will give very poor simultaneous
bounds, leading to results for the sample complexity which are exponential in $n$. 
There is a well-known decomposition of the $\mathrm{LL}(\omega_{N},G)$,
for any Bayesian network $G$, in terms of conditional entropy terms:
\begin{equation}\label{eqn:LLrelativeentropyExpansionInText}
\mathrm{LL}(\omega_{N},G)=-N\sum_{i=1}^{n}H_{p_{\omega_{N}}}(X_{i}|\mathrm{Pa}_{G}(X_{i})),
\end{equation}
where $\mathrm{Pa}_{G}(X_{i})$ is the parent set of node $X_i$ in the DAG $G$.
We can use \eqref{eqn:LLrelativeentropyExpansionInText} to show that a  \textit{polynomial
number of distinct terms appear in all the  $\mathrm{LL}(\omega_{N},G)$ 
for $G\in\mathcal{G}^d$}, namely at most,
\begin{equation}\label{eqn:numRelEntAppearingInLogLikelihood}
n^{d+1}+n^d+\cdots + n + 1=\frac{n^{d+2}-1}{n-1},
\end{equation}
which is only $O(n^{d+1})$.   If we have an estimate which applies to all of these
$O(n^{d+1})$ terms individually, then applying the union bound will give a related estimate
of all the log-likelihood terms simultaneously.  This is the key to achieving a
finite sample complexity result with polynomial dependence of $N$ on $n$.  
We stress that the idea we have just sketched out is due to H\"{o}ffgen in \cite{hoffgen1993learning}.
We begin proving an effective version of the estimates in \cite{hoffgen1993learning}.

\begin{lem}  \label{lem:polynomialNumberOfEntropyEstimates} Consider the family of Bayesian networks $\mathcal{G}^d$ on $n$ fixed
vertices with bounded in-degree $d$.  Consider the total number $M$ of different expressions 
$\tilde{p}_N\log \tilde{p}_N$,
Here, $\tilde{p}_N$ is an empirical estimate in the joint contingency table of all the variables, appearing in all the expansions $\mathrm{LL}(\omega_{N},G)$ (as given in \eqref{eqn:LLrelativeentropyExpansionInText}).
Then 
\[
M\leq 2^{d+1}\left(  \binom{n}{1}+\cdots + \binom{n}{d+1}  \right)\leq 2^{d+1}\binom{n}{d+1}\frac{n-d}{n-(2d+1)}.
\]
\end{lem}
\begin{proof}
When $\mathcal{Y},\mathcal{Z}$ are distinct
subsets of the variable set $\mathcal{X}$, and $P$ is a (joint) probability
distribution over $\mathcal{X}$, we use
\begin{equation}\label{eq:relativeEntInTermsOfOrdinaryInText}
H_P(\mathcal{Z}|\mathcal{Y}):=H_P(\mathcal{Z},\mathcal{Y}) - H_P(\mathcal{Y}),
\end{equation}
where $H_P(\mathcal{Z},\mathcal{Y})$ and $H_P(\mathcal{Y})$ are the (ordinary) entropies of the probability distributions
on the subsets $\mathcal{Z}\cup \mathcal{Y}$ and $\mathcal{Y}$, induced by $P$.
Each of the relative entropies appearing in \eqref{eqn:LLrelativeentropyExpansionInText}
can, by \eqref{eq:relativeEntInTermsOfOrdinaryInText}, be expressed in terms of ordinary entropies.
For example, in the case that $\mathrm{Pa}_G(x_i)=\{x_{j(1)},\ldots, x_{j(d)}\}$ (the largest
\nomenclature{$\mathrm{Pa}_G(x_i)$}{Parents of $x_i$ in the directed graph $G$}
possible parent set under our assumptions) and when $P=\tilde{P}$, an empirical distribution
\begin{equation}\label{eqn:conditionalOrdinaryEntropy}
H_{\tilde{P}}(x_{i}|\mathrm{Pa}_{G}(x_{i})):= H_{\tilde{P}}(x_i,x_{j(1)},\ldots, x_{j(d)}) - 
H_{\tilde{P}}(x_{j(1)},\ldots, x_{j(d)}).
\end{equation}
The number of entropy terms of the form $H_{\tilde{P}}(x_i,x_{j(1)},\ldots, x_{j(k)})$
is bounded by $\binom{n}{k+1}$, and $k$ can range from $1$ to $d$.
Each one of the entropy terms is the sum of at most $2^{d+1}$ terms
of the form $\tilde{p}\log \tilde{p}$ where $\tilde{p}$ is an empirical estimate
of $p$, derived from $N$ independent samples from a Bernoulli
random variable $X(p)$ of fixed but unknown parameter $p$.  So there
are at most a total of $M$ such terms $\tilde{p}_N\log \tilde{p}_N$ appearing in all the
expansions of log-likelihoods, with $M$ as defined in the Lemma.  The bound
on the sum of the first $d+1$ binomial coefficients is elementary and is stated and proved at \cite{LugoMathOverflow}.
\end{proof}
By computing \textit{all} of the conditional entropies in the expansion \eqref{eqn:LLrelativeentropyExpansionInText} with respect
to an arbitrary probability distribution $B$ in place of the empirical distribution $p_{\omega_N}$, we obtain
a generalization of the log-likelihood commonly called the
\textit{idealized} log-likelihood $\mathrm{LL}^{(I)}_N(B,G)$.  (For more context, see Appendix
\ref{app:relEntropyLogLikelihood}.)  Considering $\mathrm{LL}^{(I)}_N(P,G)$, we have a completely analogous statement
to Lemma \ref{lem:polynomialNumberOfEntropyEstimates} concerning the number of terms
of the form $p\log p$, where $p$ is the conditional probability of an atomic event obtained by appropriately
conditioning and marginalizing $P$.  Since terms $p\log p$
in the $\mathrm{LL}^{(I)}_N(P,G)$ correspond one-to-one with terms $\tilde{p}\log(\tilde{p})$ in $\mathrm{LL}(\omega_N,G)$,
this statement has the same proof.

We are able to prove the proposition that represents
the culmination of this section, concerning
the (empirical) log-likelihood as an ``estimate'' of the idealized log-likelihood
corresponding to the underlying distribution:
\begin{proposition}\label{prop:LLestimate}
Let $G\in \mathcal{G}^d$.  Let $\epsilon, \delta>0$.  Let  ${N>N_n(\epsilon,\delta)}$, defined
by
\[
N_n(\epsilon,\delta)=   \max\left[
\frac{3(1+\log 2)^24^{d+2}n^2}
{ \epsilon^2},24,
12\exp\left(\mathcal{W}\left(\frac{\epsilon}{n2^{d+2}}\right)\right)
\right]\cdot 
\log\frac{3\cdot 2^{d+1}\binom{n}{d+1}(n-d)}{\delta(n-(2d+1))}.
\]
Then we have
\[
\mathrm{Pr}_{\omega_N\sim P}\left\{
\left|
\mathrm{LL}(\omega_N,G) - \mathrm{LL}^{(I)}(P,G)\right|\leq N\cdot \epsilon
,\;
\forall G\in\mathcal{G}^d\right\}\geq 1-\delta.
\]
\end{proposition}
\begin{proof}
Each conditional entropy in the expression \eqref{eqn:LLrelativeentropyExpansion}  of $\mathrm{LL}^{(I)}(P,G)$ (respectively in the expression \eqref{eqn:LLrelativeentropyExpansionIdealized} of $\mathrm{LL}(\omega_N,G)$) 
 has $n$ conditional entropies.  Each of the $n$ conditional entropies may be rewritten, by \eqref{eqn:conditionalOrdinaryEntropy},
 as the difference of two ordinary entropies, as we have written out above in \eqref{eqn:conditionalOrdinaryEntropy}.  Each of the ordinary
 entropies has at most $2^{d+1}$ terms terms of the form $p\log p$ for $p$ the probability of some
event (respectively, of the form $\tilde{p}_N\log\tilde{p}_N$ for $\tilde{p}_N$ the estimated
probability of the event).  Thus, there are a total of at most $n\cdot 2\cdot 2^{d+1}=n 2^{d+2}$ terms of the form   $p\log p$ for $p$ the probability of some
event (respectively, of the form $\tilde{p}_N\log\tilde{p}_N$ for $\tilde{p}_N$ the estimated
probability of the event) in the sum for $\mathrm{LL}^{(I)}(P,G)$ (resp. $\mathrm{LL}(\omega_N,G)$).
In order to achieve accuracy $\epsilon$ in the estimate $\mathrm{LL}(\omega_N,G)$ of 
 $\mathrm{LL}^{(I)}(P,G)$ it will suffice to have accuracy $\epsilon':=\frac{\epsilon}{n\cdot 2^{d+2}}$
 in the estimate of each quantity $p\log p$ in Proposition \ref{prop:plogp}.
 
 In order to have \textit{all} the estimates with confidence $1-\delta$, we must have each individually with confidence $1-\delta/ M$, where $M$ is the number of individual
 estimates.  Lemma \ref{lem:polynomialNumberOfEntropyEstimates} gives us an upper
 bound on the number of individual estimates $M$.  So we must
 take $\delta'=\delta/M$ in the function $N_n(\epsilon,\delta)$ of Proposition \ref{prop:plogp}.
 
 We obtain the function in the Proposition by substituting $\epsilon':=\frac{\epsilon}{n\cdot 2^{d+2}}$
 for $\epsilon$ and $\delta'=\delta/M$ ($M$ from Lemma \ref{lem:polynomialNumberOfEntropyEstimates})
 into the function $N$ of Proposition \ref{prop:plogp}.
\end{proof}

\subsection{Log-likelihood term estimates}  For $P$ any probability distribution and $\mathcal{G}$ any family of BN \textit{structures} (DAGs), and any 
$\zeta>0$ we make the definitions
\begin{itemize}
 \item \boldmath $\mathcal{B}(P)\;$\unboldmath is the set of Bayesian networks of the form $(G',p_{G',P})$ for $G'$ ranging over $\mathcal{G}$.
 \item \boldmath $\mathcal{B}_{\zeta}(P)\,$\unboldmath is the subset of $\mathcal{B}(P)$ consisting of those $(G',p_{G',P})$ for which \textbf{either}  ${H(P\| p_{G',P})>\zeta}$ \textbf{or} for which $p_{G',P}=P$ (so that $H(P\| p_{G',P})=0$).
\end{itemize}
\nomenclature{$\mathcal{B}(P)$}{ $\left\{ (G',p_{G',P})\;\lvert \; G'\in\mathcal{G}  \right\}$ }%
\nomenclature{$\mathcal{B}_{\zeta}(P)$}{ $\left\{(G',p_{G',p})\in \mathcal{B}(P)\;\lvert\; H(P\lvert\lvert p_{G',P})\in\left\{0 \right\}\cup(\zeta,\infty) \right\}$}
The motivation behind introducing this notation is that in the setup for the learning algorithm, we are given
a distribution $P$ (unknown to the experimenter), for which $G$ is a perfect map (the ``faithfulness
assumption'').  Conceptually (although this is not true in practice) we may imagine that the
learning algorithm consists of two stages, a structure-learning part that returns a DAG $G'$,
followed by a parameter learning algorithm returns $P'$, maximizing the log-likelihood of the data, $LL(\omega, G')$.
We wish to be able to refer concisely to all the BNs (structure-distribution pairs
$(G',P')$ such that $G'$ is an I-map of $P'$), that are \textit{eligible} to be learned by our entire learning algorithm,
assuming that the structure learning part may make an error, but that the parameter-learning part of the algorithm is
a perfect expectation-maximizer (it can look at infinite data and 
is completely accurate as a maximizer of $\mathrm{LL}(\omega_N, G')$).
The KL-divergence $H(P\| p_{G',P})$ is an appropriate measure of the distance of an element $G'\in\mathcal{G}$ from $(G,P)$.
Consequently, a positive $H(P\| p_{G',P})$ may be considered a measure of the error of an algorithm which returns $G'$ instead of $G$.  
The set of structures $(G',P')$, for which $P'$ is ``closest" to $P$, 
given the constraint that $P'$ factors according to $G'$, is precisely $\mathcal{B}(P)$.   

The precise statement we \textit{would like} to prove in a consistency proof of 
any score-based structure learning algorithm is that the score, considered as a function on all of $\mathcal{B}(P)$,
achieves its maximum at the true network $(G,P)$.
Generally, in order to obtain a finite sample complexity result, we cannot achieve this and instead have
to content ourselves with a statement that $G$ maximizes the score on the subset of $\mathcal{B}(P)$ 
\textit{excluding those $(G',P')\neq (G,P)$ which are $\zeta$-close to $(G,P)$}.  We denote by 
 $\mathcal{B}_{\zeta}(P)$ the subset of $\mathcal{B}(P)$
which excludes the structures in $\mathcal{B}(P)$ which are $\zeta$-close to---but \textit{not} equal to---$(G,P)$. 
So  $\mathcal{B}_{\zeta}(P)$ can be thought of as $\mathcal{B}(P)$, with the $\zeta$-ball around
$(G,P)$ excised, but with the point $(G,P)$ itself added back in.  Then
the consistency results we \textit{can} prove will say that the score, considered as a function
\textit{restricted to} $\mathcal{B}_{\zeta}(P)$, achieves its maximum at the true network $(G,P)$.

We recall the following relationship (e.g. p. 17 of \cite{Friedman:UAI96})
between the empirical log-likelihood of the data $\omega_N$ under an arbitrary
network $G'\in\mathcal{G}^d$ and the KL-divergence of $P({\omega_N})$
based at $p_{G',\omega_N}$:
\begin{equation}\label{eqn:empiricalLikelihoodDecompositionRecalled}
\mathrm{LL}(\omega_N,G')=NH(P(\omega_N))-NH( p(\omega_N) \| p_{G',\omega_N} ).
\end{equation}
We also have the corresponding relationship for the idealized log-likelihood of any distribution $P$:
\begin{equation}\label{eqn:idealLikelihoodDecompositionRecalled}
\mathrm{LL}^{(I)}_N(P,G') = NH(P)-NH( P \| p_{G',P} ).
\end{equation}
\nomenclature{$\mathrm{LL}^{(I)}_N(P,G')$}{$NH(P)-NH( P \lvert\rvert p_{G',P} )$}
It is clear that the condition that $P$ happens to
factorize according to $G'$ is equivalent to $p_{G',P} =P$, and also to $NH( P \| p_{G',P} )=0$ , so that 
by \eqref{eqn:idealLikelihoodDecompositionRecalled}, the $G'$ which are I-maps
for $P$ are simultaneously the set of distributions for which $NH( P \| p_{\bm{\cdot},P} )$ takes its
minimum value of $0$ and for which $\mathrm{LL}^{(I)}_N(P,\bm{\cdot})$ takes its maximum value of  $NH(P)$.
Similar statements follow from \eqref{eqn:empiricalLikelihoodDecompositionRecalled} in relation to  $\mathrm{LL}(\omega_N,\bm{\cdot})$
and $NH(p(\omega_N)\| p_{\bm{\cdot},\omega_N})$.

The expressions  \eqref{eqn:empiricalLikelihoodDecompositionRecalled} and \eqref{eqn:idealLikelihoodDecompositionRecalled},
together with the estimate of Proposition \ref{prop:LLestimate} allow us to say (with confidence $1-\delta$) that,
scoring only on the log-likelihood, the correct structure $(G,P)$ loses to any competing structure in $\mathcal{B}(P)$ by only a certain margin, 
and furthermore, actually beats any competing structure in $\mathcal{B}_{\zeta}(P)$ by a certain positive margin 
(both margins are actually linear in $N$).

\begin{proposition}\label{prop:empiricalLLestimate}
Let $\epsilon, \delta, \zeta>0$ be fixed, and let $d$ be a fixed in-degree.  Let $G\in \mathcal{G}^d$
be a fixed Bayesian network.  Let $N_n(\epsilon,\delta)$
denote the function of Proposition \ref{prop:LLestimate}.
\begin{itemize}
\item[(a)]  We have, for all $G'\in \mathcal{G}^d$,
\[
\mathrm{Pr}_{\omega_N\sim G}\left\{\mathrm{LL}(\omega_N,G)-\mathrm{LL}(\omega_N,G') > -N\epsilon\right\}\geq 1-\delta,
\]
for
\[
N>N_n\left(\frac{\epsilon}{2},\frac{\delta}{2}\right).
\]
\item[(b)] We have, for all $G'\in \mathcal{G}^d$, $(G',P')\in \mathcal{B}_{\zeta}(P)$ and $G\neq G'$, which is to say,
for the set of Bayesian networks $(G',p_{G',P})\in \mathcal{G}^d$, excluding only those networks such that $H( P \| p_{G',P})<\zeta$, 
\[
\mathrm{Pr}_{\omega_N\sim G}\left\{\mathrm{LL}(\omega_N,G)-\mathrm{LL}(\omega_N,G') > \frac{N\zeta}{3}\right\}\geq 1-\delta,
\]
for
\[ 
N>N_n\left(\frac{\zeta}{3},\frac{\delta}{2}\right).
\]
\end{itemize}
\end{proposition}
\begin{proof}
The idea is to estimate the difference between the empirical log-likelihoods by the difference of the idealized counterparts.
With probability $1-\delta$, we have by Proposition \ref{prop:LLestimate},
\[
\begin{aligned}
\mathrm{LL}(\omega_N,G)-\mathrm{LL}(\omega_N,G') 
&=\quad\! \mathrm{LL}^{(I)}(P,G)-\mathrm{LL}^{(I)}(P,G')\\
&\quad+\mathrm{LL}(\omega_N,G) - \mathrm{LL}^{(I)}_N(P,G)\\
 &\quad+\mathrm{LL}^{(I)}_N(P,G') - \mathrm{LL}(\omega_N,G')\\
& \geq\quad\! NH(P)-NH(P\| p_{G,P})\\
&\quad  -NH(P)+NH(P\|p_{G',P}) \\
&\quad- \begin{cases}\epsilon N, & N>N_n\left(\frac{\epsilon}{2},\frac{\delta}{2}\right) \\
\frac{2\zeta}{3}N, & N> N_n\left( \frac{\zeta}{3},\frac{\delta}{2}\right).
\end{cases}
\end{aligned}
\]
Regarding the expression appearing \textit{on four lines after the last inequality}:
the quantity $NH(P)$ in the first line cancels with its opposite $-NH(P)$ in the second line.  
As discussed in the paragraph preceding the proposition, 
the function $G'\mapsto NH(P\| p_{G',P })$ takes its
minimum value of $0$ at $G'=G$, and so is positive at $G'\in \mathcal{G}^d$.  Thus, 
\[
\begin{aligned}
\mathrm{LL}(\omega_N,G)-\mathrm{LL}(\omega_N,G')
&\geq N\cdot \left(  H(P\|p_{G',P}) - \begin{cases}\epsilon, & N>N_n\left(\frac{\epsilon}{2},\frac{\delta}{2}\right) \\
\frac{2\zeta}{3}, & N> N_n\left( \frac{\zeta}{3},\frac{\delta}{2}\right)
\end{cases}   \right)\\
&\geq N\cdot\left( \begin{cases}-\epsilon, & (G',P')\in \mathcal{B}(P),\; N>N_n\left(\frac{\epsilon}{2},\frac{\delta}{2}\right) \\
\zeta-\frac{2\zeta}{3}, & (G',P')\in \mathcal{B}_{\zeta}(P),\; N> N_n\left( \frac{\zeta}{3},\frac{\delta}{2}\right)
\end{cases}  \right). 
\end{aligned}
\]
The key aspect of Proposition \ref{prop:LLestimate} which allows us to make these estimates is the quantification
over all $G'\in\mathcal{G}^d$, inside the estimate of Proposition \ref{prop:LLestimate}.
\end{proof}
In Proposition \ref{prop:empiricalLLestimate} we have summed up all we have to know about the log-likelihood
terms in order to proceed with the proof of the finite sample complexity in the $n$-node case.
\subsection{Competing Network is $\zeta$-Distant from Generating Network}\label{subsec:competingNetworkSparser}
When the competing network $G'\neq G$ belongs
to $\mathcal{B}_{\zeta}(P)$, we can use part (b) of Proposition \ref{prop:empiricalLLestimate} to show that the difference of log-likelihood terms \eqref{eqn:generalScoreDifference}
is positive.  As a result, we will be able to show that the difference function
$\mathcal{S}_{\eta,N}(G,G')$ is (with high probability) eventually positive,
without using the linear growth of any of the sparsity boost terms from $(A,B)\in G' \backslash G$.  Note that for $G\backslash G'\subseteq \Delta(G,G')$, $\sigma((A,B), G, G')$ is always $-1$ or $0$,  and $S_{A,B}(G, G') = S_{A,B}(G')$.   Thus,
ignoring the terms in the sum over $\Delta(G,G')$ from $G'\backslash G$ 
(which is always positive, since $\sigma((A,B), G, G')=1$ on $G'\backslash G$) we have  
\begin{multline}\label{eqn:generalScoreDifferenceZetaSep}
\mathcal{S}_{\eta,N}(G,G')\geq \mathrm{LL}(\omega_{N},G)-\mathrm{LL}(\omega_{N},G')+\psi_{1}(N)\left(|G'|-|G|\right)\\
-\psi_{2}(N)\left[\sum_{(A,B)\in G \backslash G'}\max_{S\in S_{A,B}(G')}\min_{s\in\mathrm{val}(S)}-\ln\left[\beta_{N}^{p^{\eta}}\tau(p(\omega_N, A,B|s))\right]\right]\\
+\psi_{2}(N)\left[\sum_{(A,B)\not\in E(G)\cup E(G')}\left(\max_{S\in S_{A,B}(G)} - \max_{S\in S_{A,B}(G')}\right)\left(\min_{s\in\mathrm{val}(S)}-\ln\left[\beta_{N}^{p^{\eta}}\tau(p(\omega_N, A,B|s))\right]\right)\right].
\end{multline}
We now explain why we have to use part (b) of Proposition \ref{prop:empiricalLLestimate} on the log-likelihood terms
in order to bound this difference from below: the other terms, namely the difference of complexity penalties and the opposites
of the sparsity boosts, are both negative.  In order to get general lower bounds on these negative terms, we have
to consider the case of $G'$ the empty network on $n$ nodes, in which case $|G'|=n\leq |G|$, and also
there are $|E(G)|$ terms in the sum of opposites of sparsity boosts.  Although, with the choice $\psi_1(N)$ logarithmic
and $\psi_2(N)$ constant, these terms are at most roughly logarithmic in $N$, we still need positivity of the difference of log-likelihoods
to achieve positivity of the entire difference.

Henceforth, we will make the following assumptions in the proof:
\begin{itemize}
\item[(a)]  We have $G$ is a perfect map for $P$, $G'\neq G$,
and $(G',p_{G',P})\in \mathcal{B}_{\zeta}(P)$, equivalently, $G'$ satisfies $H(P,p_{G',P})>\zeta$.
\item[(b)] The separating set $\mathscr{S}$ satisfies
\[
\max_{G\in\mathcal{G}}\sigma(G)<d,
\]  
\nomenclature{$\sigma(G)$}{Maximum of $\lvert S \rvert$ for $S\in S_{A,B}(G)$ and $(A,B)\in V^{\times 2}\backslash \Delta$.}
meaning that for all $G\in\mathcal{G}$ and for all $(A,B)\in V^{\times 2}\backslash \Delta$,
and for all $S\in S_{A,B}(G)$, $|S|<d$.
\item[(c)] The separating set $\mathscr{S}$ is bounded by a polynomial in $n$ of degree $d$,
\textit{i.e.},
\begin{equation}\label{eqn:SigmaScriptSBound}
\Sigma_{\mathcal{S}}(\mathcal{G})\leq
\binom{n}{d},
\end{equation}
\nomenclature{$\Sigma_{\mathcal{S}}(\mathcal{G})$}{Maximum over $(A,B)\in V^{\times 2}\backslash \Delta$ of
$\lvert \cup_{G\in\mathcal{G}}S_{A,B}(G)\rvert$ }
The condition \eqref{eqn:SigmaScriptSBound} means that
\[
\text{For all}\;
(A,B)\in V^{\times 2}\backslash \Delta,\; \left|\bigcup_{G\in\mathcal{G}}S_{A,B}(G)\right|
\leq \binom{n}{d}.
\]
\item[(d)]  We are using the following choices of weighting functions: 
\linebreak {$\psi_1(N):=\kappa\log N$}, $\psi_2(N):= 1$, $\kappa$ a constant 
(e.g., $\kappa = \frac{1}{2}$).
\item[(e)]  For $(A,B)\notin E(G)\cup E(G')$, we have $S_{A,B}(G')\subseteq S_{A,B}(G)$.
\end{itemize} 
We now incorporate these assumptions to estimate $\mathcal{S}_{\eta,N}(G,G')$.
Assumption (e) implies that the third line in \eqref{eqn:generalScoreDifferenceZetaSep},
namely the sum over $(A,B)\notin E(G)\cup E(G')$, is positive. 
The reason is that, when Assumption (e) is in force,
the maximum of the sparsity boosts over the \textit{larger}
separating set $S_{A,B}(G)$, is at least as large as the maximum
of the sparsity boosts over the possibly smaller subset $S_{A,B}(G')$.  For
the purposes of bounding $\mathcal{S}_{\eta,N}(G,G')$ from below, we can therefore ignore the third line of \eqref{eqn:generalScoreDifferenceZetaSep}.
Now, applying part (b) of Proposition \ref{prop:empiricalLLestimate} to the log-likelihood part
we obtain that, for each $\delta_1>0$, we have, with probability at least $1-\delta_1$, for all $N>N\left(\frac{\zeta}{3},\frac{\delta_1}{2}\right)$, where
this function is as defined as in part (b) of Proposition \ref{prop:empiricalLLestimate},
\begin{equation}\label{eqn:SSkeletonContainedCase}
\begin{aligned}
\mathcal{S}_{\eta,N}(G,G')&\geq  \frac{N\zeta}{3} - \kappa\log N(|G|-n)\\
& \quad - |E(G)|\max_{\stackrel{(A,B)\in G}{S\in S_{A,B}(G')}}\min_{s\in\mathrm{val}(S)}-\ln\left[\beta_{N}^{p^{\eta}}\tau(p(\omega_N, A,B|s))\right].
\end{aligned}
\end{equation}


We stress that the quantities $|G|$ and $|E(G)|$ are in general different: the former is the number of parameters in a probabilistic
model for which $G$ is a $P$-map, exponential in the degree, the latter the cardinality of the edge set of (the skeleton of) $G$.
\nomenclature{$\lvert G\rvert$}{Number of parameters in BN with DAG $G$}
\nomenclature{$\lvert E(G)\rvert$}{Number of edges in the DAG $G$}

Now we would like to apply Corollary \ref{cor:complexityPenaltyUpperBound} to bound each sparsity boost\linebreak ${-\ln\left[\beta_{N}^{p^{\eta}}\tau(p(\omega_N, A,B|s))\right]}$
from above by a universal constant $\ln \Theta^{-1}$, but the problem with applying the proposition directly is there is data fragmentation.
Of the $N$ points in the sample only $N/m$ (in the worst case) are expected to contribute to the test statistic $\tau(p(\omega_N, A,B|s))$.
Lemma \ref{lem:dataFragmentation}, below, quantifies by how much we have to increase $N$ to overcome this problem. 
Before stating and proving that lemma, we introduce some notation to help us keep track
of such things.
\begin{definition}\label{defn:dataFragmentation}
Let $N'\geq N>0$ be positive ingers.  Let $S\subset V$ and $s\in\mathrm{Val}(S)$.
Let $\omega_{N'}|_{S=s}$ denote the subsequence of $\omega_{N'}$ consisting
of those points which take the joint values $s$ in the variable $S$.  
\nomenclature{$\omega_{N'}\lvert_{S=s}$}{Subsequence of $\omega_{N'}$ consisting
of those points which take the joint values $s$ in the variable $S$}
When the number
of such points in the subsequence, $\mathrm{card}(\omega_{N'}|_{S=s})$, equals
$N$, we will sometimes use the notation
\[
Y_N := \omega_{N'}|_{S=s}.
\] 
\nomenclature{$Y_N$}{$\omega_{N'}\lvert_{S=s}$, having length $N\leq N'$.}
Note that this notation emphasizes the \textbf{number} of points $N$ in the subsequence,
at the cost of suppressing the conditioning assignment $s$ from the notation.
\end{definition}
As a result of the comment in the definition, the notation $Y_N$ is used, for the sake
of readability, in those portions of the argument where $s$ does not matter.
\begin{lem}\label{lem:dataFragmentation}  Let $\theta>0$.  
\begin{itemize}
\item[(a)]  For a fixed $S\in S_{A,B}(G)$ and $s\in \mathrm{Val}(S)$,
let 
\[
m=m_p(s):=\left[ p(S=s)\right]^{-1}.
\]
Then,
 \[\mathrm{Pr}_{\omega_{mN}\sim p}\left\{ \mathrm{card}\left(\omega_{mN}|_{S=s}\right) < (1-\theta)N \right\}\leq e^{-Nm\theta^2/3}.\]
\item[(b)]  For a fixed edge $(A,B)\in V^{\times 2}\backslash \Delta$,
in particular for $(A,B)\in G$, let
\[
m=m_p((A,B),G,\mathcal{S}):=
\max_{\stackrel{S\in S_{A,B}(G)}{s\in\mathrm{val}(S)}}
\left[
p(S=s)
\right]^{-1}.
\]
Then,
\[
\mathrm{Pr}_{\omega_{mN}\sim p}\left\{ \min_{\stackrel{S\in S_{A,B}(G)}{s\in\mathrm{val}(S)}}\mathrm{card}\left(\omega_{mN}|_{S=s}\right) < (1-\theta)N \right\}\leq|S_{A,B}(G)| 2^{
\max_{S\in S_{A,B}(G)}|S|} e^{-Nm\theta^2/3}.
\]
\item[(c)]  Recall the notations $S_{A,B}(\mathcal{G})$, $\sigma_{\mathcal{S}}(\mathcal{G})$
and $\Sigma_{\mathcal{S}}(\mathcal{G})$,
and $m(G,\mathcal{G},\mathcal{S})$ from Section \ref{subsec:separatingSets}.
In particular, set
\begin{equation}\label{eqn:nonHatmDefinition}
m=m_p(G,\mathcal{G},\mathcal{S}):=\max_{(A,B)\in G}m_p((A,B),\mathcal{G},\mathcal{S}),
\end{equation}
where
\[
m_p((A,B),\mathcal{G},\mathcal{S}):=\max_{\stackrel{S\in S_{A,B}(\mathcal{G})}{s\in\mathrm{val}(S)}}
\left[
p(S=s)
\right]^{-1}.
\]
Then
\[
\mathrm{Pr}_{\omega_{mN}\sim p}\left\{ \min_{(A,B)\in G}\min_{\stackrel{S\in S_{A,B}(\mathcal{G})}{s\in\mathrm{val}(S)}}\mathrm{card}\left(\omega_{mN}|_{S=s}\right) < 
(1-\theta)N \right\}\leq 
|E(G)|\cdot\Sigma_{\mathcal{S}}(\mathcal{G})\cdot 2^{\sigma(\mathcal{G})} e^{-Nm\theta^2/3}.
\]
\end{itemize}
\end{lem}
\begin{proof}  Part (a) is a straightforward application of the Multiplicative Chernoff bounds, second part.  Part (b) follows from the observation
that the value set $\mathrm{Val}(S)$ for $S\in \mathcal{S}_{A,B}(G)$ can  have
at most $2^{\max_{S\in S_{A,B}(G)}|S|}$ joint values, and from the union bound applied to the estimate in Part (a).  Part (c) similarly follows (a) and from the union bound, the fact
that $\sigma_{\mathcal{S}}(\mathcal{G})$ (resp., $\Sigma_{\mathcal{S}}(\mathcal{G})$) provides an upper bound for $\max_{S\in S_{A,B}(G')}|S|$ (resp., $|S_{A,B}(G')|$) where $G'$ ranges over $\mathcal{G}$ and $(A,B)\in E(G)$.
\end{proof}
\begin{corollary}\label{cor:dataFragmentation}
In accordance with Definition \ref{defn:dataFragmentation}, if $N'>m(1-\theta)^{-1}N$,
then with probability at least
\begin{equation}\label{eqn:dataNonSparseProbability}
1-|E(G)|\cdot\Sigma_{\mathcal{S}}(\mathcal{G})\cdot 2^{\sigma(\mathcal{G})} e^{-Nm\theta^2/3},
\end{equation}
 for all $(A,B)\in G$, for all $S\in S_{A,B}(\mathcal{G})$
and for all $s\in \mathrm{Val}(S)$, 
\[
\omega_{N'}|_{S=s} 
\;\text{can be written as}\; Y_N,
\] 
where $Y_N$ is a sequence of observations, all of which take the joint value $s$ at $S\subset V$.
\end{corollary}
\nomenclature{$N'$}{Integer larger than $N$, often related to $N$ by $N'=m(1-\theta)^{-1}N$}
\begin{definition}\label{defn:sHat}
Fix $p\in\mathcal{P}$.  For each $(A,B)\in G$, and for each $S\in S_{A,B}(\mathcal{G})$, let
\[
\hat{s}=\hat{s}_p((A,B),S)=\mathrm{argmax}_{s\in \mathrm{Val}(S)}\tau(p(A,B|s)).
\]
\end{definition}
The hypothesis on minimal edge strength in $G$ says precisely that for all 
$(A,B)\in G$ and for all $S\in S_{A,B}(\mathcal{G})$, we have
\[
\tau(p(A,B|\hat{s}))>\epsilon.
\]
Because Corollary \ref{cor:dataFragmentation} is quantified over \textit{all}
$(A,B)\in G$, over \text{all} $S\in S_{A,B}(\mathcal{G})$ and over \textit{all} $s\in\mathrm{Val}(S)$, we have that if $N'>m(1-\theta)^{-1}N$, then with probability \eqref{eqn:dataNonSparseProbability},
\[
\mathrm{card}\left(\omega_{{N'}}|_{S=\hat{s}}\right) > N,
\]
for all $(A,B)\in G$, $S\in S_{A,B}(\mathcal{G})$ .  That is to say 
with probability \eqref{eqn:dataNonSparseProbability}
that as $(A,B)$ varies over $E(G)$, all of the
\[
\omega_{N'}|_{S=\hat{s}_p((A,B),S)}\;\text{can be written, \textit{simultaneously}, in the form}\; Y_N.
\]
Now we are able to bound from above the sum of sparsity boosts.
\begin{lem} \label{lem:wholeSparsityBoostUpperEstimate} Let $N'$, $m$, $\theta$
be related by
\[
N'=m(1-\theta)^{-1}N.
\]  
Let
\begin{equation}\label{eqn:NcondWholeSparsityBoostUpperEstimate}
N'>m(1-\theta)^{-1}\left[
F(\Delta)
\right]^{-1}\log
\frac{24}{1-\Theta}.
\end{equation}
With probability at least
\[
1-|E(G)|\cdot\Sigma_{\mathcal{S}}(\mathcal{G})\cdot 2^{\sigma(\mathcal{G})} e^{-Nm\theta^2/3}
-|E(G)|\Sigma_{\mathcal{S}}(\mathcal{G})\mathcal{F}_N\left(
\Delta-\tilde{\mathcal{F}}^{-1}_N(1-\Theta)
\right),
\]
we have
\begin{equation}\label{eqn:singleSparsityBoost}
\max_{\stackrel{(A,B)\in G}{S\in S_{A,B}(G')}}\min_{s\in\mathrm{val}(S)}-\ln\left[\beta_{N'}^{p^{\eta}}\tau(p(\omega_{N'}, A,B|s))\right]\leq \log \Theta^{-1}.
\end{equation}
\end{lem}
\begin{proof}
We prove the bound for arbitrary $(A,B)\in G$, arbitrary $S\in S_{A,B}(\mathcal{G})$,
and for the \text{one value} $\hat{s}\in\mathrm{Val}(S)$, where
$\hat{s}:=\hat{s}((A,B),S)$, at which the maximum
separation $\geq \epsilon$ from independence is achieved. 
This will suffice because of pattern of maxes and mins 
on the left-hand side of \eqref{eqn:singleSparsityBoost}, the quantity to be estimated.  By
the comments between Definition \ref{defn:sHat} and Lemma \ref{lem:wholeSparsityBoostUpperEstimate}, we have for each $(A,B)\in G$,
$S\in S_{A,B}(\mathcal{G})$ and $\hat{s}:=\hat{s}((A,B),S)$, for $N'>m(1-\theta)^{-1}N$,
\[
-\ln\left[\beta_{N'}^{p^{\eta}}\tau(p(\omega_{N'}, A,B|s))\right] = 
-\ln\left[\beta_{N}^{p^{\eta}}\tau(p(Y_N, A,B))\right],
\]
for $Y_N$ a sequence of observations of length $N$, with probability \eqref{eqn:dataNonSparseProbability}.

Now assume that $N$ satisfies \eqref{eqn:NcondWholeSparsityBoostUpperEstimate}.
By Corollary \ref{cor:betaUpperProbableEstimate} and the union bound, 
for fixed $(A,B)\in G$, the sparsity boosts is less than or equal to
\[
-\ln\left[\beta_{N}^{p^{\eta}}\tau(p(Y_N, A,B))\right]
\leq \log \Theta^{-1},
\]
with probability
\[
1
-\Sigma_{\mathcal{S}}(\mathcal{G})\mathcal{F}_N\left(
\Delta-\tilde{\mathcal{F}}^{-1}_N(1-\Theta)\right).
\]
One further use of the union bound completes the proof of the statement over 
all $(A,B)\in E(G)$.
\end{proof}
 
We can now estimate the entire difference of objective functions.
\begin{proposition}\label{eqn:FiniteSampleComplexityMostGeneralnNodeCase}
For $N'=m(1-\theta)^{-1}N$,
\begin{equation}\label{eqn:FirstTwoNPrimeConditions}
N' > \max\left[ m(1-\theta)^{-1}\left[
F(\Delta)
\right]^{-1}\log
\frac{24}{1-\Theta},\,
N_{n}\left(
\frac{\zeta}{3},
\frac{\delta_1}{2}
\right)
\right],
\end{equation}
we have with probability at least 
\begin{equation}\label{eqn:ConfidenceWithNImplicitlyInvolved}
1-\delta_1-
|E(G)|\cdot\Sigma_{\mathcal{S}}(\mathcal{G})\cdot 2^{\sigma(\mathcal{G})} e^{-Nm\theta^2/3}
-|E(G)|\Sigma_{\mathcal{S}}(\mathcal{G})\mathcal{F}_N\left(
\Delta-\tilde{\mathcal{F}}^{-1}_N(1-\Theta)
\right)
\end{equation}
the lower bound
\begin{equation}\label{eqn:SSkeletonContainedCaseAfterApplyingSparsityBoostBound}
S_{\zeta}(G,G',\omega_{N'})=\frac{N'\zeta}{3}-\kappa\log N' \left(
|G|-n
\right)-
|E(G)|\log\Theta^{-1}.
\end{equation}
Thus, if we assume, in addition, that
\begin{equation}\label{eqn:NprimescriptWCondition}
N' > \frac{3\kappa(|G|-n)}{\zeta}\mathcal{W}
\left(
\frac{\zeta\Theta^{\frac{|E(G)|}{\kappa(|G|-n)}}}{3\kappa(|G|-n)}
\right)
\end{equation}
then we have with probability at least \eqref{eqn:ConfidenceWithNImplicitlyInvolved},
\[
S_{\eta}(G,G',\omega_N')\geq 0.
\]
\end{proposition}
\begin{proof}
The first statement follows from Lemma \ref{lem:wholeSparsityBoostUpperEstimate}, applied
to \eqref{eqn:SSkeletonContainedCase}.  The second statement
follows from solving \eqref{eqn:SSkeletonContainedCaseAfterApplyingSparsityBoostBound}
for $N'$.
\end{proof}
Proposition \ref{eqn:FiniteSampleComplexityMostGeneralnNodeCase}
gives the most general form of the finite sample complexity.  The only
thing that remains is to explain how to choose the number of samples
$N'$ so that \eqref{eqn:ConfidenceWithNImplicitlyInvolved} is at least
$1-\delta$, for $\delta>0$ an arbitrary error probability set by the experimenter.
The following gives one way (of many possible ways) of dealing with the most
complex term in the error probability.
\begin{lem}\label{lem:mostComplexTermInErrorProbability}
Let $\Gamma,\,\Delta,\,\delta'>0$.  Suppose that $N$ is an integer satisfying
\begin{equation}\label{eqn:NErrorProbabilityLastCondition}
N>
\max\left(
\left[
\tilde{F}\left(
\frac{\Delta}{2}
\right)
\right]^{-1}
\log \frac{24}{\Gamma},\,
\left[
F\left(
\frac{\Delta}{2}
\right)
\right]^{-1}
\log \frac{24}{\delta'}
\right).
\end{equation}
Then $\mathcal{F}_N^{-1}(\Gamma)<\frac{\Delta}{2}$, and
\[
\mathcal{F}_N(\Delta-\tilde{\mathcal{F}}_N^{-1}(\Gamma))<\delta'
\]
\end{lem}
\begin{proof}
By Lemma \ref{lem:functionalInversesDecreasing}, 
\[
N>(\log 24 - \log\Gamma)\cdot \left[
\tilde{F}\left(
\frac{\Delta}{2}
\right)
\right]^{-1}
\]
is equivalent to $N>\tilde{\mathcal{G}}_{\frac{\Delta}{2}}^{-1}(\Gamma)$,
which is equivalent, by the definition of $\mathcal{G}$ to \linebreak ${\tilde{\mathcal{F}}_N^{-1}(\Gamma)<
\frac{\Delta}{2}}$.  
Then the quantity inside $\mathcal{F}_N$ is at least $\frac{\Delta}{2}$.  Further,
by the same reasoning, $\mathcal{F}_N\left( \frac{\Delta}{2} \right)<\delta'$
is equivalent to
\[
N>(\log 24-\log\delta' )\left[F\left(
\frac{\Delta}{2}
\right)
\right]^{-1}.
\]
So, if \eqref{eqn:NcondWholeSparsityBoostUpperEstimate} is satisfied then,
because $\tilde{\mathcal{F}}_N$ and $\mathcal{F}_N$ are decreasing
\[
\mathcal{F}_N(\Delta-\tilde{\mathcal{F}}_N^{-1}(\Gamma))\leq
\mathcal{F}_N\left(
\frac{\Delta}{2}
\right)<\delta'.
\]
\end{proof}
\begin{lem} \label{eqn:zetaDistantLastLemma} Let $\delta>0$ be given.  Set $\delta_1=\frac{\delta}{3}$.  Also, suppose that $N'$
satisfies \eqref{eqn:FirstTwoNPrimeConditions}, so that in particular
\[
N'>N_{n}\left(\frac{\zeta}{3},\frac{\delta}{6}  \right).
\]
Also suppose that $N$ is an integer and such that $N,N'$, satisfy $N'=Nm(1-\theta)^{-1}$,
and such that
\[
N>\frac{3}{m\theta^2}\log\frac{3|E(G)|\Sigma_{\mathcal{S}}(\mathcal{G})2^{\sigma(\mathcal{G})}}{\delta}.
\]
Finally, suppose that $N$ is greater than the quantity on the right side of \eqref{eqn:NErrorProbabilityLastCondition} with $\Gamma$ set at $1-\Theta$
and $\delta'$ set at $\frac{\delta}{3 |E(G)|\Sigma_{\mathcal{S}}(\mathcal{G})  }$, that is
\[
N>
\max\left(
\left[
\tilde{F}\left(
\frac{\Delta}{2}
\right)
\right]^{-1}
\log \frac{24}{1-\Theta},\,
\left[
F\left(
\frac{\Delta}{2}
\right)
\right]^{-1}
\log \frac{72 |E(G)|\Sigma_{\mathcal{S}}(\mathcal{G})  }{\delta}
\right).
\]
  Then the quantity in \eqref{eqn:ConfidenceWithNImplicitlyInvolved} given as a function of $N'$
(and $N$) is at least $1-\delta$.
\end{lem}
\textbf{This completes the proof of the finite sample complexity result in the case when the competing network is $\zeta$-distant from the generating network.}

\subsection{Competing network is not Contained in Generating Network}
Now, assume that $\mathrm{Skel}(G')\not\subseteq\mathrm{Skel}(G)$.
In this case, there are not only terms (with negative sign) corresponding to edges \mbox{$(A,B)\in G \backslash G'$} but also terms with positive sign
corresponding to edges \mbox{$(A,B)\in G'\backslash G$}:
\begin{multline*}
\mathcal{S}_{\eta,N}(G,G')=\mathrm{LL}(\omega_{N},G)-\mathrm{LL}(\omega_{N},G')+\psi_{1}(N)\left(|G'|-|G|\right)\\
-\psi_{2}(N)\left[\sum_{(A,B)\in G \backslash G'}\max_{S\in S_{A,B}(G')}\min_{s\in\mathrm{val}(S)}-\ln\left[\beta_{N}^{p^{\eta}}\tau(p(\omega_N, A,B|s))\right]\right]\\
+\psi_{2}(N)\left[\sum_{(A,B)\in G' \backslash G}\max_{S\in S_{A,B}(G)}\min_{s\in\mathrm{val}(S)}-\ln\left[\beta_{N}^{p^{\eta}}\tau(p(\omega_N, A,B|s))\right]\right]\\
+\psi_{2}(N)\left[\sum_{(A,B)\not\in E(G)\cup E(G')}\left(\max_{S\in S_{A,B}(G)} - \max_{S\in S_{A,B}(G')}\right)\left(\min_{s\in\mathrm{val}(S)}-\ln\left[\beta_{N}^{p^{\eta}}\tau(p(\omega_N, A,B|s))\right]\right)\right]
\end{multline*}
The strategy will be different because we will prove (probable) asymptotically linear growth for the third line.  As a result,
we will need only to \textit{bound} the first two lines, rather than obtain a positivity result for them.  In particular,
it will not be necessary to use part (b) of Proposition \ref{prop:empiricalLLestimate}, and we will instead use part (a),
with no assumption that $G'$ is $\zeta$-separated from $G$ in KL-divergence. 

We continue to assume that conditions (a)--(e) from Section \ref{subsec:competingNetworkSparser} are still in
force. 

We can use the same result, Lemma \ref{lem:wholeSparsityBoostUpperEstimate},
as in Section \ref{subsec:competingNetworkSparser}
to bound the \textit{negated} sum of sparsity boosts, because this is exactly the same
sum as appeared in \linebreak $\mathcal{S}_{\eta}(G,G',\omega_N)$ for the case of $\mathrm{Skel}(G')
\subseteq \mathrm{Skel}(G)$.
Thus, by straightforward
applications of Proposition \ref{prop:empiricalLLestimate}(a)
and Lemma \ref{lem:wholeSparsityBoostUpperEstimate}, 
we get the following counterpart to Proposition \ref{eqn:FiniteSampleComplexityMostGeneralnNodeCase}.
\begin{proposition}\label{prop:InitialEstimateScoreDifferenceNotContainedCase}
Let $\epsilon, \delta_1>0$.  Let $\Theta\in(0,1)$.  Let $N'$ be a positive integer satisfying
\[
N' > \max\left[m(1-\theta)^{-1}\left[
F(\Delta)
\right]^{-1}\log\frac{24}{1-\Theta},\,
N_{n}\left(
\frac{\epsilon}{2},
\frac{\delta_1}{2}
\right)
\right],
\]
where $N_{n}$ is the function of Proposition \ref{prop:LLestimate}.  We have with probability
at least \eqref{eqn:ConfidenceWithNImplicitlyInvolved} that
\begin{multline}\label{eqn:skeletonNotContainedCase}
S_{\eta}(G,G',\omega_{N'})\geq -N'\epsilon-\kappa \log N'(|G|-n)
-|E(G)|\log\Theta^{-1}+\\
\psi_2(N')
\left[\sum_{(A,B)\in G' \backslash G}\max_{S\in S_{A,B}(G)}\min_{s\in\mathrm{val}(S)}-\ln\left[\beta_{N'}^{p^{\eta}}\tau(p(\omega_N', A,B|s))\right]\right].
\end{multline}
\end{proposition}
We now begin to study, and bound from below, 
the second line in this lower bound \eqref{eqn:skeletonNotContainedCase} for
$S_{\eta}(G,G',\omega_{N'})$.  For $(A,B)\in V^{\times 2}\backslash \Delta$,
and $S\in S_{A,B}(\mathcal{G})$, recall the notation
\[
m_p(S)=\max_{s\in\mathrm{Val}(S)}m_p(s)=\max_{s\in\mathrm{Val}(S)}\left[p(S=s)\right]^{-1}.
\]
\begin{definition}\label{def:HatSDefn}
Let $(A,B)\in G' \backslash G$.  Because $(A,B)\not\in E(G)$ and $\mathcal{S}$
is a separating collection for $G$ in $\mathcal{G}$, by Definition \ref{defn:separatingSet},
there are some (at least one and possibly more than one) $S\in S_{A,B}(G)$
with the property that  $A\ci B|S$.  Define
\begin{equation}\label{eqn:HatSDefn}
\mbox{\boldmath $\hat{S}:=\hat{S}_p((A,B),G)$\unboldmath}\in S_{A,B}(G)\;\text{such that}\; 
A\ci B|\hat{S}\;\text{and}\;m_p(\hat{S})\;\text{minimal}.
\end{equation}
\end{definition}
\nomenclature{$\hat{S}:=\hat{S}_p((A,B),G)$}{Element of $S_{A,B}(G)$ 
such that $A\ci B\lvert \hat{S}\;\text{and}\;m_p(\hat{S})\;\text{minimal}$}
Note that there are potentially several $\hat{S}\in S_{A,B}(G)$ such that
$A\ci B|\hat{S}$, and the $m_p(\hat{S})$ are all equal
and minimal for $S\in S_{A,B}(G)$ with the property that ${A\ci B| S}$.
But because there are only at most
$|S_{A,B}(G)|\leq \Sigma_{\mathcal{S}}(G)$ such $S$, there are only
finitely many choices of $\hat{S}$.   In order to
make the notation concise but technically correct, we assume a fixed but arbitrary choice of 
$\hat{S}$ for each ${(A,B)\in G'\backslash G}$
which satisfies \eqref{eqn:HatSDefn}: the choice of $\hat{S}$ has no effect
on the argument, so long as one choice is fixed throughout the argument.

Since our purpose is to estimate \textit{from below} the maximum of the sparsity boosts over
$S\in S_{A,B}(G)$, we may replace the entire maximum over $S_{A,B}(G)$
in \eqref{eqn:skeletonNotContainedCase} with the expressions
being maximized evaluated at any particular $S$, for example, at $S=\hat{S}$.  

\begin{proposition}\label{prop:scriptLSumEstimate}  Let $\mathcal{L}$ be a subset of $G'\backslash G$
of cardinality $L$.  Let
\begin{equation}\label{eqn:maxOfTwoMs}
\hat{m}:=\hat{m}_p(G,\mathcal{S}):=\max_{(A,B)\in V^{\times 2}\backslash G} 
m_p(\hat{S}_p((A,B),G)).
\end{equation}
\nomenclature{$\hat{m}$}{$\hat{m}_p(G,\mathcal{S}):=\max_{(A,B)\in V^{\times 2}\backslash G} 
m_p(\hat{S}_p((A,B),G))$}%
Denote $\hat{m}_p(G,\mathcal{S})$ more concisely by $\hat{m}$. 
Let $N'$ and $N$ be positive integers such that $N'>(1-\theta)^{-1}\hat{m}N$.
Let $\mu\in(0,1)$.  Let $\Gamma$ satisfy $\Gamma\in (0,\Gamma^{\max}(N,\mu,\eta))$.
With probability at least 
\[
1-\Sigma_{\mathcal{S}}(G)L2^{\sigma(G)}e^{-N\hat{m}\theta^2/3} -  \sigma_{\mathcal{S}}(G)L\mathcal{F}_N(\eta(1-\mu)),
\]
we have for all $(A,B)\in\mathcal{L}$, with $\hat{S}=\hat{S}((A,B),G)$ defined as above
\[
\min_{s\in \mathrm{Val}(\hat{S})}-\ln\left[
\beta_{N'}^{p^\eta}\left(
\tau(p_{\omega_{N'}},A,B|s)\right)\right]
>N\Gamma.
\]
Consequently,
\[
\left[\sum_{(A,B)\in G' \backslash G}\max_{S\in S_{A,B}(G)}\min_{s\in\mathrm{val}(S)}-\ln\left[\beta_{N'}^{p^{\eta}}\tau(p(\omega_{N'}, A,B|s))\right]\right]>LN\Gamma.
\]
\end{proposition}
The main step in the proof is to prove a version of Lemma \ref{lem:dataFragmentation}(c)
which is appropriate for the case at hand.
\begin{lem} \label{lem:scriptLestimate} Let $\mathcal{L}\subseteq\left\{
(A,B)\in G'\backslash G
\right\}$, $L:=\mathrm{Card}(\mathcal{L})$ as above,
$\hat{m}$ as in \eqref{eqn:maxOfTwoMs}.  Then
\[
\mathrm{Pr}\left\{
\min_{(A,B)\in \mathcal{L}}
\min_{s\in\mathrm{Val}
\left(
\hat{S}((A,B),G)
\right)}
\mathrm{Card}\left(
\omega_{\hat{m}N}|_{\hat{S}=s}
\right)<(1-\theta)N
\right\}\leq 
L\Sigma_{\mathcal{S}}(G)2^{\sigma(G)}e^{-N\hat{m}\theta^2/3},
\]
where
\[
\Sigma_{\mathcal{S}}(G):=\max_{(A,B)\in G}S_{A,B}(G).
\]
\end{lem}
\begin{proof}
We begin with Lemma \ref{lem:dataFragmentation}(b), in particular the version with
\linebreak ${(A,B)\in V^{\times 2}\backslash \Delta}$.  Apply Lemma \ref{lem:dataFragmentation}(b)
to all the $(A,B)\in \mathcal{L}\subset V^{\times 2}\backslash \Delta$, simultaneously.
The event whose probability we are estimating in the present Lemma
differs from the event whose probability we have estimated in Lemma \ref{lem:dataFragmentation}(c)
in only two ways: first we have $(A,B)$ ranging over the set $\mathcal{L}$,
instead of $E(G)$.  Second, $\hat{S}$ belongs to $S_{A,B}(G)$, which is
a subset of the larger set $S_{A,B}(\mathcal{G})$.  The first difference
accounts for the appearance of $L$ instead of $|E(G)|$ in the estimate.
The second difference accounts for the appearance of $\Sigma_{\mathcal{S}}(G)$,
in place of $\Sigma_{\mathcal{S}}(\mathcal{G})$ in the error estimate.
\end{proof}
\begin{proof}[Completion of Proof of Proposition \ref{prop:scriptLSumEstimate}]
Take $N'>(1-\theta)^{-1}\hat{m}N$.
Let
\[
\omega_{N'}\sim p\; \text{with}\; \mathrm{Card}(\omega_{N'})=N'
\]
By Lemma \ref{lem:scriptLestimate}, with probability at least
\[
1-L\Sigma_{\mathcal{S}}(G)2^{\sigma(G)}e^{-N\hat{m}\theta^2/3},
\]
we have that for all $s\in\mathrm{Val}(\hat{S})$,
\[
\omega_{N'}|_{\hat{S}=s}=Y_N,\;\mathrm{Card}(Y_N)=N,
\]
for $Y_N$ a sequence of observations meeting the criteria that the observations take the joint value of $s$ at the variable-set $\hat{S}$.

Then we can apply Corollary \ref{cor:sparsityBoostLowerBoundIndependentNetworkVer2}
to each one of the terms in the sum.  Using Corollary
\ref{cor:sparsityBoostLowerBoundIndependentNetworkVer2} to estimate
and the union bound, we conclude that for $\omega_{N'}\sim p$, $\mathrm{Card}(\omega_{N'})=N'$, we have, by Lemma \ref{lem:scriptLestimate}, with probability
at least
\[
1-L\Sigma_{\mathcal{S}}(G)2^{\sigma(G)}e^{-N\hat{m}\theta^2/3},
\]
that $\omega_{N'}|_{\hat{S}=s}=Y_N$, $\mathrm{Card}(Y_N)=N$.
For
\[
\Gamma\in \left(0,\Gamma^{\max}(N,\mu,\eta)\right)=\left(0,\frac{-\log 8}{N}+F(\mu\eta)\right)
\]
we have
\[
\begin{aligned}
\mathrm{Pr}_{\omega_{N'}\sim G}\left\{
\min_{s\in \mathrm{Val}(\hat{S})}
-\ln\beta_N^{p^\eta}
\left(
\tau(\omega_{N'},A,B|s)
\right)>N\Gamma
\right\}&\geq
1-|\mathrm{Val}(\hat{S})|\mathcal{F}_{N}(\eta(1-\mu))\\
&
\geq 1-\sigma_{\mathcal{S}}(G)\mathcal{F}_N(\eta(1-\mu)).
\end{aligned}
\]
To complete the proof, we need a bound that holds for \textit{all} $(A,B)\in\mathcal{L}$
simultaneously, so
we apply the union bound one more time, at the cost of one more factor of $L$
appearing with the $\sigma_{\mathcal{S}}(G)\mathcal{F}_N(\eta(1-\mu))$
term in the error probability.
\end{proof}

Combining Propositions \ref{prop:InitialEstimateScoreDifferenceNotContainedCase} and \ref{prop:scriptLSumEstimate}, we obtain
\begin{proposition} \label{prop:mostGeneralFormDenserCase} Let $\Theta,\mu\in(0,1)$, 
$\hat{m}$ as in \eqref{eqn:maxOfTwoMs}.  Let $N'$ be a positive integer satisfying
\[
N' > \max\left[m(1-\theta)^{-1}\left[
F(\Delta)
\right]^{-1}\log\frac{24}{1-\Theta},\,
N_{n}\left(
\frac{\epsilon}{2},
\frac{\delta_1}{2}
\right)
\right]
\]
and let $N$ be a smaller positive integer and $\theta\in(0,1)$ such that
\begin{equation}\label{eqn:thetaNAndNPrimeRelation}
N'=\hat{m}(1-\theta)^{-1}N.
\end{equation}
Let $\Gamma\in(0,\Gamma^{\max}(N,\mu,\eta))$.  With probability at least
\begin{multline}\label{eqn:errorProbabilityNCaseSkeletonNotContained}
1-\delta_1-
|E(G)|\cdot\Sigma_{\mathcal{S}}(\mathcal{G})\cdot 2^{\sigma(G)} e^{-Nm\theta^2/3}
-|E(G)|\Sigma_{\mathcal{S}}(\mathcal{G})\mathcal{F}_N\left(
\Delta-\tilde{\mathcal{F}}^{-1}_N(1-\Theta)
\right)\\
-\Sigma_{\mathcal{S}}(G)L2^{\sigma(G)}e^{-N\hat{m}\theta^2/3} - \sigma_{\mathcal{S}}(G)L\mathcal{F}_N(\eta(1-\mu)),
\end{multline}
we have
\begin{equation}\label{eqn:SkeletonNotContainedPenultimateEstimate}
\mathcal{S}_{\eta}(G,G',\omega_{N'})\geq \left(\frac{L\Gamma(1-\theta)}{\hat{m}}-\epsilon
\right)N' - 
\kappa\left(
|G|-n
\right)\log N' 
-|E(G)|\log\Theta^{-1}.
\end{equation}
In this case, if we also have
\[
\Gamma > \frac{\epsilon \hat{m}}{L(1-\theta)}
\]
then $S_{\eta}(G,G',\omega_{N'})\geq 0$ for all sufficiently large $N'$, in particular, for
\begin{equation}\label{eqn:nNodeSampleComplexityWFunctionPreliminary}
N'>
\left(
\frac{\hat{m}\kappa (|G|-n)}
{L\Gamma(1-\theta)-\hat{m}\epsilon}
\right)
\mathcal{W}\left(
\frac{\left(L\Gamma(1-\theta)-\hat{m}\epsilon\right)\Theta^{\frac{|E(G)|}{(|G|-n)\kappa}}}
{\hat{m}\kappa(|G|-n)}
\right)
\end{equation}
\end{proposition}
\begin{proof}
The only relation that needs further explanation is where the coefficient
in \eqref{eqn:SkeletonNotContainedPenultimateEstimate} comes from.  From
Propositions \ref{prop:InitialEstimateScoreDifferenceNotContainedCase} and \ref{prop:scriptLSumEstimate} we have the following two linear terms
in the bound
\[
S_{\eta}(G,G',\omega_{N'})\geq L\Gamma N+\epsilon N' + \cdots
\]
where $N$ and $N'$ are as in \eqref{eqn:thetaNAndNPrimeRelation}.
Then because of \eqref{eqn:thetaNAndNPrimeRelation},
\[
S_{\eta}(G,G',\omega_{N'})\geq \left(\frac{L\Gamma(1-\theta)}{\hat{m}} - \epsilon \right)N + \cdots
\]
as claimed.
\end{proof}
This most general formulation of the finite sample complexity is evidently
unsatisfactory as it stands for three reasons: it is not clear under what conditions
we have the non emptiness of the interval,
\[
\left(
\frac{\epsilon \hat{m}}{L(1-\theta)},\,\Gamma^{\max}(N,\mu,\eta)
\right)\neq \emptyset,
\]
so that there is a choice of $\Gamma$ satisfying the hypotheses.  Secondly, in
order to study the asymptotics of $N'$, the number of samples, we need all the conditions
on $N'$ to be expressed explicitly, with $``N'>"$ on one side and expressions involving
all the variables on the other side.  Third, since the error probability is a choice of the experimenter,
we should not have $\delta$ expressed in terms of the other parameters, but instead
we should have $N'$ expressed in terms of the desired confidence $1-\delta$ and any other
parameters.

As for the first task, concerning the choice of $\Gamma$,
we have to show that there is a choice of $\lambda, N,\mu$
so that the interval of possible choices of $\Gamma$ is nonempty, in other words so that
\begin{equation}\label{eqn:nonEmptyIntervalAsInequality}
\frac{\epsilon \hat{m}}{L(1-\theta)}<\Gamma^{\max}(N,\mu,\eta).
\end{equation}  
\begin{lem}  \label{lem:lambdaLemma} The condition \eqref{eqn:nonEmptyIntervalAsInequality} is equivalent to the following 
two conditions
\begin{equation}\label{eqn:conditionBetweenMuAndLambda}
F(\mu\eta)>\frac{\epsilon \hat{m}}{L(1-\theta)},
\end{equation}
and
\begin{equation}\label{eqn:NConditionOnMuEtaLambda}
N>\frac{\log 24}{F(\mu\eta) - \frac{\epsilon \hat{m}}{L(1-\theta)}}.
\end{equation}
\end{lem}
\begin{proof}
For the interval in \eqref{eqn:GammaInInterval} to be nonempty, it is equivalent to have
\begin{equation}\label{eqn:intervalNonempty}
\frac{\epsilon \hat{m}}{L(1-\theta)}< -\frac{\log 24}{N} + F(\mu\eta).
\end{equation}
For a fixed choice of $\lambda, \mu$, there will be an $N$ satisfying \eqref{eqn:intervalNonempty}
if and only if \eqref{eqn:conditionBetweenMuAndLambda} holds, because the term
$-\frac{\log 24}{N}$ must be negative.   If \eqref{eqn:conditionBetweenMuAndLambda} does hold,
then we can solve \eqref{eqn:intervalNonempty} to obtain \eqref{eqn:NConditionOnMuEtaLambda}. 
\end{proof}
We can ensure that \eqref{eqn:conditionBetweenMuAndLambda}
is satisfied by taking
\begin{equation}\label{eqn:epsilonLowerBoundnNodeCase}
\epsilon = \frac{(1-\theta)LF(\mu\eta)}{4\hat{m}}.
\end{equation}
Then the condition \eqref{eqn:NConditionOnMuEtaLambda} is
satisfied as long as
\begin{equation}\label{eqn:NConditionToHaveNonemptyGammaInterval}
N>\frac{4\log 24}{3F(\mu\eta)}.
\end{equation}
The final task, namely, to write down explicit conditions on $N'$
so that the confidence \eqref{eqn:errorProbabilityNCaseSkeletonNotContained}
is at least $1-\delta$, for given $\delta>0$, can be accomplished in a number
of ways.  We choose a very simple way, which is still sufficient to extract
the asymptotic dependence of $N'$ on $\delta$, namely, we
set each of the $5$ subtracted terms in \eqref{eqn:errorProbabilityNCaseSkeletonNotContained} equal to $\frac{\delta}{5}$.
The following Lemma sums up the explicit conditions on $N$
which we obtain thereby.
\begin{lem}\label{eqn:DenserCaseDeltaConditions}
Let $\delta>0$ be given.  Let $\mu,\Theta\in(0,1)$ be parameters
as above.   Let $m$ be as in \eqref{eqn:maxOfTwoMs}.   Then if we take
\begin{equation}\label{eqn:delta1DefinitionDenserCase}
\delta_1=\frac{\delta}{5},
\end{equation}
and $N$ larger than all of the following quantities:
\begin{equation}\label{eqn:denserNetworkErrorProbabilitymthetaFirst}
\frac{3}{m\theta^2}\log\frac{5|E(G)|\Sigma_{\mathcal{S}}(\mathcal{G})2^{\sigma(G)}}{\delta},
\end{equation}
\begin{equation}\label{eqn:denserNetworkErrorProbabilityDelta}
\max\left(
\left[
\tilde{F}\left(\frac{\Delta}{2}
\right)
\right]^{-1}
\log\frac{24}{1-\Theta},\,
\left[
F\left(\frac{\Delta}{2}
\right)
\right]^{-1}\log\frac{120|E(G)|\Sigma_{\mathcal{S}}(\mathcal{G})}{\delta}
\right),
\end{equation}
\begin{equation}\label{eqn:denserNetworkErrorProbabilitymthetaSecond}
\frac{3}{\hat{m}\theta^2}\log\frac{5L\Sigma_{\mathcal{S}}(G)2^{\sigma(G)}}{\delta},
\end{equation}
\begin{equation}\label{eqn:denserNetworkErrorProbabilityEtaMu}
\log\left(
\frac{120 \sigma_{\mathcal{S}}(G)L}{\delta}\right)
\left[
F(\eta(1-\mu))
\right]^{-1}.
\end{equation}
Then \eqref{eqn:errorProbabilityNCaseSkeletonNotContained} is at least $1-\delta$.
\end{lem} 
\begin{proof}
Each of the conditions \eqref{eqn:delta1DefinitionDenserCase}--\eqref{eqn:denserNetworkErrorProbabilityEtaMu} corresponds to one
of the subtracted terms in \eqref{eqn:errorProbabilityNCaseSkeletonNotContained}.
The derivation of \eqref{eqn:denserNetworkErrorProbabilitymthetaFirst}
and \eqref{eqn:denserNetworkErrorProbabilitymthetaSecond} requires
no further explanation, as they just result from setting the second
and fourth terms equal to $\frac{\delta}{5}$
and solving for $N'$.   In order to obtain \eqref{eqn:denserNetworkErrorProbabilityDelta},
we set the third term
\[
|E(G)|\Sigma_{\mathcal{S}}(\mathcal{G})\mathcal{F}_N\left(
\Delta-\tilde{\mathcal{F}}^{-1}_N(1-\Theta)
\right)<\frac{\delta}{5}.
\]
and apply Lemma \ref{lem:mostComplexTermInErrorProbability} with 
\[
\delta' = \frac{\delta}{5|E(G)|\Sigma_{ \mathcal{S}   } (\mathcal{G})  }.
\]
For the final condition, \eqref{eqn:denserNetworkErrorProbabilityEtaMu}, set
the final term of \eqref{eqn:errorProbabilityNCaseSkeletonNotContained} equal to $\delta/5$,
obtaining
\[
\mathcal{F}_N(\eta(1-\mu))\leq \frac{\delta}{5 \sigma_{\mathcal{S}}(G)L},
\]
which is equivalent to
\[
\mathcal{G}_{\eta(1-\mu)}(N)\leq \frac{\delta}{5 \sigma_{\mathcal{S}}(G)L}.
\]
Now use Lemma \ref{lem:functionalInversesDecreasing} to obtain that this condition is equivalent to 
\eqref{eqn:denserNetworkErrorProbabilityEtaMu}.
\end{proof}
\begin{proof}[Completion of Proof of Theorem \ref{thm:nNodeCase}]
For Part (a) start with condition \eqref{eqn:NprimescriptWCondition} and the conditions in Lemma
\ref{eqn:zetaDistantLastLemma}, changing variables from $N'$ to $N$ throughout using the relation
$N'=N(1-\theta)^{-1}$.

In order to obtain the conditions in Part (b), start with the conditions in Proposition \ref{prop:mostGeneralFormDenserCase}.  Use the choice of $\epsilon$ given in 
\eqref{eqn:epsilonLowerBoundnNodeCase} to obtain the form of the
condition \eqref{eqn:nNodeSampleComplexityWFunctionPreliminary} 
given in \eqref{eqn:nNodeSampleComplexityWFunction}.  
Since
\[
N>\frac{2\log 24}{F(\mu\eta)}
\]
by \eqref{eqn:gammaMaxDefn},
\[
\Gamma^{\max}(N,\mu,\eta)>\frac{F(\mu\eta)}{2}.
\]
Further, with the choice of $\epsilon$ in \eqref{eqn:epsilonLowerBoundnNodeCase} we also have
\[
\frac{\epsilon\hat{m}}{L(1-\theta)}=\frac{F(\mu\eta)}{4}<\frac{F(\mu\eta)}{2}.
\]
So we may take
\[
\Gamma=\frac{F(\mu\eta)}{2}.
\]
Thus
\[
L\Gamma(1-\theta)-\hat{m}\epsilon >\frac{LF(\mu\eta)}{2}(1-\theta)-\frac{LF(\mu\eta)}{4}(1-\theta)=
\frac{(1-\theta)LF(\mu\eta)}{4}.
\]
Thus \eqref{eqn:nNodeSampleComplexityWFunctionPreliminary} is bounded by \eqref{eqn:nNodeSampleComplexityWFunction}.
Finally, we obtain the remaining conditions by taking $N$ larger than the quantities in 
Lemma \ref{eqn:DenserCaseDeltaConditions}.   In order to obtain the theorem 
in the form stated, we change the notation for the number of samples back from $N'$
to $N$.
\end{proof}
\section{Refinements using Sanov's Theorem}
\subsection{Case of Two Nodes}\label{subsec:refinementTwoNodes}
We are improving the bound under the assumption that the underlying
network is independent, so throughout this section, we will work under the assumption
$G=G_0$.  We will find an $\eta_N^-\in (0,\eta)$ with the property that
\[
\mathcal{S}_{\eta,N}(\gamma)>0\;\text{for}\; [0,\eta_N^-],
\]
and we will then bound the error probability
\begin{equation}\label{eqn:errorProbabilitySanovAnalysis}
\mathrm{Pr}_{\omega_N\sim G}\left\{
\tau(\omega_N)\in(\eta_N^-,\infty)\;|\; G=G_0
\right\}.
\end{equation}
The first step is based on bounding $\beta_N^{p^\eta}(\gamma)$ for \textit{any} $\gamma\in (0,\eta)$
from above.  We will bound from above \textit{the probability of emission of a distribution in $A_{\gamma}^0$
by $p^\eta$} (which is to say, $\beta_N^{p^\eta}(\gamma)$) by the use of Sanov's Theorem.
According to Sanov's Theorem, the key statistic for estimating this probability
of emission is the KL-divergence $H(q^\gamma \| p^\eta)$ where $q^\gamma$
is defined as follows. 
\begin{definition} We let \boldmath$q^\gamma\,$\unboldmath denote the element of $\mathcal{P}$ which is the I-projection of  $p^\eta$ onto $A^0_\gamma$.
\end{definition}
\nomenclature{$q^\gamma$}{I-projection of  $p^\eta$ onto $A^0_\gamma$}
Note that even though $p^\eta$ has uniform marginals (by definition), $q^\gamma$
does not necessarily have uniform marginals.  
We now list the facts that are used to bound $H(q^\gamma \| p^\eta)$, and with it
$\beta_N^{p^\eta}(\gamma)$:
\begin{lem}\label{lem:LinfinityNormBoundedInTermsOf_t}
 For any $p(t)$ and $q(s)\in \mathcal{P}_{k,l}$, with $p\in \mathcal{P}_0$ and $q\in\mathcal{P}_0$ and  ${t\in [0,t_{\rm max}(p))}$,
 $s\in [0,t_{\rm max}(q))$, we have
 \[
  \|q(s)- p(t) \|_{\infty} \geq  (k+l+1)^{-1}\left|t-s\right|.
 \]
\end{lem}
\begin{proof}
The proof is quite elementary: see Section \ref{sec:TechnicalLemmas} for details.
\end{proof}
Recall the well-known:
\begin{lem} \label{lem:Pinsker} \textbf{Pinsker's Inequality.}  Given any pair of distributions $q,p\in\mathcal{P}$,
\[
H(q\|p)\geq 2\| q - p \|_{\infty}^2.
\]
\end{lem}
\begin{proposition}\label{prop:SanovBetaBound}
Assume that $k=l=2$ and $\gamma\in (0,\eta)$, so that $t_{\gamma}^+< t_{\eta}^+$.  Then
\[
\log \beta_N^{p^\eta}(\gamma)\leq |X|\log (N+1)-\frac{2}{25}\left(t_{\eta}^+ - t_{\gamma}^+  \right)^2\cdot N.
\]
\end{proposition}
\begin{proof}
First use Sanov's Theorem, then Pinsker's Inequality (Lemma \ref{lem:Pinsker})
and finally Lemma \ref{lem:LinfinityNormBoundedInTermsOf_t}:
\begin{equation}
\begin{aligned}
\log \beta_N^{p^\eta}(\gamma)  &\leq |X|\log(N+1) - H(q^\gamma \|p^\eta)\cdot N\\
                                                   &\leq |X|\log(N+1) - 2\| q^\gamma  - p^\eta\|_{\infty}^2\cdot N\\
                                                   &\leq |X|\log(N+1) - 2\left(\frac{1}{5}\left|t_{\gamma}^+ - t_{\eta}^+\right|\right)^2\cdot N.
\end{aligned}
\end{equation}
The reason for the coefficient $\frac{2}{25}$ is that in the case $k=l=2$,
the constant $(k+l+1)^{-1}$ in Lemma \ref{lem:LinfinityNormBoundedInTermsOf_t}
is $\frac{1}{5}$.
\end{proof}
In order to study the dependence of $t_{\gamma}^+$ and $t_{\eta}^+$
in the special case $k=l=2$, we are going to use Taylor's Theorem
with Remainder, which allows us to write out a function
as a partial Taylor series (i.e. a polynomial) where the remaining
terms are given (estimated) using the mean value theorem.
We start with the following expression for the test statistic:
\[
\tau(p)=p_{0,0}\log\frac{p_{0,0}}{p_{A,0}p_{B,0}}+
p_{0,1}\log\frac{p_{0,1}}{p_{A,0}p_{B,1}}+
p_{1,0}\log\frac{p_{1,0}}{p_{A,1}p_{B,0}}+
p_{1,1}\log\frac{p_{1,1}}{p_{A,1}p_{B,1}}.
\]
Then for any $t\in \mathbf{R}$,
\begin{multline*}
\tau(p(t))=(p_{A,0}p_{B,0}+t)\log\frac{p_{A,0}p_{B,0}+t}{p_{A,0}p_{B,0}}+
(p_{A,0}p_{B,1}-t)\log\frac{p_{A,0}p_{B,1}-t}{p_{A,0}p_{B,1}}+\\
(p_{A,1}p_{B,0}-t)\log\frac{p_{A,0}p_{B,0}-t}{p_{A,1}p_{B,0}}+
(p_{A,1}p_{B,1}+t)\log\frac{p_{A,1}p_{B,1}+t}{p_{A,1}p_{B,1}}.
\end{multline*}
Using Taylor's Theorem to expand $\tau(p(t))$ for positive $t$
around the base point $0$ we obtain:
\[
\tau(p(t))=\tau(p(0)) + \left.\frac{\partial\tau(p(t))}{\partial t}\right|_{t=0}+
\left.\frac{1}{2}\frac{\partial^2\tau(p(t))}{\partial t^2}\right|_{t=s}t^2,\;
\text{for some}\; s\in (0,t).
\]
Since the value of this function and its first derivative are both zero at the base point $t=0$,
this reduces to 
\begin{equation}\label{eqn:SecondTaylorExpansion}
\begin{aligned}
\tau(p(t)) &= \frac{1}{2}\left.\frac{\partial^2\tau(p(t))}{\partial^2 t}\right|_{t=s}t^2,
                      \;\text{for some}\; s\in(0,t)\\
               &=\frac{1}{2}\left(
                \frac{1}{p_{A,0}p_{B,0}+s}+
                \frac{1}{p_{A,1}p_{B,0}-s}+
                \frac{1}{p_{A,0}p_{B,1}-s}+
                \frac{1}{p_{A,1}p_{B,1}+s}
                \right)t^2
\end{aligned}
\end{equation}
\begin{lem} \label{lem:uniformMarginalsUpperLowerBounds} 
For $p(t)\in\mathcal{P}_{2,2}$ \textbf{with uniform marginals}, that is where
$p(0)$ is the uniform distribution, we have $\tau(p(t))=\eta$ implies
\[
\frac{1}{2\sqrt{2}}\sqrt{\frac{\eta}{2\eta+1}} \leq t_{\eta}^+\leq \frac{1}{2\sqrt{2}}\sqrt{\eta}. 
\]
\end{lem}
\begin{proof}
By using the hypothesis that the marginals are uniform, \linebreak{$p_{A,0}=p_{B,0}=\frac{1}{2}$}, we have that \eqref{eqn:SecondTaylorExpansion} specializes to
\[
\tau(p(t)=-4\left(
\frac{1}{4s-1}-
\frac{1}{4s+1}
\right)t^2,\;\text{for some}\; s\in(0,t).
\]
Define
\[
f(s)=-4\left(
\frac{1}{4s-1}-
\frac{1}{4s+1}
\right),
\]
so that
\begin{equation}\label{eqn:tauInTermsOf_f}
\tau(p(t)) = f(s)\cdot t^2\;\text{for some}\; s\in(0,t).
\end{equation}
\noindent\textbf{Claim.}  The function $f(s)$ is increasing in $(0,t)$.
\begin{figure}\label{fig:fAndfPrime}
\centering
\hspace*{-1.3cm}\mbox{\subfigure{\includegraphics[width=3in]{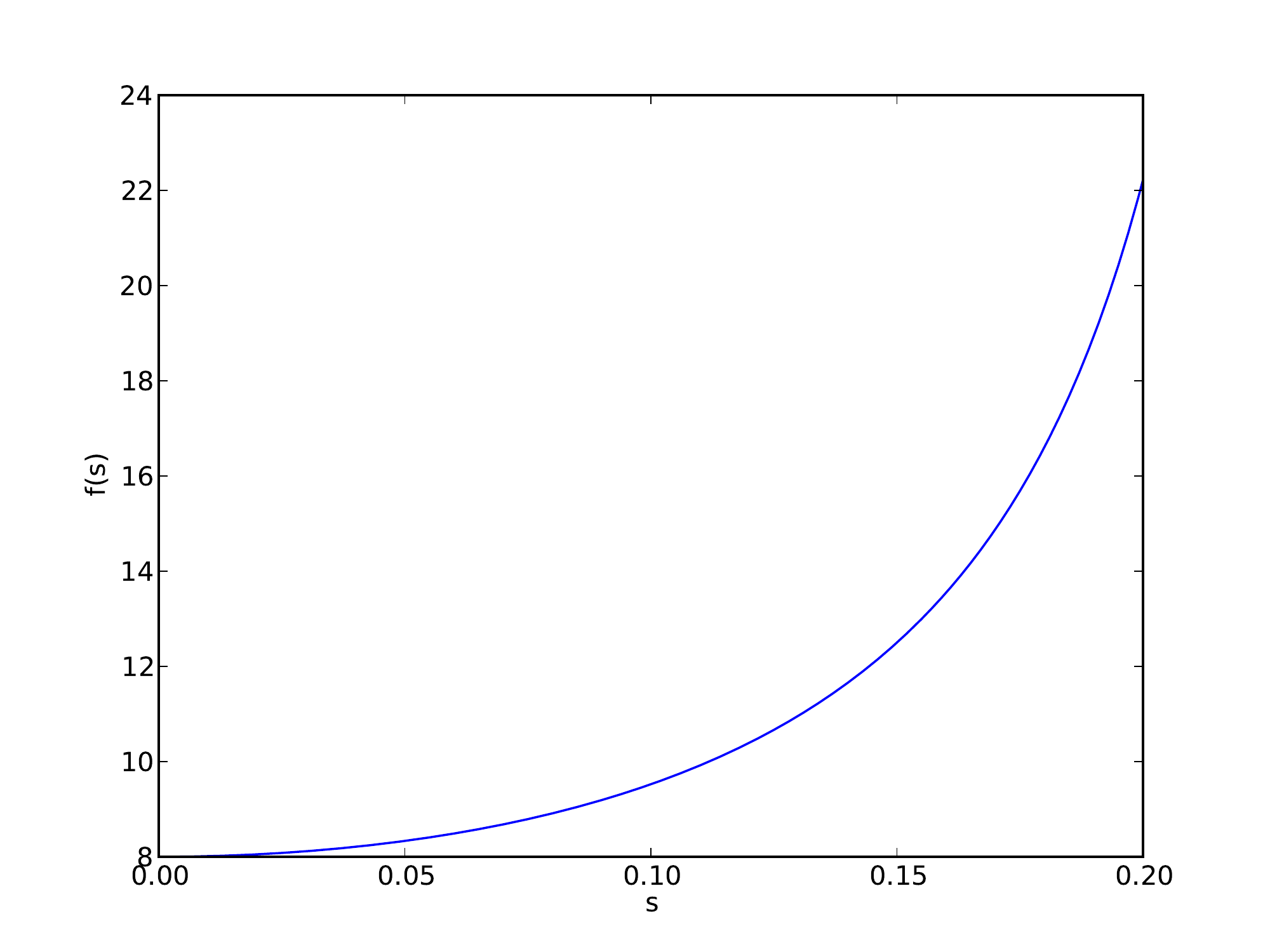}}\quad
\subfigure{\includegraphics[width=3in]{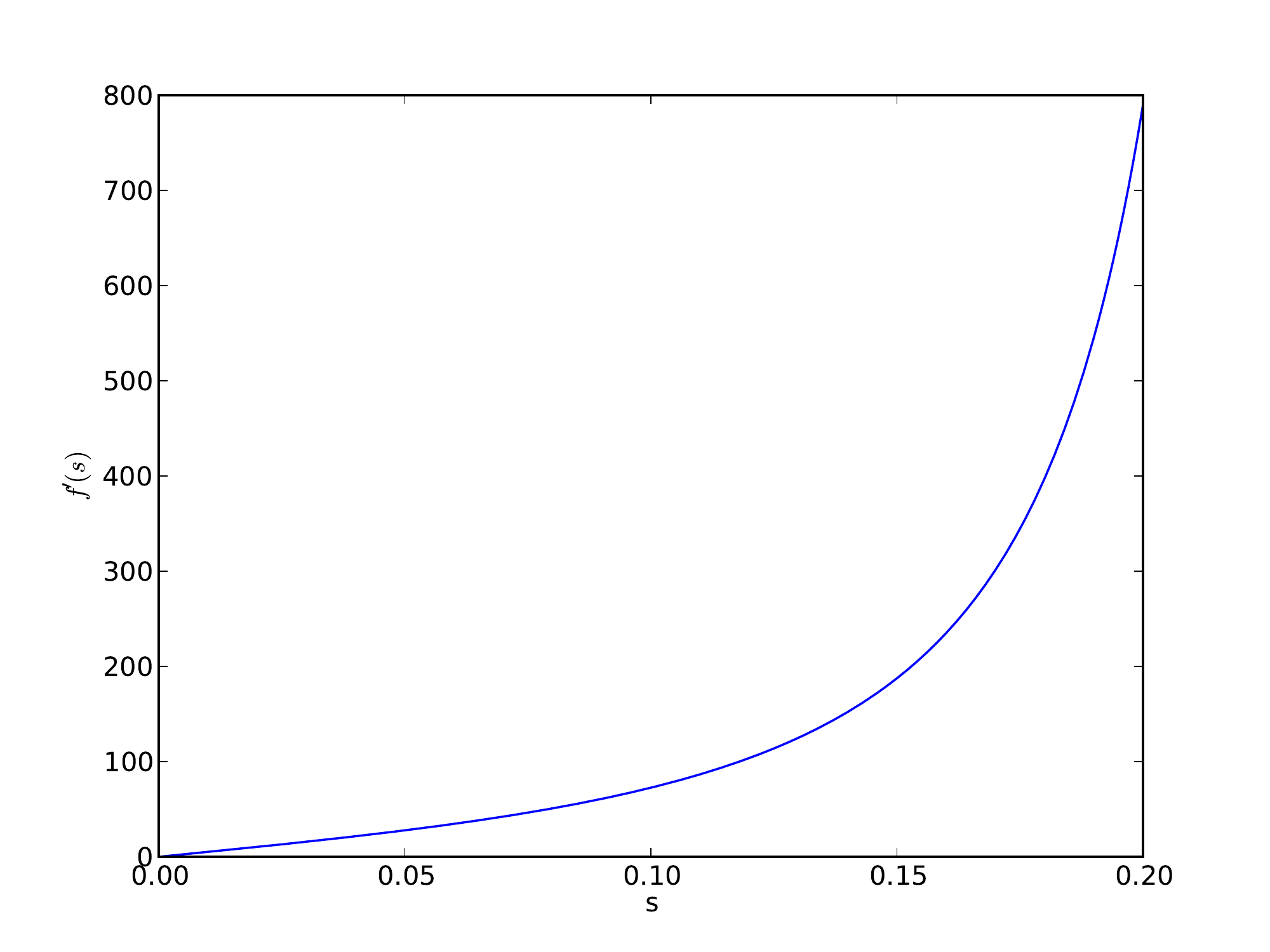} }}
\caption{Graphs of $f(s)$ and $f'(s)$ from the proof of Lemma \ref{lem:uniformMarginalsUpperLowerBounds},
showing the monotonicity of $f$ on the feasible interval.}
\end{figure}

\noindent In order to prove the claim we compute that
\begin{equation}\label{eqn:fderivativeWithRespTos}
\frac{\partial}{\partial s}f(s) = \frac{32}{(4s-1)^2} - \frac{32}{(4s+1)^2}.
\end{equation}
The derivative \eqref{eqn:fderivativeWithRespTos} has no zeros in $(0,t)$,
because we can easily solve to find that
\[
\frac{32}{(4s-1)^2} - \frac{32}{(4s+1)^2}=0 \Leftrightarrow s=0.
\]
Therefore, $f(s)$ has no critical points in $(0,t)$.  Further,
we can then show that the derivative \eqref{eqn:fderivativeWithRespTos} 
is positive for $s\in (0,t)$.  

Because of the Claim, and \eqref{eqn:tauInTermsOf_f}, we have
\[
f(0)t^2\leq \tau(p(t))\leq f(t)t^2,
\]
\[
8t^2\leq \tau(p(t))\leq -4\left(
                                         \frac{1}{4t-1}- \frac{1}{4t+1}
                                          \right)t^2,
\]
\[
8t^2\leq \eta\leq -4\left(
                                         \frac{1}{4t-1}- \frac{1}{4t+1}
                                          \right)t^2.
\]
Solving these inequalities for $t$ we obtain the bounds in Lemma \ref{lem:uniformMarginalsUpperLowerBounds}.
\end{proof}
Because $q^\gamma$ does not necessarily have uniform marginals, we need to derive
a more general, though slightly weaker, version of the upper bound on $t$ from Lemma \ref{lem:uniformMarginalsUpperLowerBounds}.
In order to do so, note that when $k=l=2$, for $i,j$ such that $0\leq i\leq k-1$ and $0\leq j\leq l-1$,
we have
\[
p_{A,i}p_{B,j}+(-1)^{i+j}s=p(s)_{i,j}
\] 
Substituting this equality in \eqref{eqn:SecondTaylorExpansion}, we may write the Taylor 
expansion in the second line of \eqref{eqn:SecondTaylorExpansion} as
\begin{equation}\label{eqn:secondTaylorAlternate}
\tau(p(t))=\frac{1}{2}\left(\sum_{i,j=0}^{1,1}p(s)_{i,j}^{-1}\right)t^2
\end{equation}
Note that, because the numbers $p(s)_{i,j}$ are the entries of the contingency
table of a probability distribution $p(s)\in\mathcal{P}_{2,2}$,
\[
\sum_{i,j=0}^{1,1}p(s)_{i,j}=1.
\]
and there are $k\cdot l$ of the numbers $p(s)_{i,j}$.  Therefore, at least one of the 
$p(s)_{i,j}$ is less than or equal to $(k\cdot l)^{-1}= \frac{1}{4}$.
Further, each of the terms in \eqref{eqn:secondTaylorAlternate} is positive.
So we have
\begin{equation}\label{eqn:secondTaylorBound}
\frac{1}{2}\left(\sum_{i,j=0}^{1,1}p(s)_{i,j}^{-1}\right)\geq \frac{1}{2}\cdot 4=2.
\end{equation}
\begin{lem}\label{lem:arbitraryMarginalsUpperBound}  For any $q(t)\in \mathcal{P}_{2,2}$, where $q(0)\in \mathcal{P}_0$ (product
distributions), we have $\tau(q(t))=\gamma$ implies
\begin{equation}\label{eqn:general_t_upperBound}
t\leq \sqrt{\frac{\gamma}{2}}.
\end{equation}
\end{lem}
\begin{proof}
By \eqref{eqn:secondTaylorAlternate} and \eqref{eqn:secondTaylorBound},
\[
\tau(p(t)) \geq 2t^2.
\]
Thus,
\[
\gamma\geq 2t^2.
\]
Solving for $t$, we obtain \eqref{eqn:general_t_upperBound}.
\end{proof}
Using Lemma \ref{lem:uniformMarginalsUpperLowerBounds}
to bound $t_{\eta}^+$ from below and
and Lemma \ref{lem:arbitraryMarginalsUpperBound} to
bound $t_{\gamma}^+$ from above, in Proposition \ref{prop:SanovBetaBound},
we obtain the following upper bound on $\log\beta_N^{p^\eta}$.
\begin{lem}  \label{lem:SanovBoundOnSparsityBoost} Assume that $\gamma\leq \frac{1}{4}\frac{\eta}{2\eta+1}$.  Then
\[
\log\beta_N^{p^\eta}(\gamma)\leq |X|\log(N+1)-\frac{1}{25}\left(
\frac{1}{2}\sqrt{\frac{\eta}{2\eta+1}}-\sqrt{\gamma}
\right)^2\cdot N.
\]
\end{lem}
We can then use this bound on $\log \beta_N^{p^\eta}(\gamma)$ to bound
the difference of objective functions.
\begin{lem}  Assume that $\gamma\leq \frac{1}{4}\frac{\eta}{2\eta+1}$.  Define
\begin{equation}\label{eqn:kappaPrimeDefn}
\kappa' := \kappa-\frac{|X|\log\left(\frac{N+1}{N}\right)}{\log N}.
\end{equation}
Then
\begin{equation}\label{eqn:objectiveFunctionDifferenceSanov}
\mathcal{S}_{\eta,N}(\gamma)\geq \left(\frac{1}{25}\left(
\frac{1}{2}\sqrt{\frac{\eta}{2\eta+1}}-\sqrt{\gamma}
\right)^2-\gamma\right) N+(\kappa'-|X|)\log N.
\end{equation}
\end{lem}
\begin{proof}
By \eqref{eqn:scriptSdefn},
\[
\begin{aligned}
\mathcal{S}_{\eta,N}(\gamma) &= -N\gamma + \kappa\log(N) -\log\beta_N^{p^\eta}(\gamma)\\
 &\geq  -N\gamma + \kappa\log(N) -|X|\log(N+1)-\frac{1}{25}\left(
\frac{1}{2}\sqrt{\frac{\eta}{2\eta+1}}-\sqrt{\gamma}
\right)^2\cdot N\\
\end{aligned}
\]
By \eqref{eqn:kappaPrimeDefn},
\[
\kappa\log(N)=\kappa'\log(N)+|X|\log(N+1)-|X|\log N.
\]
Substituting this into the above expression and gathering together
the terms involving $N$ and $\log N$ yields the estimate \eqref{eqn:objectiveFunctionDifferenceSanov}.
\end{proof}
In order to find $\eta_N^-$, what remains is to find conditions
on $\gamma$ for the right-hand side of \eqref{eqn:objectiveFunctionDifferenceSanov}
to be positive.  Conveniently, \eqref{eqn:objectiveFunctionDifferenceSanov}
amounts to a quadratic polynomial in $\sqrt{\gamma}$ which
we can solve analytically.
\begin{proposition} \label{prop:etaNminus}Assume
$N>0$ is fixed.  Let $\kappa'$ be as defined in \eqref{eqn:kappaPrimeDefn}.
Assume that
\begin{equation}\label{eqn:conditionForRadicalRealWFunction}
N>\exp \mathcal{W}\left(
\frac{\eta}
{96(2\eta+1)(|X|-\kappa')}
\right).
\end{equation}
then if we define
\begin{equation}
\eta_N^- := \left(
\frac{-\sqrt{\frac{\eta}{2\eta+1}} + \sqrt{  25\cdot\frac{\eta}{2\eta+1} -2400(|X-\kappa'|) \frac{\log N}{N}  }}
         {48}
\right)^2.
\end{equation}
Then we have
\[
\mathcal{S}_{\eta,N}(\gamma)\geq 0,\;\text{for}\; \gamma\in(0,\eta_N^-).
\]
Further,
\[
\eta_N^-\rightarrow \frac{\eta}{144(2\eta+1)},\;\text{as}\; N\rightarrow\infty.
\]
\end{proposition}
\begin{proof}
By \eqref{eqn:objectiveFunctionDifferenceSanov}, we have $\mathcal{S}_{\eta,N}(\gamma)\geq 0$ provided that
\begin{equation}\label{eqn:quadraticBeforeSubstitution}
\frac{1}{25}\left(
\frac{1}{2}\sqrt{\frac{\eta}{2\eta+1}}-\sqrt{\gamma}
\right)^2-\gamma+(\kappa'-|X|)\frac{\log N}{N}\geq 0.
\end{equation}
In order to make the formulas easier to manipulate, we set, for the purposes
of this proof
\[
C:=(|X|-\kappa')\frac{\log N}{N},
\]
\[
y:=\sqrt{\frac{\eta}{2\eta+1}},
\]
and
\[
x:=\sqrt{\gamma}.
\]
We substitute these definitions into \eqref{eqn:quadraticBeforeSubstitution} to see that we are seeking $x$ for which
\[
\frac{1}{25}\left(\frac{1}{2}y-x\right)^2-x^2-C\geq 0.
\]
Expanding the squared difference and gathering terms according to the power of $x$,
we see that this is equivalent to
\[
-100x^2+4\left(\frac{1}{2}y-x\right)^2-100C\geq 0,
\]
that is,
\[
-96x^2-4yx+(y^2-100C)\geq 0
\]
Solving the corresponding quadratic equation, we obtain the solutions
\[
x=\frac{-4y\pm 4\sqrt{25y^2-2400C}}{192}.
\]
Since $x=\sqrt{\gamma}\geq0$, we are interested in the nonnegative solution
\begin{equation}\label{eqn:xSolutionInTermsOf_yC}
x=\frac{-4y + 4\sqrt{25y^2-2400C}}{192}.
\end{equation}
In order for \eqref{eqn:xSolutionInTermsOf_yC} to define a real number, we must have 
positivity of the expression inside the radical, which is equivalent to
\[
C<\frac{y^2}{96},
\]
which is equivalent to
\begin{equation}\label{eqn:conditionForRadicalReal}
\frac{\log N}{N} < \frac{\eta}{96(2\eta+1)(|X|-\kappa')}.
\end{equation}
Condition \eqref{eqn:conditionForRadicalReal} is satisfied
as long as \eqref{eqn:conditionForRadicalRealWFunction} in the Proposition is satisfied.
Returning to the condition \eqref{eqn:xSolutionInTermsOf_yC}, in terms
of $\gamma,\eta,N$, this condition says that
\[
\gamma \leq \left(
\frac{-\sqrt{\frac{\eta}{2\eta+1}} + \sqrt{  25\cdot\frac{\eta}{2\eta+1} -2400(|X-\kappa'|) \frac{\log N}{N}  }}
         {48}
\right)^2.
\]
Thus we can set $\eta_N^-$
equal to the right-side of this inequality.  As $N\rightarrow\infty$, the subtracted term inside the second radical in the numerator approaches $0$, so that the entire numerator approaches
$4\sqrt{\frac{\eta}{2\eta+1}}$.  Thus the entire right-hand side approaches
\[
\left(
\frac{\sqrt{\frac{\eta}{2\eta+1}}}{12}
\right)^2
=
\frac{\eta}{144(2\eta+1)}.
\]
\end{proof}
It remains to carry out the second part of the strategy
outlined at the beginning of Section \ref{subsec:refinementTwoNodes} which is to bound the error probability \eqref{eqn:errorProbabilitySanovAnalysis}.
Set $\delta$ equal to this error probability, which is to say
\[
\delta :=\mathrm{Pr}_{\omega_N\sim p_0}\left\{
\tau(\omega_N)>\eta_N^-
\right\}.
\]
Then Sanov's Theorem says that
\begin{equation}\label{eqn:SanovsEstimateOfLogErrorProbability}
\log\delta \leq|X|\log(N+1)-H(\mathcal{P}_{\eta_N^-} \|p_0  )N,
\end{equation}
where by definition
\begin{equation}\label{eqn:setKLDivergenceDefinition}
H(\mathcal{P}_{\eta_N^-} \|p_0  ):=\inf\left\{
H(q\| p_0)\;|\; q\in \mathcal{P}_{\eta_N^-}
\right\}.
\end{equation}
\nomenclature{$H(A \lvert \rvert p_0  )$}{$\inf\left\{
H(q\lvert \rvert p_0)\;\lvert \; q\in A
\right\}$}
More generally we can derive an estimate, for $\gamma>0$, of the minimum KL-divergence
of an element in the set $\mathcal{P}_{\gamma}$
based at $p_0\in\mathcal{P}_0$.
\begin{lem} \label{lem:IprojectionEstimateSanov} We have for any $\gamma>0$,
\[
H(\mathcal{P}_{\gamma}\|p_0)\geq \gamma.
\]
\end{lem}
\begin{proof}
For the purposes of this proof, for $q\in \mathcal{P}_{\gamma}$,
we let
\[
 q_0 \;\text{denote the M-projection of}\; q\;\text{onto}\; \mathcal{P}_0.
\]  
By definition of the M-projection,
\[
H(q\| p_0)\geq H(q\|q_0).
\]
But $H(q\| q_0)$ is the mutual information $\tau(q)$ of $q$, and $\mathcal{P}_{\gamma}$
is defined as the set of $q\in\mathcal{P}$ such that $\tau(q)\geq \gamma$,
so that $H(q\|p_0)\geq \gamma$ for all $q\in \mathcal{P}_{\gamma}$.
So the above says that $H(q\| p_0)\geq \gamma$ for all $q\in\mathcal{P}_0$.  By
\eqref{eqn:setKLDivergenceDefinition} this completes the proof.
\end{proof}
By applying Lemma \ref{lem:IprojectionEstimateSanov} with $\gamma=\eta_N^-$ to \eqref{eqn:SanovsEstimateOfLogErrorProbability},
we obtain
\begin{equation}\label{eqn:errorProbabilityInTermsOfGamma}
\log\delta \leq|X|\log(N+1)-\eta_N^-\cdot N,
\end{equation}
which allows us to prove the following estimate of the error probability.
\begin{proposition}\label{prop:SanovErrorEstimate}
Let $p_0\in \mathcal{P}_0$.  Let $\delta>0$ be given. Assume that
\begin{equation}\label{eqn:NconditionSanovProof}
N>\frac{|X|}{\eta_N^-}
\mathcal{W}\left(
\frac{\eta_N^-\delta^{\frac{|X|}{\eta_N^-}}}
{|X|\exp\left(\frac{\eta_N^-}{|X|}\right)}
\right)
-1.
\end{equation}
Then,
\[
\mathrm{Pr}_{\omega_N\sim p_0}\left\{
\tau(\omega_N)\geq \eta_N^-
\right\}<\delta.
\]
\begin{proof}  
By using \eqref{eqn:errorProbabilityInTermsOfGamma}, we see that we will
have the claimed estimate on the error probability if
\[
N-\frac{|X|}{\eta_N}\log(N+1)\geq \log\delta^{-1}.
\]
Solving this inequality for $N$, we obtain \eqref{eqn:NconditionSanovProof}.
\end{proof}
\end{proposition}
\begin{proof}[Completion of Proof of Theorem \ref{thm:TwoNodeComparisonIndependentCase}]
Part (a) consists of repetition (of the $G$ independent case) 
from Theorem \ref{thm:finiteSampleComplexity}.  
For part (b), combine Propositions \ref{prop:etaNminus} and \ref{prop:SanovErrorEstimate}.
The only modification that has to be made is that in Proposition \ref{prop:etaNminus},
the parameter $\kappa'$ depends on $N$.  In order to make the conditions
on $N$ explicit, we have chosen a $\mu\in(0,1)$ such that
\[
\log N > \frac{|X|}{\kappa\mu}.
\]
Thus
\[
\frac{|X|\log\frac{N+1}{N}}{\log N}\leq \frac{|X|}{\log N} <\kappa\mu.
\]
Thus by \eqref{eqn:kappaPrimeDefn}, $\kappa'\geq \kappa(1-\mu)$.

For part (c), the only asymptotic statement we have not yet proved is for $N^{\rm S}$.
The expression responsible for the asymptotics is the final element of the 
maximum, namely the one involving $\eta_N^-$.  By Lemma \ref{lem:LamberWasymptoticsApplied}, below, we know that
\[
\frac{|X|}{\eta_N^-}
\mathcal{W}\left(
\frac{\eta_N^-\delta^{\frac{|X|}{\eta_N^-}}}
{|X|\exp\left(\frac{\eta_N^-}{|X|}\right)}
\right) = \tilde{O}\left(\frac{1}{\eta_N^-}\log\left(\frac{1}{\eta_N^-}\left(\frac{1}{\delta}\right)^{\frac{|X|}{\eta_N^-}}\right)\right)
\]
as $\eta_N^-$, $\delta\rightarrow0^+$.  (The $\eta_N^-$ in the exponent
in the denominator inside $\mathcal{W}$ just makes the exponentiated
factor approach $1$, and can be ignored).
Further 
\[
\frac{1}{\eta_N^-}\log\left(\frac{1}{\eta_N^-}\left(\frac{1}{\delta}\right)^{\frac{|X|}{\eta_N^-}}\right)=
\frac{1}{\eta_N^-}\left(\log\frac{1}{\eta_N^-} + \frac{|X|}{{\eta_N^-}}\log\frac{1}{\delta}\right).
\]
So that 
\[
\frac{|X|}{\eta_N^-}
\mathcal{W}\left(
\frac{\eta_N^-\delta^{\frac{|X|}{\eta_N^-}}}
{|X|\exp\left(\frac{\eta_N^-}{|X|}\right)}
\right) = \tilde{O}\left(\frac{1}{{\eta_N^-}^2}\log\frac{1}{\delta}\right).
\]
But as $N\rightarrow\infty$, $\eta_N^-$ converges to $\eta$, whose asymptotics
are the same as the asymptotics of $\epsilon$.
\end{proof}
\begin{proof}[Completion of Proof of Theorem \ref{thm:twoNodeCaseOptimalAsymptotics}]
For the case of $(G,P)=(G_1,p_1^{\epsilon})$ (dependent case) we use
Proposition \ref{prop:finiteSampleComplexityDependentCase} in Section \ref{subsec:dependentNetwork}.  Using
the calculations of Section \ref{subsec:GeneralCase},
we transform the conditions of Proposition \ref{prop:finiteSampleComplexityDependentCase}
into the form seen here.  For the case of $(G,P)=(G_0,p_0)$ (independent
case), we use the result of part (b) of Theorem \ref{thm:TwoNodeComparisonIndependentCase}.  The asymptotics
are still controlled by the same bounds as in the proof of Theorem 
\ref{thm:TwoNodeComparisonIndependentCase}.
\end{proof}

\subsection{Case of $n$ Nodes}  The case $\mathrm{Skel}(G')\not\subseteq
\mathrm{Skel}(G)$ is the only one where the estimates we want
to refine can arise, so throughout this subsection we are under the
assumption $\mathrm{Skel}(G')\not\subseteq
\mathrm{Skel}(G)$.

We prove a counterpart of Proposition 
\ref{prop:scriptLSumEstimate} which uses an estimate of $\beta_N^{p^\eta}(\tau(Y_N))$
derived from Lemma \ref{lem:SanovBoundOnSparsityBoost}, in place of Corollary \ref{cor:sparsityBoostLowerBoundIndependentNetworkVer2}. 
That is we replace a bound ultimately derived from Chernoff's Theorem
with one ultimately derived from Sanov's Theorem.
First we formulate Lemma \ref{lem:SanovBoundOnSparsityBoost} in terms
of the probability that drawing a sequence from an independent
joint distribution leads to a sparsity boost with specified linear growth.
\begin{lem} \label{lem:SanovAppliedToLSparsityBoosts} Let $p_0\in \mathcal{P}_0$.  Let $\gamma\leq \frac{1}{4}\frac{\eta}{2\eta+1}$.
Then 
\[
\mathrm{Pr}_{Y_N\sim p_0}\left\{
-\log\beta_N^{p^\eta}(\tau(Y_N))< \frac{1}{25}\left(\frac{1}{2}
\sqrt{\frac{\eta}{2\eta+1}} - \sqrt{\gamma}\right)^2N
-|X|\log(N+1) \right\}
\]
is no greater than $(N+1)^{|X|}e^{-\gamma N}$.
Let $L>0$.  Let $p_{i,0},\ldots,p_{L,0}\in  \mathcal{P}_0$ (not necessarily distinct,
but all product distributions) and suppose $Y_{i,N}$, $1\leq i\leq L$ are empirical sequences
with $Y_{i,N}$ consisting of $N$ samples drawn i.i.d. from $p_{i,0}$.
Then the probability is no greater than
\[
L\cdot (N+1)^{|X|}e^{-\gamma N}.
\]
that for \textbf{any} $i$, $1\leq i \leq L$, we have
\[
-\log\beta_N^{p^\eta}(\tau(Y_N))< \frac{1}{25}\left(\frac{1}{2}
\sqrt{\frac{\eta}{2\eta+1}} - \sqrt{\gamma}\right)^2N-|X|\log(N+1).
\]

\end{lem}
\begin{proof}
By Sanov's Theorem, for every $\gamma>0$,
\begin{equation}\label{eqn:simpleSanov}
\mathrm{Pr}_{Y_N\sim p_0}\left\{\tau(Y_N)> \gamma\right\}\leq
(N+1)^{|X|}e^{-\gamma N}.
\end{equation}
Consider the events
\[
E_1:=\left\{\tau(Y_N)> \gamma\right\}
\]
and the event
\[
E_2:=\left\{
-\log\beta_N^{p^\eta}(\tau(Y_N))<
-|X|\log(N+1) + \frac{1}{25}\left(\frac{1}{2}
\sqrt{\frac{\eta}{2\eta+1}} - \sqrt{\gamma}\right)^2
\right\}
\]
By Lemma \ref{lem:SanovBoundOnSparsityBoost}, for $\gamma$
satisfying the hypothesis, $E_2\subseteq E_1$.
By this containment of events and \eqref{eqn:simpleSanov}
\[
\mathrm{Pr}_{Y_N\sim p_0}(E_2)\leq \mathrm{Pr}_{Y_N\sim p_0}(E_1)\leq (N+1)^{|X|}e^{-\gamma N}.
\]
By using the definition of the event $E_2$ in this chain of inequalities,
we obtain the first statement.
For the second statement, we apply the union bound to the first statement.
\end{proof}
\begin{proposition}\label{prop:scriptLSumEstimateSanov}  Let $\mathcal{L}$ be a subset of 
\[
\left\{
(A,B)\in G'\backslash G
\right\}
\]
of cardinality $L$.  
Let $\hat{m}$ be as in \eqref{eqn:maxOfTwoMs}.
Let $\theta\in(0,1)$ and 
let $N'$ and $N$ be positive integers such that $N'=(1-\theta)^{-1}\hat{m}N$.
Let $\gamma < \frac{1}{4}\frac{\eta}{2\eta+1}$.
With probability at least 
\[
1-\Sigma_{\mathcal{S}}(G)L2^{\sigma(G)}e^{-N\hat{m}\theta^2/3} - L(N+1)^{|X|}e^{-\gamma N},
\]
we have for all $(A,B)\in\mathcal{L}$, with $\hat{S}=\hat{S}((A,B),G)$ defined in
Definition \ref{def:HatSDefn},
\[
\min_{s\in \mathrm{Val}(\hat{S})}-\ln\left[
\beta_{N'}^{p^\eta}\left(
\tau(p_{\omega_{N'}},A,B|s)\right)\right]
>\frac{1-\theta}{\hat{m}}\frac{1}{25}\left(\frac{1}{2}
\sqrt{\frac{\eta}{2\eta+1}} - \sqrt{\gamma} \right)^2 N' - |X|\log(N+1).
\]
Consequently,
\begin{multline*}
\left[\sum_{(A,B)\in G' \backslash G}\max_{S\in S_{A,B}(G)}\min_{s\in\mathrm{val}(S)}-\ln\left[\beta_{N'}^{p^{\eta}}\tau(p(\omega_{N'}, A,B|s))\right]\right]>\\ \frac{1-\theta}{\hat{m}}\frac{L}{25}\left(\frac{1}{2}
\sqrt{\frac{\eta}{2\eta+1}} - \sqrt{\gamma} \right)^2 N'-L|X|\log(N+1).
\end{multline*}
\end{proposition}
\begin{proof}
Take $N'>(1-\theta)^{-1}\hat{m}N$.
Let
\[
\omega_{N'}\sim p\; \text{with}\; \mathrm{Card}(\omega_{N'})=N'.
\]
By Lemma \ref{lem:scriptLestimate}, with probability at least
\[
1-L\Sigma_{\mathcal{S}}(G)2^{\sigma(G)}e^{-N\hat{m}\theta^2/3},
\]
we have that for \textit{all} $s\in\mathrm{Val}(\hat{S})$,
\[
\omega_{N'}|_{\hat{S}=s}=Y_N,\;\mathrm{Card}(Y_N)=N,
\]
for $Y_N$ a sequence of observations meeting the criteria that the observations take the joint value of $\hat{S}$ is $s$.

Then we can apply Lemma  \ref{lem:SanovAppliedToLSparsityBoosts}
to each one of the terms in the sum.  Using Lemma
\ref{lem:SanovAppliedToLSparsityBoosts} to estimate
and the union bound, we conclude that for $\omega_{N'}|_{\hat{S}=s}=Y_N$, $\mathrm{Card}(Y_N)=N$, and
for $\gamma < \frac{1}{4}\frac{\eta}{2\eta+1}$,
\begin{multline*}
\mathrm{Pr}_{\omega_{N'}\sim G}\left\{
\min_{s\in \mathrm{Val}(\hat{S})}
-\ln\beta_{N'}^{p^\eta}
\left(
\tau(\omega_{N'},A,B|s)
\right)>
\frac{1-\theta}{\hat{m}}\frac{1}{25}\left(\frac{1}{2}
\sqrt{\frac{\eta}{2\eta+1}} - \sqrt{\gamma} \right)^2 N' - |X|\log(N+1)
\right\}\\ 
\geq
1-(N+1)^{|X|}e^{-\gamma N}\\
\end{multline*}
To complete the proof, we need a bound that holds for \textit{all} $(A,B)\in\mathcal{L}$
simultaneously, so
we apply the union bound one more time, at the cost of a factor of $L:=\mathrm{card}(\mathcal{L})$
multiplying the $(N+1)^{|X|}e^{-\gamma N}$
term in the error probability.
\end{proof}
This allows us to deduce the most general form of the finite sample
complexity for this case.
\begin{proposition} \label{prop:nNodeCaseSanovMostGeneral} 
Under the hypotheses of the preceding Proposition \ref{prop:scriptLSumEstimateSanov}, 
with probability at least
\begin{equation}\label{eqn:Prob_nNodeCaseSanovMostGeneral}
1-\delta_1 - |E(G)|\Sigma_{\mathcal{S}}(G)2^{\sigma(G)}e^{-N\hat{m}\theta^2/3}-
L(N+1)^{|X|}e^{-\gamma N},
\end{equation}
we have
\begin{multline*}
\mathcal{S}_{\eta}(G,G',\omega_{N'})\geq 
\left(\frac{1-\theta}{\hat{m}}\frac{L}{25}\left(\frac{1}{2}
\sqrt{\frac{\eta}{2\eta+1}} - \sqrt{\gamma} \right)^2-
\epsilon\right)N' \\-\left(\kappa(|G|-n)-L|X|\right)\log N'-|E(G)|\log\Theta^{-1}.
\end{multline*}
As a consequence, provided that $\epsilon,\gamma,\eta$ satisfy
\begin{equation}\label{eqn:etaGammaEpsilonCondition}
\frac{1-\theta}{\hat{m}}\frac{L}{25}\left(\frac{1}{2}
\sqrt{\frac{\eta}{2\eta+1}} - \sqrt{\gamma} \right)^2-
\epsilon>0,
\end{equation}
there is a function $N(\eta,\epsilon,\theta,L,\kappa,\Theta)$ such that
\[
\text{For all}\; N'>N(\eta,\epsilon,\theta,L,\kappa,\Theta),\;
\mathcal{S}_{\eta}(G,G',\omega_{N'})\geq 0,\;\text{with probability}\;
\eqref{eqn:Prob_nNodeCaseSanovMostGeneral}.
\]
\end{proposition}
\begin{proof}
The estimate for $\mathcal{S}_{\eta}(G,G',\omega_{N'})$ 
follows immediately from Proposition \ref{prop:scriptLSumEstimateSanov},
except that we replace the subtracted term $L|X|\log(N+1)$
with the larger subtracted term $L|X|\log(N')$.  The second statement
follows from the first, because the linear term in the estimate
dominates the logarithmic and constant terms for large $N'$.
\end{proof}
This most general form of the finite sample complexity
has the usual three deficiencies: the statement does not make clear that the condition
\eqref{eqn:etaGammaEpsilonCondition}
can be satisfied; the error probability is stated in terms of functions
involving $N$ instead of being in terms of a confidence $1-\delta$
chosen by the experimenter; and the asymptotics of $N'$ are not
immediately clear because we have not derived an explicit expression
for $N'$ in terms of the parameters.

We address this by making a somewhat arbitrary, but adequate choice
of $\epsilon,\gamma$, in terms of $\theta$.  First, we know
that we must have
\[
\sqrt{\gamma}\leq \frac{1}{2}\sqrt{\frac{\eta}{2\eta+1}},
\]
in order for the estimate of Proposition \ref{prop:nNodeCaseSanovMostGeneral}
to hold.  In order to achieve a positive margin (which $\epsilon$ is going to decrease)
for the coefficient of $N'$ in the estimate, we take, $\gamma$ to be a quarter of
this upper bound,
\begin{equation}\label{eqn:gammaChoice}
\sqrt{\gamma} = \frac{1}{4}\sqrt{\frac{\eta}{2\eta+1}}
\end{equation}
Then the left-hand side of \eqref{eqn:etaGammaEpsilonCondition} is at least
\[
\frac{1-\theta}{\hat{m}}\frac{L}{400}\frac{\eta}{2\eta+1}-
\epsilon.
\]
So we choose $\epsilon$ so that subtracting $\epsilon$ only decreases the linear
coefficient by half:
\begin{equation}\label{eqn:EpsilonChoice}
\epsilon=\frac{1-\theta}{\hat{m}}\frac{L}{800}\frac{\eta}{2\eta+1}.
\end{equation}
As a result of the choices of $\gamma$, and $\epsilon$ in \eqref{eqn:gammaChoice}
and \eqref{eqn:EpsilonChoice}, we obtain that the left-hand side of 
\eqref{eqn:etaGammaEpsilonCondition} is 
\[
\frac{1-\theta}{\hat{m}}\frac{L}{800}\frac{\eta}{2\eta+1}>0,
\]
meaning that condition \eqref{eqn:etaGammaEpsilonCondition} is satisfied.
\begin{proof}[Completion of the proof of Theorem \ref{thm:nNodeCaseSanov}]
Each of the conditions preceding the bound corresponds to setting
one of the subtracted terms
of \eqref{eqn:Prob_nNodeCaseSanovMostGeneral} equal to $\frac{\delta}{3}$.
In particular, we obtain the third bound by setting $\gamma=\frac{1}{16}\frac{\eta}{2\eta+1}$
in accordance with \eqref{eqn:gammaChoice}.
For $N$ larger than the first three quantities
in the statement of the Theorem, with probability at least $1-\delta$, we have 
\[
\mathcal{S}_{\eta}(G,G',\omega_{N})\geq 
\frac{1-\theta}{\hat{m}}\frac{L}{800}\frac{\eta}{2\eta+1}N -\left(\kappa(|G|-n)-L|X|\right)\log N-|E(G)|\log\Theta^{-1}.
\]
Therefore, for $N$ larger than the last quantity in the 
statement of the Theorem, $\mathcal{S}_{\eta}(G,G',\omega_{N})\geq 0$
with probability $1-\delta$.
\end{proof}

\section{Proofs of Technical Lemmas}\label{sec:TechnicalLemmas}
We have derived our asymptotic statements in the finite-sample complexity theorems with the aid of the 
following elementary lemma in the theory of the Lambert-W function.
\begin{lem}\label{lem:LambertWasymptotics}
 With $\mathcal{W}(x):=-W_{-1}(-x)$, the function $\mathcal{W}$ has the following asymptotic behavior as $x\rightarrow 0^+$:
 \begin{equation}\label{eq:LambertWasymptotics}
  \mathcal{W}(x)=-\ln(x)+\ln(-\ln(x)) - O\left( \frac{\ln (-\ln (x) )}{\ln (x)} \right).
 \end{equation}
\end{lem}
\begin{proof} 
 Start with the inequality 
 \[
  -\frac{1}{e}\leq x\leq \frac{\exp(1/x)}{x}\;\text{for}\, -\frac{1}{e}< x < 0
 \]
Since $W_{-1}$ is the decreasing real branch of the Lambert-W function on $[-e^{-1},0]$, by applying $W_{-1}$ to these inequalities we obtain
\[
 1<-W_{-1}(x)\leq -\frac{1}{x}.
\]
Taking logs,
\[
 0 < \ln(-W_{-1}(x))\leq -\ln(-x).
\]
Further, by taking the natural log of the relation $W_{-1}(x)\exp(W_{-1}(x))=x$, we have the ``quasi-recursive'' identity for $-W_{-1}(x)$
\[
 -W_{-1}(x)=-\ln(x)+\ln(-W_{-1}(x)).
\]
Applying the previous two-sided inequalities to this we have
\[
 1< -W_{-1}(x) < -2\ln(-x).
\]
Taking the log again,
\[
 0 < \ln(-W_{-1}(x)) < 2\ln (-\ln (-x)).
\]
Applying this inequality to the right side of the ``quasi-recursive'' identity for $-W_{-1}(x)$ yields the asymptotic expression
\[
 W_{-1}(x)=\ln(-x)+O(\ln(-\ln(-x))) \;\text{as}\; x\rightarrow 0^-.
\]
Finally, we obtain the claimed asymptotic by substituting this identity again into the quasi-recursive identity.
\end{proof}

We recall a form of the big-O notation appropriate to multivariable
situations, which may be less familiar to some readers than the single-variable form:
if $\mathbf{x}$ is a vector of variables $x_I$, then we say that
a function $f(\mathbf{x})=O(g(\mathbf{x}))$ as $\mathbf{x}\rightarrow 0^+$
if there are constants $C, M$, such that provided all $x_i>0$ and $x_i^{-1}>M$,
we have $f(\mathbf{x})\leq Cg(\mathbf{x})$.  An analogous definition can be made for $x\rightarrow\infty$.
The following Lemma is an elaboration of Lemma \ref{lem:LambertWasymptotics} which we can apply directly in the proofs of our finite sample complexity theorems:
\begin{lem}
 \label{lem:LamberWasymptoticsApplied}
  With $\mathcal{W}(x):=-W_{-1}(-x)$, we have the following asymptotic behavior:.
  \begin{itemize}
 \item[(a)]
 $\mathcal{W}(x)=O(\ln x^{-1})$ as $x\rightarrow 0^+$.
 \item[(b)]
 Let $x_i$ $i=1,\ldots k$ be variables and denote the vector $\left(x_i \right)_{i}$ by  $\mathbf{x}$. Let $e_i, f_i\leq 0$ be non-positive 
 integer exponents and $C_1, C_2$ real constants. We have
 \[
  C_1\prod_i x_i^{e_i}\mathcal{W}\left(  
  C_2 \prod_i x_i^{-f_i}
  \right) = 
  O\left(
  \prod_i x_i^{e_i}\sum_i \ln \frac{1}{x_i}
  \right)\;\text{as}\; \mathbf{x}\rightarrow \mathbf{0}^+.
 \]
\item[(c)]  Let the $x_i$ and $\mathbf{x}$, as well as the
$e_i, f_i\leq 0$ and $C_{1,j}, C_{2,j}$ be as in part (b).  Let $\lambda_j$ be functions such that
\[
 \lambda_j= C_{1,j}\prod_i x_i^{e_{i,j}} \mathcal{W}\left(  
C_{2,j}  
  \prod_i x_i^{-f_{i,j}}
  \right) 
\]
Let $\lambda = \max(\lambda_j)$.  Then we have 
\[
\lambda =O\left(
  \prod_i x_i^{\max_j(e_{i,j})}\prod_i \ln \frac{1}{x_i}
  \right)\;\text{as}\; \mathbf{x}\rightarrow \mathbf{0}^+.
\]
  \end{itemize}
\end{lem}
\begin{proof}
Part (a) follows immediately from Lemma \ref{lem:LambertWasymptotics}.  Part (b) follows from the observation that for fixed $M>0$, if we set $M' = M^{1/\sum_i e_i}$, then if all $x_i <1/ M'$, the monomial $\prod_i x_i^{e_i}< M$.  Note that we can replace
the sum on the right-hand side of (b) with a product simply because
for $x_i^{-1}$ large all the terms $\ln\frac{1}{x_i}>0$.   With this observation, it is not difficult to prove part (b) from part (c) using (for example) induction on the number of functions $\lambda_j$.    
\end{proof}
An almost identical Lemma can be formulated for the case of $x\rightarrow\infty$, whose statement and proof we leave to the reader.

\begin{proof}[Proof of Lemma \ref{lem:LinfinityNormBoundedInTermsOf_t}]
First, use
\[
\begin{aligned}
 q_{A,0} &= \sum_{j=0}^{l-1} q_{0,j}(s)\quad &q_{B,0} = \sum_{i=0}^{k-1} q_{i,0}(s)  \\
 p_{A,0} &= \sum_{j=0}^{l-1} p_{0,j}(t),\quad &p_{B,0} = \sum_{i=0}^{k-1} p_{i,0}(s).
\end{aligned}
\]
Thus,
\[
\begin{aligned}
 \left|q_{A,0}-p_{A,0} \right|&\leq l\|p(t)-q(s) \|_{\infty}\\
 \left|q_{B,0}-p_{B,0} \right|&\leq k\|p(t)-q(s) \|_{\infty}.
\end{aligned}
 \]
So we compute that
\begin{equation}\label{eqn:productDistParamsDiffEst}
\begin{aligned}
 | q_{A,0}q_{B,0}-p_{A,0}p_{B,0}| &= | q_{A,0}q_{B,0}-q_{A,0}p_{B,0}+q_{A,0}p_{B,0}-p_{A,0}p_{B,0}| \\
                                  &\leq q_{A,0} |q_{B,0}-p_{B,0}| + p_{B,0}|q_{A,0}-p_{A,0}| \\
                                  &\leq (k+l)\|p(t)-q(s) \|_{\infty}
\end{aligned}
\end{equation}
 Next note that
\[
\begin{aligned}
 t&=p(t)_{0,0}-p_{0,0}(0)=p(t)_{0,0}-p_{A,0}p_{B,0}\\
 s&=q(s)_{0,0}-q_{0,0}(0)=q(s)_{0,0}-q_{A,0}q_{B,0}.
 \end{aligned}
\]
Thus
\[
\begin{aligned}
 |t-s|&\leq \left| p(t)_{0,0}-q(t)_{0,0} \right| + \left|q_{A,0}q_{B,0}-p_{A,0}p_{B,0}\right|\\
 \\ &\leq \|p(t)-q(s) \|_{\infty} + (k+1) \|p(t)-q(s) \|_{\infty}
      & \leq (k+l+1) \|p(t)-q(s) \|_{\infty},\\
 \end{aligned}
\]
where we have obtained the next-to-last inequality by using definition of the norm $\|p(t)-q(s) \|_{\infty}$ and \eqref{eqn:productDistParamsDiffEst}.
\end{proof}

\begin{proof}[Proof of Lemma \ref{lem:ChernoffApplication}]
We divide the event $\left\{\left|  1-\frac{\tilde{p}_N}{p} \right|\right\}$
 into two cases, $\tilde{p}_N>p$ and $\tilde{p}_N<p$, handle
 both with the appropriate form of the Chernoff inequality, and then 
 apply the union bound.  First, if $\tilde{p}_N>p$, then
 \[
 \left|  1-\frac{\tilde{p}_N}{p} \right| = \frac{\tilde{p}_N}{p}  - 1
 \]
 By the multiplicative Chernoff bound, first part
 \[
 \mathrm{Pr}\left\{ \frac{\tilde{p}_N}{p}  - 1\geq \epsilon \right\}\leq e^{-Np\epsilon^2/3}.
 \]
 Second, if $\tilde{p}_N<p$, then 
 \[
 \left|  1-\frac{\tilde{p}_N}{p} \right| = 1-\frac{\tilde{p}_N}{p}  
 \] 
 By the multiplicative Chernoff bound, second part
\[
 \mathrm{Pr}\left\{ 1-\frac{\tilde{p}_N}{p} \geq \epsilon \right\}\leq e^{-Np\epsilon^2/2}.
\] 
 Thus
 \[
 \mathrm{Pr}\left\{\left|  1-\frac{\tilde{p}_N}{p} \right|\geq \epsilon \right\}\leq  e^{-Np\epsilon^2/3}
 +e^{-Np\epsilon^2/2} \leq  2e^{-Np\epsilon^2/3}.
\] 
\end{proof}

\chapter{Computation of $\beta$-values}
\label{sec:computationOfBeta}
We now explain the main computational part of this work, which is the computer calculation and tabulation of  (estimates for) error probabilities $\beta^{p^{\eta}}$.  We describe
in concrete terms not only the theory behind this computation but also how it is implemented in Python/Numpy/Scipy code in the packages betaMonteCarlo.src 
and betaMonteCarlo.test.  We have made the source of these packages publicly available at \cite{Brenner:2013:Online}.

\textbf{Notation.}  We recall the following notation from earlier in the paper or from the standard theory of types.  For $N>0$  and an alphabet of symbols $X$, we let 
$\mathcal{T}_N$ denote the \textit{type classes} of size $N$ over $X$.
\nomenclature{$\mathcal{T}_N$}{Type classes of size $N$ over $X$; set of outcomes of $N$ independent draws with replacement
from $X$, considered without respect to order in the sequence of draws.}
\nomenclature{$T$}{Element of $\mathcal{T}_N$: $\lvert X\rvert $-vector whose entries are nonnegative integers summing to $N$}
  That is, $\mathcal{T}_N$ 
is the set of outcomes of $N$ independent draws with replacement of elements from $X$, considered without 
respect to order in the sequence of draws.  An element $T\in\mathcal{T}_N$ can 
be identified with an $|X|$-vector whose entries are nonnegative integers summing to $N$.  The origin of the terminology is the
fact that the map associating each empirical sequence $\omega_N\in X^N$ with
the type class, by taking frequency counts, is a partitioning of $X^N$
into disjoint equivalence classes.  Thus, we can write $\omega_N\in T$
to indicate that $\omega_N$ belongs to the type class $T$.

Denote by $p_T$ 
\nomenclature{$p_T$}{The probability distribution in $\mathcal{P}=\mathcal{P}_{\lvert X\rvert }$ associated to type $T$}
the probability distribution in $\mathcal{P}=\mathcal{P}_{|X|}$ associated to type $T$, defined as the distribution whose parameters are the entries of $T$ divided by $N$.
For $p\in\mathcal{P}$, $\gamma>0$, $\beta^p_N(\gamma)$ is defined as in Section \ref{sec:preliminaries}, namely,
\begin{equation}\label{eqn:defOfBeta}
\beta^p_N(\gamma) = \mathrm{Pr}_{\omega_N\sim p}\left\{\tau(P_{\omega_N}) < \gamma\right\}.
\end{equation}
For a set $A\subset\mathcal{P}$, $\mathbf{1}_{A}$ denotes the characteristic (indicator) function of $A$.  
\nomenclature{$\mathbf{1}_A$}{Characteristic function of a subset $A\subset \mathcal{P}$}
\nomenclature{$\mathbf{1}_A(T)$}{Value $\mathbf{1}_{A}(p_T)$}
We use $\mathbf{1}_A(T)$
to denote the value $\mathbf{1}_A(p_T)$. Note that we clearly have
\begin{equation}\label{eqn:MIDependsOnlyOnTypeClass}
\tau(p_{\omega_N})=\tau(p_T)\;\text{for}\; \omega_N\in T.
\end{equation}
Along the same lines, we can expand the emission probability of $T\in\mathcal{T}_N$, as
\begin{equation}\label{eqn:decompositionOfTypeEmissionProbability}
\mathrm{Pr}_p(T)=\sum_{\omega_N\in T}\mathrm{Pr}_p(\omega_N).
\end{equation}
\nomenclature{$\mathrm{Pr}_p(T)$}{$\sum_{\omega_N\in T}\mathrm{Pr}_p(\omega_N)$}
For an in-depth treatment of the Theory of Types see, e.g., Chapter 11 in \cite{cover2006elements}).

In this chapter we will denote $p^\eta$ whenever possible simply by $p$.  This not only  reduces
notational clutter, but also reminds the reader that the majority of the arguments apply to any $p\in \mathcal{P}$, not only to distributions with uniform marginals.

\section{Overview of Methods}
In our computation of $\beta_N^p(\gamma)$, we carry out two computationally distinct tasks:
 \begin{enumerate}
 \item Prior to the running of the structure learning algorithm, compute (offline) 
 estimates of $\beta^p_N(\gamma)$ for a reasonably large number (on the order of $10^3$) of pairs of $N$, 
 $\gamma$, and we store the results in a data structure/file.
 \item In the deployed version of the learning algorithm, we call a function that first looks up values of $\beta$ in the tables produced in the first step, 
 in a neighborhood of the requested $N$ and $\gamma$,
 and then we use linear interpolation to return a result approximating $\beta^p_N(\gamma)$.
 \end{enumerate}
 Each of these components raises its own issues, which are actually intertwined.  For the offline computation, the main issues are how to compute $\beta$
 in a way that trades off precision of the result with speed/tractibility in a quantifiable manner.  For the lookup/interpolation step, the main issue is precisely how
 to carry out the linear interpolation, that is on which statistics derived from $N,\gamma, \eta$, to linearly interpolate.  
 The answers to these questions determine the makeup of the grid of $N$, $\gamma$
 for which the computations in the first step should be carried out and stored. 
 
Because of the functioning of each step in the overall algorithm, each step has different practical constraints.
The first step can be as slow as necessary to achieve the desired accuracy, but the second step must be
 very fast so as not to impede the learning algorithm.  These consideration must be balanced against one another.  On the one hand, by making the table in the first step
 larger, we make it easier to obtain accurate estimates of $\beta^p_N(\gamma)$ for the full range of $N,\gamma$.  On the other hand, we do not want to make the table too large, because we want the table to be able to fit into memory (RAM) in order to make the second step as fast as possible.  We will find these issues are much easier to address after we have explored the results of calculating $\beta$'s for strategically chosen pairs of $N,\gamma$. 
 
 In more detail, for the first task, our computation of an individual $\beta_N^p(\gamma)$, we follow, up to a certain point, well known and accepted patterns
 in numerical methods.  First (Section \ref{sec:exactComputation}), we give
 an algorithm for the exact computation of $\beta_N^p(\gamma)$,
 using its definition as the cdf of a certain discrete random variable.
 Second, (Section \ref{subsec:summationIntegrationReplacement}), we replace the original definition of $\beta$
 as combinatorial object (sum of many different, small probabilities) with the integral of a continuous function in $\mathbf{R}^3$, approximating the original $\beta$.
 Having made this replacement, we can introduce the standard method of integration of a continuous function over a domain in $\mathbf{R}^n$
 by ``Monte Carlo" sampling, and specify (Section \ref{subsec:montecarlo}) a particular scheme for determining a stopping criterion for the Monte Carlo integration.
 At this point, the problem becomes how to choose a probability distribution for the Monte Carlo sampling which reduces
 the variance of the approximations and allows the stopping criterion to be satisfied in as few iterations as possible.  We explain
 (Section \ref{subsec:importanceSampling}) how we can achieve this goal via the strategy of \textit{importance sampling}, which uses the information we already have concerning the integrand to construct the sampling distribution.  In order to carry out the importance sampling strategy, we have to use some features
 which are particular to the situation at hand.  Namely, in this situation, the selection of a point from $\mathbf{R}^3$ according to the
``importance-sampling" probability distribution is equivalent to picking a distribution $p_0(t)\in \mathcal{P}_{2,2}$ according to a weighting
scheme, with some regions of the space $\mathcal{P}_{2,2}$ weighted much more heavily than the rest, in a way based on the integrand.  We break down this selection into two parts: first, the selection of the marginals $p_{0,A}$, $p_{0,B}$ and second, the selection of the $t$-parameter.  We explain how our knowledge of the integrand guides each part, covering the weighting scheme for the marginals in Section \ref{subsec:samplingmarginals} and our weighting scheme for the $t$-parameter in Section \ref{subsec:fixedmarginals}.
 
 \section{Exact Computation of the Entire CDF}
 \label{sec:exactComputation}
We develop a method of computing the \textit{function} 
$\beta_N^p(\cdot)$ exactly.  The reason this is feasible
is that $\beta_N^p$ is the CDF of a certain discrete random variable,
namely the random variable which maps $\omega_N$,
in the finite outcome space $X^N$, to the mutual information $\tau(\omega_N)$.
As a result, $\beta_N^p(\cdot)$ has a very simple form as a function:
it is an increasing ``step function", with a certain finite number of 
jump discontinuities on a fixed interval between $0$ and the maximum
value of $\tau(p)$ for $p\in P_{k,l}$.  This allows us to store $\beta_N^p(\cdot)$
concisely in a data structure that keeps track only of the discontinuity
points $\gamma$ and the values $\beta_N^p(\gamma)$ of the CDF
at those discontinuity points.  For this purpose, a dictionary (Python for ``hash map")
is a suitable data structure.  It is convenient to define a Python class
called CDF, such that every CDF object has such a dictionary as its main
data structure, and also several methods as part of its API, as we will describe below.

\subsection{Serial Computation}
\label{subsec:serialComputation}
By \eqref{eqn:MIDependsOnlyOnTypeClass} and \eqref{eqn:decompositionOfTypeEmissionProbability}
we can express $\beta^p_N(\gamma)$, 
 defined as in \eqref{eqn:defOfBeta}, as follows, 
\begin{equation}\label{eqn:TypesDefnOfBeta}
\begin{aligned}
\beta^p_N(\gamma) =& \sum_{\omega_N\in X^N}\mathrm{Pr}_p(\omega_N)
\cdot \mathbf{1}_{[\tau(p_{\omega_N})<\gamma]}\\
=& \sum_{T\in \mathcal{T}_N} \mathrm{Pr}_p(T)\cdot \mathbf{1}_{A^0_{\gamma}}(p_T).\\
\end{aligned}
\end{equation}
This equation allows us to develop a practical method for computing \textit{the entire} function $\beta^p_N(\cdot)$ exactly.  

Any such method must have two main components: first, a method for \textit{iterating over} all the
types of a fixed size $N$ and length $n=|X|$; second, a method for \text{processing} a given type.
We first describe our method for iterating
over the types, which is one of only one of many possible methods.  The advantage of our method
is that it makes it easy to implement a parallelized version of the algorithm, as we will see in
Section \ref{subsec:parallelComputation} below.  Then we will explain our method of processing
each type, which in this case amounts to accounting for the type in the CDF.  The advantage
of separating out the processing method from the iteration method is that it allows
us to ``plug in" various other processing methods, as we will do in Section \ref{subsec:parallelComputation} below.

In order to describe the iteration method, consider the list of \textit{prefixes} of $T\in\mathcal{T}_N$.
For example, the list of prefixes of the type $(0120)\in T_3$ of length $4$ 
consists of the $5$ prefixes $[(),\, (0), \,(01), \,(012),\, (0120) ]$.  We can
think of \textit{all} the prefixes of \textit{all} elements of $T_N$ as being arranged in a rooted tree structure of height $n+1$,
where the root of the tree is the empty prefix, the prefixes of length $m$ are located at level $m+1$,
and the parent of a prefix of length $m+1$ is its prefix of length $m$.  
Further, the elements of $\mathcal{T}_N$ reside in the leaf level of the tree.  With this
definition of the tree, the path from the root down to $T\in T_N$ in level $n+1$ (e.g., $(0120)$) passes
through the $n+1$ prefixes of the type (e.g., $[(),\, (0), \,(01), \,(012),\, (0120) ]$).  In Figure
\ref{fig:typePrefixTree} we
have shown the entire tree of prefixes for the case $T_3$ and $n=|X|=4$.
\vspace*{3cm}
\begin{sidewaysfigure}
\begin{flushleft}
\vspace*{9cm}
\includegraphics[scale = 0.8, trim = 0mm 0mm 0mm 0mm, clip]{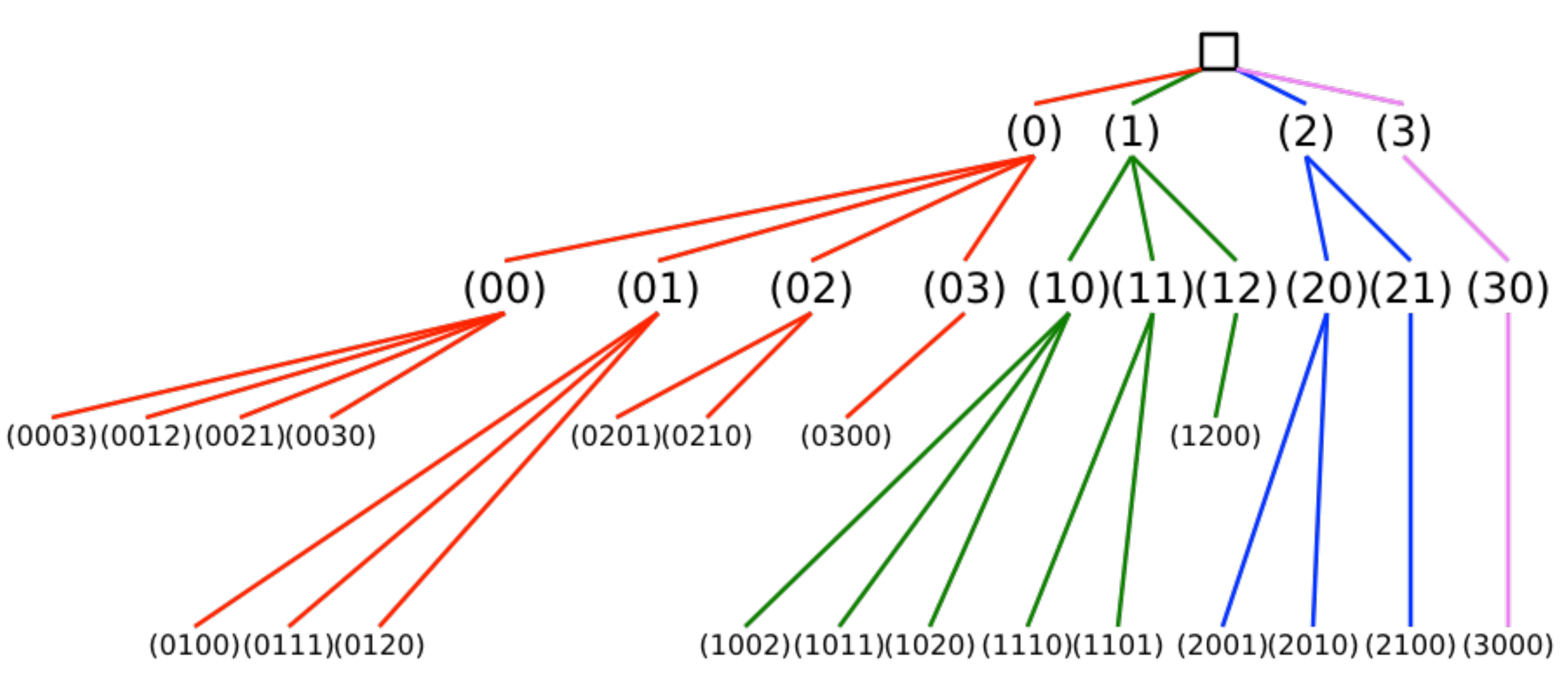}\\
\caption{Complete tree of Type Prefixes for $|X|=4$
and $N=3$, Types appearing as leaves.}
\end{flushleft}
\label{fig:typePrefixTree}
\end{sidewaysfigure}

Once we have set up the tree, we can specify a method of iterating over $\mathcal{T}_N$
by giving any method of tree traversal.  For reasons which will be seen below in Section 
\ref{subsec:parallelComputation} we use depth-first traversal (DFT).  We present
pseudocode for the TypePrefix class and recursive DFT processing of the tree nodes: 
\begin{lstlisting}[language=Python]
class TypePrefix

    #Constructor
    def typePrefix(N, data = [], n = 4):
        New typePrefix with 
            size := N
            length := n, 
            data := data, where len(data) <= n sum(data) <= N
            processed 	:= False
            
    #Methods 
    def boolean hasChildren():
         return len(data) < n

    def generator<typePrefix> childrenGenerator():
        r := N - sum(data)
        if len(data) < n - 1:
            return generator(typePrefix(N,data+[i],n) for i in range(r+1))
        else:
            return generator( typePrefix(N, data+[r],n) for i in range(1))
            
    def void DFSProcess(processMethod):
        if hasChildren():
            for child in childrenGenerator():
                if not child.processed:
                    child.DFSProcess(processMethod)
        else:
            processMethod(data)
        processed := True
\end{lstlisting}
Because only the leaves of the tree are true elements of $\mathcal{T}_N$,
it is only at the leaf level that any actual processing is done: for internal
nodes, which correspond to proper prefixes of types, processing the node
\textit{means} processing the node's children.  In our implementation,
we have the flexibility to pass in to the method \textbf{DFSProcess}
any method \textbf{processMethod} which takes a \textbf{TypePrefix} as input.

It will be easier to explain the \textbf{processMethod} we use in the context
of the API of the \textbf{CDF} object.  A minimal implementation of the \text{CDF}
object must have the following methods:
\begin{itemize}
\item \textbf{setReferenceDistribution(eta)} Set the underlying distribution to
be $p^\eta$ with reference to which the emission probabilities of types are calculated.
\item \textbf{accountForEvent(probability,value)}  Recall that the CDF object
stores a representation of a locally constant, monotone increasing function,
with finitely many discontinuities.  This method modifies the CDF function
by adding the step function $\textbf{probability}\times \mathbf{1}_{\left\{\geq \textbf{value}\right\}}$
to the CDF function.
\item \textbf{assignCumulativeProbability(value)} returns the value of the CDF
function at \textbf{value}.      
\end{itemize}
The implementation of these methods will depend on the data structures
used internally by the \textbf{CDF} object to store the CDF function.  We
have used a sorted list of discontinuity values and a dictionary of
pairs (discontinuity, CDF(discontinuity)), because this is the easiest
to implement, but more space/time efficient implementations are certainly possible.

In this framework, we can use \eqref{eqn:TypesDefnOfBeta} to see how
to define the \textbf{processMethod} which will give us the exact CDF:
\begin{lstlisting}[language=Python]
def accountForType(T):  #T a type class
    p := referenceDistribution
    p_T := empirical distribution of T
    accountForEvent(tau(p_T), Pr_p(T))  #Pr_p(T) emission probability of T
\end{lstlisting}
In order to compute the entire exact CDF, we make calls to the driver method
\textbf{accountForTypesWithPrefix} of the \textbf{CDF} class:
\begin{lstlisting}[language=Python]
   def accountForTypesWithPrefix(aTypePrefix):
        aTypePrefix.DFSProcess(accountForType)
                        
    def accountForAllTypes(self):
        rootPrefix = typePrefix(N, [ ], n)
        accountForTypesWithPrefix(rootPrefix)
\end{lstlisting}
The method \textbf{accountForTypesWithPrefix} will be useful
in the parallelization, Section \ref{subsec:parallelComputation}.

\begin{figure}\label{fig:runningTime}
\centering
\hspace*{-1.3cm}\mbox{\subfigure{\includegraphics[width=3in]{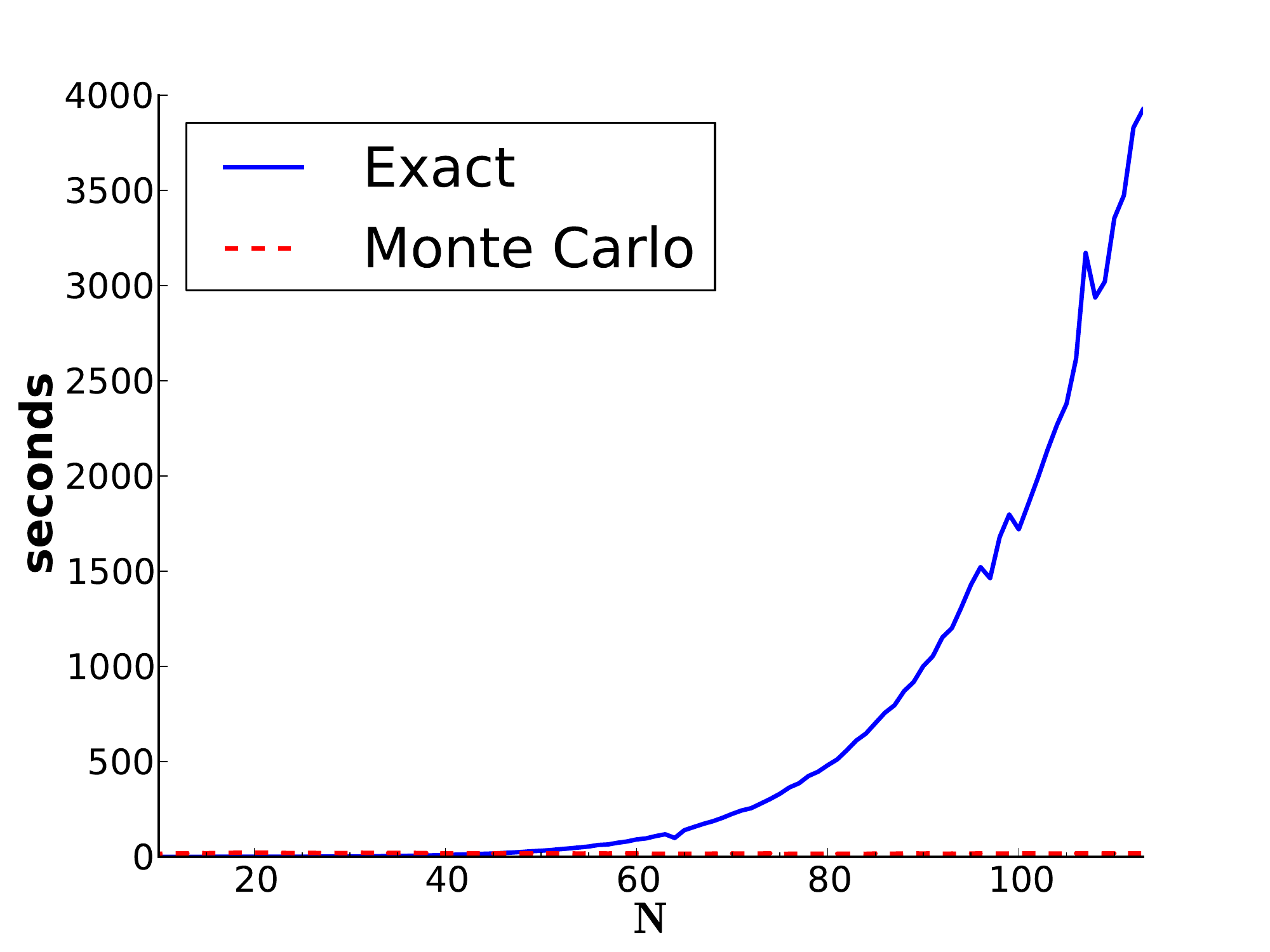}}\quad
\subfigure{\includegraphics[width=3in]{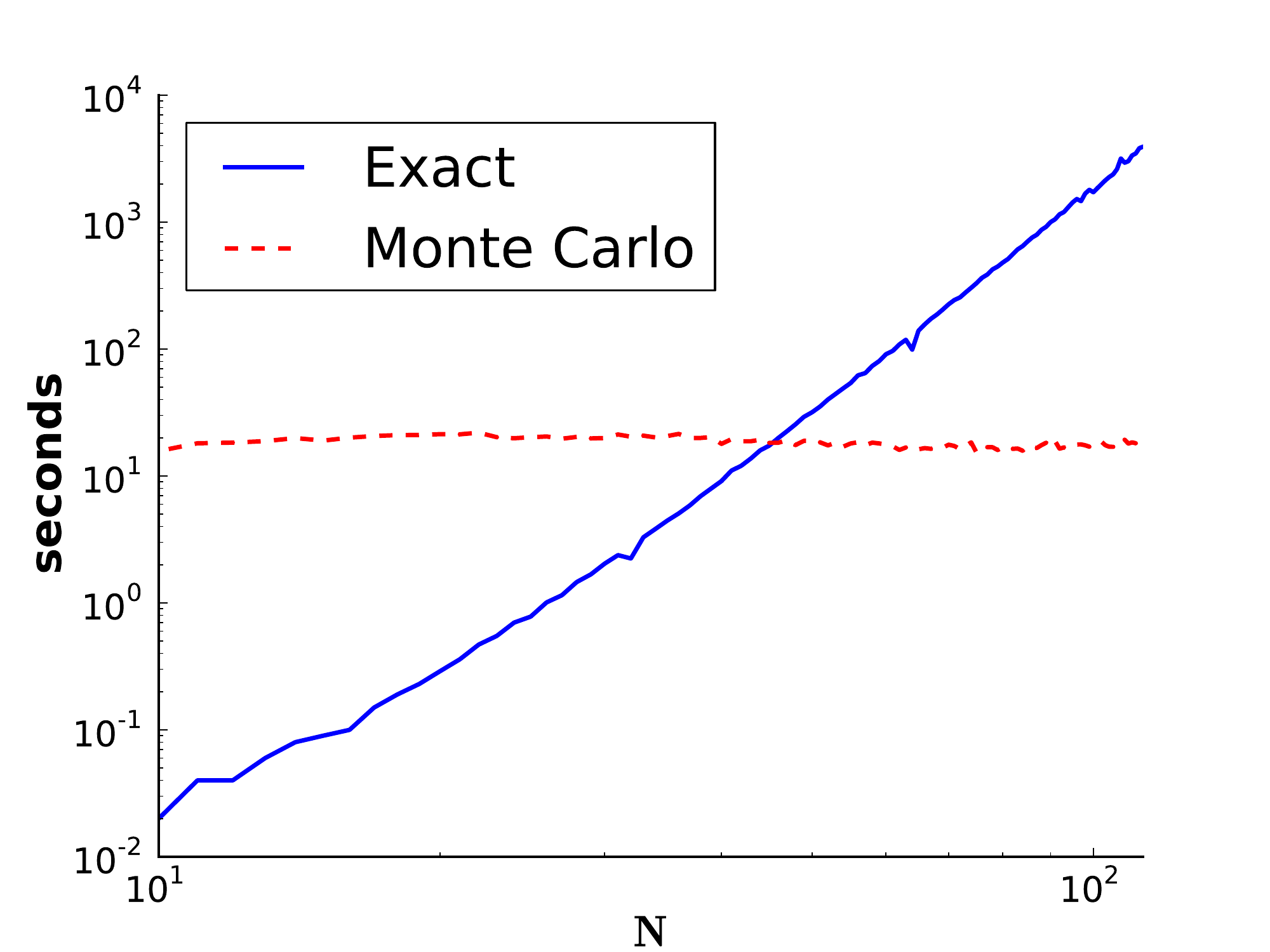} }}
\caption{Running time of exact calculations (serial implementation) and Monte Carlo estimations
of $\beta_N^{p^{\eta}}$ for a pair of binary marginal random variables, i.e.,
$n=|X|=4$, as a function of $N$.  Log-log plot shows the approximately quartic dependence of running
time on $N$.}
\end{figure}
Since the number of types $|\mathcal{T}_N|$ is exactly $\binom{N+n-1}{n-1}$, the running 
time of the exact calculation algorithm is $\Omega(N^{n-1})$.
In our experiments, Figure \ref{fig:runningTime}, the algorithm for the case
of $n=|X|=4$, appears to run in at least quartic time.  It might
be possible to wring out some improvement in the running
time by further optimizations to the serial algorithm, but it is much
easier to achieve a speed-up through a simple parallelization of the algorithm.

 \subsection{Parallel Computation}
  \label{subsec:parallelComputation}
The idea of the parallel computation is roughly to break down the task of
computing the CDF by ``branch" of the tree of type prefixes, where
a branch is a subtree having a single type prefix as common ancestor.  For each
branch, we form a separate process.  
The processes are distributed out to different cores and each process returns
a \textbf{CDF} object accounting for the types in that branch.  When all the processes
have returned their \textbf{CDF} objects, we merge (add) the CDFs into a single CDF,
accounting for all the types.  Although the computation is not embarrassingly
parallel, the parallelization scheme we have implemented agrees closely
with the MapReduce paradigm \cite{dean2008mapreduce}.

In order to carry out the parallelization, we need a few auxiliary methods:
\begin{itemize}
\item   \textbf{typePrefix.childrenListLastEntry\_k\_Mod\_m(k,m)}: return a list of the children
of a typePrefix whose last (right-most) entry is congruent to $k$ modulo $m$.
\item  \textbf{accountForListOfTypePrefixes(typePrefixes[ ], N, n, refDistribution)}.  Returns
a new CDF object accounting for all the types which are descended from \textbf{any} of the 
prefixes contained in the given list of typePrefixes.
\item  \textbf{CDF.mergeInto(anotherCDF)}  Merges anotherCDF into the given CDF object,
producing a CDF whose function is the sum of the functions of the given CDF and the other CDF.  
\end{itemize}
The first two methods are simple generalizations of the procedures \textbf{childrenGenerator}
and \textbf{accountForTypesWithPrefix}, for which we have already given pseudocode above.
The implementation of \textbf{CDF.mergeInto(anotherCDF)} is dependent
on the representation of the CDF function internal to the \textbf{CDF} object, so we will
not give pseudocode for it here.  

All that remains is to give pseudocode for a ``driver" function that carries
out the parallelization strategy:
\begin{lstlisting}[language=Python]
def void accountForAllTypesParallelized(m):  #m: modulus
    rootPrefix = typePrefix(N, [], n) 
    rootsForParallelJobs = [rootPrefix.childrenListLastEntry_k_Mod_m(k,m) for k in range(modulus)]
    #Mapper step: distribute to collection of nodes
    ListOfCDFs = [accountForListOfTypePrefixes(aRootList,N, n,refDistribution) 
                                        for aRootList in rootsForParallelJobs]
    #Reducer step
    for CDF in ListOfCDFs:
            mergeInto(CDF)
\end{lstlisting} 
Because the mapper step is embarrassingly parallelizable, we can
use the \textit{Parallel...delayed} idiom of the Python joblib
library to parallelize the ``mapper" step and to farm out each one 
of the \textbf{accountForListOfTypePrefixes}
jobs to one of the available cores of the processor.  With this very 
simple implementation, we were able to calculate the entire
CDF for the case $N=200$ on an eight-core processor in roughly 2 hours.  If we wished
instead to use a cluster of distributed processors, we might use a MapReduce
implementation such as Hadoop, and then the call to \textbf{accountForListOfTypePrefixes}
would constitute the action of each of the Mappers.  The loop
closing the pseudocode would constitute the action of the Reducer.

The colors of the branches in Figure \ref{fig:betaDepOnKLDivergence}
illustrate this parallelization procedure in the toy-example case of $N=3$, $n=4$.
Each of the colors indicates those types which fall under a given element of
the list \textbf{rootsForParallelJobs}, provided we take the modulus $m\geq 4$.
For this very small example, there is a severe imbalance between
the number of types in each branch, with the branch under the prefix $(0)$
containing ten times as many types as the branch under the prefix $(3)$.
In the case when both $N$ and the modulus $m$ are significantly larger than the number of cores,
which is normally the case, this problem essentially
resolves itself because, whenever a node happens to be assigned a branch with very few
types, the node will finish its computation early and will then be free to be
assigned another branch to work on.

 \section{Monte Carlo Integration}
 \label{sec:monteCarloIntegration}
 \subsection{Replacing Summation with Integration}\label{subsec:summationIntegrationReplacement}
 We will  begin by addressing the question of how to estimate $\beta^p_N(\gamma)$, in a way that is scalable to several thousand pairs of $N,\gamma$, with some $N$ in the 
 range of $10^4$ to $10^5$.  
 In order to obtain a scalable method of computing $\beta$, we would like to find a way to apply
standard techniques of numerical integration such as quadrature (e.g., Riemann sums) and Monte Carlo Integration, because
one can run such methods iteratively with a stopping criteria and obtain estimates that are (provably) good enough
for certain purposes.   Our first significant step is to replace the sum in
 \eqref{eqn:TypesDefnOfBeta}
 with an integral approximating the sum.  Naturally, the region of integration for this integral is $A^0_{\gamma}$.  The choice of integrand
 is based on the following basic estimate in information theory.
 \begin{lem}\label{lem:RobbinsApproximation}
Let $T\in\mathscr{T}_N$ have the property that $(T_N)_i\neq 0$ for $i=1,\ldots,|X|$; equivalently, the
 associated probability distribution $p_T\in\mathcal{P}$ assigns nonzero probability to each atomic event $\{\omega\}$,  
 for $\omega\in X$.  For some $0\leq \vartheta(T)\leq 1$, we have
\begin{equation}\label{eqn:probOfTypeExactExpression}
 \mathrm{Pr}_{p}(T) = \exp(-NH( p_T \|p) -\vartheta(T)|X|/12)\cdot (2\pi N)^{-(|X|-1)/2}\left(\prod_{i,j}(p_T)_{i,j}\right)^{-\frac{1}{2}}.
\end{equation}
Consequently, we have the upper bound
\begin{equation}\label{eqn:probOfTypeUpperEstimate}
 \mathrm{Pr}_{p}(T) \leq  \exp(-NH( p_T \|p))\cdot (2\pi N)^{-(|X|-1)/2}\left(\prod_{i,j}(p_T)_{i,j}\right)^{-\frac{1}{2}},
 \end{equation}
 \end{lem}
\begin{proof}
It is well known (and follows immediately from Theorem 11.1.2 of \cite{cover2006elements}) that  
\[
 \mathrm{Pr}_{p}(T) = |T|\exp(-N( H(p_T) + H(p_T \| p)) ),
\]
where $|T|$ is the cardinality of the type class $T$.
From Robbins' sharpening of Stirling's formula, it is possible 
to find an estimate for $\log|T|$ (see p. 39 of \cite{csiszar2011information} for details).  Exponentiating
Robbins' formulation we obtain:
\[
|T|= \exp(NH( p_T ) -\vartheta(T)|X|/12)\cdot (2\pi N)^{-(|X|-1)/2}\left(\prod_{i,j}(p_T)_{i,j}\right)^{-\frac{1}{2}},
\]
for some $\vartheta(T)\in [0,1]$.
We use this expression, under the assumption that  $(T_N)_i\neq 0$ for $i=1,\ldots,|X|$, or equivalently, that $p_T(a)>0$ for every
$a\in X$, to obtain \eqref{eqn:probOfTypeExactExpression}.  (Note that the definition of $\log$ used in  \cite{csiszar2011information} is $\log_2$, i.e., number of bits, which
explains the appearance of the $\ln e$ which does not appear in our formula).   As an immediate consequence of the nonpositivity of $-\vartheta(T)|X|/12$, we
then obtain \eqref{eqn:probOfTypeUpperEstimate}.
\end{proof}
We use the right-hand side of \eqref{eqn:probOfTypeUpperEstimate} to define a function $I$ ``extending"
$\mathrm{Pr}_p(\cdot)$ to \textit{all} of $\mathcal{P}$.
 Namely, we set
\[
B:=\{p\in \mathcal{P}|\; p(\{\omega\})>0 \;\forall\;\omega\in X\}.                                                                                                            
\]
\nomenclature{$B$}{$\{p\in \mathcal{P}\lvert \; p(\{\omega\})>0 \;\forall\;\omega\in X\}$}
For $q\notin B$, $\prod_{i,j}(q)_{i,j}=0$, so the factor
on the right-hand side of \eqref{eqn:probOfTypeUpperEstimate} 
which is the reciprocal of the square root of this product
is  $\infty$.  Thus, in our definition of $I$, we have to account
specially for points $q\in \mathcal{P}-B$.
For any $q\in\mathcal{P}$, we set
\begin{equation}\label{eqn:Idefinition}
 I(q):=
 \begin{cases}
 0                                                                                                      &          \text{if}\; q\in \mathcal{P}- B,\, q\notin p_{\mathcal{T}_N}\\
 \mathrm{Pr}_p(q)                                                                            &           \text{if}\; q\in \mathcal{P}- B,\, q\in p_{\mathcal{T}_N}\\
 \exp(-NH( q || p))\cdot (2\pi N)^{-(|X|-1)/2}                                        
 \left(\prod_{i,j}(q)_{i,j}\right)^{-\frac{1}{2}}   &        \text{otherwise, i.e., if}\; q\in B.
\end{cases}
\end{equation}
\nomenclature{$I(q)$}{$\begin{cases}
 0                                                                                                      &          \text{if}\; q\in \mathcal{P}- B,\, q\notin p_{\mathcal{T}_N}\\
 \mathrm{Pr}_p(q)                                                                            &           \text{if}\; q\in \mathcal{P}- B,\, q\in p_{\mathcal{T}_N}\\
 \exp(-NH( q \lvert \rvert p))\cdot (2\pi N)^{-(\lvert X\vert -1)/2}                                        
 \left(\prod_{i,j}(q)_{i,j}\right)^{-\frac{1}{2}}   &        \text{otherwise, i.e., if}\; q\in B.
\end{cases}$}
By Lemma \ref{lem:RobbinsApproximation}, $I$ has the following property:
\begin{equation}\label{eqn:Iupperbound}
I\geq \mathrm{Pr}_p\;\text{where they are both defined, namely on}\; p_{\mathcal{T}_N}.
\end{equation}
Also, since the interior $\mathrm{Int}(\mathcal{P})\subseteq B$, and since
the third line of \eqref{eqn:Idefinition} is a continuous function of $q$,
we have
\begin{equation}\label{eqn:Icontinuity}
I\;\text{is continuous on}\; \mathrm{Int}(\mathcal{P}).
\end{equation}
The practical advantage of $I(\cdot)$ over $ \mathrm{Pr}_{p}(\cdot)$ is that $I$ is defined on \textit{every} point of $\mathrm{Int}(\mathcal{P})$,
not just the discrete set of points $p_{T_N}$ corresponding to $\mathcal{T}_N$.  Attempts to extend the emission probability
``directly" (or using some sort of interpolation) to the whole probability simplex entail complicated and slow and error-prone rounding and truncation procedures.  
Since we are concerned with obtaining a numerical approximation to quantities based on $\mathrm{Pr}_{p}(\cdot)$, we derive
a considerable advantage to use a continuous approximation such as $I(\cdot)$, instead of $\mathrm{Pr}_{p}(\cdot)$ itself.

Because of the property \eqref{eqn:Iupperbound}, we have
\begin{equation}\label{eqn:betaBoundedByISum}
\beta_N^{p}(\gamma)=\sum_{T\in \mathcal{T}_N} \mathrm{Pr}_{p}(T) \mathbf{1}_{A_0^{\gamma}}(p_T)\leq 
\sum_{T\in \mathcal{T}_N} I(p_T) \mathbf{1}_{A_0^{\gamma}}(p_T),
 \end{equation}
where $\mathbf{1}_{A_0^\gamma}$ is the characteristic function
of the set $A_0^\gamma=\{p\in \mathcal{P}\;|\; \tau(p)\leq \gamma\}$.
By Lemma \ref{lem:RobbinsApproximation}, we only have
that $I$ provides an upper bound for $\mathrm{Pr}_p$, but
we have no information on how tight this upper bound is.  In order
to establish that the bound is tight for large $N$, we have
carried out both summations in \eqref{eqn:betaBoundedByISum},
in the script \textbf{scripts.robbinsEstimationAccuracy.py},
and reported the results in Figure \ref{fig:exactVersusRobbins}.  The CDFs
graphed in Figure \ref{fig:exactVersusRobbins}
show that the bound in \eqref{eqn:betaBoundedByISum} becomes
tight for large $N$, particularly for $\gamma <\eta$, which
is the region on which the approximation has to be most
accurate for the purposes of our algorithm.  So we may write
\[
\beta_N^{p}(\gamma)\approx 
\sum_{T\in \mathcal{T}_N} I(p_T) \mathbf{1}_{A_0^{\gamma}}(p_T)\;\text{
for large $N$}.
\]
\begin{figure}\label{fig:exactVersusRobbins}
\hspace*{-1.3cm}\mbox{\subfigure{\includegraphics[width=3in]{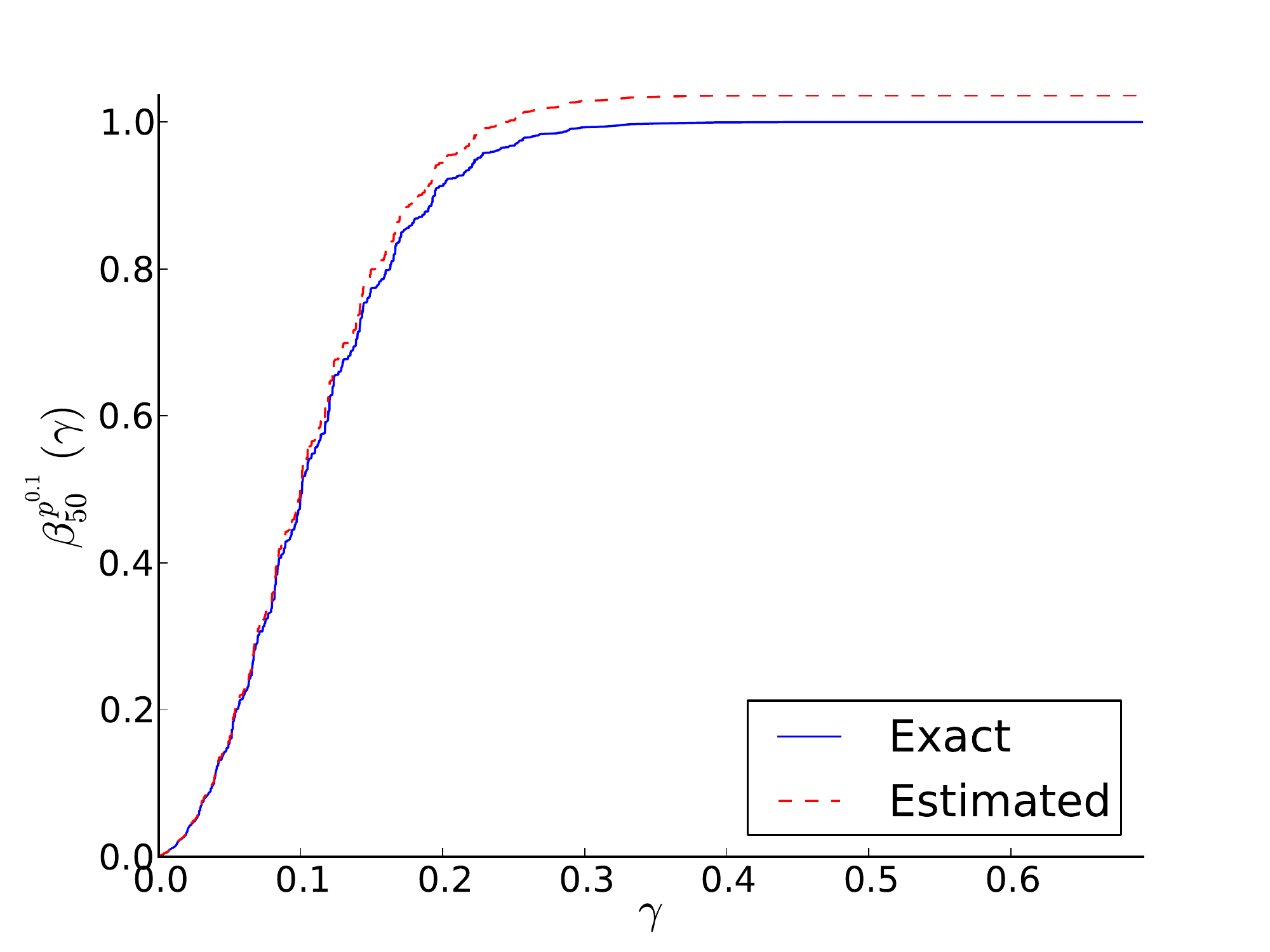}}\quad
\subfigure{\includegraphics[width=3in]{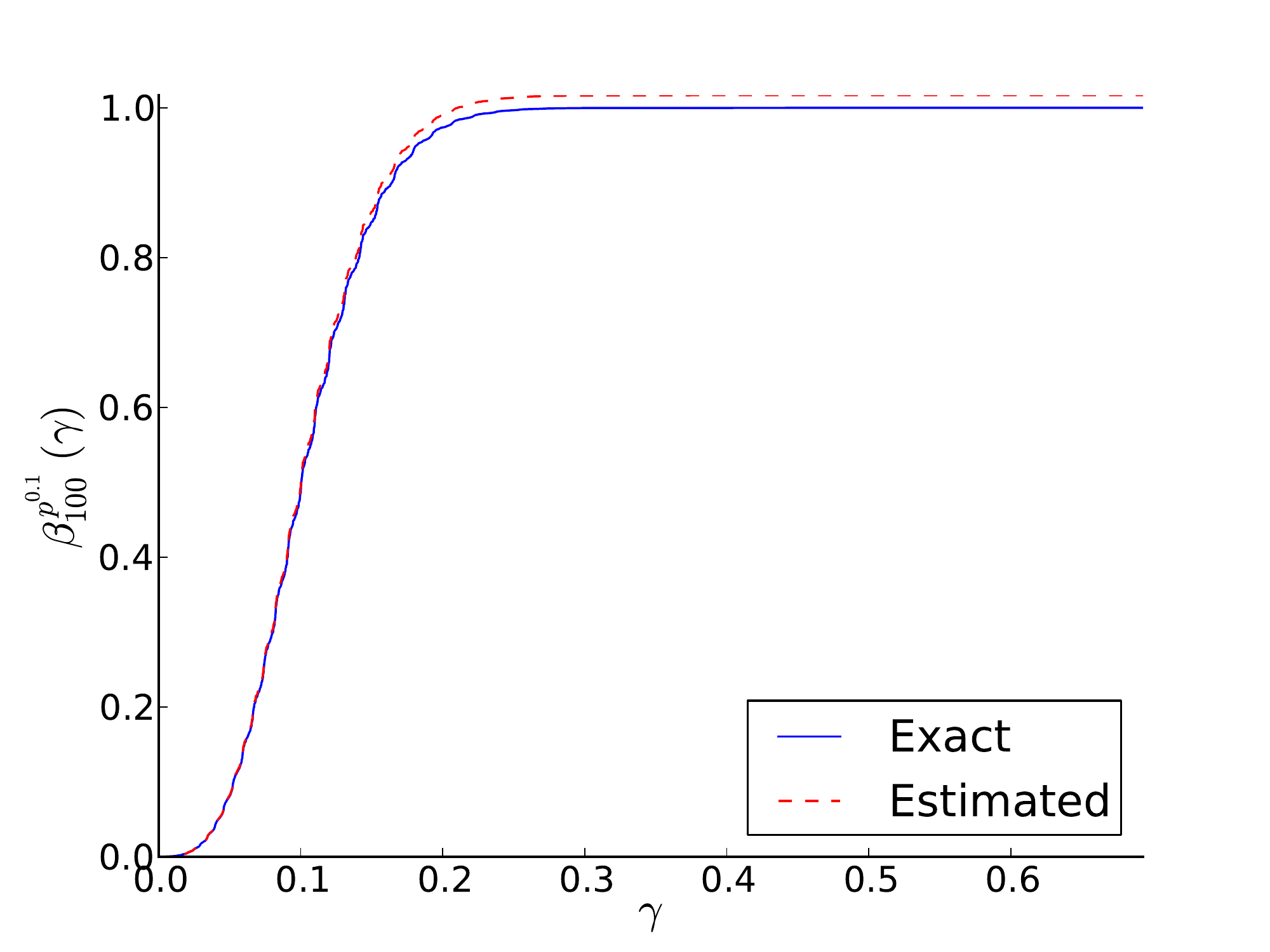} }}
\hspace*{-1.3cm}\mbox{\subfigure{\includegraphics[width=3in]{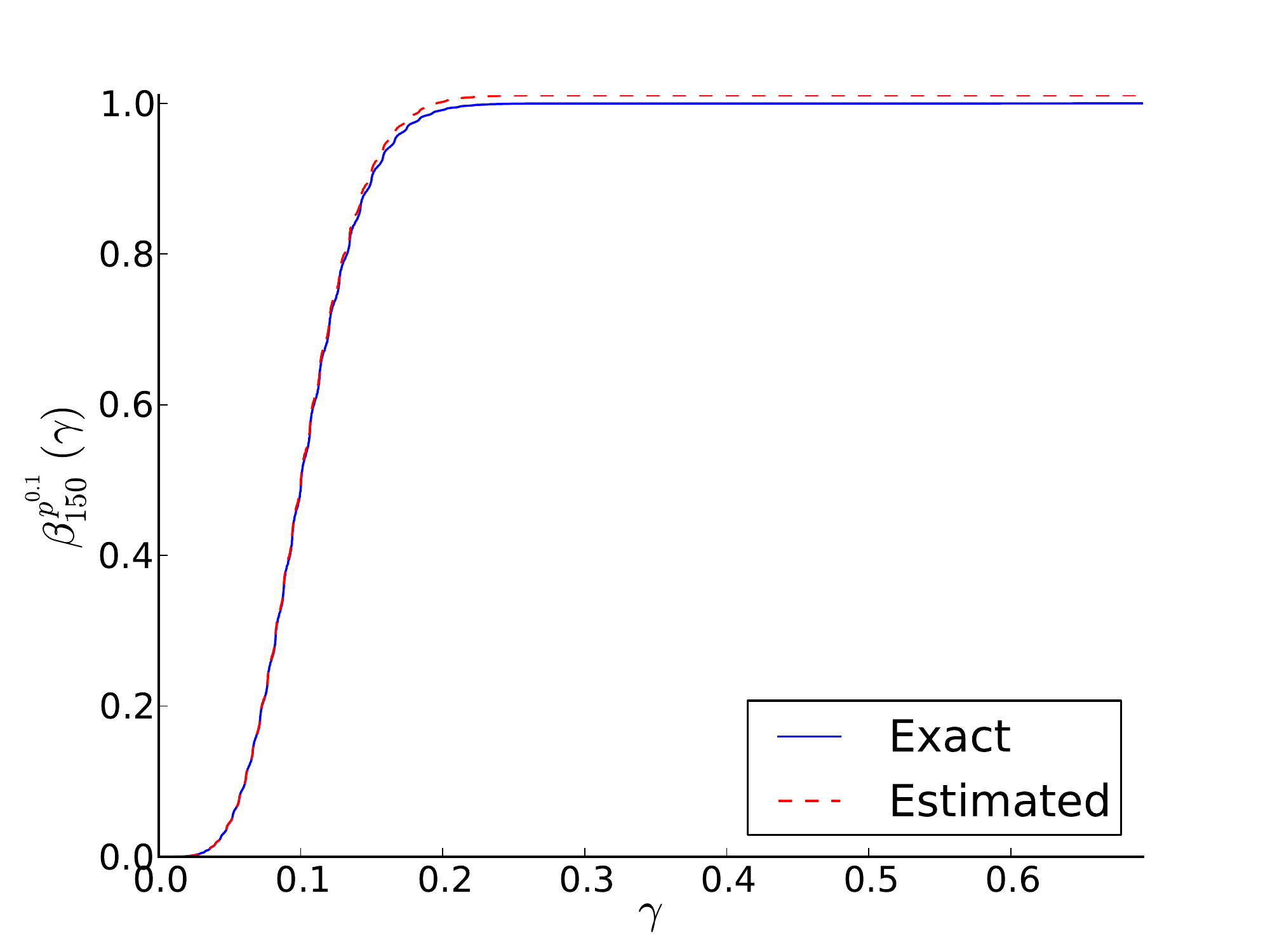}}\quad
\subfigure{\includegraphics[width=3in]{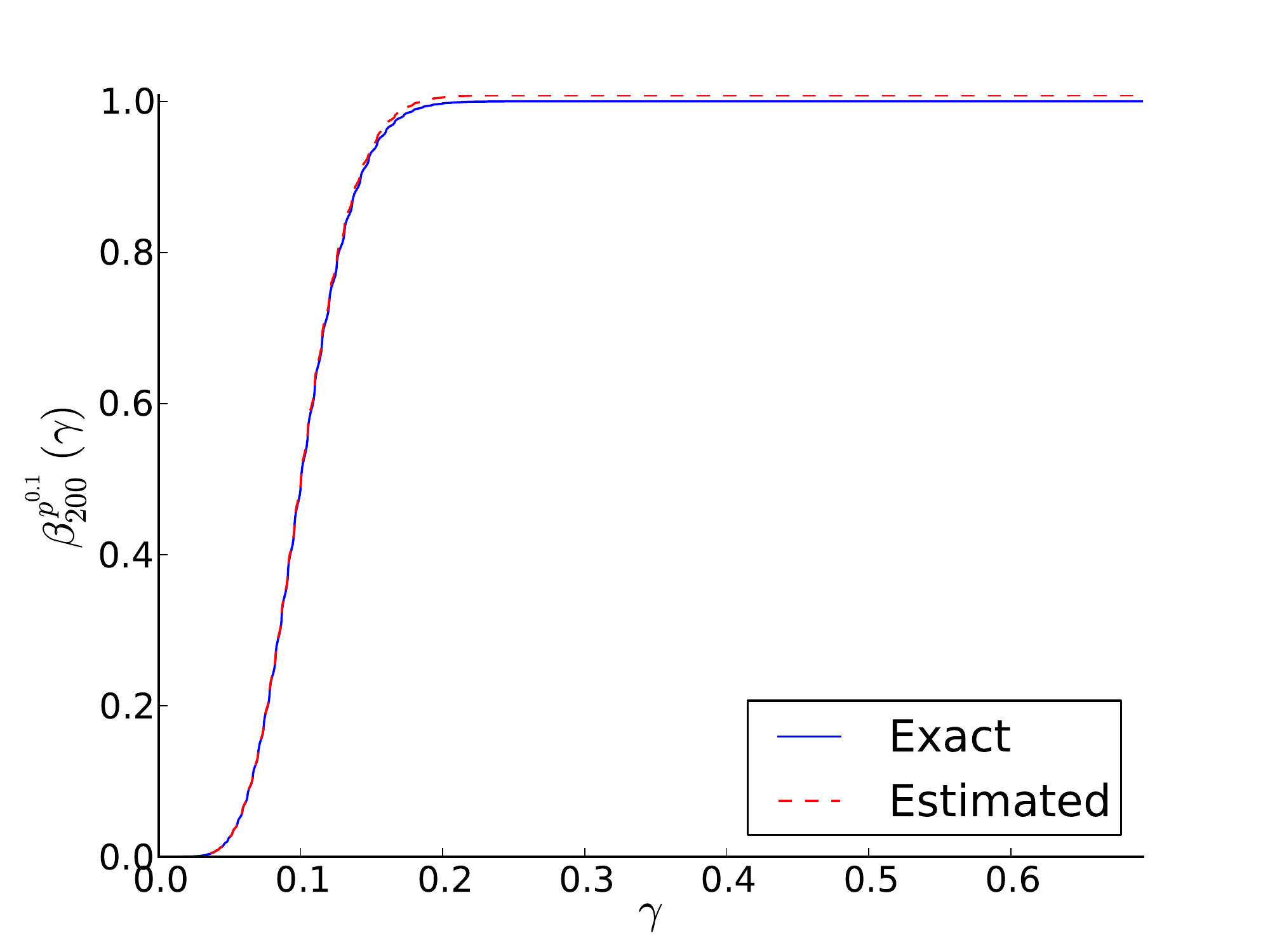} }}
\caption{Values of $\beta_N^{p^{0.1}}$ calculated by summation
over all type classes, using the exact probability function $\mathrm{Pr}_{p^{0.1}}$
and also using the estimating function $I$ evaluated on 
the points $p_{\mathcal{T}_N}$.  Results, shown for $N=50,100,150,200$,
essentially show the L1-error of $I$ as an approximation of $\mathrm{Pr}_{p^{0.1}}$, on $A_{0}^\gamma$, for all feasible $\gamma$ simultaneously.}
\end{figure}

Now extend $I$ from $\mathcal{P}$ to a function $\tilde{I}$ on all of 
$\mathbf{R}^n$ by
 \[
 \tilde{I}(q)=\begin{cases}
            I(q)& q\in \mathcal{P}\\
               0&\;\text{otherwise}.\end{cases}
\] 
\nomenclature{ $\tilde{I}(q)$}{$\begin{cases}
            I(q)& q\in \mathcal{P}\\
               0&\;\text{otherwise}.\end{cases}$}
 The properties \eqref{eqn:Iupperbound} and \eqref{eqn:Icontinuity}
 hold for $\tilde{I}$ since they hold for $I$.
Let $A\subseteq \mathcal{P}$ and let $\mathrm{d}q$ denote the ordinary Lebesgue measure.  Then by \eqref{eqn:Iupperbound} and the properties
of Riemann sums, we have 
\[
 N^{-\dim\mathcal{P}}\sum_{T\in \mathcal{T}_N} \tilde{I}(p_T) \mathbf{1}_{A}(p_T) \rightarrow \int_{A}  \tilde{I}(q)\mathrm{d}q  \;
 \text{as}\; N\rightarrow\infty.
\]
Since $\mathrm{dim}(\mathcal{P})=|X|-1$, we have
\[
\sum_{T\in \mathcal{T}_N} \tilde{I}(p_T) \mathbf{1}_{A}(p_T) \approx
N^{|X|-1}
\int_{A}  \tilde{I}(q)\mathrm{d}q  \;
 \text{for large}\; N.
\]
Putting the previous two approximations together, we have
\begin{equation}\label{eqn:betaapproxbyintegral}
  \beta^p_N(\gamma)\approx N^{|X|-1}\int_{A_\gamma^0}\tilde{I}(q)\mathrm{d}q 
   \;
 \text{for large}\; N.
\end{equation}

More explicitly,
\[
\begin{aligned}
 \beta^p_N(\gamma)\approx N^{|X|-1}\int_{A_\gamma^0}\tilde{I}(q)\,\mathrm{d}q &=  N^{|X|-1} \cdot \int_{A_0^\gamma} \exp(-NH( q \|p))\cdot(2\pi N)^{-(|X|-1)/2}                                        
 \left(\prod_{i,j}(q)_{i,j}\right)^{-\frac{1}{2}}\,\mathrm{d}q\\
&= (N/2\pi)^{(|X|-1)/2}\int_{A_\gamma^0}\exp(-NH( q \|p))\left(\prod_{i,j}(q)_{i,j}\right)^{-\frac{1}{2}}\,\mathrm{d}q\\
  &= (N/2\pi)^{(|X|-1)/2}\int_{\mathcal{P}}\exp(-NH( q \|p))\left(\prod_{i,j}(q)_{i,j}\right)^{-\frac{1}{2}}\,\mathbf{1}_{A_\gamma^0}\,\mathrm{d}q.
\end{aligned}
\]
The numerical evaluation of this integral breaks down into two distinct tasks.
\begin{itemize}
\item[(1)] Coding the integrand as a function 
\item[(2)] Writing a function that integrates such a function over $\mathcal{P}$, using Monte Carlo integration based on importance sampling with a standard stopping criterion. 
\end{itemize}
For task (1),  we form an object of type \textbf{emissionProbabilityCalculator}, parameterized by the variables ($k$, $l$, $\eta$, $N$), which remain fixed throughout the computation.  In our software, the
method  \textbf{RobbinsEstimateOfEmissionProbabilityTimesCharFunctionOfTauMinusGamma} takes the marginals and $t$-parameter of $p\in\mathcal{P}$ and returns
the value of the integrand at $p$. 

Task (2), the integration, is made up of two subtasks.  The first is implementing a general procedure for MonteCarlo integration.  The procedure iterates until the measured variance is ``small".  By ``small" we mean small enough that we, for
a fixed percentage accuracy, say \text{PrecPerc}, close to $0$ and a confidence, say \textbf{Conf}, close to $1$, we conclude that \textit{the result is
 \textbf{PrecPerc} percent accurate, with confidence \textbf{Conf}}.
For the first subtask we implement a standard scheme following Section 4.5 of \cite{bucklew:rare}, as decribed in Section \ref{subsec:montecarlo} below.  The second subtask is choosing a probability distribution on the space $\mathcal{P}$ 
for the MonteCarlo integration, based on our prior knowledge of the integrand, that reduces the variance and induces convergence in as few Monte Carlo iterations as possible.
This part, which is explained in Section 
\ref{subsec:importanceSampling} and is by its nature particular to the situation of the particular integrand at hand, is essentially our novel contribution.

\subsection{Monte Carlo Integration with stopping criterion}\label{subsec:montecarlo}
We now give a high-level view in pseudocode of the \textbf{MonteCarloIntegrate} procedure.  As input the procedure takes the following parameters, to be specified below in more detail relevant to our situation:
\begin{itemize}\item \textbf{f}, the integrand, a function on $\mathcal{P}$.
In the application $f$ is $\tilde{I}\cdot\mathbf{1}_{A_0^\gamma}$.
\item\textbf{pdf}, a nonnegative function on $\mathcal{P}$ representing a probability distribution 
absolutely continuous with respect to Lebesgue measure on $\mathcal{P}$.
\item \textbf{MaxIt}, the  maximum number of iterations, a positive integer.
\item\textbf{freqRec},
a frequency of recording intermediate results, a positive  integer.
\item\textbf{freqCrit}, a frequency of testing for the stopping criterion, a positive integer
\item\textbf{PrecPerc},  the desired precision percent of the result, a small integer, typically  $10$.
\item \textbf{Conf}, the desired confidence level in the precision of the result, a float near $1$, typically $0.95$.
\end{itemize}
 As output, the procedure has the lists \textbf{itVals} and \textbf{rhoHatVals} such that for each index \textbf{i}$>0$, the following hold:
 \begin{itemize} \item \textbf{rhoHatVals[i]} is the Monte Carlo estimate of the integral at iteration \textbf{itVals[i]} and 
 \item \textbf{rhoHatVals[-1]}, the last value Monte Carlo estimate in \textbf{rhoHatVals}, 
 is within \textbf{PrecPerc} percent of the true value of the integral with probability \textbf{Conf}.
 \end{itemize}
 The latter is a heuristic statement rather than a guarantee: it must be verified empirically (see below).
 The framework procedure has the following pseudocode:
\begin{lstlisting}[language=Python]
def MonteCarloIntegrate(f, pdf, MaxIt, freqRec, freqStop, PrecPerc, Conf):
	t_Conf := two-sided quantile of the standard Gaussian random variable at Conf
	K := (t_Conf*100/PrecPerc)^2
	
	rhoHatVals := [] 
	FHatValues := []
	itVals := []
	s :=0
	F_s := 0 

	Iterate until stopping criterion or MaxIt iterations has been reached:
	
		Draw distribution q according to pdf
		s_summand := f(q)/pdf(q)
		F_s_summand := s_summand^2
		s += s_summand
		F_s += F_s_summand
		Once every freqRec iterations:
			I := s/(Number of iteration)
			rhoHatVals.append(I)
			itVals.append(Number of iteration)
		Once every freqStop iterations:
			F := F_s/(Number of Iterations)
			if Number of  Iterations >= K*(F/I^2 - 1 ):
				Invoke Stopping Criterion and break from loop
	
	return itVals, rhoHatVals
     \end{lstlisting}

A few comments on the procedure \textbf{MonteCarloIntegrate} are in order: the two-sided quantile of the standard Gaussian random variable at \textbf{Conf} 
is defined as the unique value $t_y$ satisfying $P(|Z|\leq t_y)=\textbf{Conf}$ for $Z$ the standard Gaussian random variable.  In Python/Scipy this can be computed
easily by
\begin{lstlisting}[language=Python]
from scipy import stats
t_y = stats.norm.isf((1-y)/2.0)
     \end{lstlisting}
and in our application $y$ is \textbf{Conf}.  For example, for \textbf{Conf} at $0.90$, we have $t_{\textbf{Conf}} \approx  1.644$, for \textbf{Conf} at $0.95$, we have $t_{\textbf{Conf}} \approx  1.959$,
and for \textbf{Conf} at $0.99$ we have $t_{\textbf{Conf}}\approx 2.57582$.

The stopping criterion we use
is based on a heuristic rather than a theorem.
To explain this heuristic, let $\rho$
be the \textbf{true value of an integral} which we are computing
by the Monte Carlo method, and let $\hat{\rho}$
be \textbf{the empirical estimate for \boldmath $\rho$ \unboldmath} achieved after a certain number of iterations.
Thus, the random variable $\hat{\rho}-\rho$
is \textbf{the empirical error in the Monte Carlo estimate of the integral}.
\nomenclature{$\hat{\rho}-\rho$}{Random variable quantifying the empirical error in the Monte Carlo
estimate of an integral after $i$ iterations}%
For our stopping criterion, we would like to
prove that if a certain condition is satisfied, then we have
\begin{equation}\label{eqn:stoppingCriterionMeaningOf}
\mathrm{Pr}\left( \left|
\hat{\rho}-\rho
\right|\leq \frac{\textbf{PrecPerc}}{100}
\right)\geq \textbf{Conf}.
\end{equation}
Under the assumption that $\hat{\rho}-\rho$
is a Gaussian r.v.$\!$ centered at zero and with 
standard deviation $\sqrt{\mathrm{Var}(\hat{\rho})}$,
it is not difficult to prove the relation (see, e.g., (4.6) in \cite{bucklew:rare})
\[
\mathrm{Var}(\hat{\rho})=\frac{F-\rho^2}{\textbf{Number of iteration}},
\]
where $F$ is the quantity calculated in the procedure \textbf{MonteCarloIntegrate}.  
Using some simple algebraic manipulations (see p. 71 of \cite{bucklew:rare})
we can obtain from this relation that \eqref{eqn:stoppingCriterionMeaningOf}
is satisfied provided that the condition 
\begin{lstlisting}[language=Python]
Number of  Iterations >= K*(F/I^2 - 1 )
\end{lstlisting}
on line 24 of the pseudocode for \textbf{MonteCarloIntegrate} is satisfied.

As for the assumption that $\hat{\rho}-\rho$ is has a Gaussian distribution, we do not have any
a priori reason to believe this assumption is satisfied in our situation. We nevertheless
use the stopping criterion outlined above, because it is simple 
to implement.    The justification for using the stopping
criterion we have adopted in
\textbf{MonteCarloIntegrate} is an empirical one, namely that, for large $N$ at least, by using this stopping criterion with \textbf{PrecPerc} equal to $10$
and \textbf{Conf} equal to $0.95$, we do achieve
estimates to the true values of $\beta$ that are well within the
$10\%$ error margin in 95\% of cases.  We have recorded
the empirical results supporting these assertions in Figure
\ref{fig:stoppingCriterion}.  These empirical results do not validate the assumption on $\hat{\rho}-\rho$ per se,
but they do suggest that a Gaussian distribution models $\hat{\rho}-\rho$ well enough in practice that our
entire procedure produces estimates $\hat{\rho}$ that are within \textbf{PrecPerc} percent of $\rho$ at least \textbf{Conf} of the time.  
\begin{figure}\label{fig:stoppingCriterion}
\hspace*{-1.3cm}\mbox{\subfigure{\includegraphics[width=3in]{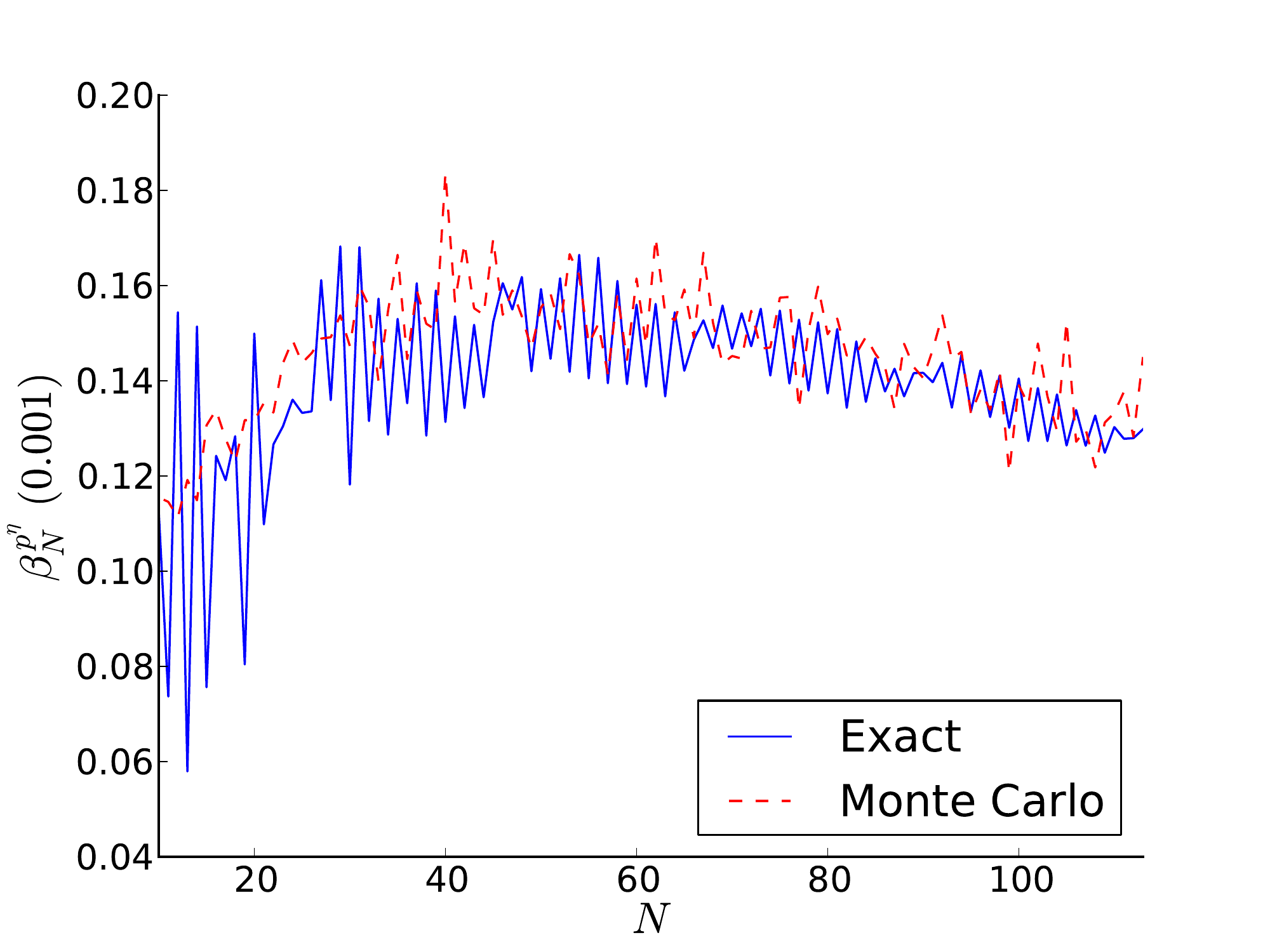}}\quad
\subfigure{\includegraphics[width=3in]{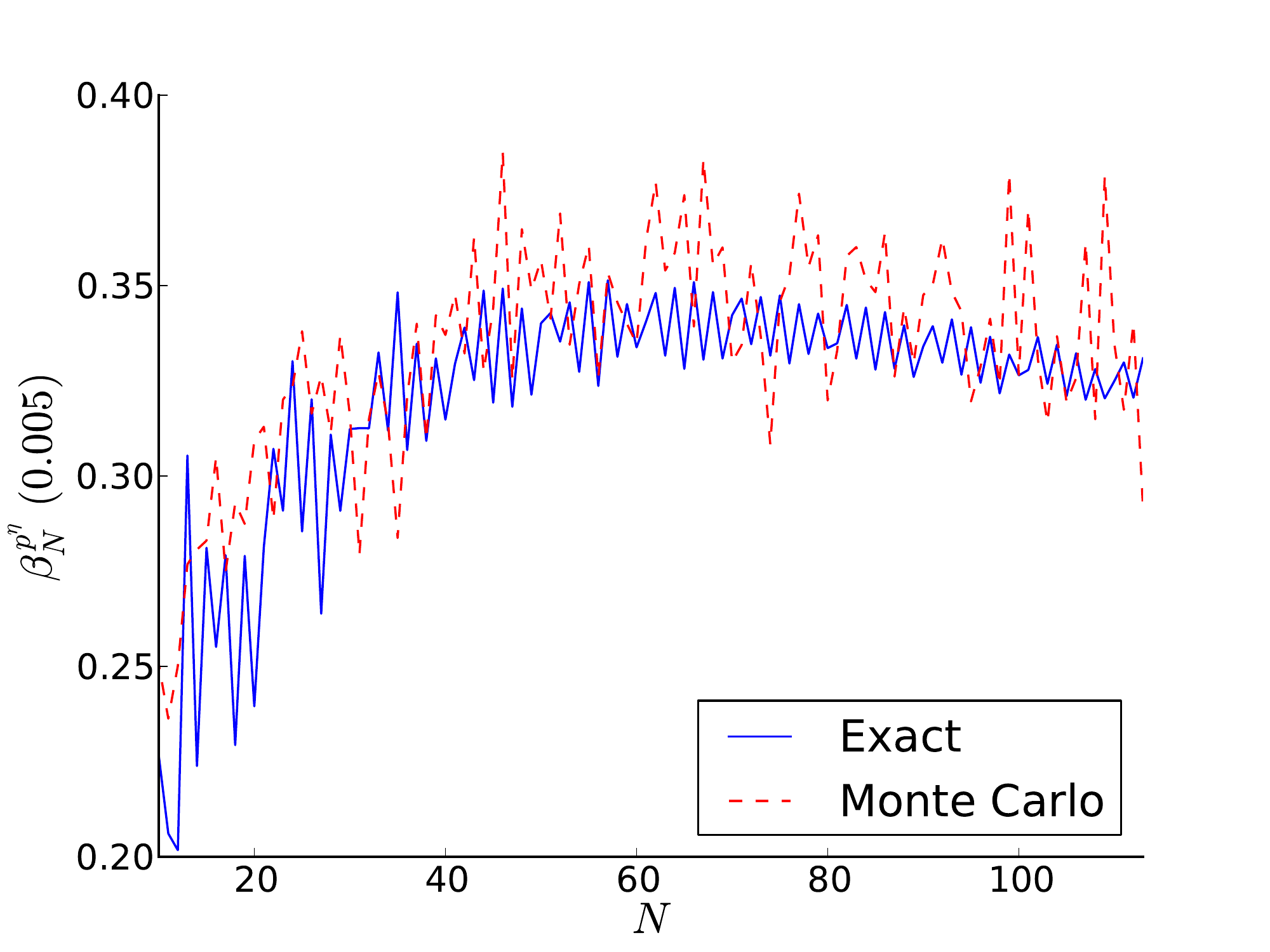} }}
\hspace*{-1.3cm}\mbox{\subfigure{\includegraphics[width=3in]{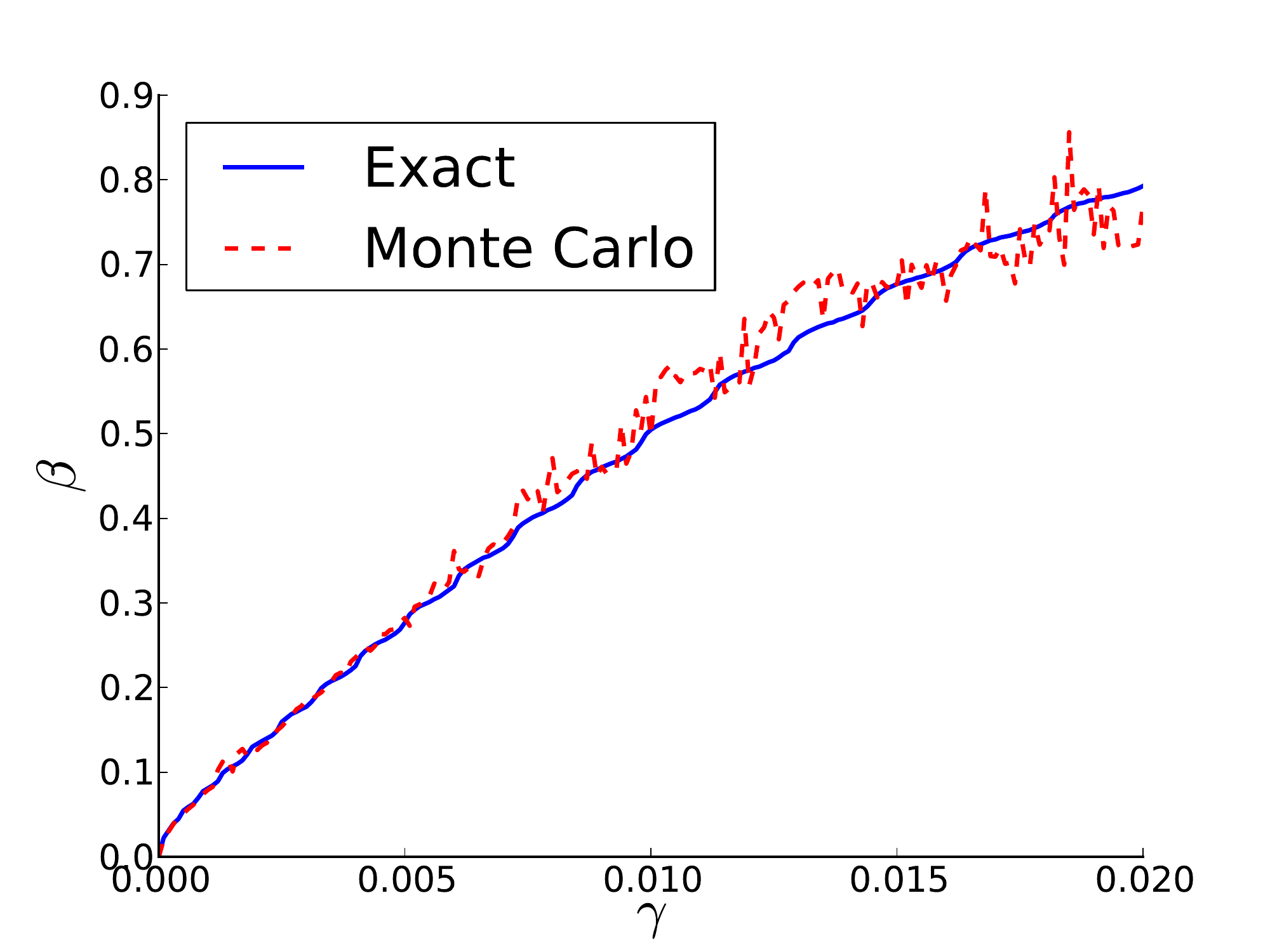}}\quad
\subfigure{\includegraphics[width=3in]{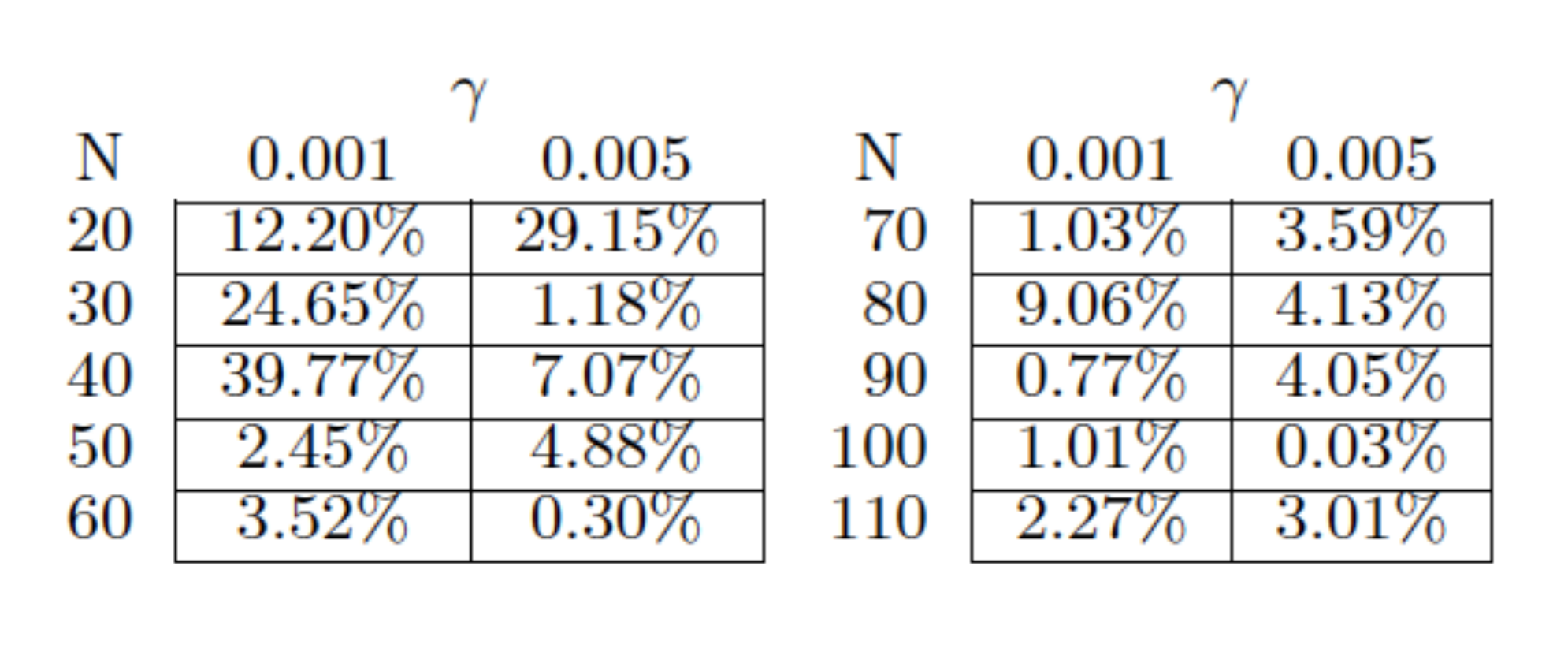} }}
\caption{Illustrations of the effectiveness of Monte Carlo
method with stopping criterion at producing estimates
$\hat{\beta}_N^{p^{0.01}}(\gamma)$
of $\beta_N^{p^{0.01}}(\gamma)$ within $10\%$ of the true value.
Top row: improvement in the estimates of $\beta_N^{p^{0.01}}(0.001)$
and $\beta_N^{p^{0.01}}(0.005)$ as $N$ grows from $10$ to $110$.
Bottom left: for $N$ fixed at $200$ improvement in accuracy
of estimated $\beta_{200}^{p^{0.01}}(\gamma)$ as $\gamma\rightarrow 0^+$.
Bottom right: Multiplicative error $|\beta_N^{p^\eta} - \hat{\beta}_N^{p^\eta}|/
\beta_N^{p^\eta}$ as $N$ varies from $20$ to $110$.
}
\end{figure}

In the procedure, it may make sense to use the uniform distribution as the sampling distribution if, for example, we know nothing about the integral besides its integrability.  In the case of 
the integrand $\tilde{I}$, because only a very small part $\mathcal{P}$ lies in $A_{\gamma}$, especially for small $\gamma$, it is an extremely rare
event that a point chosen by the uniform distribution contributes to $s$
anything at all, much less, anything significant.
As a result, for a random, very small proportion of the iterations, 
the integral estimate \textbf{rhoHatVals} will record a (relatively speaking) very large jump.
Further, even less frequently will a point of $\mathcal{P}$ chosen randomly according to the uniform distribution have test statistic
$\tau$ less than $\gamma$ and simultaneously lie close to the I-projection
of $p^\eta$ onto $A_{\gamma}$.  For $N$ large, such
points account for most of the mass of the integrand (as can be seen qualitatively from Sanov's Theorem), so such a small proportion of the
iterations should lead to even more significant jumps in \textbf{rhoHatVals}.  
The result is that it will take very many iterations for the variance of the estimates to descrease enough for the stopping criterion to be satisfied.

These problems with uniform sampling affect most markedly
the cases of small $\gamma$ and large $N$.  From the point of view of the application
of the algorithm these cases are of the most interest to us, so we must
address them.  Fortunately, in the above Monte Carlo algorithm, we can use a 
nonuniform sampling distribution \textbf{pdf} tailored to reduce the variance of the estimates of the integral of $f$.
Such a variance reduction method is called an ``importance sampling" scheme.

\subsection{Importance Sampling Overview}\label{subsec:importanceSampling}
What now remains to explain is the part of the calculation that is the most delicate and potentially open to further refinement: the choice of the measure for the Monte Carlo sampling.
In order to explain this, we first specialize to the case of $k=l=2$, for which the map $F$ from the marginal-$t$ coordinates to the contingency-table coordinates
is defined explicitly by
\[
 F: \left(
 p_A, p_B, t
 \right)
 \mapsto 
 \begin{bmatrix}
 p_Ap_B+t & (1-p_A)p_B-t\\
 p_A(1-p_B)-t & *
\end{bmatrix}=
\begin{bmatrix}
 p_{1,1} & p_{1,2}\\
 p_{2,1} & *
\end{bmatrix}.
\]
The map $F$ is bijective from $\mathbf{R}^3$ to the copy of $\mathbf{R}^3$ 
embedded in the upper-left corner of the space of $2$-by-$2$ contingency tables.  
\nomenclature{$F$}{Unimodular change-of-coordinates mapping on $\mathcal{P}$ 
used in the definition of importance sampling scheme.}
(The lower right coordinate is a function  of the first three coordinates, namely $*=1-p_{1,1}-p_{1,2}-p_{2,1}$).  The inverse of the bijection $F$ is given explicitly
by
\[
 F^{-1}: \begin{bmatrix}
 p_{1,1} & p_{1,2}\\
 p_{2,1} & *
\end{bmatrix}\mapsto
\begin{pmatrix}p_{1,1}+p_{2,1}\\ p_{1,1}+p_{1,2}\\ p_{1,1} - (p_{1,1}+p_{2,1})(p_{1,1}+p_{1,2})
\end{pmatrix}.
\]
The mapping $F$ in effect provides a change of coordinates map on $\mathcal{P}\subset\mathbf{R}^3$.
On the one hand, it is more convenient to think of the domain $\mathcal{P}$ 
as defined in terms of equations in the coordinates $p_{1,1},\, p_{1,2},\, p_{2,1}$ 
in $\mathbf{R}^3$, since in these coordinates $\mathcal{P}$ is simply
the unit cube.  On the other hand, it is more convenient to think to define the pdf of the
sampling distribution in terms of the coordinates $p_{A}$, $p_{B}$, $t$,
because of our knowledge of the integrand $\tilde{I}$ and the way we can use
the theory of Large Deviations to extract information about $\tilde{I}$.  We will
explain the procedure in detail in Sections \ref{subsec:samplingmarginals} and \ref{subsec:fixedmarginals} below.  
First, we validate this procedure through the following lemma.
\begin{lem}\label{lem:unimodularityLemma}
 The Jacobian $\mathcal{J}F$ of the bijective mapping $F$ is unimodular.  Consequently, the Jacobian
 $\mathcal{J}F^{-1}$ of the inverse mapping $F^{-1}$ is unimodular.
\end{lem}
\begin{proof}
 It is a straightforward calculation, from the partial deriviates of $F$ taken in the order $\frac{\partial}{\partial p_A}$, $\frac{\partial}{\partial p_B}$
 $\frac{\partial}{\partial t}$ that
 \[
  \mathcal{J}F(p_A, p_B, t) = 
\left(\begin{array}{ccc}
p_{B} & 1-p_{B} & -p_{B}\\
p_{A} & -p_{A} & 1-p_{A}\\
1 & -1 & -1
\end{array}\right).
\]
(The Jacobian $\mathcal{J}F(p_A, p_B, t)$ turns out to be constant in the variable $t$.) 
From this one computes that $\det\mathcal{J}F(p_A, p_B, t)\equiv 1$, as claimed.
Therefore,
\[
\det\mathcal{J}F^{-1}\left( p_{1,1},p_{1,2},p_{2,1} \right)=\det\mathcal{J}F^{-1}\left(F(p_A,p_B,t)\right)=\left(\det\mathcal{J}F(p_A,p_B,t)\right)^{-1}\equiv 1,
\]
also, as claimed.
\end{proof}
\nomenclature{$\mathcal{J}F$}{Jacobian matrix of mapping $F$.}
The significance of Lemma \ref{lem:unimodularityLemma} is as follows: if we denote the Lebesgue measure of $\mathbf{R}^3$ with respect to the coordinates $\mathrm{d}p_{1,1}\mathrm{d}p_{1,2}\mathrm{d}p_{2,1}$ and
define $\mu$ to be a measure with density function $\Phi$ with respect to the $(p_A,p_B,t)$ coordinates, i.e. so that
\[
\mathrm{d}\mu:=\Phi\,\mathrm{d}p_A\mathrm{d}p_B\mathrm{d}t,
\]
then the following holds for all measurable $A\subseteq \mathbf{R}^3$:
\begin{equation}\label{eqn:unimodularChangeOfCoordinates}
\mu(A):= \int_{A}\mathrm{d}\mu=\int_{A} \Phi \mathrm{d}p_A\mathrm{d}p_B\mathrm{d}t = 
\int_{F(A)} \Phi\cdot \det\mathcal{J}F^{-1} \mathrm{d}p_{1,1}\mathrm{d}p_{1,2}\mathrm{d}p_{2,1}=
 \int_{F(A)} \Phi\mathrm{d}q.
\end{equation}
In \eqref{eqn:unimodularChangeOfCoordinates}, we have, as above, used the notation $\mathrm{d}q$ for the Lebesgue measure on $\mathbf{R}^3$ with respect to the coordinates $p_{i,j}$.

Now we will define $\Phi$ (thus defining $\mu$) as a product,
\begin{equation}\label{eqn:RNderivativeDefinition}
 \Phi(p_A,p_B,t):=\varphi_N(p_A)\varphi_N(p_B)\psi_{N,\gamma}(t),
\end{equation}
where each of the functions $\varphi_N$, $\psi_{N,\gamma}$ is a Gaussian pdf on $\mathbf{R}$ with total integral $1$ and scale depending on $N$.  In the
case of $\psi_{N,\gamma}$ both the scale and the center (i.e. location of the mode) of the Gaussian depend on $N$ and on $\gamma$.  

The main reason for adopting such a simple formula for $\Phi(p_A,p_B,t)$ 
is expediency, but we also have
to justify the choice of $\Phi$ in \eqref{eqn:RNderivativeDefinition} by showing that
it really leads to a bona fide probability measure $\mu$ on $\mathbf{R}^3$.
Fortunately, this easy easy to accomplish using \eqref{eqn:unimodularChangeOfCoordinates} and Fubini's Theorem,
and we obtain the following result.
\begin{lem}
With $\Phi$ defined by \eqref{eqn:RNderivativeDefinition} and the $\varphi_N$
and $\psi_{N,\gamma}$ Gaussian pdf's on $\mathbf{R}$ with total
integral $1$, we have
\[
\int_{\mathbf{R}^3}\Phi\left(
F^{-1}(p_{1,1},p_{1,2},p_{2,1})\right)\,
\mathrm{d}q=1
,
\]
where $\mathrm{d}q$ is the Lebesgue measure on $\mathbf{R}^3$
with respect to the coordinates $(p_{1,1},p_{1,2},p_{2,1})$.
That is, the measure $\Phi \circ F^{-1}\mathrm{d}q$ is a probability measure on $\mathbf{R}^3$.
\end{lem}
As a result we are justified in using $\mu$ with density function \eqref{eqn:RNderivativeDefinition} as the parameter
\textbf{pdf} in the function \textbf{MonteCarloIntegrate} procedure listed above.
\subsection{Sampling the Marginals}\label{subsec:samplingmarginals}
In constructing $\varphi_N$, according to the general principles of variance reduction in Monte Carlo integration, we wish to choose $\varphi_N$
so that the following proportionality relationship holds:
\begin{equation}\label{eqn:fiberAverage}
 \varphi_N(p_A)
 \propto
 \int_{\mathbf{R}\times\mathbf{R}}\left(\tilde{I}\circ F\right)\cdot \left(\mathbf{1}_{A_{\gamma}}\circ F\right)
 \left(
 p_A,p_B,t
 \right)\,\mathrm{d}p_B\,\mathrm{d}t.
\end{equation}
Recall that $\tilde{I}(q)$ approximates the probability of emission of $q$ by $p^{\eta}$ 
Numerical
experimentation shows that, to first approximation, the variation in the magnitude of \eqref{eqn:fiberAverage} 
is primarily due to the variation in the factor $\exp(-NH(q\| p^\eta))$ inside $\tilde{I}$,
as opposed to the factor $\mathbf{1}_{A_{\gamma}}$.  Therefore, 
we replace \eqref{eqn:fiberAverage} with the simpler condition 
\begin{equation}\label{eqn:reducedFiberAverage}
 \varphi_N(p_A)
 \propto
 \int_{\mathbf{R}\times\mathbf{R}}\left(\tilde{I}\circ F\right)\left(
 p_A,p_B,t
 \right)\,\mathrm{d}p_B\,\mathrm{d}t.
\end{equation}
By standard information theory the right-hand side of \eqref{eqn:reducedFiberAverage}, 
as a function of $p_A$, closely correllates with
\[
 \mathrm{Pr}_{\omega_N\sim p^{\eta}}\left\{ p(\omega_N)_A = p_A\right\}.
\]  
By the central limit theorem, the most likely values of $p(\omega_N)_A$ are concentrated
around $p^\eta_A=\frac{1}{k}$.  Thus, we should choose the mode
of $\varphi_N$ for all $N$ to be located at $\frac{1}{k}=0.5$.

Now that we have determined the mode of the Gaussian $\varphi_N$ it remains to determine how to
scale $\varphi_N$.  Recall that Chernoff's inequality implies that
\[
 \mathrm{Pr}_{\omega_N\sim p^{\eta}}\left\{ \left| p(\omega_N)_A - \frac{1}{2} \right| > t\right\}\leq e^{-Nt^2}
\]
We derive a heuristic from this inequality which says that $\varphi_N$ should
be defined so that there is a ``suitable" sequence $t_N$ (to
be specified precisely below) so that
\begin{equation}\label{eqn:ChernoffAsEqualityNoncentral}
\int_{\left\{x \,| \,|x-1/2| > t_N \right\}}\varphi_N(x) \mathrm{d}x= e^{-N{t_N}^2}.
\end{equation}
Since $\varphi_N$ is a pdf on $\mathbf{R}$, we may write \eqref{eqn:ChernoffAsEqualityNoncentral} in the more convenient form
\begin{equation}\label{eqn:ChernoffAsEquality}
\int_{\frac{1}{2}-t_N}^{\frac{1}{2}+t_N}\varphi_N(x)\,\mathrm{d}x = 
1-e^{-N{t_N}^2}.
\end{equation}
Since \eqref{eqn:ChernoffAsEqualityNoncentral} and \eqref{eqn:ChernoffAsEquality} can hold true
for one sequence of $\left\{t_N\right\}$,
we will choose $\left\{t_N\right\}$ so that the right-hand side of 
\eqref{eqn:ChernoffAsEqualityNoncentral} and \eqref{eqn:ChernoffAsEquality}
are both fixed at a certain (convenient) constant value.  Namely,
we will choose a constant \textbf{CentralProbability} in $(0,1)$, and
set
\begin{equation}\label{eqn:ChernoffRadiusDefinition}
t_N=\sqrt{\frac{-\log(1-\textbf{CentralProbability})}{N}},
\end{equation}
so that \eqref{eqn:ChernoffAsEquality} becomes
\begin{equation}\label{eqn:ChernoffAsEqualityCentral}
\int_{\frac{1}{2}-t_N}^{\frac{1}{2}+t_N}\varphi_N(x)\,\mathrm{d}x =
\textbf{CentralProbability}.
\end{equation}
The most convenient value to take for \textbf{CentralProbability}
is
\[
\label{eqn:CentralProbabilityConstantDefinition}
\textbf{CentralProbability}:=0.6826894921,
\] the proportion
of the mass lying within one standard deviation of the mean of a Gaussian, because then \eqref{eqn:ChernoffAsEqualityCentral}
simply says that $t_N$ is the standard deviation (scale) of the
distribution with pdf $\varphi_N$.  

In our source module \textbf{informationTheory}, we implement
the above calculation of $\varphi_N$ by defining a function,
called \textbf{ChernoffRadius}.  Then $t$, the scale of $\varphi_N$,
is computed with a call to the function as follows:
\begin{lstlisting}[language=Python]
 t = src.informationTheory.ChernoffRadius(CentralProbability, N)
\end{lstlisting}
where \textbf{CentralProbability} is defined by \eqref{eqn:CentralProbabilityConstantDefinition}.

\subsection{Sampling from a family of fixed marginals}\label{subsec:fixedmarginals}
For the description of $\psi_{N,\gamma}$, we focus on the most important case: that is when
 when $\gamma < \eta$ and the marginals match the marginals of $p^\eta$.  
In this case, that means the case when the marginals $(p_A,p_B)=\left( \frac{1}{k},\frac{1}{l}\right)=
\left( \frac{1}{2},\frac{1}{2} \right):=p_0$.  We define the following auxiliary notations:
\[
 t_{\gamma}^+ = \max_{t>0}\{ \tau(p_0(t))\leq \gamma\},
\]
\[
 t_{\gamma}^- = \min_{t<0}\{ \tau(p_0(t)) \leq \gamma\},
\]
\nomenclature{$t_{\gamma}^+$}{$\max_{t>0}\{ \tau(p_0(t))\leq \gamma\}$}
\nomenclature{$t_{\gamma}^-$}{$ \min_{t<0}\{ \tau(p_0(t)) \leq \gamma\}$}
and 
\begin{equation}\label{eqn:lenOfRelevantSegment}
 \ell_{\gamma}^0 = t_{\gamma}^+ - t_{\gamma}^-.
\end{equation}
\nomenclature{ $\ell_{\gamma}^0$}{$t_{\gamma}^+ - t_{\gamma}^-$}
 We therefore center the the Gaussian $\psi_{N,\gamma}$ at $t_{\gamma}^+$, the maximum of the integrand within $A_{\gamma}^0$.
 
The scale of $\psi_{N,\gamma}$ is $t_{\gamma}^+ - t_N$, where $t_N$
is a sequence of points belonging to $\left[ t_{\gamma}^-,t_{\gamma}^+ \right]$ whose construction we now explain as follows: restrict the integrand $\tilde{I}$ to the set
\[
\left\{ (p_{A},p_B,t)\in \mathcal{P}\;|\; (p_A,p_B)=\left( \frac{1}{2},\frac{1}{2} \right),\; t\in \left[ t_{\gamma}^-, t_{\gamma}^+ \right] \right\}.
\]
The integrand  can be written as a function of one variable
$t\in [t_{\gamma}^-,t_{\gamma}^+]$, parameterized by $N\in\mathbf{N}$,
\[
 f_N(t):= \tilde{I}(p_0(t))= \exp(-NH( p_0(t) \|p^\eta))\left(\prod_{i,j}(p_0(t))_{i,j}\right)^{-\frac{1}{2}}.
\]
It is not difficult to see that the family $\left\{ f_N\right\}_{N\in \mathbf{N}}$
of functions from $(t_{\gamma}^-,t_{\gamma}^+)$
to the reals has the following properties:
\begin{itemize}
\item Each $f_N$ is defined on the same interval $\left[t_{\gamma}^-,t_{\gamma}^+\right]$.
\item For each $N\in\mathbf{N}$, $f_N(t_{\gamma}^+)>0$.
\item For all but finitely (small) $N\in\mathbf{N}$, $f_N$ is increasing as a function of $t$ on $\left[ t_{\gamma}^-,t_{\gamma}^+\right]$.
\item Let $\rho$ be a fixed ratio $\rho<1$.  Then by the preceding two properties, for all but finitely many $N\in\mathbf{N}$,  
\begin{equation}\label{eqn:rhodefinition}
 \begin{aligned}
 \text{\bf{Either}}\;\rho\cdot f_n(t_{\gamma}^+) < f_N(t_{\gamma}^-)\;\text{\bf{or} there exists a unique}\\
 t_N\in [t_{\gamma}^-,t_{\gamma}^+]\;\text{such that}\; f_N(t_N)/f_n(t_{\gamma}^+)=\rho.
 \end{aligned}
\end{equation}
\end{itemize}
Following the last property, for all but finitely many $N$, we can define 
$t_N$ as the unique value in
$[t_{\gamma}^-,t_{\gamma}^+]$ 
such that \mbox{$f_N(t_N)/f_N(t_{\gamma}^+)=\rho$}, if such a $t_N$ exists,
and otherwise set $t_N=t_{\gamma}^-$. 
\nomenclature{$t_N$}
{the unique value in
$[t_{\gamma}^-,t_{\gamma}^+]$ 
such that \mbox{$f_N(t_N)/f_N(t_{\gamma}^+)=\rho$}, if such a $t_N$ exists,
and otherwise set $t_N=t_{\gamma}^-$.}
For the remaining finitely many $N$ (those for which $f_N$
is not increasing on $[t_{\gamma}^-,t_{\gamma}^+]$),
we can define $t_N$ to be the supremum of those 
$t_N\in [t_{\gamma}^-,t_{\gamma}^+]$ for which \mbox{$f_N(t_N)/f_N(t_{\gamma}^+)=\rho$}, if such a $t_N$ exists,
and otherwise set $t_N=t_{\gamma}^-$.
 
Therefore, in order to specify $t_N$ completely, it remains to define the fixed ratio $\rho$.
Let $Z(z)$ be the pdf of the standard Gaussian random variable, where $z$
measured is in units of $\sigma$, and set 
\[
 \rho := Z(1)/Z(0)\approx 0.60653\ldots
\]
\nomenclature{$\rho$}{$Z(1)/Z(0)\approx 0.60653\ldots$}
\nomenclature{$Z(z)$}{pdf of the standard Gaussian random variable, $z$
measured in units of $\sigma$}
It follows from the definition of $t_N$ in terms of $\rho$ in the previous paragraph that, for all but finitely many $N$, $\psi_{N,\gamma}$ has the properties that
\[
\mathrm{argmax}_{[t_{\gamma}^-,t_{\gamma}^+]} \psi_{N,\gamma}(t) = \mathrm{argmax}f_N(t)=t_{\gamma}^+,\quad\text{and}\quad \frac{\psi_{N,\gamma}(t_N)}{\psi_{N,\gamma}(t_{\gamma}^+)} = \frac{f_N(t_N)}{f_N(t_{\gamma}^+)} =\rho.
\]
Thus, $\psi_{N,\gamma}$ realizes a Gaussian that approximates the shape of the integrand $\tilde{I}$ on the path $p^0(t)$, $t\in\left[ t_{\gamma}^-, 
t_{\gamma}^+\right]$ of uniform marginals. 

This explains the choice of $\psi_{N,\gamma}$ on the fiber of uniform marginals.  For the other fibers 
$\{p(t) \}$ (with $p\in \mathcal{P}_0$), we again pick the Gaussian $\psi_{N,\gamma}$
to have its maximum at the point $p_0\left(t_{\gamma}^+\right)$ at which the path $\{p(t)\}$ 
meets the boundary of $A_{0}^\gamma$,
 because at least on the fibers close to the fiber
containing $p^0$, $p\left(t_{\gamma}^+\right)$ is still the maximum point for the integrand.  It is too time consuming though to recompute $t_N$ separately
for each fiber.  Although it does not make sense to reuse $t_N$ from the uniform-marginals path we can extract
a statistic from $t_N$ which we call the \textbf{scaleRatio} and use that instead.  Namely, the \textbf{scaleRatio} is defined as the ratio of the scale 
from the previous calculation to the length of the entire path:
\[
\textbf{scaleRatio}(N) := \frac{t_{\gamma}^+ - t_N}{\ell_{\gamma}^0}.
\]
\nomenclature{$\textbf{scaleRatio}(N)$}
{$\frac{t_{\gamma}^+ - t_N}{\ell_{\gamma}^0}$}
For the general path of fixed, but not necessarily uniform, marginals, we define $\ell_{\gamma}$ in an analogous manner, as the length of the segment
of the path contained in $A_{\gamma}^0$.  We then set $\psi_{t,\gamma}$ to be the Gaussian centered at $t_{\gamma}^+$ and with scale equal to $\mathbf{scaleRatio}\cdot \ell_{\gamma}$ (where both $t_{\gamma}^+$ and $\ell_{\gamma}$ are calculated with reference to the path in question, not the $p^0$-based path).
\nomenclature{$\psi_{t,\gamma}$}
{Gaussian centered at $t_{\gamma}^+$ and with scale equal to $\mathbf{scaleRatio}\cdot \ell_{\gamma}$}
\section{Interpolation}
In this section, we address the following problem: on the one hand, while any table of precomputed values $\beta^{p^\eta}(\gamma, N)$ must be finite, indeed must fit in RAM
to be practical.  On the other hand, in the implementation of the algorithm to learn structures from data, we will inevitably have to evaluate $\beta^{p^\eta}$ at a continuum of $\gamma$
values and for arbitrary integer values $N$.  This reality means that to put the algorithm into practice, we have to carry out some sort of interpolation.
For reasons of speed, linear interpolation is the most feasible method, but a few moments thought will suffice
to show that the dependence of $\beta^{p^\eta}$ on $\gamma,N$ cannot be close to linear.
Fortunately, through a combination of heuristic reasoning inspired by the Theory of Large Deviations and exploratory data analysis,
we can discover the following: the dependence of $\log\beta$ on $N$ \textit{is} roughly
linear.  Likewise, the relationship of $\log\beta$  to a certain fairly simple statistic ($\mathrm{KL}(p^\gamma \| p^\eta)$) derived from $\gamma, \eta$ \textit{is} roughly
linear. The emphasis in these statements is on the range of values that we most care about, namely ``moderately" small $\gamma$ and moderately large $N$.  Our tasks
in this Section are as follows: first explain our choice of the statistics; second, demonstrate the claimed
approximate linear dependence of $\log\beta^{p^\eta}$ on these statistics;
and, third, explain how the choice of interpolating statistics affects the choices  $\beta^{p^\eta}_N(\gamma)$ we choose to compute and store in the table.
\subsection{Statistics for linear interpolation}\label{subsec:statsLinearInterpolation}
The relevant statistic for linear interpolation of $\log\beta$ in $\gamma$
is the KL-divergence $\mathrm{KL}(p^\gamma \| p^\eta)=H(p^\gamma,p^\eta)$. 
(Recall that $p^\gamma$, like $p^\eta$, by definition has uniform marginals.)
 The reason for adopting this statistic
as the basis for the linear interpolation is ultimately empirical, as will be illustrated below.  However, we also give an intuitive
reason for why it would make sense to expect a linear relationship between $\log\beta$ and $\mathrm{KL}(p^\gamma \| p^\eta)$.  Namely,
Conjecture \ref{conj:Iprojection} says that in the cases that are of greatest interest to us $p^\gamma$ is the I-projection of $p^\eta$ 
onto $A_{\gamma}^0$.  One form of Sanov's Theorem, applied to the closed set $A_{\gamma}^0$, yields
\[
\lim_{N\rightarrow \infty} \frac{1}{N}\log\beta^{p^\eta}_N(\gamma) = -\mathrm{KL}(p^\gamma \| p^\eta).
\]
In fact, fixing several (reasonably small) values of $N$, namely $N=90,\, N=200,\, N=900,\, N=9000$ and graphing $\beta$ 
versus $\mathrm{KL}(p^\gamma \| p^\eta)$, we clearly see the linear relationship emerging, especially as $N$ grows (Figure \ref{fig:betaDepOnKLDivergence}).
\begin{figure}[htb]
 \centering
 \hspace*{-5em}
\includegraphics[scale = 0.8, trim = 0mm 0mm 0mm 0mm, clip]{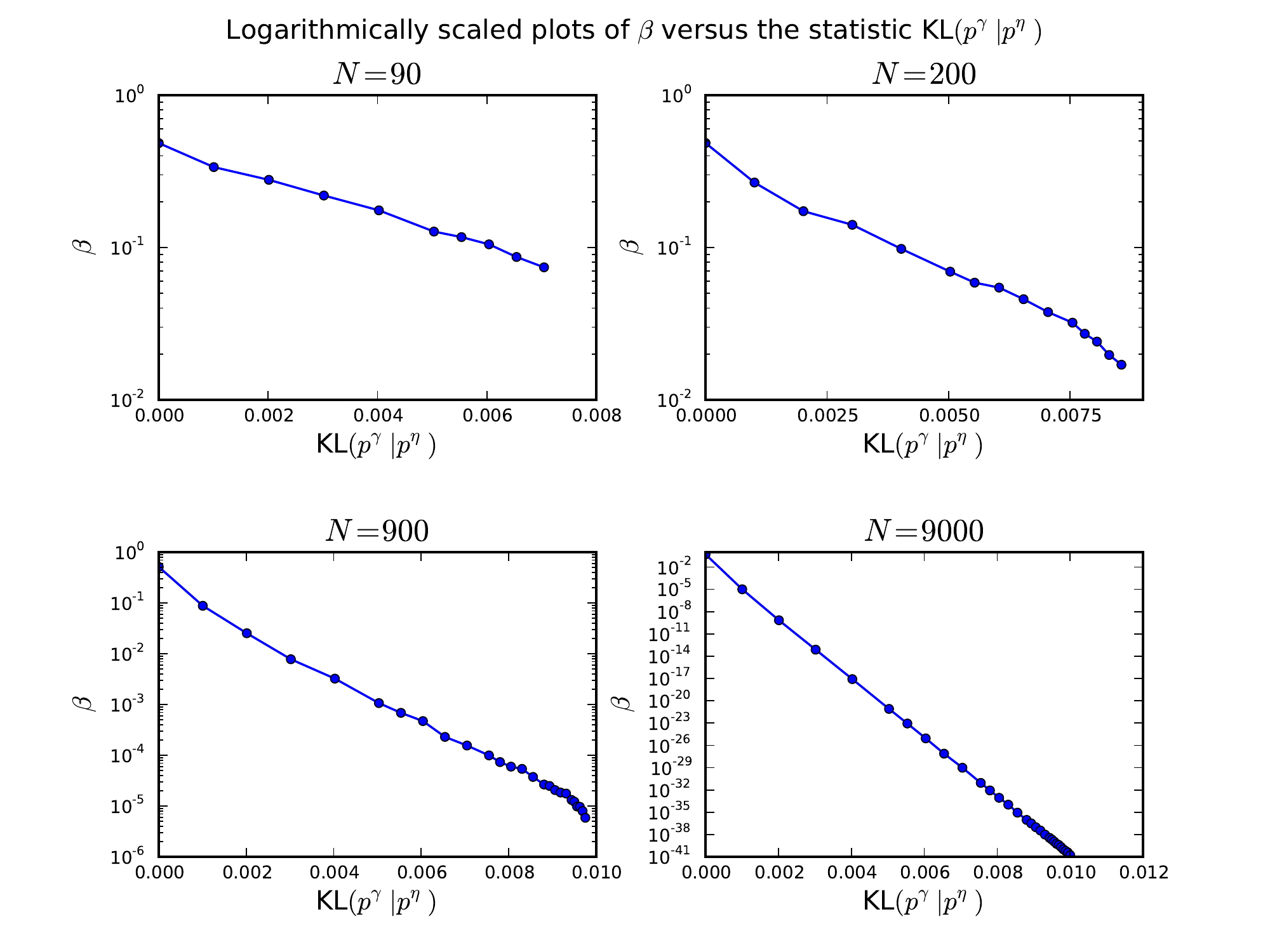}\\
\caption{Approximate linear exponential decay of $\beta$ in the $KL(p^\gamma \|p^\eta)$,
$\eta=0.01$, as $\gamma$ varies, in the top row,
from approximately $1\times 10^{-3}$ and in the bottom row,
from approximately $0$ up to to $\eta=0.01$.}
\label{fig:betaDepOnKLDivergence} 
\end{figure}

\begin{figure}[htb]
 \centering
 \hspace*{-5em}
\includegraphics[scale = 0.8]{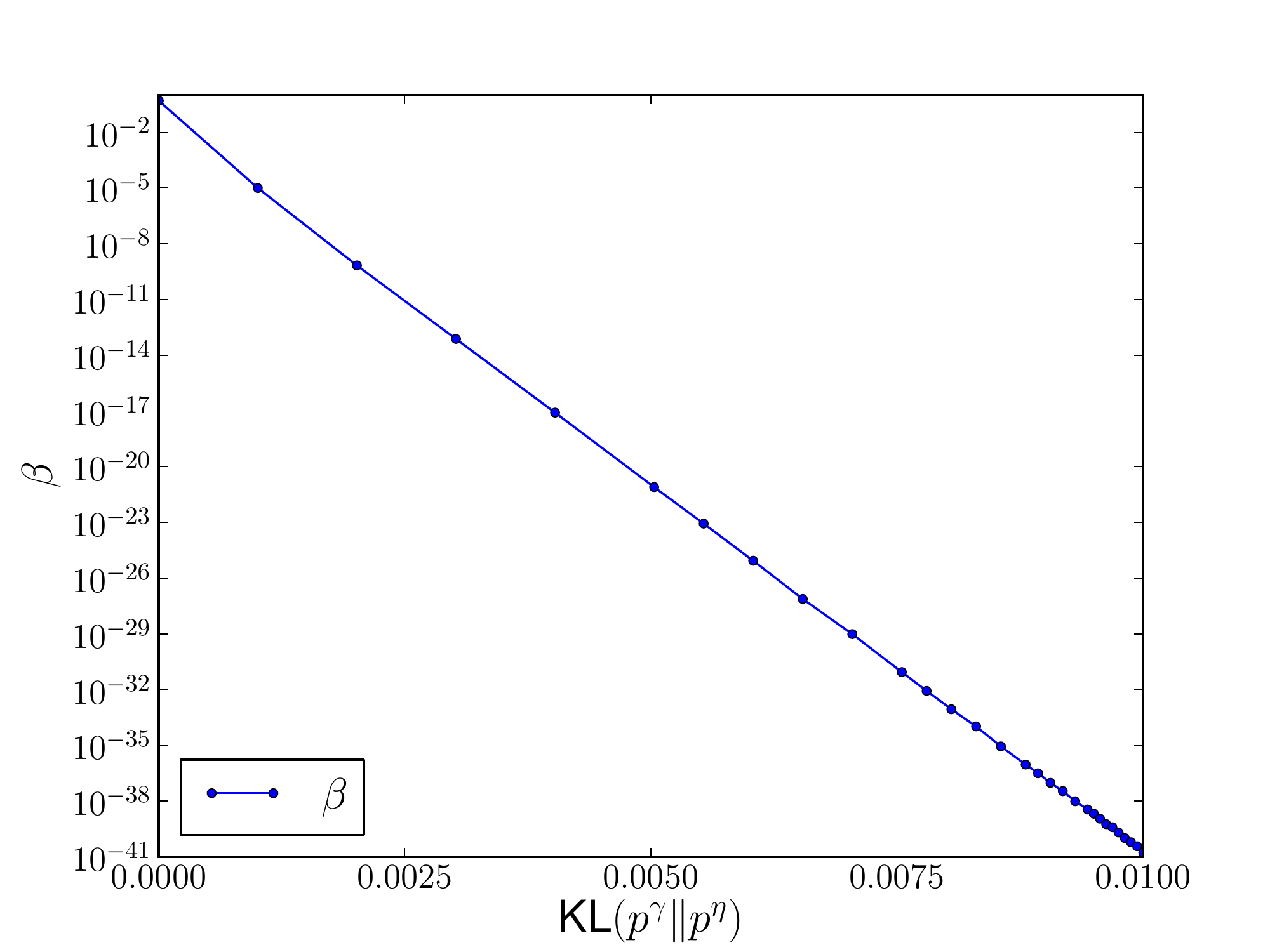}
\caption{Same data as in Figure \ref{fig:betaDepOnKLDivergence}, but plotted with respect
to the $\gamma$ instead of the statistic $KL(p^\gamma \|p^\eta)$.  This illustrates why we interpolate
with respect to the statistic $KL(p^\gamma \|p^\eta)$ instead of $\gamma$: the growth
of $\beta(\gamma)$ with respect to $\gamma$ itself is \textbf{not} linear exponential.}
\label{fig:betaDepOnGamma} 
\end{figure}
Note that, because we are only presenting data for the case $\gamma < \eta$, which is the case which matters
most for the structure learning algorithm, a larger KL-divergence on the $x$-axis means a smaller $\gamma$.
The reason that, in the data we present, the maximum of  $\mathrm{KL}(p^\gamma \| p^\eta)$ grows from $0.007$ to $0.01$
as $N$ grows will be explained in Section \ref{subsec:tableBuilding}, below.

The linear relationship between $N$ and $\log\beta^{p^\eta}_N$ is empirically demonstrated
if we fix $\gamma$ and graph the values
of $\log\beta^{p^\eta}_N$ for fixed $\gamma$ and $N$ varying
in the range from $N=5$ to $N=9000$ (Figure \ref{fig:betaDepOnN}).  It is clear that (for the values of $\gamma$ we care about, a sample of which are shown here),
the linear relationship is sufficiently strong once $N$ reaches the thousands range.
\begin{figure}[htb]
 \centering
 \hspace*{-5em}
\includegraphics[scale = 0.8, trim = 0mm 0mm 0mm 0mm, clip]{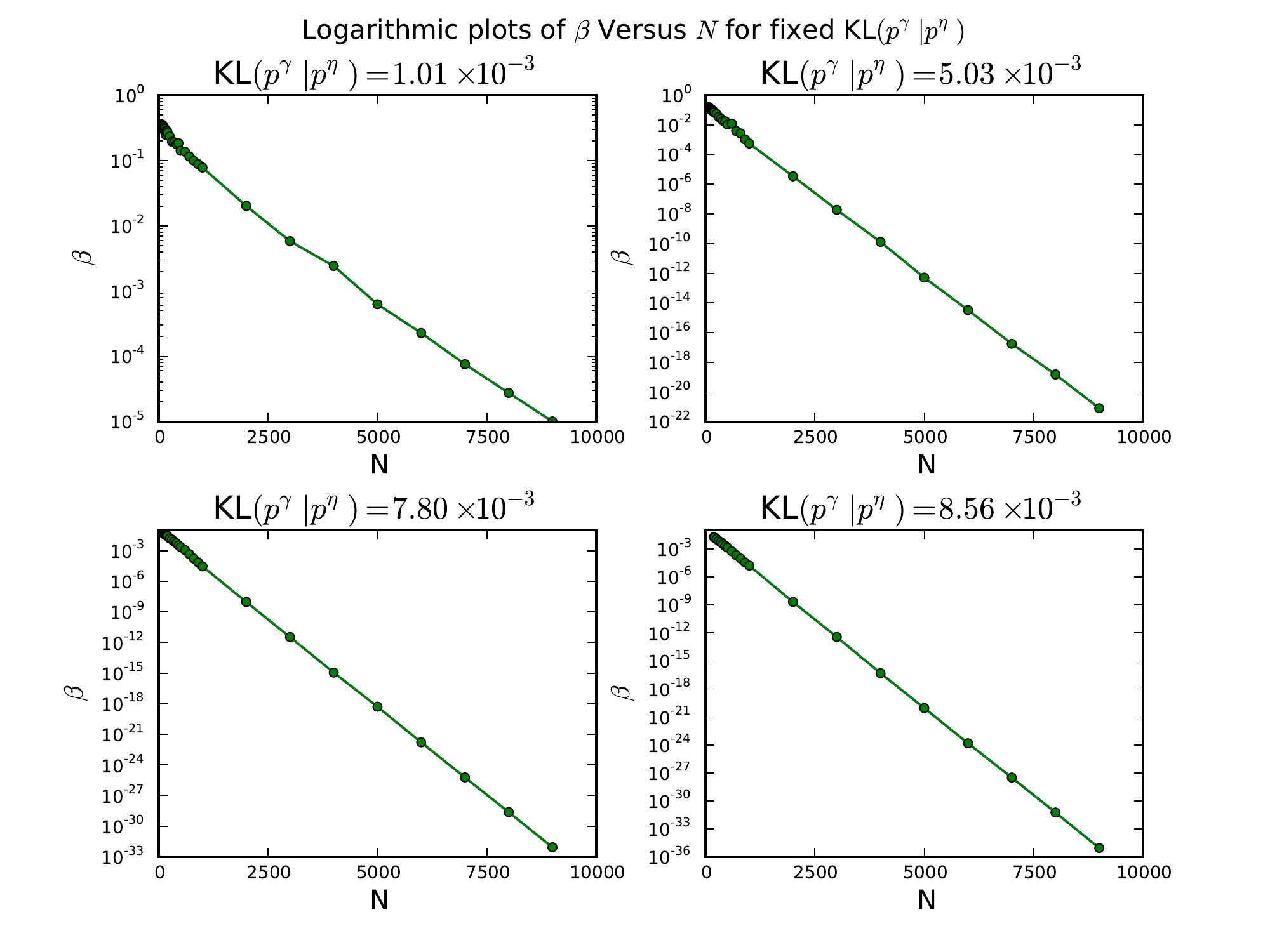}\\
\caption{Approximate linear exponential decay of $\beta$ as a function of $N$, in thousands.  Values
of $\gamma$ in first row, from left to right are $4.67\times 10^{-3}$, $8.61\times 10^{-4}$, in the second
row, from left to right are $1.43\times10^{-4}$ and $6.09\times 10^{-5}$.  Plot is shown for $\eta=0.01$.}
\label{fig:betaDepOnN} 
\end{figure}
When $N$ is in the range of hundreds, as shown in Figure \ref{fig:betaDepOnSmallN}, the linear relationship between $\log\beta^{p^\eta}_N$ and $N$ is already strong enough except in the case of the largest $\gamma$ considered.  Because the sparsity boost in the algorithm makes the least difference
when $\gamma$ is large (equivalently, $\gamma$ is close to  $\eta$, equivalently, the KL-divergence considered is small), we find
it safe to adopt the linear interpolation method in this range of $N$. 
\begin{figure}[htb]
 \centering
 \hspace*{-7em}
\includegraphics[scale = 0.8, trim = 0mm 0mm 0mm 0mm, clip]{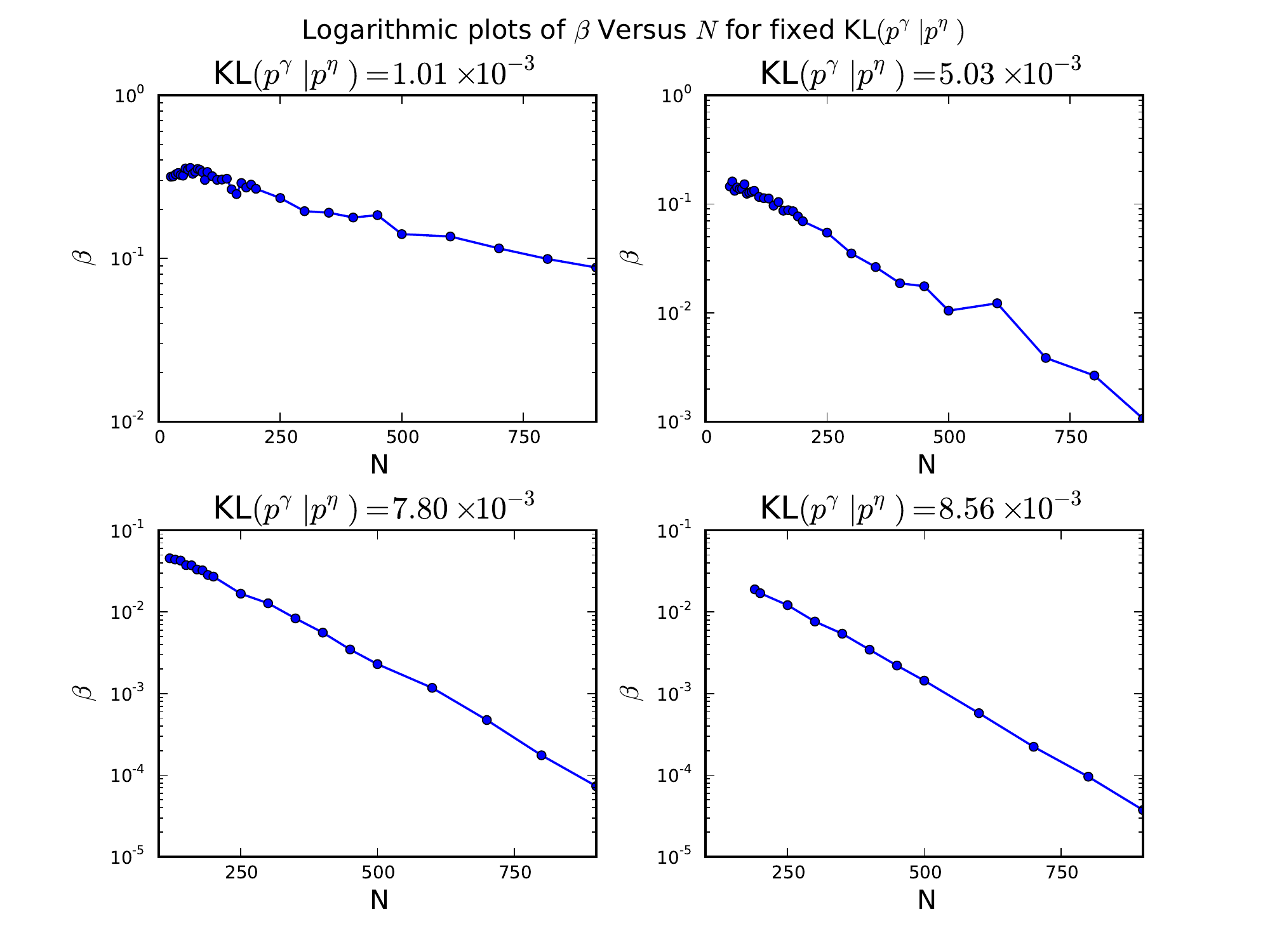}\\
\caption{Approximate linear exponential decay of $\beta$ as a function of $N$ in hundreds.  Results shown
for same values of $\gamma$, $\eta$ as in Figure \ref{fig:betaDepOnN}.}
\label{fig:betaDepOnSmallN} 
\end{figure}
Finally, in the case of $N$ in the range less than $200$, see Figure \ref{fig:betaDepOnVerySmallN}, we see that the linear relationship breaks down,
and indeed any definite trend in $\beta^{p^\eta}_N$ is difficult to discern. 
(The estimates of $\beta_N^{p^\eta}$ also become more noisy for $N$
in this range.)  In this range we can exactly calculate $\beta^{p^\eta}_N$, so one approach would be to simply
use the exact values for each $N$.  Doing this ends up adding unnecessary complexity to the algorithm, so we keep using estimated values and
linear interpolation.  We compensate by decreasing the interval between $N$ whose $\log\beta_N^{p^\eta}\gamma$
we are recording in the table from $N=1000$ in the largest ranges,
to $N=100$ in the intermediate ranges, and finally to $N=10$ or $N=5$ in the smallest
range $N<100$.  With these choices of $N$ intervals, the linear interpolation procedure yields a value for $\log\beta^{p^\eta}_N(\gamma)$ which is close in magnitude to the ones for very nearby $N$ actually recorded in the table.
\begin{figure}[htb]
 \centering
 \hspace*{-3em}
\includegraphics[scale = 0.8, trim = 0mm 0mm 0mm 0mm, clip]{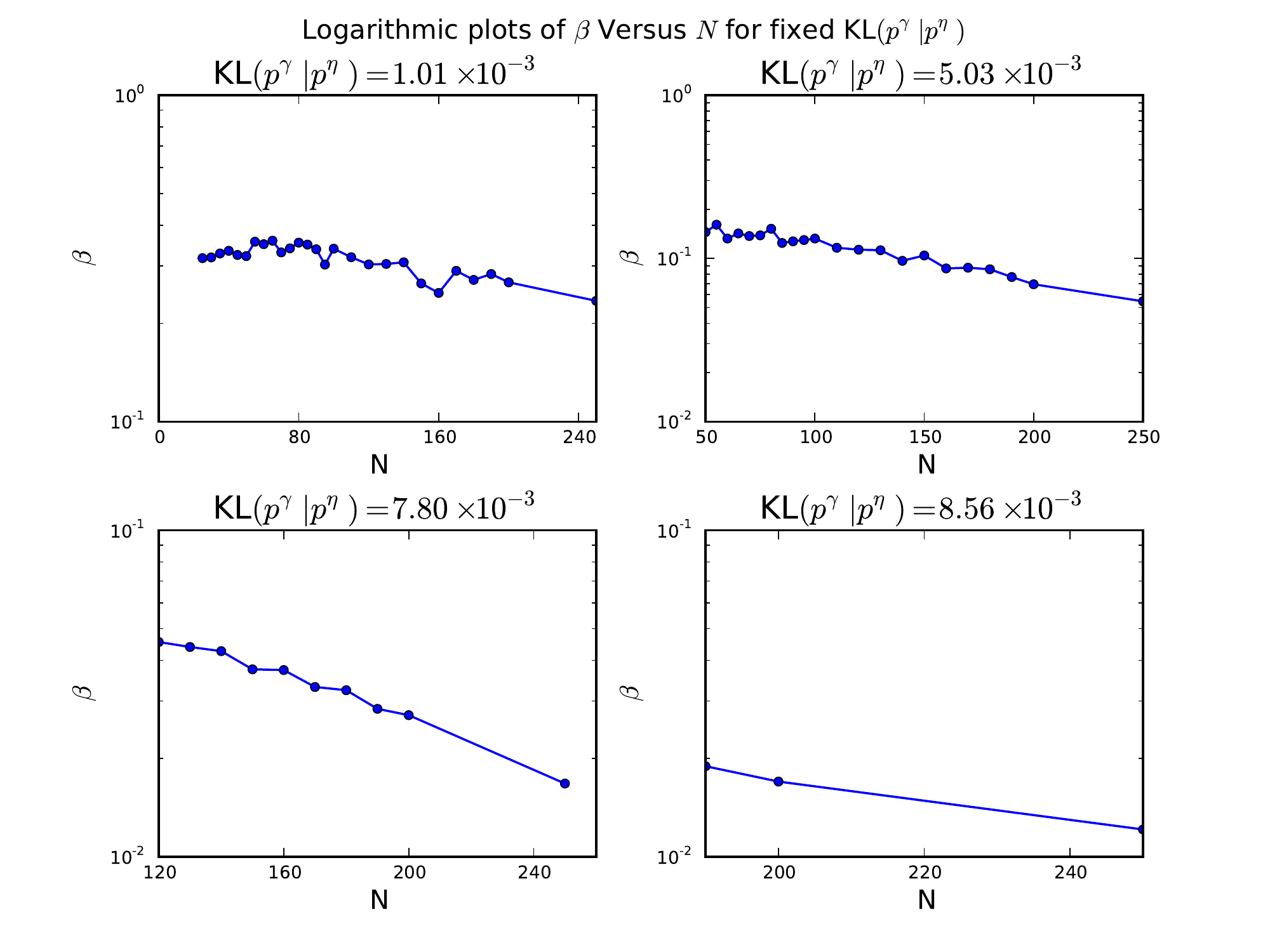}\\
\caption{Values of $\beta$ for small $N$ in intervals of $5$ or $10$.  Results shown
for same values of $\gamma$, $\eta$ as in Figure \ref{fig:betaDepOnN}.}
\label{fig:betaDepOnVerySmallN} 
\end{figure}
Note that for the smaller values of $\gamma$ considered in the ranges of smaller $N$, we haven't recorded enough
data to graph.  This is reflected in Figure \ref{fig:betaDepOnVerySmallN} by 
the increasing value of $N$ at the left edge of the graph as we move from upper left to lower right.
The reason for this will be explained in Section \ref{subsec:tableBuilding}, below.

\subsection{Building the table}\label{subsec:tableBuilding}
We now come to the final topic, namely choosing the two-dimensional array of $(N,\gamma)$-pairs, for which we actually compute
and tabulate $\beta_N^{p^\eta}(\gamma)$.  We have saved this topic for last because the decisions we make in building the table are informed by all the previous
considerations coming both from the structure-learning algorithm itself and from the empirical results of Section \ref{subsec:statsLinearInterpolation} on the numerical
properties of $\log\beta^{p^\eta}_N(\gamma)$.  The three main topics to be covered are as follows: first, the choice of $N$, second the choice of $\gamma$
for each $N$, and third the choice of $\gamma_0(N)$, which is the smallest \textit{meaningful} test statistic, as  function of $N$. We will explain below the precise meaning in this context of ``smallest meaningful".

\textbf{Choice of $N$ for the table.} The principal determining factor is that, as illustrated in Figures \ref{fig:betaDepOnN} through \ref{fig:betaDepOnVerySmallN},
the linear dependence of $\log\beta_N$ on $N$ becomes more reliable for large $N$, allowing us to make the interval $\Delta N$ between $N$
in the table much larger as $N$ grows.  In practice, it seems to make sense, in our experience, to adopt a scheme such as
\begin{itemize}
\item $\Delta N = 5$ for $N<100$,
\item $\Delta N = 10$ for $<100<N<200$,
\item $\Delta N = 50$ for $<200<N<500$,
\item $\Delta N = 100$ for $500<N<1000$,
\item $\Delta N = 1000$ for $1000<N<10000$,
\end{itemize}
which is what we have adopted for the tables we built for our experiments.  

\textbf{Choice of $\gamma$ in the table for each $N$.} For the case of $\gamma < \eta$, we refer back to Figure \ref{fig:betaDepOnKLDivergence}, which illustrates that $\log\beta_N^{p^\eta}(\gamma)$
has a roughly linear dependence on $KL(p^\gamma | p^\eta)$.  
It will be convenient to re-parametrize the path-segment from $p^0$ to $p^\eta$
using the unit interval.  In order the accomplish this we first parametrize the $t$-interval from $t=0$ to $t_{\eta}^+$ 
(recall that $t_{\eta}^+$ denotes the unique positive real value satisfying
the condition $\tau(p^0(t_{\eta}^+))=\eta$ where $p^0$ is the uniform distribution) using the statistic
\[
\zeta(t) = \frac{\mathrm{KL}(p^0(t)\| p^\eta)}{\mathrm{KL}(p^0 \| p^\eta)}.
\]
Note that
\[
\zeta(0) = \frac{\mathrm{KL}(p^0\| p^\eta)}{\mathrm{KL}(p^0 \| p^\eta)}=1,\; \zeta(t_{\eta}^+)=
\frac{\mathrm{KL}(p^\eta\| p^\eta)}{\mathrm{KL}(p^0 \| p^\eta)}=0.
\]
As a result, $\zeta$ is a smooth decreasing function which parameterizes the unit interval
with $[0,t_{\eta}^+]$, in \textit{reverse} order.  Thus, the functional inverse $\zeta^{-1}$
parameterizes $[0,t_{\eta}^+]$ with the unit interval in \textit{reverse} direction:
\[
\zeta^{-1}: [0,1]\rightarrow [0,t_{\eta}^+],\;\text{with}\; \zeta^{-1}(0)=t_{\eta}^+\;\text{and}\; \zeta^{-1}(1)=0.
\]
So if we \textit{compose} $\zeta^{-1}$
with $p^0(\cdot)$, we obtain that $p^0(\cdot)\circ \zeta^{-1}$
is a parameterization of the path-segment from $p^\eta=p^0(t_{\eta}^+)$ to $p^0=p^0(0)$ with the unit interval,
satisfying
\[
p^0(\cdot)\circ \zeta^{-1}: [0,1]\rightarrow [p^\eta,p^0],\;\text{with} \;p^0(\zeta^{-1}(0))=p^{\eta}\;\text{and}\; p^0(\zeta^{-1}(1))=p^0.
\] 
Once we have parameterized the path from $p^{\eta}$ to $p^0$
by the unit interval, the best way of thinking of our choice of $\gamma$ for the table is to think of our marking off a certain set of points on the unit interval, which we call the
the ``normalized KL-divergence list".  Since we care more about getting accurate values for $\beta(\gamma)$ for $\gamma$  close to $0$, that is for points of the
unit interval close to $1$, the basic idea is to mark off the points of the ``normalized KL-divergence list" so that they become more concentrated towards $1$.  Many schemes
could produce this effect, but the scheme we have chosen is as follows.  We create an object of type \textbf{betaTable},
and first invoke the method \textbf{generate\_N\_List} to generate the list of $N$ values and store it as a data member.  When
we invoke on this same \textbf{betaTable} object the method \textbf{generateNormalizedKL\_divergenceList} with signature
\begin{lstlisting}[language=Python]
generateNormalizedKL_divergenceList(self, stepsize, levelRatio, numLevels)
\end{lstlisting}
The method produces a list of rational numbers separated by 
\begin{itemize}
\item
\textbf{stepsize} units from $0$ to $0.5$, 
\item
\textbf{stepsize}/\textbf{levelRatio} units from $0.5$ to $0.75$, 
\item
\textbf{stepsize}/\textbf{levelRatio}${}^2$,
from $0.75$ to $0.875$ and so on \mbox{\textbf{numLevels}} times, so that the ``ticks" closest to $1$ are separated by only
\item \mbox{\textbf{stepsize}/\textbf{levelRatio}${}^{\textbf{numLevels}}$} units.
\end{itemize}
To clarify, here is an example of a test client invoking the method \textbf{generateNormalizedKL\_divergenceList}
on an object of type \textbf{betaTable} class, with the output below the horizontal rule:
\begin{lstlisting}[language=Python]
import src.betaTable as bt
class prettyfloat(float):
            def __repr__(self):
                return "%0.4f" % self
eta=0.01
newBetaTable = bt.betaTable(eta)
newBetaTable.generateNormalizedKL_divergenceList(0.1, 2, 4)
print map(prettyfloat, newBetaTable.normalizedKL_divergenceList)
print len(newBetaTable.normalizedKL_divergenceList)
-------------------------------------------------------------------
[0.0000, 0.1000, 0.2000, 0.3000, 0.4000, 0.5000,
0.5500, 0.6000, 0.6500, 0.7000, 0.7500, 
0.7750, 0.8000, 0.8250, 0.8500, 0.8750, 0.8875, 0.9000,
0.9125, 0.9250, 0.9375, 0.9500, 0.9625, 0.9750, 0.9875]
25
\end{lstlisting} 
Because the interpolation of $\log\beta_N^{p^\eta}$ is actually performed on the statistic 
$\mathrm{KL}(p^\gamma \| p^\eta)$, 
\nomenclature{$\mathrm{KL}(p \lvert\rvert q)$}{$H(p\lvert\rvert q)$}
not $\gamma$, in practice the table is
stored in memory as two separate tables of $\log\beta$'s associated to pairs $(N,\mathrm{KL}(p^\gamma \| p^\eta))$.  The first table contains values for $\gamma < \eta$,
and the second table contains values $\gamma > \eta$.  In practice, as soon as a $\gamma = \tau(p(\omega_N,i,j))$ is encountered in the evaluation
of the objective function, $\gamma$ is converted to a KL-divergence and sent to the appropriate table for the interpolation.

We have already explained how the $\gamma$ for $\gamma < \eta$ are selected, and it remains to explain how the $\gamma$ for $\gamma > \eta$
are chosen for tabulation.  For the case of $\gamma >\eta$, accuracy is far less crucial because in this case, the sparsity boost
in the objective function is small.  Therefore, we simply pick $10$ evenly spaced points (on the $\gamma>\eta$ side of the path) 
between $0$ and $\mathrm{KL}\left(p^0\left(\frac{1}{4} \right) | p^\eta\right)$ at which
to evaluate and store the values of $\beta$.

Finally, we discuss how to set the ``threshold" value $\gamma_0(N)$.  In the course of the algorithm, $\gamma_0(N)$
plays the following role.  If we calculate a test statistic $\gamma=\tau(p(\omega_N),i,j)<\gamma_0(N)$,
we will replace $\gamma$ with $\gamma_0(N)$ before computing $\beta$.  In other words, in practice, the sparsity boost for
$\gamma:=\tau(\omega, A,B| s)$ is
\[
-\log \beta^{p^\eta}_N(\max(\gamma,\gamma_0(N))).
\]
It is easiest to explain our reason for introducing such a $\gamma_0(N)$ in the case when $N$, the number of data points, is very small.  For example, when $N=4$, 
there are only $35$ contingency tables that the data could give rise to, and $25$ of them correspond to distributions in $\mathcal{P}_0$.  Thus, without
such a $\gamma_0(4)$, there is a very high chance of the data sampled from \textit{any} distribution being assigned the largest possible sparsity boost, namely $-\log\beta^{p^\eta}_N(0)$.
We have to find the smallest $\gamma_0$ such that the difference between $\beta_N(\gamma_0)$ and $\beta_N(\gamma')$ for $\gamma' < \gamma_0$ is meaningful, that is, is not a matter of arithmetic coincidence.
In order to find a reasonable value for $\gamma_0$, we return to the definition of $\beta_N(\gamma)$,
\[
\beta_N^{p^\eta}(\gamma) := \sum_{T\in \mathcal{T}_N}\mathrm{Pr}_{p^\eta}(T)\cdot \mathrm{1}_{A_\gamma^0}(p_T).
\]
The main contribution to the sum comes from $T$ such that $p_T$ belongs to a
small number of paths (given by fixing the marginals) of probability distributions, namely those paths passing through the submanifold 
$\mathcal{P}_0$
at points close to $p^0$, the uniform distribution.  The reason that only a few paths make a significant
contribution to the probability is that the probability factor drops off sharply as a function of the marginals, for $T$
such that the marginals of $p_T$ are far from uniform.  
The locations of the points in the lattice $p_{\mathcal{T}_N}$ have nothing to do with $\gamma$: they are determined by arithmetic.  Thus, in the case that the length of the segment, $\ell_{\gamma}$, is small in comparison to the spacing of the points of
the lattice, whether
a path contains any points of $p_{\mathcal{T}_N}$, and therefore makes a contribution to the sum is a matter of arithmetic
coincidence.   Therefore, we adopt a heuristic to avoid this circumstance, and choose $\gamma_0$ so that
\begin{equation}\label{eqn:gammaZeroDefn}
\gamma_0\;\text{satisfies}\;
\ell_{\gamma_0} = 1/N.
\end{equation}
Here $\ell_{\gamma_0}$ is as defined in \eqref{eqn:lenOfRelevantSegment}.

Let $\gamma_0$ be defined by \eqref{eqn:gammaZeroDefn}.  
Then $\gamma_0$ is the threshold such that for all $\gamma>\gamma_0$
 we expect to have at least one point of $p_{\mathcal{T}_N}$ fall on 
the segment of $p^0(t)$ between $t_{\gamma}^-$
and $t_{\gamma}^+$, that is the portion of that path within the set $A_{\gamma^0}$.  
The point of the foregoing discussion is that if we obtain, empirically, a $\gamma$
value which is smaller than $\gamma_0=\gamma_0(N)$, then the empirical
sparsity boost $-\ln\beta_N(\gamma)$ will be ``huge" simply because the path
of distributions with uniform marginals contains no element of $p_{\mathcal{T}_N}$.
We prevent large sparsity boosts from showing up by chance 
for small $N$ by putting a lower threshold of $\gamma_0$
on the empirical mutual information.

We implement the computation of $\gamma_0(N)$,
defined by \eqref{eqn:gammaZeroDefn}, with the method \textbf{generate\_N\_toMinimumGammaDict}
of the class \textbf{betaTable}.  The results are plotted in Figure \ref{fig:depdendenceGammaZeroOnNloglog}
to show the clear power law dependence of $\gamma_0(N)$ on $N$.  

\begin{figure}[htb]
 \centering
 \hspace*{-2cm}
\mbox{\includegraphics[scale = 0.8, trim = 0mm 0mm 0mm 0mm, clip]{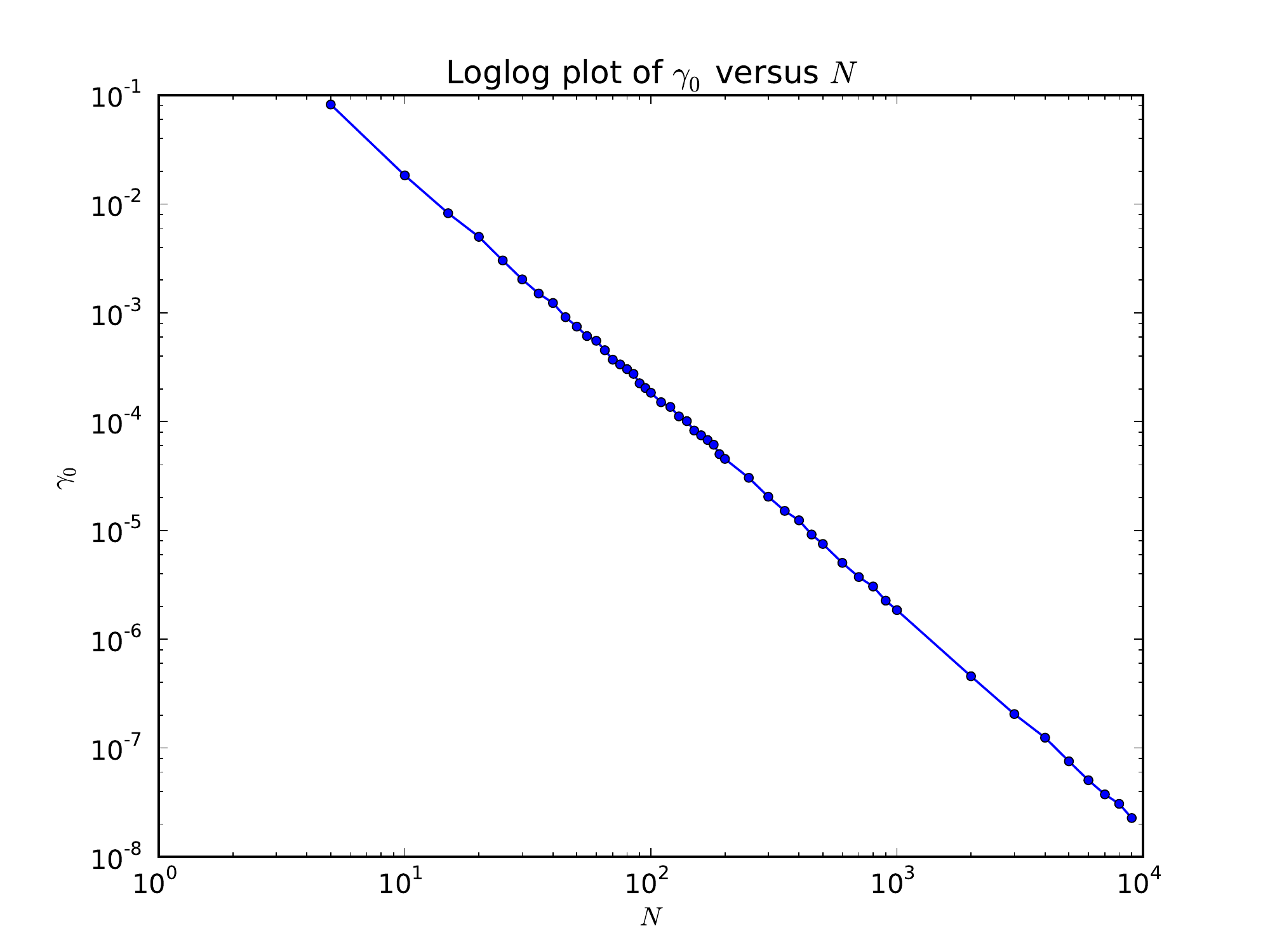}}\\
\caption{Power law dependence of $\gamma_0(N)$ on $N$}
\label{fig:depdendenceGammaZeroOnNloglog} 
\end{figure}

Bayesian-minded readers will doubtless observe that the problem that thresholding by $\gamma_0(N)$
solves is an artifact of our making a prior distribution concentrated entirely at the point $p^\eta$.
An alternative to compensating for the choice of $p^\eta$ in this way would be to make
the marginals of $p^\eta$ ``responsive" to the marginals of the empirical distribution,
in particular for $p^\eta$ to be a distribution over distributions, instead of a point in $\mathcal{P}$.
We discuss this and other directions for future research in Chapter \ref{chap:Discussion}.

To summarize, the algorithm for the interpolation (or extrapolation) is as follows:
\begin{lstlisting}[language=Python]
def float table.interpolateAndReturnLogBeta(N, gamma):
    #table: data structure {(N,gamma,logBeta)}
    if N < smallest tabulated N:
        return 0
    if N > largest tabulated N:
        return extrapolateAndReturnLogBeta(N,gamma)
    #gamma_0(N) smallest tabulated for N 'near' given N 
    if gamma < gamma_0(N):
         gamma = gamma_0(N)
    return table.LinearInterpolate(N,gamma)

def float table.extrapolatedValue(N,gamma):
    Extract LargestN and SecondLargestN from table
    valLargest = interpolateAndReturnLogBeta( LargestN, gamma)
    valSecondLargest = interpolateAndReturnLogBeta(SecondLargestN, gamma)
    slope = (valLargestN - valSecondLargest)/( LargestN - SecondLargestN)
    return valLargest + slope*(N-largestN)
\end{lstlisting}
In practice, for the linear interpolation itself, we do not implement the function \textbf{table.LinearInterpolate(N,gamma)}
but instead call \textbf{scipy.interpolate.griddata} with the default method, ``linear".  Internally,
this ``tesselates the input point set to 2-dimensional simplices, using Qhull, and interpolates linearly on each simplex",
which allows the interpolation to take account of all the neighboring points of the grid in a mathematically
rigorous manner. 
\chapter{Discussion}\label{chap:Discussion}
We close by giving some further discussion of the relation of our results
to those found in the literature, and outlining some directions for future research.

\section{Relation to previous literature}\label{sec:discussionLiterature}
Readers who are familiar with the results of \cite{Friedman:UAI96} will note
that we have an exponent of $2$ on $\zeta$ (corresponding to their quantity $\epsilon$),
where \cite{Friedman:UAI96} have an exponent of $\frac{4}{3}$.
We believe it may be possible to improve our exponent to their value
through further analysis.  Our exponent on $\zeta$ matches
that obtained by \cite{hoffgen1993learning}, since we have used
the methods of \cite{hoeffding1965asymptotically} in estimating
the log-likelihood terms.  Second, our exponent on $m$ is $4$,
where \cite{Friedman:UAI96} have $2$.  In our case, the higher exponent
results from the analysis of the sparsity boost terms.  There is a trade-off
in terms of finite sample complexity in that, although using the sparsity
boost terms gives worse finite sample complexity in terms of $m$, using the sparsity
boost does allow the objective function to rule out \textit{all} false
structures that have even one false edge in the skeleton.  The
sample complexity results of \cite{Friedman:UAI96}, for the usual BIC
score, without the sparsity boost, make no such claim, ruling out only
false structures in $\mathcal{B}_{\zeta}$.  Finally, we point out that
because we have used the methods of \cite{hoffgen1993learning},
we are, unlike \cite{Friedman:UAI96}, obtaining a result that depends polynomially on $n=|V|$, although unlike \cite{Friedman:UAI96}, we are restricting ourselves
to an constant bound $d$ for the in-degree of all graphs to be considered.

The approaches in the literature which are most closely related to our approach 
are the ``hybrid" approaches, which use both constraint-
and score-based approaches and attempt, through heuristic methods, to get the best of both worlds.  Since the constraint-based
approaches usually rely on independence tests, the resemblance to our method, at least at a philosophical level, is clear.
The leading representatives of this approach are Fast's ``Greedy Relaxation Algorithm" (\textsc{Relax}) described in \cite{fast2010learning} and 
Tsamardinos et al.'s Max-Min Hill-Climbing (MMHC) described in \cite{mmhc}.  The \textsc{Relax} algorithm 
starts by running constraint identification to learn constraints followed by edge orientation
to produce an initial model. After the first model has been identified, \textsc{Relax} uses a local greedy search over possible relaxations of the constraints,
at each step choosing the single constraint which, if relaxed,
leads to the largest improvement in the score.  The MMHC algorithm, in contrast, runs a constraint identification algorithm
to learn an undirected skeleton only. Then MMHC builds up the final, directed network starting
from the empty network, and then adding, deleting, and reversing edges, guided by a greedy optimization of the scoring metric.
Note that both these approaches, though hybrid in a sense, still separate out the constraint- and score-based
approaches into two distinct steps, rather than incorporating the constraint (conditional independence tests)
as a term directly into the score itself.  An a priori advantage of our approach is that it removes the needs for
heuristics from these hybrid approaches, and lends itself more to self-contained finite-sample complexity results such as
the ones we have derived in this paper.  This advantage we believe could be valuable in and of itself in making the subject
more accessible to practitioners outside the specialty of Bayesian networks.  Implicit in the experiments envisioned in the ``Directions
for future research" section below is to compare the experimental results obtained with our method to the published results of \textsc{Relax}
and MMHC algorithms.

\section{Directions for future research}\label{sec:DirectionsForFutureResearch}
As is well known, there is a finite sample complexity
result for the BIC score \cite{Friedman:UAI96}. Intuitively we expect
that under our scoring function, as compared to the BIC score, the
true network will win out more quickly and by a larger margin over
its competitors, because while the complexity penalty scales the same
for every parameter in the model, the sparsity boost for different
nonexistent edges differs. The boost increases to infinity precisely
for those nonexistent edges for which accumulating evidence from the
hypothesis tests points to independence. This will become important
in the second step of the algorithm, the optimization.

In order to optimize our score, virtually any optimization procedure
can be used, but we propose using linear programming methods, similar
to \cite{JaaSonGloMei_aistats10,Cussens11}. In preliminary experiments, \cite{brennerSontagUAI},
we have found that as
the amount of data increases LP relaxations become tight, thus providing
the optimal solution to this maximization problem, with little computational
overhead. This also suggests that local search-based methods would
also be able to find the best-scoring Bayesian network in the large-sample
limit.

A key issue in the optimization is finding a good way of pruning a
set of possible parents of each node. Even under the typical assumption
that the size of all parent sets is bounded by a constant $d$, there
are still far too many candidate parent sets for each variable. For
the BIC score, there are criteria saying when one may exclude from
consideration certain possible parent sets of a variable, and also,
more significantly, all supersets of the excluded parent set (Section
2 of \cite{Campos:ICML09}). The criterion significantly reduces the
number of parent sets whose information (contribution to the score
function) has to be cached, but we would still like to cut down on
the number further. In this regard, the heuristic of \cite{mmhc}
for generating a list of candidate parent sets for each variable is
of particular interest. We envision using a similar heuristic from
the point of view of ordering the possible parent sets, but using
our objective function instead of statistical independence tests.
For more details, see our forthcoming paper \cite{brennerSontagUAI}.

Another issue to be explored as a future line of investigation
is the choice of $p^\eta$ in $\beta^{p^\eta}$ within
the objective function.  The overall motivation for $\beta^{p^\eta}$ is to capture the probability of Type II
error of a statistical test with independent distributions $\mathcal{P}_0$ as null hypothesis $H_0$
and all distributions $\mathcal{P}_{\eta}$ as alternative hypothesis $H_1$.
The currently implemented choice of $p^{\eta}$ is a distribution with uniform marginals.
This corresponds philosophically to replacing the more natural composite alternative hypothesis $H_1=\mathcal{P}_{\eta}$
with the more restrictive simple alternative hypothesis $H_1=p^\eta$.
Alternatively, fixing   $p^{\eta}$ this way, so that it is completely insensitive to the marginals
of the empirical distribution $p(\omega_N)$, is somewhat like adopting a strong Bayesian prior belief
that the marginals of the underlying distribution are at or near uniform.  It must be said that
overall the choice of uniform marginals for $p^\eta$ represents an expedient choice, 
providing an objective function that is manageable
to implement and compute, yet still has a reasonable sample complexity.  The more philosophically grounded 
method of choosing $H_1$, taking account of every possible distribution in $\mathcal{P}_\eta$,
would likely lead to an objective function that is prohibitively expensive to compute in practice.  A compromise approach,
which might yield the best results in the long run, is to use a simple alternative $p^{\eta}$
whose marginals can vary, approximating without necessarily matching those of $p(\omega_N)$.  This is one concrete possibility for
future investigations to explore.  More generally, one can contemplate incorporating various
sorts of statistically derived probabilities into the objective function in order to arrive at a good sparsity boost,
and it may very well turn out that some of them will be even more effective for certain purposes
than the one considered here.

This leads to the broader point that objective functions, and the optimization of them
over discrete spaces of structures, are ubiquitous throughout probabilistic
graphical models and fields of machine learning.
The present investigations may suggest a new paradigm for incorporating information from ``classical'' hypothesis
tests into the objective functions used in machine learning.  These investigations, therefore, should be of broader interest to the machine learning
community, far beyond researchers in the field of Bayesian networks.

\begin{appendix}
\chapter{Investigation of the $I$-projection}\label{A}
In the first paragraph of Section \ref{subsec:statsLinearInterpolation}, 
we expressed our interest in characterizing more precisely the 
$I$-projection $q^\gamma$ of $p^\eta$ onto $A^0_\gamma$ (at least for the case $k=l=2$)
and promised to discuss this issue further in the appendices.  While, according to general results,
for example Theorem 8.6, p. 277, in \cite{koller}, the distribution $p^\gamma$ is
the \textit{unique} $M$-projection of $p^\eta$ onto $A_0^\gamma$, there do exist $\eta$ such that $p^\eta$
has two $I$-projections onto $\mathcal{P}_0$, with equal KL-divergence from $p^\eta$  (for example the two I-projections 
of $p^\eta$ with $\eta=0.2$ are two product distributions with equal marginals $(0.25,0.75)$ and $(0.75,0.25)$, respectively).
This illustrates the general principle that it is difficult to compute or characterize the $I$-projection of an arbitrary
distribution.  We do not strictly need the results of this appendices for any of the proofs of the main text, but
a better understanding of the $I$-projection would considerably
simplify the discussion of Chater \ref{chap:finitesamplecomplexity} by allowing us to conclude that the marginals of $q^{\gamma}$ are
\textit{always} (at least in many cases of interest) close to the marginals of $p^{\eta}$ and therefore entirely dispense
with one branch of the ``dichotomy" discussed in that section.  This would lead to better constants in the finite sample
complexity result and a simpler characterization of $\eta_N^-$.

This serves to justify our interest in the following conjecture, which characterizes the $I$-projection for those $p^\eta$ such that
$\eta$ is sufficiently close to $0$.

\begin{conj}\label{conj:Iprojection}
For all $p^\eta$ with $t_{\eta}^+ < \frac{1-e^{-1}}{4(1+e^{-1})}\approx 0.115529289315002$ (which is to say all $\eta$ less than approximately $0.11094054602671935$)
we have that the $I$-projection of $p^\eta$ onto $A_0^\gamma$ is the distribution $p^\gamma$ with uniform marginals
on the boundary component of $A_0^\gamma$ closer to $p^\eta$.
\end{conj}
According to the conjecture, we may assume without loss of generality (\textit{e.g.}, by decreasing $\eta$ to $0.1109$ if necessary) that
\begin{equation}\label{eq:identifyClosestElementOfAGamma}
q^\gamma=p^\gamma
 \end{equation}
\begin{proposition}\label{prop:Iprojection}
Conjecture \ref{conj:Iprojection} is true for the special case of $\gamma=0$.
\end{proposition}
\begin{proof}
 In that special case we are considering the $I$-projection of the 
distribution $p^\eta$ onto the space of product distributions $\mathcal{P}_0$.  By symmetry of $p^\eta$ we can see that any $I$-projection
of $p^\eta$ onto $\mathcal{P}_0$ is a product distribution $p_0(x)$ of two marginal distributions with identical parameters $(x,1-x)$ with $x\in [0,1]$.
Further if one of the $I$-projections is the product distribution $p_0(x)$ of $(x,1-x)$ with itself, then the the other $I$-projection
is the product distribution $p_0(1-x)$ of $(1-x,x)$ with itself (including the ``degenerate" case of $x=\frac{1}{2}$ when these two product distributions are coincidentally equal to one another ).  As a result, we may assume without generality that $x\in [\frac{1}{2},1]$.  Now, we set $t=t^\eta$ and write the KL-divergence $H(p_0(x)\|p^{\eta})$ in terms of the variables
$x,t$ by setting $H(x,t):=H(p_0(x) \| p^\eta )$.  Let $t$ be fixed (because $t=t_{\eta}^+$ and $\eta$ is fixed).  The parameter $x$ can correspond to an $I$-projection $p(x)$
of $p^\eta$ if and only if $\frac{\partial}{\partial x}H(x,t)$ vanishes.  We can compute the partial derivative explicitly and obtain
\[
 \frac{\partial}{\partial x}H(x,t)=(x-1)\log\left( \frac{1-x}{x} \frac{\frac{1}{4}-t}{\frac{1}{4}+t} \right)-x\log\left(
 \frac{1-x}{x}\frac{\frac{1}{4}+t}{\frac{1}{4}-t}
 \right)
\]
Make the change of variable
\begin{equation}
\label{eqn:Iprojectionchangeofvariables}
 y:=\frac{1-x}{x},\; z :=  \frac{\frac{1}{4}-t}{\frac{1}{4}+t}.
\end{equation}
The bounds new variables
\[
 0\leq y \leq 1,\; 0 \leq z \leq 1,
\]
say that the new variables lie in the unit square.
Regrouping terms we obtain that the condition that the partial derivative vanishes is
\begin{equation}
\label{eqn:yzeqn}
 \log y + y\log z - \log z = 0.
\end{equation}
For each $z$ in the unit interval equality \eqref{eqn:yzeqn} has two real solutions $y$, namely,
\begin{equation}
\label{eqn:yW}
 y=W_{0}(z\log z)/(\log z),\quad\text{and}\; y=W_{-1}(z\log z)/(\log z),
\end{equation}
where $W_{0}$ and $W_{-1}$ are, principal real and unique non-principal real branches of the Lambert-W function respectively.
We have to detemine \textit{which of these solutions yields a  $y$ in the unit interval}.

The function $g(z)=z\log z$ maps the unit interval $(0,1)$ onto the interval $(0,-e^{-1})$ in a two-to-one fashion:
namely, $g(z)$ decreases on the interval $(0,e^{-1})$ from the value $0$ (which is how we define $0\log 0$ by convention) to the value $-e^{-1}$.
At that point, the graph of $g(z)$ has a critical point (minimum) and thereafter increases on the interval $(e^{-1},1)$
from the value $-e^{-1}$ to the value $0$.  

The two points $z,z'$ on opposite sides of the critical point $e^{-1}$ which
map to the same value $g(z)=g(z')$ satisfy
\begin{equation}\label{eqn:zzprime}
 \frac{\log z}{\log z'} = \frac{z'}{z}.
\end{equation}
Since $g(z)=e^{\log z}\log z=(\log z)e^{\log z}$, the two compositions $W_0(g(z))$  and $W_{-1}(g(z))$ must equal $\log z$
and $\log z'$.  Thus we obtain the two solutions to \eqref{eqn:yzeqn}:
\[
 y=\frac{\log(z)}{\log(z)}=1,\quad y=\frac{\log(z')}{\log(z)}=\frac{z}{z'}.
\]
where the latter equality follows from \eqref{eqn:zzprime}.  In order for $y\leq 1$ to be satisfied, 
we must have $z\leq z'$.  By the above description of $g(z)$, $z$ must lie in the interval $[0,e^{-1}]$.
According to the change ofvariable \eqref{eqn:Iprojectionchangeofvariables} the point $z=e^{-1}$ corresponds uniquely to
\[
 t_0=\frac{1-e^{-1}}{4(1+e^{-1})}\approx 0.115529289315002.
\]

The trivial solution $y=1$ corresponds to the critical point at $x=\frac{1}{2}$, and this is a critical point for every choice of $t$.
In addition, we have shown that there is one additional critical point $x$, located in $(\frac{1}{2},1)$, if and only if $t>t_0\approx 0.115529289\ldots$, 
so by symmetry there is also a critical point at $1-x$.  Since for $t<t_0$, the critical point at $x=\frac{1}{2}$ is the only critical point,
it is easy to see that this unique critical point is a minimum, so the $I$-projection is the uniform distribution as claimed.
\end{proof}
\begin{rem}
In the above proof, it is possible to verify that the non-trivial solution comes from the principal branch $W_0$ of the Lambert-W function in \eqref{eqn:yW},
while the non-trivial non-solution, that is the $y>1$ values in the $z$-interval $[e^{-1},1]$ comes from the non-principal branch $W_{-1}$.
\end{rem}
\begin{figure}[htb]
 \centering
 \hspace*{-3em}
\includegraphics[scale = 0.7, trim = 0mm 0mm 0mm 0mm, clip]{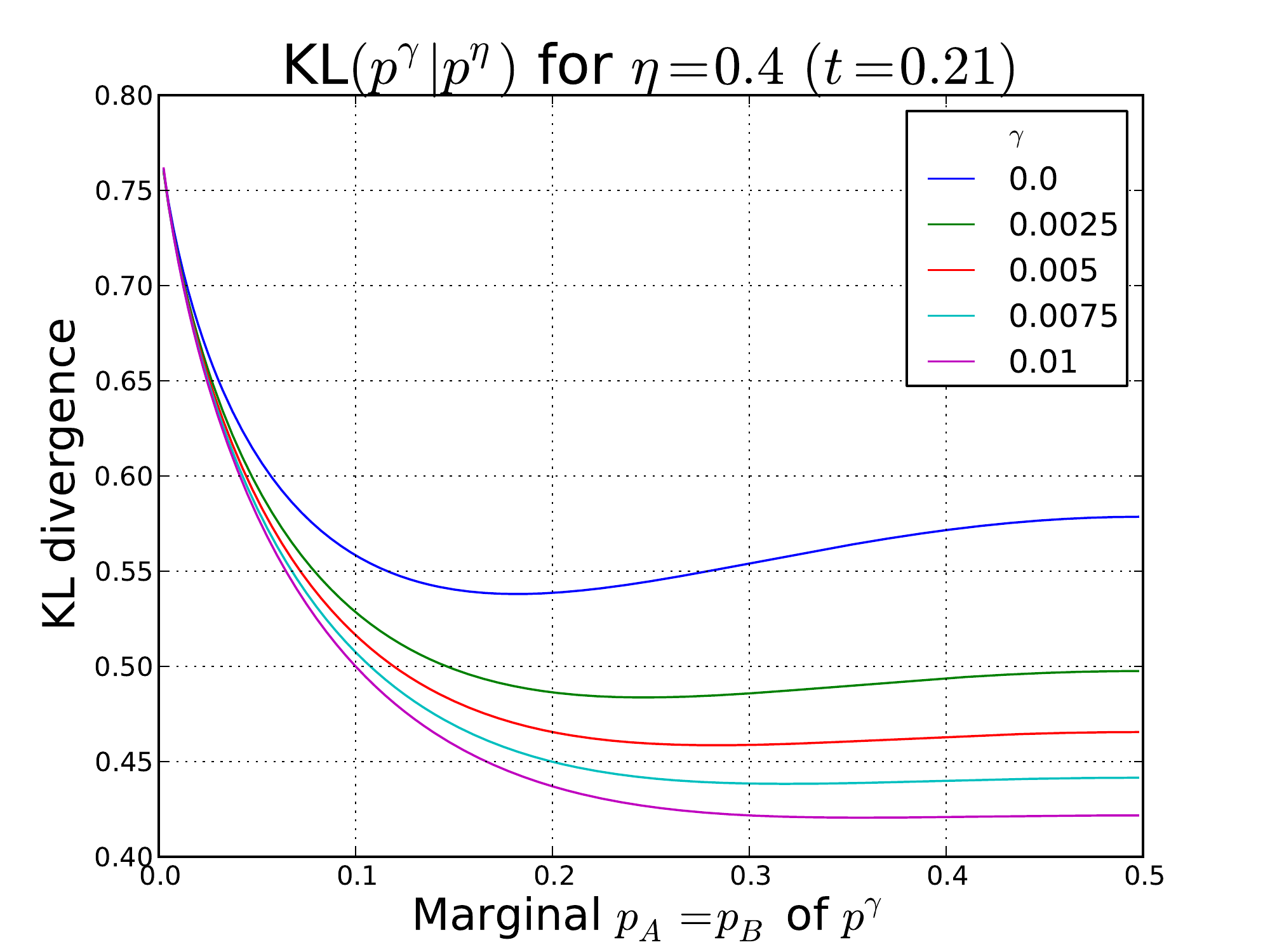}\\
\caption{Dependence on $\gamma$ of I-projection of $p^\eta$ onto $A^{\gamma}_0$}
\label{fig:IProjection} 
\end{figure}
We conclude by giving some empirical evidence for the truth of Conjecture \ref{conj:Iprojection}.  In Figure
\ref{fig:IProjection}, the marginal $p_A=p_B$ of the I-projection
of $p^\gamma$ onto $A_{0}^\gamma$ (i.e., the marginal called $x$ in the above proof) corresponds
to the minimum point of the KL-divergence curve.  We have considered the case
of $\eta=0.4$, which is very nearly the case considered in \textbf{Example 8.14}, p. 275 of \cite{koller}.  In
that example, they consider only the case $\gamma=0$ and state only that there are two distinct minima
of the function $\mathrm{KL}(p(x),p^\eta)$.  Here we can see from the shape of the curve the qualitative
reason for this behavior: the KL-divergence of $p^\gamma$ from $p^\eta$
always decreases very quickly as $p_A$ moves towards the central value of $0.5$
away from the wall, but when $p^\eta$ is sufficiently far from the uniform distribution
$\left(\frac{1}{2},\frac{1}{2}\right)$, the KL divergence levels off when $p_A$
moves close enough to $1/2$, that is to say when $p^\gamma$ moves close enough to the uniform.
When in this case $\gamma$ is sufficiently small, this behavior is pronounced
enough to lead to an increase in $\mathrm{KL}(p(x),p^\eta)$ after $x$ gets close enough to $0.5$.
Thus, for example, in the cases of $\gamma=0.0$ and $\gamma=0.0025$, we can see a relative
minimum when $x$ is roughly $0.15$ and roughly $0.25$, respectively.  By the time $\gamma$ is $0.005$,
we no longer see this behavior occurring strongly enough to lead to a relative minimum.

This then is the basis for believing that Conjecture \ref{conj:Iprojection} is true for all $t_{\eta}^+$ smaller than the value
stated in the Conjecture: the failure of the statement in the conjecture is associated with the existence of a 
relative minimum of the function $\mathrm{KL}(p(x),p^\eta)$ for some $x\in (0,0.5)$.  The combination
of empirical experiments (a sample of which is shown in Figure \ref{fig:IProjection}) and heuristic
reasoning shows that the relatve minimum has the ``best chance" of appearing for $\gamma=0$,
and Proposition \ref{prop:Iprojection} shows that no such relative minimum exists for $\gamma=0$
for $t_{\eta}^+$ in the stated range.  Thus, we expect that no relative minima exist for any $\gamma>0$
 for $t_{\eta}^+$ in the stated range.

\chapter{Background on Entropy and Log Likelihood}\label{app:relEntropyLogLikelihood}
We now recall some well-known notions related to the log-likelihood term and estimates of this term.
Experts in the field will probably find it possible to skip the entirety of this section with the exception
of the characterization of the error-tolerance $\zeta$ as a KL-divergence at the very end of the section.
We first recall the notation for conditional entropy.  When $\mathcal{Y},\mathcal{Z}$ are distinct
subsets of the variable set $\mathcal{X}$, and $P$ is a (joint) probability
distribution over $\mathcal{X}$, we use
\begin{equation}\label{eq:relativeEntInTermsOfOrdinary}
H_P(\mathcal{Z}\lvert \mathcal{Y}):=H_P(\mathcal{Z},\mathcal{Y}) - H_P(\mathcal{Y})
\end{equation}
where $H_P(\mathcal{Z},\mathcal{Y})$ and $H_P(\mathcal{Y})$ are the (ordinary) entropies of the probability distributions
on the subsets $\mathcal{Z}\cup \mathcal{Y}$ and $\mathcal{Y}$, induced by $P$.
\nomenclature{$H_P(\mathcal{Z}\lvert \mathcal{Y})$}{$H_P(\mathcal{Z},\mathcal{Y}) - H_P(\mathcal{Y})$}

There is a well-known decomposition of the $\mathrm{LL}(\omega_{N},G)$,
for any Bayesian network $G$, in terms of conditional entropy terms:
\begin{equation}\label{eqn:LLrelativeentropyExpansion}
\mathrm{LL}(\omega_{N},G)=-N\sum_{i=1}^{n}H_{p_{\omega_{N}}}(X_{i}|\mathrm{Pa}_{G}(X_{i})),
\end{equation}
where $\mathrm{Pa}_{G}(X_{i})$ is the parent set of node $X_i$ in the DAG $G$.
\nomenclature{$\mathrm{LL}(\omega_{N},G)$}{Log likelihood of the data $\omega_N$ under the network $(G,P)$ for optimal distribution $P$ which is an I-map for $G$}

The expansion \eqref{eqn:LLrelativeentropyExpansion} allows us to define
a more general version of the log-likelihood for \textit{any} distribution $B$
in place of the empirical distribution $\hat{p}(\omega_N)$.  Namely, we may use $B$ to evaluate
each of the conditional entropy terms,
\begin{equation}\label{eqn:LLrelativeentropyExpansionIdealized}
\mathrm{LL}^{(I)}_N(B,G) = -N\sum_{i=1}^n H_B(X_i \lvert  \mathrm{Pa}_G(X_i)).
\end{equation}
This expression is known as an \textit{idealized} log-likelihood, because for many values of $N$
it may not be possible to realize $B$ as an empirical distribution $\hat{p}(\omega_N)$.
\nomenclature{$\mathrm{LL}^{(I)}_N(B,G)$}{Idealized log-likelihood for arbitrary probability distribution; $$-N\sum_{i=1}^n H_B(X_i \lvert \mathrm{Pa}_G(X_i))$$}

It is well-known that a probability distribution $B$ is an $I$-map
for $G$ (that is to say, $B$ satisfies the independence relationships expressed by the DAG $G$),
if and only if $B$ factors as follows:
\begin{equation}\label{eqn:LLrelativeentropyExpansionParents}
 B(X)=\prod_{i=1}^n B_i(X_i|\mathrm{Pa}_G X_i)
\end{equation}
for some collection of CPD's $B_i(X_i|\mathrm{Pa}_G X_i)$.  It is well-known from the study of BN learning
as optimization that $LL(\omega_N,G)$ is the maximum
of $LL(\omega_N,p)$ among all the distributions $p$ factoring according to $G$.
Following e.g. p. 5 of \cite{Friedman:UAI96}) we know that this maximum is achieved by
$p=p_{G,\omega_N}$, where
\[
p_{G,\omega_N}(\mathcal{X})=\prod_{i=1}^n \hat{P}_{\omega_N,i}(X_i|\mathrm{Pa}_G x_i)
\]
and where the parameters of the factor CPD's are equal to the parameters
induced by the appropriately marginalized frequency counts of $\omega_N$. 
\nomenclature{$p_{G,\omega_N}$}
{Argmax of $p\mapsto LL(\omega_N,p)$ among distributions $p$
which factor according to $G$; M-projection of $p_{\omega}$ onto the set distributions having $G$ as an I-map.}

Because \eqref{eqn:LLrelativeentropyExpansionIdealized} extends the notion,
\eqref{eqn:LLrelativeentropyExpansion}, of log-likelihood of an empirical sequence
under a network $G$ to \textit{idealized} log-likelihood of any probability distribution $B$ under $G$,
we can extend this concept naturally and ask for probability distribution $p_{G,B}$
which maximizes the ``likelihood of $B$'' under $G$ subject to the condition that $G$ is an I-map for $P_{G,B}$.

The key to defining $p_{G,B}$  is to use the following relationship (see, e.g., p. 17 of \cite{Friedman:UAI96})
between the log-likelihood of the data $\omega_N$ under an arbitrary
network $G\in\mathcal{G}$ and the KL-divergence of $\hat{p}_{\omega_N}$
based at $p_{G,\omega_N}$:
\[
\mathrm{LL}(\omega_N,G)=NH(\hat{p}(\omega_N))-NH( \hat{p}(\omega_N) \| p_{G,\omega_N} ).
\]
Since the first term (involving entropy of $\omega_N$) is fixed, and the second term has a minus sign,
it is immediately clear that $p_{\omega_N,G}$ can be characterized as the distribution
\textit{minimizing} $H( \hat{p}(\omega_N) \| p_{G,\omega_N})$ (among those which factorize according to $G$). 
More generally, the same computation yields, for any probability distribution over $\mathcal{X}$,
\begin{equation}\label{eqn:idealLikelihoodDecomposition}
\mathrm{LL}^{(I)}_N(B,G) = NH(B)-NH( B \| p_{G,B} ),
\end{equation}
where the distribution $p_{G,B}$ is defined by
\[
p_{G,B}(X)=\prod_{i=1}^n p_{B}(X_i|\mathrm{Pa}_G X_i),
\]
and where the factor distributions in the product on the right-hand side
are the conditional joint factors induced by the restriction
and marginalization of $B$. 
That is to say, among all $p$ which factor according to $G$, $p_{G,B}$ minimizes the KL-divergence $H( B | p_{G,B} )$
and also simultaneously is the distribution
for which $ NH(B)-NH( B \| p )$ achieves its maximum value of $\mathrm{LL}^{(I)}_N(B,G)$.%
\nomenclature{$p_{G,B}$}
{Argmin of $p\mapsto H(B \lvert p)$ among distributions $p$
which factor according to $G$; M-projection of $B$ onto the set distributions having $G$ is an I-map.}

For $B$ any probability distribution and $\mathcal{G}$ any family of BN's, and any $\zeta>0$ we make the definitions
\begin{itemize}
 \item \boldmath $\mathcal{B}\;$\unboldmath is the set of Bayesian networks of the form $(G,P_{G,B})$ for all $G\in \mathcal{G}$.
 \item \boldmath $\mathcal{B}_{\zeta}\,$\unboldmath is the subset of $\mathcal{B}$ consisting of those $(G,P_{G,B})$ for which
 $P_{G,B}=B$ or for which $H(B\| P_{G,B})>\zeta$.
\end{itemize}
It is clear that if $B$ (or $P({\omega_N})$ for that matter) happens to
factorize according to $G$, then $P_{G,B} = B$ (resp. $P_{\omega_N,G}=P(\omega_N)$).
This explains why we have included the $(G,P_{G,B})$ for which
$P_{G,B}=B$ in the definition of $\mathcal{B}_{\zeta}$: we are mostly interested in applying the definition
when $B$ is the probability distribution for an unknown underlying Bayesian network
from which a sequence $\omega_N$ of samples is drawn.  Then our finite-sample complexity
theorem will give conditions under which the scoring function assigns its highest score to
the true network $(G,B)$, among all networks in  $\mathcal{B}_{\zeta}$, where $\zeta$ is an error
tolerance that is or may be under the control of the experimenter (learner).

\textbf{In order to understand the proofs of our main results, the reader has to keep in mind
only the following fact concerning $\mathcal{B}$ and $\mathcal{B}_{\zeta}$}
\begin{quotation}  The complement $\mathcal{B}\backslash \mathcal{B}_{\zeta}$ consists
of those BN's $(G,P)$ such that $H(B\| P)<\zeta$ and $P$ is the M-projection of $B$ onto the set of networks which
are I-maps for $B$.
\end{quotation}

\begin{rem}  The quantity $\zeta$, not our $\epsilon$, corresponds to the $\epsilon$
of \cite{Friedman:UAI96}.  Our $\epsilon$ is more closely (though somewhat more general) than the minimium ``information content''
considered in \cite{DBLP:conf/uai/ZukMD06}.
\end{rem}

 \begin{proposition}\label{prop:plogpPositiveDeviation}
Let $\epsilon\geq 0$, $p\in [0,1]$.  Denote by $\tilde{p}_N\in(0,1)$ the empirical average of a sequence drawn from the Bernoulli distribution $X(p)$.  
\boldmath \textbf{Assume that $\tilde{p}_N > p$.}\unboldmath.  Then we have 
 \[
  \mathrm{Pr}\left\{ 
\left|
\tilde{p}_N\log\tilde{p}_N-p\log p
\right|\geq \epsilon
 \right\}\leq 5\mathrm{exp}\left(
 \frac{-N\epsilon^2}
 {18}
 \right)
 \]
\end{proposition}
\begin{proof}
By elementary estimates,
\begin{align*}
|p\log p-\tilde{p}_{N}\log\tilde{p}_{N}| & =|p\log p-\tilde{p}_{N}\log p+\tilde{p}_{N}\log p-\tilde{p}_{N}\log\tilde{p}_{N}|\\
 & \leq|p\log p-\tilde{p}_{N}\log p|+|\tilde{p}_{N}\log p-\tilde{p}_{N}\log\tilde{p}_{N}|\\
 & =|p\log p|\left|1-\frac{\tilde{p}_{N}}{p}\right|+\tilde{p}_{N}|\log p-\log\tilde{p}_{N}|\\
 & \leq|p\log p|\left|\frac{\tilde{p}_{N}}{p}-1\right|+|\tilde{p}_{N}-p|\left|\log\frac{\tilde{p}_{N}}{p}\right|+p\left|\log\frac{\tilde{p}_{N}}{p}\right|.
\end{align*}
We now define each of the three events
\[
E_1 = \left\{ |p\log p|\left|\frac{\tilde{p}_{N}}{p}-1\right|\geq \frac{\epsilon}{3} \right\},
\]

\[
E_2= \left\{   |\tilde{p}_{N}-p|\left|\log\frac{\tilde{p}_{N}}{p}\right| \geq \frac{\epsilon}{3}  \right\},
\]

\[
E_3= \left\{ p\left|\log\frac{\tilde{p}_{N}}{p}\right| \geq \frac{\epsilon}{3} \right\},
\]
and then use the union bound.

By applying Lemma \ref{lem:ChernoffApplication} (stated below in Section \ref{sec:TechnicalLemmas}) directly to $E_1$, we obtain 
\[
\mathrm{Pr}E_1\leq 2\mathrm{exp}\left( \frac{-N\epsilon^2}{27p|\log p|^2} \right).
\]
Since $p\in (0,1)$, we have, first that $p^{-1}\geq 1$ and second that $\frac{1}{p|\log p|^2}$
has its minimum (in the unit interval) of $\frac{e^2}{4}$ at $e^{-2}$.  Therefore,
\begin{equation}\label{eqn:E1estimate}
\mathrm{Pr}E_1\leq 2\mathrm{exp}\left( \frac{-Ne^2\epsilon^2}{4\cdot 27} \right).
\end{equation}
In order to estimate the probability of $E_2$, we first prove the estimate
\begin{equation}\label{eqn:E2estimatePreliminary}
 |\tilde{p}_{N}-p|\left|\log\frac{\tilde{p}_{N}}{p}\right|\leq p\left|
1-\frac{\tilde{p}_N}{p}
\right|^2.
\end{equation}
There are two cases: $\tilde{p}_N>p$ and $p>\tilde{p}_N$.  In case $\tilde{p}_N>p$, 
\[
  |\tilde{p}_{N}-p|\left|\log\frac{\tilde{p}_{N}}{p}\right|=(\tilde{p}_{N}-p)\log\frac{\tilde{p}_{N}}{p}\leq p\left(\frac{\tilde{p}_{N}}{p}-1\right)^2=
  p\left|
1-\frac{\tilde{p}_N}{p}
\right|^2.
\]
In case $p>\tilde{p}_N$, 
\[
\begin{aligned}
  |\tilde{p}_{N}-p|\left|\log\frac{\tilde{p}_{N}}{p}\right|&=(-(\tilde{p}_{N}-p))\left(-\log\frac{\tilde{p}_{N}}{p}\right)\\
  &=(\tilde{p}_{N}-p)\left(\log\frac{\tilde{p}_{N}}{p}\right)\leq p\left(\frac{\tilde{p}_N}{p}-1\right)^2= p\left|
1-\frac{\tilde{p}_N}{p}
\right|^2.
\end{aligned}
\]
Because of \eqref{eqn:E2estimatePreliminary},
\[
\mathrm{Pr}E_2 \leq \mathrm{Pr}\left\{
p\left|
1-\frac{\tilde{p}_N}{p}
\right|^2\geq \frac{\epsilon}{3}\right\}=\mathrm{Pr}\left\{
\left|
1-\frac{\tilde{p}_N}{p}
\right|\geq \sqrt{\frac{\epsilon}{3p}}
\right\},
\]
so that by Lemma \ref{lem:ChernoffApplication}, 
\begin{equation}\label{eqn:E2estimate}
\mathrm{Pr}E_2\leq 2\exp\left(\frac{-N\epsilon}{9}\right)
\end{equation}
In order to estimate $\mathrm{Pr}E_3$, we first use the assumption $\tilde{p}>p$ to eliminate the absolute value signs around the logarithm, then use \eqref{eqn:tangencyInequality}:
\[
\mathrm{Pr}E_3=\mathrm{Pr}\left\{\log\frac{\tilde{p}_{N}}{p} \geq \frac{\epsilon}{3p} \right\}\leq
\mathrm{Pr}\left\{\frac{\tilde{p}_{N}}{p}-1\geq \frac{\epsilon}{3p} \right\}=\mathrm{Pr}\left\{  \frac{\tilde{p}_{N}}{p}  \geq 1+ \frac{\epsilon}{3p} \right\}
\]
So by Lemma \ref{lem:ChernoffMultiplicative}(a), 
\begin{equation}
\label{eqn:E3estimate}
\mathrm{Pr}E_3\leq \mathrm{exp}\left( \frac{-N\epsilon^2}{18p} \right)\leq \mathrm{exp}\left( \frac{-N\epsilon^2}{18} \right),
\end{equation}
where again we have used the condition that $p^{-1}>1$.
Among the three terms on the right side of \eqref{eqn:E1estimate}, \eqref{eqn:E2estimate}, \eqref{eqn:E3estimate}, the one with the smallest coefficient of $-N$ inside the exponent
 is \eqref{eqn:E3estimate}, so that, the other terms can be estimated by $2$ times this term.  We conclude from these three just cited estimates and the union bound that
 \[
 \mathrm{Pr}|\tilde{p}_N\log \tilde{p}- p\log p|\leq 5\mathrm{exp}\left(
 \frac{-N\epsilon^2}
 {18}
 \right)
 \]
\end{proof}

\end{appendix}
\bibliographystyle{alpha}   
\bibliography{structureLearningThesis}  
\end{document}